\documentclass[mnsc]{informs3}
\usepackage[final]{showkeys}
\OneAndAHalfSpacedXI % current default line spacing
\usepackage{endnotes}
\let\footnote=\endnote

\usepackage{enumitem}
\usepackage{amsmath,amssymb}

\usepackage{bm} 
\usepackage{framed}
\usepackage[outdir=./]{epstopdf}
\usepackage{mathtools}
\usepackage{enumitem}
\usepackage{bbm}
\usepackage[pageanchor]{hyperref}
\usepackage{url, xurl}
\usepackage[linesnumbered,ruled,vlined]{algorithm2e}
\usepackage{algpseudocode}
\usepackage[normalem]{ulem}

\usepackage{graphicx}
\usepackage[font=small]{caption}
\usepackage{subcaption}
\usepackage{xcolor,cancel}
\usepackage{tabularx}
\usepackage{multirow}
\usepackage{threeparttable}

\newcommand{\E}{\mathbb{E}}
\newcommand{\Pp}{\mathbb{P}}
\newcommand{\N}{\mathbb{N}}
\newcommand{\Z}{\mathbb{Z}}
\newcommand{\R}{\mathbb{R}}
\newcommand{\ind}{\mathbf{1}}
\newcommand{\Bern}{\operatorname{Bernoulli}}
\newcommand{\Var}{\operatorname{Var}}
\newcommand{\klbin}{\operatorname{kl}}
\newcommand{\pos}[1]{\left[#1\right]_+}
\newcommand{\cdata}{c_{\mathrm{data}}}
\newcommand{\cAI}{c_{\mathrm{AI}}}
\newcommand{\cH}{c_{\mathrm H}}
\newcommand{\calA}{\mathcal A}
\newcommand{\calR}{\mathcal R}
\newcommand{\calI}{\mathcal I}
\newcommand{\calH}{\mathcal H}
\newcommand{\calF}{\mathcal F}
\newcommand{\calT}{\mathcal T}
\newcommand{\calE}{\mathcal E}
\newcommand{\AI}{\mathrm{AI}}
\newcommand{\Hum}{\mathrm H}
\newcommand{\stime}{t^\pi_{\text{stop}}}
\newcommand{\ltime}{T^\pi_{\text{stop}}}

\newcommand{\main}{\mathrm{main}}
\newcommand{\dir}{\mathrm{dir}}

\newcommand{\esc}{\mathrm{esc}}
\newcommand{\LB}{\mathrm{LB}}

\newcommand{\tot}{\mathrm{tot}}

\newcommand{\fixed}{\mathrm{fixed}}

\providecommand{\argmin}{\operatorname*{arg\,min}}
\providecommand{\calJ}{\mathcal J}
\providecommand{\fb}{\mathrm{fb}}
\providecommand{\sense}{\mathrm{sense}}

\providecommand{\pilot}{\mathrm{pilot}}
\providecommand{\eff}{\mathrm{eff}}
\providecommand{\calD}{\mathcal D}

\newcolumntype{C}[1]{>{\centering\arraybackslash}p{#1}}

\usepackage{natbib}
 \bibpunct[, ]{(}{)}{,}{a}{}{,}%
 \def\bibfont{\small}%
 \def\bibsep{\smallskipamount}%
 \def\bibhang{24pt}%
\TheoremsNumberedThrough     % Preferred (Theorem 1, Lemma 1, Theorem 2)
\ECRepeatTheorems

\EquationsNumberedThrough    % Default: (1), (2), ...
\let\INFORMSendproof\endproof
\newif\ifHalmosDone

\def\endproof{%
  \ifHalmosDone\else\Halmos\fi
  \global\HalmosDonefalse
  \INFORMSendproof}

\usepackage{diagbox}
\begin{document}
%%%%%%%%%%%%%%%%

%\setcounter{page}{1}
%\DOI{10.1287/xxxx.0000.0000}%

\RUNAUTHOR{Ham et al.}
%\RUNAUTHOR{Chen and Smichi-Levi and Wang}

% \RUNTITLE{}

% \TITLE{Bundling Information Goods of Decreasing Value}
\TITLE{Human–AI-Powered Hypothesis Testing: \\ Cost-Aware Selective AI Scoring and Sequential Human Escalation}


\ARTICLEAUTHORS{
\AUTHOR{Dae Woong Ham}
\AFF{University of Michigan, Ann Arbor, MI 48109, USA,
\\ \EMAIL{daewoong@umich.edu}}

\AUTHOR{Xuejun Zhao}
\AFF{University of North Carolina at Charlotte, Charlotte, NC 28223, USA,
\\ \EMAIL{xzhao19@charlotte.edu}}

\AUTHOR{Stefanus Jasin, Fenghua Yang}
\AFF{University of Michigan, Ann Arbor, MI 48109, USA,
\\ \EMAIL{sjasin@umich.edu, yfenghua@umich.edu}}
}
%\ARTICLEAUTHORS{ (Authors’ names blinded for peer review) }
% \ARTICLEAUTHORS{
% \AUTHOR{}
% \AFF{}
% \AUTHOR{}
% \AFF{}
% \AUTHOR{}
% \AFF{}
% }% end of the block

\ABSTRACT{Large language models are increasingly used as inexpensive judges to evaluate outputs, label data, and assess whether a system meets a desired quality standard. Yet using AI judgments for formal statistical inference is fundamentally different from simply treating them as ground-truth labels: AI evaluations can be biased or noisy, and rigorous hypothesis testing requires explicit control of type-I and type-II errors. We study how to use AI judgments, together with selective human verification, to conduct a valid hypothesis test at minimum cost. We consider a population of items with hidden binary labels. After choosing a fixed pool of items, the decision maker can selectively query AI, send an item directly to a human, escalate an AI-scored item to a human after observing the AI report, or stop once sufficient evidence has accumulated. We derive an information-theoretic lower bound that captures the minimum cost of achieving prescribed testing errors and characterizes the value of AI information and human verification through a report-dependent information frontier. Motivated by this characterization, we develop SCALE, a sequential cost-aware policy that combines selective AI scoring with adaptive human escalation. SCALE is valid at finite sample sizes and matches the lower bound to first order as the target error probabilities vanish. We further extend the framework to an unknown AI-output model using paired AI--human pilot data. Numerically, SCALE approaches Human-only or AI-only testing when one source clearly dominates, while achieving its largest savings when inexpensive AI judgments and selective human verification are both valuable.

}


%\HISTORY{Submitted November, 2019.}

\date{}

\maketitle

\section{Introduction} \label{sec:intro}

Large language models are increasingly being used not only to generate
content, but also to evaluate it. This is particularly attractive in
settings where the ultimate goal is to make a population-level quality
statement from many individual evaluations. Three application domains
illustrate this opportunity particularly clearly.

First, consider software reliability. A company deploying an AI coding
assistant may want to certify that the fraction of generated programs that
are semantically correct exceeds a prescribed reliability threshold. When
exhaustive test cases are unavailable, expert human review provides a
natural gold-standard assessment, but conducting such review at scale can be
costly. A growing literature therefore studies LLMs themselves as code
judges. \citet{tong2024codejudge} develop CodeJudge, which uses LLMs to
assess the semantic correctness of generated code without requiring test
cases. \citet{zhao2025codejudgeeval} introduce CodeJudge-Eval, which asks
LLMs to determine whether submitted code solutions are correct across
different error types and compilation issues. More recently,
\citet{jiang2026codejudgebench} benchmark 26 LLM judges across code
generation, code repair, and unit-test generation. Their results also
illustrate why a statistical treatment is needed: LLM code judges can be
sensitive to seemingly irrelevant changes such as response ordering,
variable names, and misleading comments. Thus, an LLM judge can provide an
inexpensive and scalable signal of code correctness, but its judgment cannot
automatically be treated as ground truth.

Second, similar ideas are already being used to screen large collections of
clinical documents. A health system may, for example, want to assess whether
the prevalence of biased or stigmatizing language in its clinical notes
exceeds an acceptable level. Manual review by clinical experts provides the
natural reference standard but is difficult to perform at the scale of a
large electronic-health-record system. \citet{apakama2025bias} apply GPT-4
to 50,000 emergency-department medical and nursing notes to identify several
categories of biased language and use human reviewers to verify the model's
detections. \citet{zhang2025equitable} compare ChatGPT-4 with human expert
annotations for identifying stigmatizing language in electronic health
records and find that performance varies across categories and is sensitive
to prompt design. \citet{sethi2026stigma} study LLM-based detection of
stigmatizing language using more than 77,000 ICU notes and externally
validate the approach on a substantially larger health-system data set.
These studies demonstrate both sides of the opportunity: AI can make
large-scale screening feasible, while expert human labels remain important
for establishing what is actually present in the underlying records.

Third, automated judging has become central to AI safety and content
moderation. An AI provider may want to certify that the fraction of model
responses violating a safety policy remains below a prescribed threshold.
Reviewing every generated response by hand is prohibitively expensive, so
large-scale safety evaluations increasingly rely on automated judges or
moderation classifiers. \citet{mazeika2024harmbench} develop HarmBench, a
standardized framework for large-scale evaluation of harmful model behavior
that uses automated evaluation classifiers calibrated against human-labeled
examples. \citet{movva2024annotation} directly compare GPT-4 safety
annotations with human judgments of conversational safety.
\citet{chen2025safer} evaluate 11 LLM judges used to assess the safety of
generated content and show that seemingly superficial artifacts, including
apologetic and verbose phrasing, can substantially change the resulting
verdicts. Here again, automated evaluation is valuable precisely because it
is scalable, but its errors become consequential when the objective is to
make a formal statement about an underlying safety rate.

Across these applications, the same basic structure emerges: there is a
population of items with latent ground-truth labels, an AI system provides
inexpensive but imperfect assessments of those labels, and human review
provides costly ground truth. Much of the emerging literature focuses on
improving the quality of the AI judge itself. Researchers develop more
effective judging procedures, richer evaluation rubrics, better prompting
and reasoning mechanisms, specialized evaluation models, robustness
benchmarks, and aggregation schemes intended to bring automated judgments
closer to human ground truth. Related work in clinical applications similarly
evaluates and refines AI screening procedures by comparing their outputs with
expert human annotations. These efforts are important: a better AI judge can
clearly reduce the amount of expensive human review that is needed. However,
improving the judge does not eliminate the underlying statistical problem.
Even a highly accurate and carefully calibrated AI evaluator remains an
imperfect measurement of the ground truth, and high average agreement with
humans does not by itself guarantee valid type-I and type-II error control
for a population-level hypothesis test. If the ultimate goal is rigorous population-level statistical inference, then building a better AI judge is not enough: its residual errors must be incorporated into a formal inferential procedure with explicit statistical guarantees, so that conclusions about the underlying population are scientifically defensible rather than merely reflections of the judge's average agreement with humans.

One natural approach is to combine inexpensive AI judgments with selective
human verification. Rather than treating AI as a complete substitute for
human evaluation, a decision maker can use AI as a cheap source of
preliminary information and purchase a human ground-truth label only when
the additional information is worth its cost. This idea is related to
prediction-powered inference (PPI), which combines machine predictions with
a smaller number of gold-standard labels to obtain valid statistical
inference \citep{angelopoulos2023prediction}. Our setting adds an explicitly operational dimension: rather than taking AI predictions and human labels as given, we jointly decide how each acquired item should be evaluated and when sufficient evidence has been collected to stop the test. This leads to the central question of the paper:
\emph{How should AI and human evaluations be combined to reach a
statistically valid conclusion at minimum cost?}

We study this question in a simple hypothesis-testing model for a population
proportion. Each acquired item has an unobserved binary ground-truth label
$X_i\in\{0,1\}$, and the objective is to distinguish
$$
    H_0:p=p_0
    \qquad\text{from}\qquad
    H_1:p=p_1,
    \qquad 0<p_0<p_1<1,
$$
where $p$ is the population fraction of positive labels. A human query
reveals the true label, whereas an AI query returns a potentially noisy
finite-valued report that may contain a predicted label, confidence score,
or other metadata. Acquiring an item, querying the AI, and querying a human
all carry potentially different costs. The policy chooses a fixed pool of
items in advance, but information acquisition within that pool is adaptive:
it may query AI, query a human directly, query a human after observing an AI
report on the same item, or stop and decide. We minimize worst-case expected
total cost subject to prescribed type-I and type-II error guarantees.

Two features are central. First, \emph{AI scoring itself is selective}: the
decision maker need not run AI on every acquired item. When direct human
information is sufficiently valuable, it may be preferable to skip the AI
and query a human immediately. Conversely, when AI information is
inexpensive and sufficiently informative, an AI report may by itself provide
useful statistical evidence. Second, human review after an AI report is a \emph{nested information
action}. Once report $r$ has been observed, a subsequent human label
provides only the residual information in the ground truth conditional on
that report, and this residual value can vary substantially across reports.
Some AI reports may already be highly informative, whereas others may leave
considerable uncertainty and therefore be particularly valuable to verify.
Moreover, because the test is sequential, the desirable mix of AI and human
information also depends on the likelihood evidence accumulated across
previous items. The policy and our paper therefore study the joint decision of \emph{when AI is
worth querying, when an AI judgment is worth verifying, and when the test
should stop}.

\vspace{1mm}
\textbf{Contributions and main results.}
Our first contribution is a lower bound on the minimum cost that any valid
human--AI testing policy must incur. The basic idea is simple: to control
type-I and type-II errors, every policy must collect enough statistical
evidence to distinguish the two hypotheses. We use KL-divergence arguments
to quantify how much evidence can be obtained from an AI judgment, from a
human label, and from a human label obtained after an AI judgment has already
been observed. We also account for the fact that the number of items must be
chosen in advance. Combining these requirements gives a computable lower
bound on the best possible cost of any policy; this lower bound is developed
in Section~\ref{sec:lower-bound}.

Our second contribution is SCALE, the \emph{Sequential Cost-Aware
Likelihood-Guided Escalation} policy. SCALE builds on the basic logic of the
classical sequential probability ratio test (SPRT) \citep{waldSPRT}: it
continually tracks the likelihood ratio between the two hypotheses and stops
once the accumulated evidence is sufficiently strong. The key difference is
that, in a classical SPRT, each new observation is drawn from a fixed
information source. In our setting, the policy must also decide \emph{how}
to obtain the next piece of evidence. SCALE therefore chooses whether the
next item should be evaluated by AI only, by a human directly, or by AI
first followed by selective human verification based on the observed AI
report. In this sense, SCALE extends the SPRT from a stopping rule with a
fixed observation channel to a joint sensing-and-stopping rule for a
human--AI system. Because the number of items is acquired in advance, SCALE
also includes a pre-specified fallback test in case the pool is exhausted
before either likelihood-ratio boundary is reached, which guarantees the
desired type-I and type-II errors at every finite sample size. Our main
theoretical result shows that, as the target error probabilities become
small, SCALE achieves the lower bound to first order. In other words, no
other valid adaptive policy can have a meaningfully lower leading-order
cost. Section~\ref{sec:upper-bound} develops the policy and proves this
result.

Our third contribution considers the practically important case in which the
AI judge's error behavior is not known in advance. In practice, one typically
has to learn how the AI's judgments relate to ground truth from a calibration
sample. We therefore collect a paired pilot sample in which each item is
evaluated by both the AI and a human, use it to estimate the AI-output model,
and then run a guarded plug-in version of SCALE. The guardrails account for
the fact that the estimated AI model is itself uncertain, so that this
estimation error does not undermine the statistical guarantees of the test.
We show that, when the pilot is sufficiently informative, the resulting
procedure achieves the same first-order cost as if the AI model were known
from the outset. We also characterize when the additional cost of collecting
a fresh pilot is small enough not to change this first-order benchmark.

Finally, our numerical experiments show when combining AI and human
judgments is most valuable. When human review is inexpensive, SCALE behaves
almost like Human-only testing; when human review is very expensive, it
approaches AI-only testing. The largest gains arise in the intermediate
regime, where AI provides useful low-cost information but selective human
verification is still worth purchasing. This is precisely the regime in
which committing in advance to either a Human-only or an AI-only architecture
is most costly. In our benchmark instance, SCALE reduces cost by as much as
$36.37\%$ relative to the cheaper single-source baseline as the human-review
cost varies, and its savings reach $39.16\%$ as the testing requirements
become more stringent. These results illustrate the main operational value
of the hybrid design: AI is useful not only because it can replace some human
review, but because it helps determine \emph{which} items still warrant
costly human verification. Section~\ref{sec:scale-simulations} reports the
full numerical results.

\vspace{1mm}
\textbf{Relation to existing approaches.}
Our paper is most closely related to prediction-powered inference (PPI),
two-phase sampling with selective gold-standard verification, and active
hypothesis testing; Section~\ref{sec:lit_rev} provides a detailed review. The connection to prediction-powered inference is especially close: both
settings combine inexpensive machine-generated information with a smaller
amount of costly ground truth to obtain valid statistical inference
\citep{angelopoulos2023prediction}. Our main distinction is operational.
AI and human information are not taken as given; the policy decides how each
item should be evaluated, whether an AI-scored item should be escalated to a
human, and when to stop. Our nested AI-then-human structure is also related to classical two-phase
sampling, where a cheap first-stage measurement is followed by selective
gold-standard verification
\citep{tenenbein1970double,begg1983assessment}. In our setting, however, the
first-stage AI measurement is itself optional and the verification decision
depends on both the observed AI report and the accumulated evidence. Finally, SCALE is related to active hypothesis testing and controlled
sensing, which adaptively choose among information sources as evidence
accumulates
\citep{chernoff1959sequential,naghshvar2013active,
nitinawarat2013controlled}. The distinctive feature here is that querying AI
creates a report-dependent option to reveal the same item's ground-truth
label through human verification. This nested action structure is central to
both our lower bound and our policy.

\vspace{1mm}
\textbf{Organization of the paper.} The remainder of the paper is organized as follows.
Section~\ref{sec:lit_rev} reviews the related literature in greater detail.
Section~\ref{sec:model} introduces the model, policy class, cost objective,
and information quantities. Section~\ref{sec:lower-bound} derives the lower
bound on the optimal cost. Section~\ref{sec:upper-bound} develops SCALE,
establishes its finite-sample feasibility and first-order optimality, and
contains the numerical experiments in
Section~\ref{sec:scale-simulations}. Section~\ref{sec:unknown-ai} develops
the pilot-calibrated procedure for an unknown AI-output model.
Section~\ref{sec:conclusion} concludes the paper.

\section{Literature Review}\label{sec:lit_rev}
Our paper lies at the intersection of three statistical literature. The first
uses machine predictions to economize on expensive gold-standard labels while
preserving valid inference. The second studies two-phase designs that combine
broadly available but fallible measurements with selectively acquired exact
outcomes. The third studies active hypothesis testing, in which observation
channels and stopping decisions are chosen as evidence accumulates. At an
architectural level, selective prediction and learning to defer also motivate
endogenous human escalation, but those models primarily optimize item-level
prediction or delegation loss \citep{madras2018predict,mozannar2020defer}; our
terminal decision concerns a population hypothesis. We therefore organize the
review around the three statistical streams that map most directly to our
inferential target, nested observation structure, and adaptive policy.

\subsection{Prediction-Powered and Active Statistical Inference}
\label{sec:lit_ppi}

A growing statistical literature uses machine predictions to reduce the need
for expensive labels while retaining inferential validity.
\citet{angelopoulos2023prediction} introduce prediction-powered inference
(PPI), which combines abundant machine predictions with a smaller labeled
sample to correct prediction error. \citet{kossen2021active} select which test
examples to label for sample-efficient model evaluation, while
\citet{zrnic2024active} assign observation-dependent labeling probabilities
and construct valid confidence intervals and hypothesis tests. Recent work learns data-adaptive labeling policies or develops anytime-valid
and sequential prediction-assisted tests
\citep{ma2026opal,csillag2025prediction,tenzer2026betting}. Most closely related operationally,
\citet{angelopoulos2025costoptimal} derive cost-aware allocations between a
cheap weak rater and an expensive strong rater to estimate the mean strong
rating accurately under an annotation budget.

The main distinction is our objective and action space. Most work in this
stream evaluates estimator variance, confidence-interval width, or testing
power under a labeling or annotation budget. We instead study a simple
population hypothesis test in which AI scoring is itself optional and costly,
direct human review is an alternative sensing action, and the acquired pool is
chosen jointly with the within-pool policy. A non-escalated AI report remains
direct likelihood evidence, whereas a human label obtained after that report
contributes only report-conditional residual information. Furthermore, in the PPI literature there is no action to further escalate a machine prediction to a ``human'' label as we explicitly model.
We minimize
worst-case expected data, AI, and human cost subject to separate type-I and
type-II error constraints. Consequently, labeling decisions depend on both
the global likelihood-ratio direction and direction-specific residual KL
information, rather than only on predictive uncertainty, influence, or
conditional squared error.

\subsection{Two-Phase Sampling and Selective Gold-Standard Verification}
\label{sec:lit_two_phase}

A longstanding literature studies two-phase or double-sampling designs in
which an inexpensive but fallible measurement is collected broadly and an
expensive gold-standard outcome is obtained for a validation subsample.
\citet{tenenbein1970double} estimate a binomial proportion using a fallible
classifier on a first-stage sample and an exact classifier on a subsample,
explicitly optimizing the two sample sizes against cost and precision.
\citet{begg1983assessment} show that selective verification based on the
initial screen must be incorporated into inference to avoid verification bias.
Subsequent work develops efficient prevalence-survey designs and inference or
tests under partial validation
\citep{shrout1989design,mcnamee2003efficiency,pepe1992inference,
      alonzo2003estimating,tang2012test}.
Design-oriented contributions choose second-phase sampling fractions using
relative costs, pilot information, or semiparametric efficiency criteria
\citep{reilly1996optimal,tao2020optimal}; recent work similarly targets
expensive manual chart review using noisy electronic-health-record phenotypes
and covariates \citep{marksanglin2025optimal}.

This literature provides the closest static analogue to our nested
AI-then-human observation structure. However, conventional two-phase designs
generally collect the phase-one measurement for the full cohort and choose a
validation sample through prespecified strata or observation-dependent
sampling probabilities, with the goal of correcting bias or improving
estimation precision. In our model, even the phase-one AI measurement is
optional: an acquired item may receive direct human review, AI scoring alone,
AI scoring followed by human verification, or no query. Moreover, the
follow-up rule changes with the accumulated likelihood ratio and values a
human label by its directional residual KL information. Thus the validation
design is not merely tailored to an estimand; it is embedded in an adaptive
hypothesis test. This additional history dependence leads to the active-testing
literature discussed next.

\subsection{Active Hypothesis Testing and Controlled Sensing}
\label{sec:lit_active_ht}

Active hypothesis testing and controlled sensing study how a decision maker
should choose among observation channels while learning which hypothesis is
true. Wald's sequential probability ratio test endogenizes stopping under a
fixed observation channel \citep{waldSPRT}, whereas
\citet{chernoff1959sequential} allows the experiment itself to be selected
adaptively from past observations. \citet{naghshvar2013active} derive
information-acquisition bounds and asymptotically optimal sensing policies
under sampling costs and wrong-decision penalties; companion work separates
the gains from sequential stopping and adaptive experiment selection
\citep{naghshvar2013sequentiality}. \citet{nitinawarat2013controlled} analyze
controlled sensing in both fixed-sample and sequential multihypothesis testing,
deriving error-exponent bounds and Chernoff-type policies under decision-risk
constraints. \citet{kartik2022fixed} study fixed-horizon active tests with
adaptive experiment selection and the option of an inconclusive decision.
Extensions allow controlled Markovian observations and nonuniform control
costs \citep{nitinawarat2015controlled}; more recently,
\citet{vershinin2026active} study heterogeneous action costs and show that the
relevant efficiency criterion is expected information gain divided by expected
cost.

These studies generally model each action as selecting an experiment-level
observation law. In the standard binary fixed-sample model, a stationary
open-loop control can already attain the optimal error exponent
\citep{nitinawarat2013controlled}; the value of adaptivity in our setting
instead comes from sequential stopping, direction-dependent cost efficiency,
and report-contingent follow-up. Our model has a hybrid fixed-pool and
sequential structure: the number of items is chosen and paid for in advance, but sensing
and stopping within that pool remain adaptive. The sensing technologies are
also nested within an item. An AI query first generates a noisy report, after
which the policy may purchase an exact human label on the same item. The
current likelihood ratio selects a direction-specific mixture of direct-human
and AI-first sensing, while the realized AI report determines human follow-up
through a direction-specific residual-information frontier. We minimize
worst-case expected acquisition, AI, and human cost subject to separate type-I
and type-II error constraints. This structure creates a full-label pool-size
floor and a capacity-aware transcript-KL lower bound, and it supports a
finite-sample-valid policy that matches the lower bound to first order.

\section{Model Setup}\label{sec:model}

We study a hypothesis testing problem in which each
data item $i$ carries a binary label $X_i\in\{0,1\}$
with population mean $p$. Although the binary label
assumption may appear restrictive, it covers a wide
range of practically important settings. A binary
label naturally captures any pass/fail, accept/reject,
or present/absent decision, which are among the most
common output types in AI-assisted workflows. Examples
include automated correctness checking of documents,
code, and mathematical proofs \citep{kryscinski2020evaluating, tong2024codejudge, dekoninck2026open},
AI-assisted legal adjudication of decisions such as
``guilty'' or ``innocent''
\citep{imai_judge}, and medical
screening of patient records for a binary health
outcome.

Formally, we test the simple hypotheses
\begin{equation}\label{eq:hypotheses}
    H_0:p=p_0,
    \qquad
    H_1:p=p_1,
    \qquad
    0<p_0<p_1<1.
\end{equation}
The ordering $p_0<p_1$ is adopted for concreteness;
the case $p_1<p_0$ is symmetric and can be analyzed
analogously. To fix ideas, consider a company that wants to assess
whether an AI coding assistant produces correct code
more than $p_0$ fraction of the time. The company
can collect a batch of code submissions but even collecting each
submission carries a cost (e.g., compute time,
storage, or API fees), and the correct label
$X_i\in\{0,1\}$ of each submission is not immediately
known. To obtain a label, the company can either ask
an AI checker to review the submission (fast and
cheap, but potentially inaccurate) or hire a human
expert to verify it (slow and expensive, but exact). 
The central question is: can we design a policy that
intelligently mixes AI and human queries (e.g.,
escalating to a human only when the AI evidence is
insufficient) and still retain the same statistical
guarantees as full human labeling, but at a fraction
of the cost?

To make this precise, for any positive integer $m$
we write $[m]:=\{1,2,\ldots,m\}$. A policy $\pi$
first acquires $N^\pi$ i.i.d.\ items. Each item
$i\in[N^\pi]$ has a hidden binary label
$X_i\in\{0,1\}$; under $H_h$ the labels are
distributed as $X_i\sim\Bern(p_h)$, but they are
not observed unless the policy explicitly queries
a human on item $i$. A cost $\cdata>0$ is incurred
for each acquired item regardless of whether it is
subsequently queried. Importantly, we focus on a
\emph{fixed-sample} design: the number of acquired items
$N^\pi$ is chosen in advance and does not depend
on the data, in contrast to sequential testing
where the sample size is itself an adaptive
stopping decision. Given the fixed pool of $N^\pi$
items, the policy may then query the AI, query a
human, both, or neither on each item. The central
decision is therefore how to jointly choose $N^\pi$
and how to allocate AI and human queries across
the acquired items, so as to minimize total cost
while meeting the target type-I and type-II error
constraints (to be specified below). We formalize
the AI output model next.

\subsection{AI Outputs}

If the policy queries the AI on item $i$, it receives
a report $R_i$, which may encode a predicted label, a
confidence score, or any other finite metadata. For
example, in the code correctness setting, a typical
AI response might be $R_i = (1, 0.99)$, indicating
that the AI predicts $X_i = 1$ (correct) with $99\%$
confidence. We allow $R_i$ to be as flexible as the
practitioner requires; the only constraint is that
its range is finite. Formally, let $\calR$ be the finite set of all
possible AI outputs. Conditional on $X_i = x$, the
report takes value $r\in\calR$ with probability
\[
    f_x(r) := \Pp(R_i = r \mid X_i = x),
    \qquad r\in\calR,\quad x\in\{0,1\}.
\]
We assume throughout the main analysis (relaxed later in Section~\ref{sec:unknown-ai}) that these
conditional distributions are known:

\vspace{1mm}
\begin{assumption}[Known strictly positive finite-alphabet
AI-output model]\label{ass:score-model}
The conditional probability mass functions $f_0$ and
$f_1$ are known. They satisfy 
$
    f_x(r)>0, \, \forall r\in\calR, \, x\in\{0,1\},
$
and
$
    \sum_{r\in\calR}f_x(r)=1,
    \, \forall x\in\{0,1\}.
$
\end{assumption}
\vspace{1mm}

The strict positivity condition $f_x(r)>0$ rules out
report values that are structurally impossible under
one of the labels, which simplifies the analysis
without materially restricting the model. In the simplest case where $R_i$ is just a predicted
label $\hat{X}_i\in\{0,1\}$, Assumption~\ref{ass:score-model}
reduces to knowing the AI's confusion matrix: the
probability of predicting $1$ when the true label is
$1$, and the probability of predicting $0$ when the
true label is $0$. More generally, knowing $f_x(r)$
means knowing the full conditional distribution of
the AI output given the true label, which
characterizes how informative and reliable the AI is.

The known-$f$ assumption provides a useful benchmark
for understanding how AI and human information should
be combined when the AI's accuracy profile is known.
In practice, however, $f_0$ and $f_1$ generally must
be estimated from data for which both the AI report
and the human-verified label are observed (e.g., using pilot data). For completeness, we return
to this issue in
Section~\ref{sec:unknown-ai}, where we relax the
known-$f$ assumption, estimate $f_0$ and $f_1$ from
an independent paired pilot sample, and study how
estimation uncertainty affects both statistical
validity and cost.

\iffalse
In practice, $f_0$ and $f_1$ can be estimated from
a small labeled pilot sample or from prior domain
knowledge about the AI system's performance. There
is a growing literature on LLM calibration that
develops principled methods for estimating exactly
these conditional output distributions from data
{\color{red} ST: citation}. We treat the
known-$f$ case as a natural and tractable starting
point.
\fi

In addition to Assumption \ref{ass:score-model}, we also impose the standard assumption that items
are independent and identically distributed (i.i.d.).

\vspace{1mm}
\begin{assumption}[Independent items]\label{ass:iid}
Under each hypothesis $H_h$, the pairs
$(X_1,R_1),\ldots, (X_N,$ $R_N)$ are independent and
identically distributed.
\end{assumption}
\vspace{1mm}

Let $\Pp_h$ denote the joint distribution of
$(X_i,R_i)$ under $H_h$. The marginal probability
that the AI report takes value $r$ under $H_h$ is
\begin{equation}\label{eq:gh-def}
    g_h(r)
    :=\Pp_h(R_i=r)
    =p_h f_1(r)+(1-p_h)f_0(r),
    \qquad r\in\calR,
\end{equation}
where the second equality follows by the law of
total probability, conditioning on $X_i\in\{0,1\}$.
By Assumption~\ref{ass:score-model} and $p_h\in(0,1)$,
we have $g_h(r)>0$ for every $r\in\calR$ and both
$h\in\{0,1\}$.

After observing the AI report $R_i=r$, the posterior
probability that the hidden label equals one under
$H_h$ follows from Bayes' rule:
\begin{equation}\label{eq:q-h-def}
    q_h(r):=\Pp_h(X_i=1\mid R_i=r)
    =\frac{p_hf_1(r)}{p_hf_1(r)+(1-p_h)f_0(r)}
    =\frac{p_hf_1(r)}{g_h(r)}.
\end{equation}
Conditional on the AI report $R_i=r$, a subsequently
revealed human label $X_i$ is therefore Bernoulli
with success probability $q_h(r)$, with mass function
\begin{equation}\label{eq:rho-def}
    \rho_h(x\mid r)
    :=q_h(r)^x(1-q_h(r))^{1-x},
    \qquad x\in\{0,1\},\quad r\in\calR.
\end{equation}
Throughout the paper, all quantities carrying a
subscript $h$ depend on whether the true proportion
is $p_0$ or $p_1$, and our analysis proceeds under
each hypothesis separately.

\subsection{Feasible Policy}

We consider a general setting where a policy can be a dynamic, data-dependent, and possibly
randomized decision rule. At each time step
$t = 1, 2, \dots$, the policy observes the history
of past actions and outcomes, and chooses the next
action, stopping at some time $t \leq 2N^\pi + 1$.
The upper bound $2N^\pi + 1$ arises as follows. Each
of the $N^\pi$ acquired items can be queried at most
once by the AI and at most once by a human, giving
at most $2N^\pi$ paid queries in total. The
additional $+1$ accounts for the final stopping
decision, at which the policy rejects or accepts
$H_0$. The policy is free to stop at any time and
act on the items in any order: for instance, it may
query the AI on item $i=3$ at $t=1$, then query a
human on item $i=1$ at $t=2$, and stop and decide
at $t=3$ without ever querying the remaining items.
We now formalize this.

Given $N^\pi$ acquired items, at each time step $t$
the policy chooses one of three actions:
\vspace{1mm}
\begin{enumerate}[label=(\roman*)]
    \item query the AI on item $i\in[N^\pi]$,
    denoted $A_t = \AI(i)$, which returns the AI
    report $O_t = R_i\in\calR$;
    \item query a human evaluator on item
    $i\in[N^\pi]$, denoted $A_t = \Hum(i)$, which
    reveals the exact label $O_t = X_i\in\{0,1\}$;
    \item stop, denoted $A_t = \text{STOP}$, and
    emit a final decision $\delta^\pi\in\{0,1\}$,
    where $\delta^\pi = 1$ means reject $H_0$.
\end{enumerate}

\vspace{1mm}
Each item may be AI-queried at most once and
human-queried at most once. The policy may query
the AI on item $i$ before or after querying a human
on the same item. However, once the human reveals
the exact label $X_i$, any subsequent AI query on
item $i$ yields no additional information about the
hypothesis and wastes cost $\cAI$. As we show
formally in Section~\ref{sec:lower-bound}, such
queries contribute zero statistical evidence and are therefore never used by an
optimal policy.

To formalize the information available to the policy
at each time step, let
\[
    \mathcal{H}_t
    := \sigma(A_1, O_1,\ldots, A_{t}, O_{t})
\]
be the filtration generated by the actions and
observations up to the end of epoch $t$, with
$\mathcal{H}_0 = \emptyset$. This filtration is
defined for all $t$ such that $A_{t'}\neq\text{STOP}$
for all $t'\leq t$. To allow for randomized policies, let
$U = (U_t)_{t=1}^{2N^\pi+1}$ be a vector of i.i.d.\
random seeds, independent of $(X_i,R_i)_{i=1}^{N^\pi}$
under both hypotheses, with common law $\mu$ that
does not depend on the hypothesis. The seed $U_t$
is used at epoch $t$ to randomize the policy's
action. 

At each epoch $t$, the set of actions
available to the policy is
\begin{align*}
    \calA_t(\mathcal{H}_{t-1})
    :=\{&\AI(i): i\in [N^\pi],
    i\text{ not yet AI-queried}\}\\
    & \cup \, \{\Hum(i): i\in [N^\pi],
    i\text{ not yet human-queried}\}\\
    & \cup \, \{\text{STOP}\},
\end{align*}
which encodes the constraint that each item may be
AI-queried at most once and human-queried at most
once. We now formally define an admissible policy.

\vspace{1mm}
\begin{definition}[Admissible policy]
\label{def:admissible_policy}
An admissible policy is a sequence of measurable
functions $\pi = (\pi_1, \pi_2, \ldots)$ such that
$A_t = \text{STOP}$ for some $t\in[2N^\pi+1]$, and
\begin{equation*}
\begin{split}
    \pi_t: \mathcal{H}_{t-1}\times U_t
    \to A_t\in\calA_t(\mathcal{H}_{t-1}),
    &\quad\text{if }A_{t'}\neq\text{STOP}
    \text{ for all }t'<t,\\
    \pi_t: \mathcal{H}_{t-1}\times U_t
    \to \delta^\pi\in\{0,1\},
    &\quad\text{if }A_t=\text{STOP}.
\end{split}
\end{equation*}
\end{definition}
\vspace{1mm}

Definition~\ref{def:admissible_policy} captures the
key requirements of a valid policy: actions are
chosen based only on the observed history and the
current random seed, the policy is guaranteed to
stop within $2N^\pi+1$ steps, and upon stopping it
emits a binary decision to reject or accept $H_0$.

Not all admissible policies have good statistical
properties. We restrict attention to policies that
simultaneously control both type-I and type-II
errors at prescribed levels. Let
$\ltime := \min\{t: A_t = \text{STOP}\}$ be the
stopping time of the policy, and let $\Pp^\pi_h$
and $\E^\pi_h$ denote probability and expectation
under hypothesis $H_h$ and policy $\pi$.

\vspace{1mm}
\begin{definition}[Feasible policy]
\label{def:feasible}
A policy $\pi$ is feasible for target errors
$(\alpha,\beta)$ with $0<\alpha<1$, $0<\beta<1$,
and $\alpha+\beta<1$, if it is admissible and
satisfies
\begin{equation}\label{eq:error-constraints}
    \underbrace{\Pp^\pi_0(\delta^\pi=1)}_{\text{type-I error}}\le\alpha,
    \qquad
    \underbrace{\Pp^\pi_1(\delta^\pi=0)}_{\text{type-II error}}\le\beta.
\end{equation}
We let $\mathcal{F}(\alpha,\beta)$ denote the class
of all feasible policies for target errors
$(\alpha,\beta)$.
\end{definition}
\vspace{1mm}

By restricting to
$\mathcal{F}(\alpha,\beta)$, we ensure that the
policies we consider provide meaningful statistical
guarantees, and our goal is to find the one among
them that minimizes cost.

\subsection{Optimal Cost}

There are three cost components: a data acquisition
cost $\cdata>0$ per acquired item, an AI query cost
$\cAI>0$ per AI query, and a human query cost
$\cH>0$ per human query. While in most practical
settings we have $\cH\gg\cAI$, we do not impose this
ordering in our analysis; our results hold for any positive cost
parameters. Let
\[
    N^{\mathrm{tot}}_{\AI}
    := \sum_{t=1}^{2N^\pi+1}
    \ind\{A_t=\AI(i)\text{ for some }i\in[N^\pi]\}
\]
be the total number of AI queries, and
\[
    N^{\mathrm{tot}}_{\Hum}
    := \sum_{t=1}^{2N^\pi+1}
    \ind\{A_t=\Hum(i)\text{ for some }i\in[N^\pi]\}
\]
be the total number of human queries. The total
cost incurred by policy $\pi$ is
\begin{equation*}\label{eq:cost-def}
    C^\pi:=\cdata N^\pi+\cAI N^{\mathrm{tot}}_{\AI}+\cH N^{\mathrm{tot}}_{\Hum}.
\end{equation*}

Since the true hypothesis is unknown, we evaluate
a policy by its worst-case expected cost over the
two hypotheses as typically done in the hypothesis testing literature \citep{waldSPRT, baraud2002nonasymptotic}. Specifically, the optimal cost is
\begin{equation}\label{eq:objective}
    C^*(\alpha,\beta)
    :=\inf_{\pi\in\mathcal{F}(\alpha,\beta)}
      \max\bigl\{\E^\pi_0[C^\pi],\E^\pi_1[C^\pi]
      \bigr\}.
\end{equation}

An exact closed-form solution of \eqref{eq:objective}
is generally intractable. Thus, we will study the
asymptotic regime in which the target errors
$\alpha$ and $\beta$ vanish. This regime is
practically relevant because in many high-stakes
applications (such as medical diagnosis, legal
adjudication, or quality control) decision makers often
require stringent error guarantees, and it is
precisely in this low-error regime that the
structure of the optimal policy becomes clear and
analytically tractable. In this regime, we seek
a policy that is asymptotically optimal, i.e., a
feasible policy $\pi\in\calF(\alpha,\beta)$
satisfying
\[
    \frac{\max\{\E^\pi_0[C^\pi],\E^\pi_1[C^\pi]\}}
         {C^*(\alpha,\beta)}
    \to 1
    \qquad\text{as }\alpha,\beta\to 0.
\]
In other words, the policy achieves the same
leading-order cost as the best possible feasible
policy, with only lower-order terms left unmatched.

We construct such a policy in
Section~\ref{sec:upper-bound}. To guide its
construction and to certify its optimality, we
first derive in Section~\ref{sec:lower-bound} a
lower bound on $C^*(\alpha,\beta)$ that any feasible
policy must satisfy.

\section{Lower Bound on the Optimal Cost}
\label{sec:lower-bound}

In this section, we derive a lower bound on the
optimal cost $C^*(\alpha,\beta)$ defined in
\eqref{eq:objective}. The analysis of the lower bound has two
components. First, in
Section~\ref{sec:lower-bound-samples}, we show
that any feasible policy must acquire at least
$N_{\fixed,\Hum}(\alpha,\beta)$ items, the minimum number
of items needed to test \eqref{eq:hypotheses} even
with full access to all labels. Second, in
Sections~\ref{sec:information-quantity}
and~\ref{sec:lower-bound-optimization}, we use
information-theoretic arguments to show that any
feasible policy must spend enough on AI and human
queries to accumulate sufficient statistical
evidence to distinguish $H_0$ from $H_1$. Each
item can contribute evidence in one of three
informative ways: through an AI query alone,
through a human query alone, or through both an
AI query and a human query on the same item. Each
of these contributes a quantifiable amount of
statistical evidence at a certain cost, and the
lower bound captures the minimum cost of assembling
enough evidence to meet the error constraints.
Together, these two components yield the lower
bound $\LB(\alpha,\beta)$, which we show in
Section~\ref{sec:upper-bound} is achievable to
first order by our proposed policy. 

\subsection{Lower Bound on Number of Acquired
Samples}\label{sec:lower-bound-samples}

We begin by establishing a fundamental lower bound
on the number of items any feasible policy must
acquire. The benchmark is $N_{\fixed,\Hum}(\alpha,\beta)$,
the minimum number of items needed to test
\eqref{eq:hypotheses} even in the idealized setting
where all labels are observed directly, with no AI
or human query costs. Any feasible policy in our
setting, which has access to strictly less
information per item than the full label, cannot
possibly require fewer items.

\vspace{1mm}
\begin{definition}[Full-label fixed-sample-size
benchmark]\label{def:Nfull}
For $0<\alpha< 1$ and $0<\beta<1$, let $N_{\fixed,\Hum}$ $(\alpha,\beta)$
be the smallest integer $N\in\N$ for which there
exists a randomized test
$
    \phi_N:\{0,1\}^N\to[0,1]
$
satisfying
\[
    \E_0[\phi_N(X_1,\ldots,X_N)]\le\alpha,
    \qquad
    \E_1[1-\phi_N(X_1,\ldots,X_N)]\le\beta,
\]
where $\phi_N(x)\in[0,1]$ is the probability
of rejecting $H_0$ upon observing the full label
vector $x\in\{0,1\}^N$, and $\E_h$ denotes
expectation under $H_h$.
\end{definition}
\vspace{1mm}

We first show that any feasible policy
$\pi\in\calF(\alpha,\beta)$ must acquire at least
$N_{\fixed,\Hum}(\alpha,\beta)$ items.

\vspace{1mm}
\begin{lemma}[Full-label data-pool lower
bound]\label{lem:data-pool}
Suppose Assumptions~\ref{ass:score-model}
and~\ref{ass:iid} hold. If a feasible policy
$\pi\in\calF(\alpha,\beta)$ acquires $N$ items,
then
$
    N\ge N_{\fixed,\Hum}(\alpha,\beta).
$
\end{lemma}

\vspace{1mm}
The intuition behind Lemma~\ref{lem:data-pool} is
straightforward. Even if a policy could somehow
observe all $N$ true labels at no cost, it would
still need at least $N_{\fixed,\Hum}(\alpha,\beta)$ items
to meet the error constraints. A policy in
our setting, which must pay for each label it
observes and receives only noisy AI reports on
some items, has access to no more information than
the full-label benchmark. It therefore cannot
satisfy the same error constraints with fewer
items. 

We now describe how to compute $N_{\fixed,\Hum}(\alpha,\beta)$
explicitly. This construction is also needed for the
fallback test in Section~\ref{sec:upper-bound}. By the Neyman--Pearson lemma~\citep{neyman1933ix},
for any fixed sample size $N$, the most powerful
test at type-I level $\alpha$ is the
likelihood-ratio test. Since the likelihood ratio
is strictly increasing in the label sum $L_N :=
\sum_{i=1}^N X_i$ (as we show below), this reduces
to a simple threshold test on $L_N$. We can
therefore compute $N_{\fixed,\Hum}(\alpha,\beta)$ by
increasing $N$ from $1$ upward and checking whether
the threshold test meets the type-II target $\beta$.

Fix $N\in\N$ and suppose the full label vector
$X=(X_1,\ldots,X_N)$ is observed. Under $H_h$,
$L_N$ has a binomial distribution with parameters
$N$ and $p_h$. The likelihood ratio under $H_1$
versus $H_0$ is
\begin{align}\label{eq:full-label-lr}
\notag
    \Lambda_N(X)
    =\prod_{i=1}^N
      \frac{p_1^{X_i}(1-p_1)^{1-X_i}}
           {p_0^{X_i}(1-p_0)^{1-X_i}}
    =\left(\frac{p_1}{p_0}\right)^{L_N}
      \left(\frac{1-p_1}{1-p_0}\right)^{N-L_N}.
\end{align}
Since $p_1>p_0$, the log-likelihood ratio
\begin{equation}\label{eq:full-label-llr}
\notag
    \log\Lambda_N(X)
    =L_N\log\frac{p_1}{p_0}
    +(N-L_N)\log\frac{1-p_1}{1-p_0}
\end{equation}
has a strictly positive coefficient on $L_N$, so
$\Lambda_N(X)$ is strictly increasing in $L_N$.
The optimal test therefore rejects $H_0$ when
$L_N$ is large.

For a target type-I level $\alpha$, define $c_N$
to be the smallest integer in $\{0,1,\ldots,N\}$
such that $\Pp_0(L_N>c_N)\le\alpha$, and set
\begin{equation}\label{eq:gammaN-def}
\notag
    \gamma_N:=\frac{\alpha-\Pp_0(L_N>c_N)}
                   {\Pp_0(L_N=c_N)}.
\end{equation}
The denominator is positive because $0<p_0<1$,
and $\gamma_N\in[0,1]$ by the definition of $c_N$.
This gives rise to the following test.

\vspace{1mm}
\begin{definition}[Neyman--Pearson
test~\citep{casella2024statistical}]
The randomized threshold test
\begin{equation*}\label{eq:phi-star}
    \phi_N^*(x)=
    \begin{cases}
        1, & \sum^N_{i=1} x_i>c_N,\\
        \gamma_N, & \sum^N_{i=1} x_i=c_N,\\
        0, & \sum^N_{i=1} x_i<c_N,
    \end{cases}
\end{equation*}
where $x=(x_1,\ldots,x_N)\in\{0,1\}^N$, is called
the Neyman--Pearson test.
\end{definition}
\vspace{1mm}

By construction, $\phi_N^*$ meets the type-I
constraint exactly: $\E_0[\phi_N^*(X)]=\alpha$.
Its type-II error is
\begin{equation}\label{eq:typeII-full}
    \E_1[1-\phi_N^*(X)]
    =\Pp_1(L_N<c_N)+(1-\gamma_N)\Pp_1(L_N=c_N).
\end{equation}
The Neyman--Pearson lemma guarantees that this
test is optimal, which we state formally below.

\vspace{1mm}
\begin{lemma}[Neyman--Pearson
lemma~\citep{neyman1933ix}]
The test $\phi_N^*(X)$ has the smallest type-II
error among all tests $\phi(X)$ satisfying
$\E_0[\phi(X)]\leq\alpha$. That is, $\phi_N^*$
is the most powerful test at significance level
$\alpha$.
\end{lemma}

\vspace{1mm}
It follows that $N_{\fixed,\Hum}(\alpha,\beta)$ can be
computed exactly by increasing $N$ from $1$ upward
and checking whether \eqref{eq:typeII-full} is at
most $\beta$. The logic follows from the Neyman--Pearson lemma: at each $N$,
$\phi_N^*$ has the smallest type-II error of any
test with type-I error at most $\alpha$. So if
$\phi_N^*$ fails to meet the type-II target $\beta$,
no test at that sample size can. Conversely, if
$\phi_N^*$ does meet $\beta$, it is itself a valid
test meeting both error targets. The first $N$ at
which \eqref{eq:typeII-full} is at most $\beta$ is
therefore $N_{\fixed,\Hum}(\alpha,\beta)$.

\subsection{Information Quantities}
\label{sec:information-quantity}

At its core, any feasible policy must accumulate
enough statistical evidence to reliably distinguish
$H_0$ from $H_1$. The fundamental currency of this
evidence is KL divergence: it quantifies how
distinguishable the distribution of observations
is under $H_1$ versus $H_0$, and therefore measures
how much each query contributes toward meeting the
error constraints. We now define the KL-divergence
quantities associated with each possible query type.

We use two standard KL-divergence quantities
throughout. For $a,b\in(0,1)$, the Bernoulli KL
divergence is given by
\[
    \klbin(a\Vert b)
    :=a\log\frac{a}{b}
    +(1-a)\log\frac{1-a}{1-b}.
\]
For probability mass functions $P$ and $Q$ on a
common finite set $S$, the KL divergence is
\[
    D(P\Vert Q)
    :=\sum_{s\in S}P(s)\log\frac{P(s)}{Q(s)},
\]
with conventions $0\log(0/q)=0$ and
$p\log(p/0)=+\infty$ for $p>0$.

\vspace{1mm}
\paragraph{AI query only.}
When the policy queries only the AI on item $i$,
it observes the report $R_i$ with marginal
distribution $G_h$ (with probability mass function
$g_h$) under $H_h$. Since the two error constraints
\eqref{eq:error-constraints} are asymmetric
(one applies when $H_0$ is true and the other when
$H_1$ is true) we need to track the
discriminating power of an AI query under each
hypothesis separately. Specifically, $I_R^{(1)}$
measures how informative the AI report is for
distinguishing $H_1$ from $H_0$ when $H_1$ is the
true hypothesis, and $I_R^{(0)}$ measures the same
when $H_0$ is true:
\begin{equation}\label{eq:IR1}
    I_R^{(1)}:=D(G_1\Vert G_0)
    =\sum_{r\in\calR}g_1(r)\log\frac{g_1(r)}{g_0(r)},
\end{equation}
\begin{equation}\label{eq:IR0}
    I_R^{(0)}:=D(G_0\Vert G_1)
    =\sum_{r\in\calR}g_0(r)\log\frac{g_0(r)}{g_1(r)}.
\end{equation}
In general $I_R^{(1)}\ne I_R^{(0)}$, reflecting
the fact that the AI report may be more informative
under one hypothesis than the other.

\vspace{1mm}
\paragraph{Human query only.}
If the policy queries a human on item $i$ with
no prior AI query, the revealed label $X_i$ is
distributed as $\Bern(p_h)$ under $H_h$. By the
same reasoning as above, we track the
discriminating power of a human query under each
hypothesis separately. Specifically, $J_X^{(1)}$
measures how informative a directly revealed label
is for distinguishing $H_1$ from $H_0$ when $H_1$
is true, and $J_X^{(0)}$ measures the same when
$H_0$ is true:
\begin{equation}\label{eq:JX-def}
    J_X^{(1)}:=\klbin(p_1\Vert p_0),
    \qquad
    J_X^{(0)}:=\klbin(p_0\Vert p_1),
\end{equation}
both of which are strictly positive because
$0<p_0<p_1<1$. Note that $J_X^{(1)}\ne J_X^{(0)}$
in general, for the same reason as above.

\vspace{1mm}
\paragraph{Human query followed by AI query.}
If item $i$ is first queried by a human, revealing
the exact label $X_i$, and is then queried by the
AI, the AI report $R_i$ carries no additional
information: the true label $X_i$ is already known,
so the AI output is redundant. This query ordering
therefore contributes zero additional KL divergence
and is never used by an optimal policy.

\vspace{1mm}
\paragraph{AI query followed by human query.}
If the policy first queries the AI on item $i$,
receiving report $R_i = r$, and then queries a
human on the same item, the human reveals the
exact label $X_i$. Conditional on $R_i = r$, the
label $X_i$ is distributed as $\Bern(q_h(r))$
under $H_h$ by \eqref{eq:rho-def}. The information
this human label carries for discriminating $H_1$
from $H_0$, beyond what the AI report $r$ already
provided, depends on which hypothesis is true.
When $H_1$ is true, the additional information is
\begin{equation}\label{eq:d1-def}
    d^{(1)}(r)
    :=D(\rho_1(\cdot\mid r)\Vert\rho_0(\cdot\mid r))
    =\klbin(q_1(r)\Vert q_0(r)),
\end{equation}
and when $H_0$ is true it is
\begin{equation}\label{eq:d0-def}
    d^{(0)}(r)
    :=D(\rho_0(\cdot\mid r)\Vert\rho_1(\cdot\mid r))
    =\klbin(q_0(r)\Vert q_1(r)).
\end{equation}
Note that $d^{(h)}(r)$ depends on the AI report
$r$: a more informative AI report leaves less
residual uncertainty about $X_i$, and hence
contributes less additional information when the
human is subsequently queried.

\subsection{Supporting Results for the Lower Bound}
\label{sec:preliminary-results}

%{\color{red} ST: i added this new subsection. The lemmas and theorem here are important to understand the logic behind the info theoretic approach to lower bound so they need to be presented earlier. }

To derive a lower bound on $C^*(\alpha,\beta)$, we
need to understand how much statistical evidence
any admissible policy can accumulate, and at what
cost. The key insight is that any feasible policy
must accumulate enough statistical evidence to
distinguish $H_0$ from $H_1$, and this evidence
can only be obtained by paying for AI queries,
human queries, or both. We formalize this by
measuring the statistical evidence in terms of KL
divergence between the distributions of the
policy's actions and observations under the two
hypotheses.

For a policy $\pi$ with stopping time $\ltime$,
we call the realized sequence
\[
    (A_1, O_1, \ldots, A_{\ltime-1},
    O_{\ltime-1}, A_{\ltime}, \delta^\pi)
\]
the \emph{transcript} of the policy.
Let $P^\pi_h$ denote the distribution of the
transcript under $H_h$. The KL divergence $D(P_1^\pi\Vert P_0^\pi)$
is the expected log-likelihood ratio of the
transcript distribution under $H_1$, measuring
how much $P_1^\pi$ favors $H_1$ over $H_0$.
Similarly, $D(P_0^\pi\Vert P_1^\pi)$ is the
expected log-likelihood ratio under $H_0$,
measuring how much $P_0^\pi$ favors $H_0$ over
$H_1$. Both must be large enough for the policy
to reliably distinguish the two hypotheses and
meet the error constraints
\eqref{eq:error-constraints}.

\vspace{1mm}
\begin{lemma}[Testing errors imply transcript KL
requirements]\label{lem:error-to-kl}
Every feasible policy $\pi\in\calF(\alpha,\beta)$
satisfies
\[
    D(P_1^\pi\Vert P_0^\pi)\ge A,
    \qquad
    D(P_0^\pi\Vert P_1^\pi)\ge B,
\]
where
\begin{equation*}\label{eq:A-B-def}
    A:=\klbin(1-\beta\Vert\alpha),
    \qquad
    B:=\klbin(1-\alpha\Vert\beta).
\end{equation*}
\end{lemma}
\vspace{1mm}

Lemma~\ref{lem:error-to-kl} says that meeting
the error constraints forces the transcript
distributions $P_1^\pi$ and $P_0^\pi$ to be
sufficiently separated: $P_1^\pi$ must be at
least $A$ units of KL divergence away from
$P_0^\pi$, and $P_0^\pi$ must be at least $B$
units away from $P_1^\pi$. We next characterize
exactly how this total KL information accumulates
across the policy's query decisions.

For a policy $\pi$, let $\calI_{\AI}^\pi$ be the
set of items whose first query is an AI query
(which may or may not be followed by a human
query), $\calI_{\Hum}^\pi$ the set of items whose
first query is a human query with no preceding AI
query, % {\color{red} ST: shall we change all ``dir'' with ``H'' and ``scored'' with ``AI''?}\xz{Updated.}, 
and $\calI_{\AI\Hum}^\pi$ the set of items
that are human-queried after their AI report has
been observed. Let $N^{\dir, \pi}_{\AI}:=|\calI_{\AI}^\pi|$,
$N^{\dir,\pi}_{\Hum}:=|\calI_{\Hum}^\pi|$, and
$N^{\pi}_{\AI\Hum}:=|\calI_{\AI\Hum}^\pi|$ be the
corresponding random counts. Note that these counts
are random and their expectations under $H_0$ and
$H_1$ differ in general. An AI query on an item
that has already received a human query contributes
zero KL information, since the report law $f_x$
does not depend on the hypothesis once the label
$x$ is known.

\vspace{1mm}
\begin{theorem}[Adaptive KL
decomposition]\label{thm:kl-decomposition}
Suppose Assumptions~\ref{ass:score-model}
and~\ref{ass:iid} hold.
For every admissible policy $\pi$ using $N^\pi$
items,
\begin{align}
    D(P_1^\pi\Vert P_0^\pi)
    &=\E^\pi_1\!\left[
        N^{\dir, \pi}_{\AI} I_R^{(1)}
        +N^{\dir,\pi}_{\Hum} J_X^{(1)}
        +\sum_{i\in\calI_{\AI\Hum}^\pi}d^{(1)}(R_i)
      \right],
      \label{eq:forward-kl-decomp}\\
    D(P_0^\pi\Vert P_1^\pi)
    &=\E^\pi_0\!\left[
        N^{\dir, \pi}_{\AI} I_R^{(0)}
        +N^{\dir,\pi}_{\Hum} J_X^{(0)}
        +\sum_{i\in\calI_{\AI\Hum}^\pi}d^{(0)}(R_i)
      \right].
      \label{eq:reverse-kl-decomp}
\end{align}
\end{theorem}
\vspace{1mm}

Theorem~\ref{thm:kl-decomposition} shows that the
transcript KL divergence decomposes exactly into
three per-item contributions: each AI-first item
contributes $I_R^{(h)}$, each human-first item
contributes $J_X^{(h)}$, and each AI-first item
that is subsequently human-queried contributes an
additional $d^{(h)}(R_i)$ depending on its
realized report $R_i$. The escalation term
$\sum_{i\in\calI_{\AI\Hum}^\pi}d^{(h)}(R_i)$
depends on the realized reports, making it
difficult to work with directly. We therefore define, for $s\in[0,1]$,
\begin{equation}\label{eq:Psi-frontier-def}
    \Psi_h(s)
    :=\inf_{\lambda\ge0}
    \left\{
        \lambda s+
        \sum_{r\in\calR}g_h(r)
        \pos{d^{(h)}(r)-\lambda}
    \right\},
\end{equation}
which provides a deterministic upper bound on the
expected escalation term, as shown in
Lemma~\ref{lem:selected-info-bound} below. The
following lemma establishes the key properties
of $\Psi_h$, including a dual representation
that gives $\Psi_h(s)$ a natural interpretation.

\vspace{1mm}
\begin{lemma}[Properties of the follow-up
frontier]\label{lem:capacity-frontier-properties}
For each $h\in\{0,1\}$, $\Psi_h$ is nondecreasing,
concave, and continuous on $[0,1]$. By linear
programming duality, it admits the equivalent
dual representation
\begin{equation}\label{eq:Psi-dual}
    \Psi_h(s)
    =\max_{\eta:\calR\to[0,1]}
    \left\{
        \sum_{r\in\calR}g_h(r)\eta(r)d^{(h)}(r)
        :
        \sum_{r\in\calR}g_h(r)\eta(r)\le s
    \right\}.
\end{equation}

\vspace{1mm}
\noindent
Moreover,
\begin{equation}\label{eq:Psi-one-chain-rule}
    \Psi_h(1)=J_X^{(h)}-I_R^{(h)}.
\end{equation}
\end{lemma}
\vspace{1mm}

The dual representation \eqref{eq:Psi-dual}
reveals the interpretation of $\Psi_h(s)$. The
variable $\eta(r)\in[0,1]$ can be interpreted as the probability of
querying a human on an AI-first item with report
$r$. Under $H_h$, the report equals $r$ with
probability $g_h(r)$, so the expected KL
information gained from human follow-up queries
under rule $\eta$ is
$\sum_{r\in\calR}g_h(r)\eta(r)d^{(h)}(r)$, and
the expected number of follow-up queries per
AI-first item is $\sum_{r\in\calR}g_h(r)\eta(r)$.
Thus $\Psi_h(s)$ is the maximum expected KL
information extractable per AI-first item, over
all report-dependent follow-up rules, when the
expected follow-up rate is at most $s$. Setting
$s=0$ forbids any human follow-up and gives
$\Psi_h(0)=0$; setting $s=1$ permits following
up every AI-first item. 

Identity
\eqref{eq:Psi-one-chain-rule} then has a clean
interpretation: when every AI-first item is also
followed up by a human ($s=1$), the policy
observes both the AI report and the true label.
The total KL information from observing both
equals $J_X^{(h)}$ by the KL chain rule, of which
$I_R^{(h)}$ was already contributed by the AI
report alone. The remaining $J_X^{(h)}-I_R^{(h)}$
is the additional information the human follow-up
provides, which is exactly $\Psi_h(1)$.

It is worth noting here that the optimal $\eta(r)$ achieving
$\Psi_h(s)$ provides an implementable report-dependent
escalation rule that we actually use to construct the
matching policy in Section~\ref{sec:upper-bound}. Specifically, to construct the frontier optimizer,
without loss of generality we can order the reports as $r^{(h)}_1,\ldots,r^{(h)}_{|\calR|}$ so that
\[
    d^{(h)}(r^{(h)}_1)\ge d^{(h)}(r^{(h)}_2)\ge\cdots\ge
    d^{(h)}(r^{(h)}_{|\calR|}),
\]
with deterministic tie-breaking. 
Let
\[
    W^{(h)}_0:=0,
    \qquad
    W^{(h)}_j:=\sum_{\ell=1}^j g_h(r^{(h)}_\ell),
    \qquad j=1,\ldots,|\calR|.
\]
Then the optimal solution to \eqref{eq:Psi-dual} is to
escalate reports in decreasing order of conditional information and then randomize
only at the marginal report value, given by Proposition~\ref{prop:escalation-solution} below.

\vspace{1mm}
\begin{proposition}\label{prop:escalation-solution} 
    For every $s\in[0,1]$, let 
    \begin{equation}\label{eq:eta-frontier-closed-form}
    \eta^*_h(r^{(h)}_j;s):=
    \begin{cases}
        1, & W^{(h)}_j\le s,\\[1ex]
        \displaystyle\frac{s-W^{(h)}_{j-1}}{g_h(r^{(h)}_j)},
            & W^{(h)}_{j-1}<s<W^{(h)}_j,\\[2ex]
        0, & W^{(h)}_{j-1}\ge s.
    \end{cases}
\end{equation}
Then $\{\eta^*_h(r; s)\}_{r\in\calR}$ is an optimal solution to \eqref{eq:Psi-dual} under direction $h$.
\end{proposition}
\vspace{1mm}

We can now bound the expected
escalation term in Theorem~\ref{thm:kl-decomposition}
by a deterministic quantity.

\vspace{1mm}
\begin{lemma}[Selective escalation information
bound]\label{lem:selected-info-bound}
Suppose Assumption~\ref{ass:iid} holds.
Fix $h\in\{0,1\}$. For any admissible policy $\pi$, we have:
\begin{equation}\label{eq:selected-info-bound}
    \E^\pi_h\!\left[
        \sum_{i\in\calI_{\AI\Hum}^\pi}d^{(h)}(R_i)
    \right]
    \le
    \E^\pi_h[N^{\dir, \pi}_{\AI}]\,
    \Psi_h\!\left(
        \frac{\E^\pi_h[N^{\pi}_{\AI\Hum}]}
             {\E^\pi_h[N^{\dir, \pi}_{\AI}]}
    \right).
\end{equation}

\vspace{1mm}
\noindent
If $\E^\pi_h[N^{\dir, \pi}_{\AI}]=0$, then
$\E^\pi_h[N^{\pi}_{\AI\Hum}]=0$ and the right side is
defined as $0$. 
\end{lemma}
\vspace{1mm}

With all the above results in hand, we are ready to
derive the lower bound. Lemma~\ref{lem:error-to-kl}
tells us that any feasible policy must achieve KL
divergences of at least $A$ and $B$ in the two
directions; Theorem~\ref{thm:kl-decomposition}
decomposes these KL divergences exactly into
per-item contributions from AI queries, human
queries, and human follow-up queries;
Lemma~\ref{lem:selected-info-bound} then bounds
the human follow-up contribution by a deterministic
quantity involving $\Psi_h$. Combining Lemma~\ref{lem:error-to-kl},
Theorem~\ref{thm:kl-decomposition}, and
Lemma~\ref{lem:selected-info-bound}, any feasible
policy $\pi\in\calF(\alpha,\beta)$ satisfies, for
each $h\in\{0,1\}$,
\begin{equation}\label{eq:combined-bound}
    T_h
    \le D(P_h^\pi\Vert P_{1-h}^\pi)
    \le \E_h^\pi[N^{\dir,\pi}_{\Hum}]\,J_X^{(h)}
      +\E_h^\pi[N^{\dir, \pi}_{\AI}]\,I_R^{(h)}
      +\E_h^\pi[N^{\dir, \pi}_{\AI}]\,
        \Psi_h\!\left(
          \frac{\E_h^\pi[N^{\pi}_{\AI\Hum}]}
               {\E_h^\pi[N^{\dir, \pi}_{\AI}]}
        \right),
\end{equation}

\vspace{1mm}
\noindent
where $T_1:=A$ and $T_0:=B$.  Crucially, the right-hand side depends only on
the expected query counts,
reducing a stochastic constraint on the policy to
a deterministic feasibility condition, which we
exploit in the next subsection to derive the lower bound
$\LB(\alpha,\beta)$.

\subsection{Information-Theoretic Lower Bound}
\label{sec:lower-bound-optimization}

We are now ready to derive the lower bound on
$C^*(\alpha,\beta)$. The three results in
Section~\ref{sec:preliminary-results} together
imply that any feasible policy must incur a
minimum cost to generate enough KL information
to distinguish $H_0$ from $H_1$. We formalize
this as a cost minimization program.

For $h\in\{0,1\}$, $T\ge0$, and integer $N\ge1$,
define
\begin{equation}\label{eq:Gamma-def}
\begin{aligned}
    \Gamma_h(T,N):=
    \min_{n_{\Hum},\,n_{\AI},\,n_{\esc}\ge0}\quad
        &\cH n_{\Hum}+\cAI n_{\AI}+\cH n_{\esc}\\
    \text{subject to}\quad
        &n_{\Hum}+n_{\AI}\le N,\\
        &0\le n_{\esc}\le n_{\AI},\\
        &n_{\Hum}J_X^{(h)}+n_{\AI}I_R^{(h)}
          +n_{\AI}\Psi_h(n_{\esc}/n_{\AI})\ge T.
\end{aligned}
\end{equation}

\vspace{1mm}
\noindent
Here $n_{\Hum}$, $n_{\AI}$, and $n_{\esc}$
are continuous optimization variables representing,
respectively, the expected number of human-first items,
AI-first items, and human follow-up (escalation) queries in
direction $h$. The objective is the total expected
cost: $\cH$ per human-first item, $\cAI$ per
AI-first item, and $\cH$ per human follow-up
query. The first constraint reflects that
human-first and AI-first items are drawn from the
same pool of $N$ acquired items. The second
constraint ensures that human follow-up queries
can only be applied to AI-first items. The information constraint, justified by \eqref{eq:combined-bound}, requires
the total KL information from all three query
types to meet the target $T$, where $T$ will be
set to $A$ when $h=1$ and to $B$ when $h=0$ in
the lower bound \eqref{eq:LB-def} below.

Problem~\eqref{eq:Gamma-def} can be interpreted
as follows: given a pool of $N$ items and a KL
information target $T$, what is the cheapest mix
of human-first items, AI-first items, and human
follow-up queries that meets the target? The
answer depends on the relative costs $\cH$,
$\cAI$, and the information yields $J_X^{(h)}$,
$I_R^{(h)}$, $\Psi_h$. For example, if AI queries
are cheap and informative, the optimizer will
favor AI-first items over human-first items. The minimum in \eqref{eq:Gamma-def} is attained
whenever the feasible set is nonempty, since all
variables are bounded to $[0,N]$ and the objective
is continuous and bounded below by zero. If no
feasible triple satisfies the information
constraint, we set $\Gamma_h(T,N)=+\infty$.

By Theorem~\ref{thm:kl-decomposition},
Lemma~\ref{lem:selected-info-bound}, and
Lemma~\ref{lem:error-to-kl}, the expected query
counts of any feasible policy $\pi$ under $H_h$
satisfy all constraints of \eqref{eq:Gamma-def}
with $T=A$ when $h=1$ and $T=B$ when $h=0$.
Since $\Gamma_h(T,N^\pi)$ is the minimum cost
over all such feasible triples, the expected cost
of $\pi$ under $H_h$ satisfies
\[
    \E^\pi_h[C^\pi]
    \ge N^\pi\cdata+\Gamma_h(T,N^\pi).
\]
Since we minimize the worst-case expected cost
over both hypotheses, we take the maximum of the
two direction-wise lower bounds. Combining with
the sample size floor from Lemma~\ref{lem:data-pool},
we get
\begin{equation}\label{eq:LB-def}
    \LB(\alpha,\beta)
    :=\min_{N\ge N_{\fixed,\Hum}(\alpha,\beta)}
      \left\{
        N\cdata+
        \max\{\Gamma_1(A,N),\Gamma_0(B,N)\}
      \right\}.
\end{equation}
The outer minimum searches over all integers $N\ge N_{\fixed,\Hum}(\alpha,\beta)$. The minimum in
\eqref{eq:LB-def} is attained because only
finitely many sample sizes need to be checked:
once a finite value $V_0$ is achieved at some
$N_0$, every $N>V_0/\cdata$ has objective value
at least $N\cdata>V_0$ and cannot be optimal.
Moreover, $\LB(\alpha,\beta)$ is computationally
tractable, since each inner program $\Gamma_h(T,N)$
is convex.

\vspace{1mm}
\begin{proposition}\label{prop:lower-bound-convexity}
For each $h\in\{0,1\}$, $T\ge0$, and integer
$N\ge1$, the optimization
\eqref{eq:Gamma-def} is convex in
$(n_{\Hum},n_{\AI},n_{\esc})$. Thus,
$\LB(\alpha,\beta)$ can be computed by solving a
finite sequence of convex programs.
\end{proposition}
\vspace{1mm}

The following theorem formalizes the validity of
$\LB(\alpha,\beta)$ as a lower bound on the
optimal cost.

\vspace{1mm}
\begin{theorem}[Selective-scoring cost lower
bound]\label{thm:lower-bound}
Under Assumptions~\ref{ass:score-model}
and~\ref{ass:iid}, we have
$
    C^*(\alpha,\beta)\ge \LB(\alpha,\beta).
$
\end{theorem}
\vspace{1mm}

The lower bound $\LB(\alpha,\beta)$ serves as the
benchmark for our proposed policy. We show in
Section~\ref{sec:upper-bound} that our policy
achieves $\LB(\alpha,\beta)$ to first order as
$\alpha,\beta\to0$, establishing its asymptotic
optimality. The gap between $\LB(\alpha,\beta)$
and $C^*(\alpha,\beta)$ at finite $(\alpha,\beta)$
arises from the information-theoretic relaxation
in Lemma~\ref{lem:error-to-kl}, and vanishes
asymptotically.

\section{A Sequential Cost-Aware Policy} \label{sec:upper-bound}

Section~\ref{sec:lower-bound} established that any
feasible policy must incur cost at least
$\LB(\alpha,\beta)$. We now construct a policy, which we call SCALE (Sequential Cost-Aware Likelihood-Guided Escalation) policy,
that matches this lower bound asymptotically. Specifically, we consider a sequence of
problem instances indexed by $k=1,2,\ldots$, with
type-I and type-II error levels $\alpha_k\downarrow0$
and $\beta_k\downarrow0$ as $k\to\infty$ (we refer to this as first-order asymptotics). For each
$k$, our proposed policy, which we refer to as $\pi_k$, is constructed with respect to $(\alpha_k, \beta_k)$. We describe $\pi_k$ at a high level in
Section~\ref{sec:policy-description}. In Section~\ref{sec:policy-construction}, we give its
formal construction without specifying the values for the tuning parameters  and show that the policy is feasible regardless of those values. In Section~\ref{sec:main-theorem}, we discuss the choice of tuning parameters and
establish the asymptotics of $\pi_k$ with those choices. For interested readers, we lay out the main technical arguments in Appendix~\ref{sec-appendix:LB-growth}.

% , and satisfies
% \[
%     \frac{\max\{\E^{\pi_k}_0[C^{\pi_k}],
%                 \E^{\pi_k}_1[C^{\pi_k}]\}}
%          {\LB(\alpha_k,\beta_k)}
%     \to 1
%     \qquad\text{as }k\to\infty,
% \]
% and consequently
% \[
%     \frac{\max\{\E^{\pi_k}_0[C^{\pi_k}],
%                 \E^{\pi_k}_1[C^{\pi_k}]\}}
%          {C^*(\alpha_k,\beta_k)}
%     \to 1
%     \qquad\text{as }k\to\infty.
% \]
% In particular, define:
% \begin{equation}\label{eq:Lk-def}
%     L_k:=\max\!\left\{\log\frac{1}{\alpha_k},\,
%                       \log\frac{1}{\beta_k}\right\}.
% \end{equation}
% As shown in Lemma~\ref{lem:LB-perturb}, the
% lower bound $\LB_k:=\LB(\alpha_k,\beta_k)$ scales
% as $\LB_k=\Theta(L_k)$, so first-order optimality
% is equivalent to matching $\LB_k$ up to an
% $o(L_k)$ additive term. 

% We describe the policy at a high level in
% Section~\ref{sec:policy-description}, give its
% formal construction and show that the policy is feasible in
% Section~\ref{sec:policy-construction}. In Section~\ref{sec:main-theorem}, we discuss the choice of tuning parameters and
% establish the asymptotics of the algorithm with the tuning parameter.

\subsection{High Level Description of SCALE}\label{sec:policy-description}

The policy proceeds in two stages. The first is the sequential main stage. In this stage, the policy maintains a running log-likelihood-ratio
statistic and uses a Wald-style sequential probability ratio test
\citep{waldSPRT} whose increments are generated by the three types of queries
introduced in Section~\ref{sec:lower-bound}. At the beginning of each epoch
of the main stage, the policy uses the running statistic to determine a
sensing rule, which in turn specifies the probability of querying a human
versus an AI, and the probability of escalating an AI-scored item to human
review. The main stage stops as soon as the statistic crosses $+a_k$
(decide $H_1$) or $-b_k$ (decide $H_0$). If the main stage ends without a
crossing, the policy enters the second stage, a fallback stage that employs
a fixed-sample-size test using either full human labels or full AI labels,
depending on a comparison of the two tests' costs.

\subsection{Policy Construction}\label{sec:policy-construction}

We now formalize the two-stage procedure described
in Section~\ref{sec:policy-description}. To start, we split the type-I and type-II budgets into two parts: one for the sequential main stage, and one for the fallback. Let $f_{\fb, k}\in(0,1)$ be the fraction of type-I and type-II budgets allocated to the fallback. As classically done through the union bound, the type-I and type-II budgets allocated to sequential main stage and
fallback are:
\begin{equation}\label{eq:budget-split-alpha}
    \alpha_{1,k}:=(1-f_{\fb,k})\alpha_k,
    \qquad
    \alpha_{2,k}:=f_{\fb,k}\alpha_k,
\end{equation}
\begin{equation}\label{eq:budget-split-beta}
    \beta_{1,k}:=(1-f_{\fb,k})\beta_k,
    \qquad
    \beta_{2,k}:=f_{\fb,k}\beta_k.
\end{equation}
Then $\alpha_{1,k}+\alpha_{2,k}=\alpha_k$ and
$\beta_{1,k}+\beta_{2,k}=\beta_k$, and the likelihood-ratio boundaries for the main stage sequential test are
\begin{equation}\label{eq:boundaries}
    a_k:=\log\frac1{\alpha_{1,k}},
    \qquad
    b_k:=\log\frac1{\beta_{1,k}}.
\end{equation}

Additionally, we also acquire $N_{\main,k}$ data items. As the fallback test revisits a fixed-sample-size test with type-I and type-II guarantees $\alpha_{2,k}$ and $\beta_{2,k}$, we require $N_{\main,k} \geq N_{\fixed,\Hum}(\alpha_{2,k}, \beta_{2,k})$ to make sure we have enough data items for finite-sample feasibility. In other words, the condition $N_{\main,k} \geq N_{\fixed,\Hum}(\alpha_{2,k}, \beta_{2,k})$ is only required so that our fallback stage has the proper type-I and type-II error guarantees while the main sequential stage immediately guarantees the right $\alpha_{1,k}, \beta_{1,k}$ errors through the sequential boundaries $a_k, b_k$.

\subsubsection*{Sequential main stage.}

This main stage processes the fixed pool in order 
$1,2,\ldots,N_{\main,k}$.  Items are sensed one at a time, in this order,
until a boundary is crossed or the pool is exhausted. In other words, the number of sensed items is a stopping time determined by the likelihood-ratio path. The running statistic starts at $S_0=0$. At the beginning of each epoch during the sequential main stage, we compare the running statistic with the boundary $z_k$ and $-z_k$ to determine the sensing rule for that epoch. We refer to $z_k$ as the ``hypothesis boundary'' as these determine whether we focus on $H_0$ or $H_1$. Specifically, for item $i$, let $S_{i-1}$ be the
statistic before item $i$ is sensed and set
\begin{equation}\label{eq:rule-by-sign}
    J_i:=
    \begin{cases}
        1,& S_{i-1}>z_k,\\
        0,& S_{i-1}<-z_k,\\
        *,& |S_{i-1}|\le z_k.
    \end{cases}
\end{equation}
Then we end up with three types of sensing rules for each item $i$: one favoring $H_0$ (i.e. $J_i = 0$), one favoring $H_1$ (i.e. $J_i = 1$), and another undecided (i.e. $J_i = *$). We note that $J_i$ is only used to determine which hypothesis direction $h = 0,1$ to focus on. As shown in Sections~\ref{sec:preliminary-results}-\ref{sec:lower-bound-optimization}, many important quantities relevant for our policy (e.g., $\Psi_h(s), \eta^*_h(r^{(h)}_j;s)$) differ according to the null $h=0$ or the alternative $h = 1$. Hence $J_i$ aims to determine which $h$-direction to focus on sequentially.

Based on the sensing rule $J_i = j\in\{0,1,*\}$, we choose to query human, query AI, or escalate the item to human after querying AI. With $J_i=j$, the policy directly queries a human with probability
$\eta^{\Hum}_{j,k}\in[0,1]$; otherwise it queries the AI with probability $1-\eta^{\Hum}_{j,k}$, observes the report $R_i$, and escalates
with probability $\eta^{\esc}_{j,k}(R_i)\in[0,1]$. At this point, we emphasize that $z_k$, $\eta^\Hum_{h,k}$ and $\eta^\esc_{h,k}$ are all tuning parameters that do not affect the feasibility and can be tuned to improve the finite sample and asymptotic performance of the algorithm. We will discuss the choice of these tuning parameters in Section~\ref{sec:main-theorem}.

After each AI or human query, we update our running statistic. If a human is queried directly, before any AI query on the same item, we increment the running statistic by the single-observation log-likelihood ratio for the observed label $X=x$ under $H_1$ versus $H_0$:
\begin{equation}\label{eq:direct-human-increment}
    \ell_X(x):=\log\frac{p_1^x(1-p_1)^{1-x}}
                       {p_0^x(1-p_0)^{1-x}},
    \qquad x\in\{0,1\}.
\end{equation}
After a direct AI query with $R = r$, we increment the running statistic by the single-observation log-likelihood ratio of $R = r$ under $H_1$ versus $H_0$:
\begin{equation}\label{eq:AI-report-increment}
    \ell_R(r):=\log\frac{g_1(r)}{g_0(r)},
    \qquad r\in\calR,
\end{equation}
If the query is a human query preceded by an AI query on the same item, we increment the running statistic by the conditional log-likelihood ratio of \(X=x\) given the previously observed AI report \(R=r\) under \(H_1\) versus \(H_0\):
\begin{equation}\label{eq:escalation-increment}
    \ell_H(x,r):=\log\frac{\rho_1(x\mid r)}{\rho_0(x\mid r)},
    \qquad x\in\{0,1\},\ r\in\calR.
\end{equation}
Thus $\ell_X(x)$ and $\ell_R(r)$ capture the evidence provided by a human label and AI label respectively, and $\ell_H(x,r)$ captures the additional evidence provided by the human label beyond the evidence already contained in the AI report. After each increment, the policy stops and rejects $H_0$ if $S_i\ge a_k$, and
stops and accepts $H_0$ if $S_i\le -b_k$.

\subsubsection*{Fallback benchmarks.}
The main sequential stage does not guarantee that we reach a terminal decision to reject or accept the null hypothesis (i.e., we may never cross either $a_k$ or $-b_k$). The fallback stage guarantees, with potentially more queries, that we conclude our hypothesis test with a rejection/acceptance decision.
The human fallback uses the exact full-label fixed-sample-size benchmark $N_{\fixed,\Hum}$ from
Definition~\ref{def:Nfull}.  The AI fallback uses the analogous exact benchmark $N_{\fixed, \AI}$
for i.i.d. AI reports. 

\vspace{1mm}
\begin{definition}[Exact AI-report sample-size benchmark]\label{def:NAI}
For $0<\alpha< 1$ and $0<\beta<1$, we define $N_{\fixed, \AI}(\alpha,\beta)$ as the smallest
$N\in\N$ for which there exists a randomized test
$\psi_N:\calR^N\to[0,1]$ satisfying
\[
    \E_0[\psi_N(R_1,\ldots,R_N)]\le\alpha,
    \qquad
    \E_1[1-\psi_N(R_1,\ldots,R_N)]\le\beta,
\]
where, under $H_h$, the reports $R_1,\ldots,R_N$ are i.i.d. with mass
function $g_h$.  If no finite $N$ exists, set
$N_{\fixed, \AI}(\alpha,\beta)=+\infty$.
\end{definition}

\vspace{1mm}
If $G_0\ne G_1$, the Neyman--Pearson lemma applied to the report likelihood
ratio gives a finite $N_{\fixed, \AI}(\alpha,\beta)$.  If $G_0=G_1$ and
$\alpha+\beta<1$, report-only observations have the same law under the two
hypotheses, so no report-only test can satisfy both error constraints and
$N_{\fixed, \AI}(\alpha,\beta)=+\infty$. As the computation of $N_{\fixed,\AI}(\alpha, \beta)$ follows similarly as that of $N_{\fixed,\Hum}(\alpha, \beta)$  according to Neyman-Pearson lemma, we refer the readers to Appendix~\ref{appendix-sec:computation-NAI} for more details on computation of $N_{\fixed,\AI}(\alpha, \beta)$.

\subsubsection*{Pre-committed fallback.}

If no boundary is crossed by the time the pool is exhausted, the policy uses a
single fallback completion chosen before any data are observed. The choice is
made by computing a non-data dependent deterministic cost of running only AI-based or only human-based Neyman Pearson tests. Formally, the human-only completion cost
upper bound is
\begin{equation*}\label{eq:precommit-human-cost}
    \overline C_{\Hum,k}:=\cH N_{\fixed,\Hum}(\alpha_{2,k}, \beta_{2,k}).
\end{equation*}
The AI-only completion is available only if $N_{\fixed, \AI}(\alpha_{2,k},\beta_{2,k})<+\infty$ and
$N_{\fixed, \AI}(\alpha_{2,k},\beta_{2,k})\le N_{\main,k}$, because this section keeps the fixed-pool convention.  If
it is available, its primitive cost upper bound is
\begin{equation*}\label{eq:precommit-AI-cost}
    \overline C_{\AI,k}:=\cAI N_{\fixed, \AI}(\alpha_{2,k},\beta_{2,k}).
\end{equation*}
If $N_{\fixed, \AI}(\alpha_{2,k},\beta_{2,k})=+\infty$ or $N_{\fixed, \AI}(\alpha_{2,k},\beta_{2,k})>N_{\main,k}$, set
$\overline C_{\AI,k}:=+\infty$.  The pre-committed fallback mode is
\begin{equation}\label{eq:precommitted-mode}
    \mathsf F_k:=
    \begin{cases}
        \Hum,& \overline C_{\Hum,k}\le \overline C_{\AI,k},\\
        \AI,& \overline C_{\AI,k}< \overline C_{\Hum,k}.
    \end{cases}
\end{equation}
The value of $\mathsf F_k$ is a deterministic function of the primitives,
error targets, costs, and precomputed sample sizes.  It is not a function of
any realized label, AI report, randomization seed, stopping event, or terminal
statistic, thus can be computed before collecting any data. 

If $\mathsf F_k=\Hum$, the fallback uses the fixed index set
\begin{equation*}\label{eq:fixed-human-set}
    \calJ_{\Hum,k}:=\{1,2,\ldots,N_{\fixed,\Hum}(\alpha_{2,k}, \beta_{2,k})\}.
\end{equation*}
It queries a human on every item in $\calJ_{\Hum,k}$ whose label has not already
been revealed by human, and then applies a Neyman--Pearson full-label test
$\phi^*_{N_{\fixed,\Hum}(\alpha_{2,k}, \beta_{2,k})}$ satisfying
\begin{equation}\label{eq:human-fallback-level}
    \E^{\pi_k}_0[\phi^*_{N_{\fixed,\Hum}(\alpha_{2,k}, \beta_{2,k})}]\le \alpha_{2,k},
    \qquad
    \E^{\pi_k}_1[1-\phi^*_{N_{\fixed,\Hum}(\alpha_{2,k}, \beta_{2,k})}]\le \beta_{2,k}.
\end{equation}
The test is applied to $(X_i)_{i\in\calJ_{\Hum,k}}$.  The set is fixed in
advance, so these labels are i.i.d. Bernoulli under each hypothesis. If $\mathsf F_k=\AI$, by design, we must have $N_{\fixed, \AI}(\alpha_{2,k},\beta_{2,k})\le N_{\main,k}$. The fallback uses the fixed
report set $\{R_1,\ldots,R_{N_{\fixed, \AI}(\alpha_{2,k},\beta_{2,k})}\}$, querying any missing reports in that
set, and applies a Neyman--Pearson report test $\psi^*_{N_{\fixed, \AI}(\alpha_{2,k},\beta_{2,k})}$ satisfying
\begin{equation}\label{eq:AI-fallback-level}
    \E^{\pi_k}_0[\psi^*_{N_{\fixed, \AI}(\alpha_{2,k},\beta_{2,k})}]\le \alpha_{2,k},
    \qquad
    \E^{\pi_k}_1[1-\psi^*_{N_{\fixed, \AI}(\alpha_{2,k},\beta_{2,k})}]\le \beta_{2,k}.
\end{equation}
Again the set is fixed in advance, so the reports are i.i.d. with mass
$G_h$ under $H_h$. The full description of SCALE is summarized in Algorithm~\ref{algo:seq-policy}.

\subsubsection*{Feasibility.} We show in Theorem~\ref{thm:feasibility} that $\pi_k$ remains feasible regardless of the choice of the tuning parameters. We give the proof of Theorem~\ref{thm:feasibility} in Appendix~\ref{appendix-sec:feasibility-proof}.

\vspace{1mm}
\begin{theorem}[Feasibility of SCALE]\label{thm:feasibility}
Suppose Assumptions~\ref{ass:score-model}
and~\ref{ass:iid} hold.
SCALE (formally presented as policy $\pi_k$ in Algorithm~\ref{algo:seq-policy}) is feasible for
every $k$:
\[
    \Pp^{\pi_k}_0(\delta^{\pi_k}=1)\le\alpha_k,
    \qquad
    \Pp^{\pi_k}_1(\delta^{\pi_k}=0)\le\beta_k.
\]
\end{theorem}

\vspace{1mm}
Feasibility follows from a two-step argument.
The policy terminates either by crossing a
boundary during the sequential main stage or by invoking the
fallback test. Let $E_+$
and $E_-$ denote the events of
stopping at the upper and lower boundaries
respectively, and $E_{\mathrm{fb}}$ the event
of reaching fallback stage. Analogous to Wald's sequential hypothesis testing, we show that
$
    \Pp^{\pi_k}_0(E_+)\le
    \alpha_{1,k}$ and 
    $
    \Pp^{\pi_k}_1(E_-)\le
    \beta_{1,k},$
while the fallback test either satisfies \eqref{eq:human-fallback-level} or satisfies \eqref{eq:AI-fallback-level},
so
\[
    \Pp^{\pi_k}_0\bigl(E_{\mathrm{fb}}\cap
    \{\text{fallback rejects}\}\bigr)\le
    \alpha_{2,k},
    \qquad
    \Pp^{\pi_k}_1\bigl(E_{\mathrm{fb}}\cap
    \{\text{fallback accepts}\}\bigr)\le
    \beta_{2,k}.
\]
Since these routes are mutually exclusive and
exhaustive, $\Pp^{\pi_k}_0(\delta^{\pi_k}=1)\le\alpha_k$ and $\Pp^{\pi_k}_1(\delta^{\pi_k}=0)\le\beta_k$. 
    
Theorem~\ref{thm:feasibility} implies that we can choose the values of the tuning parameters to improve its finite sample and asymptotic performance of $\pi_k$ without affecting its feasibility. In the next section, we discuss one specific choice of the tuning parameters which leads to asymptotic optimality of $\pi_k$.

\begin{algorithm}[t]
\caption{SCALE — Sequential Cost-Aware Likelihood-Guided Escalation ($\pi_k$) \label{algo:seq-policy}}
\DontPrintSemicolon
\KwIn{Primitives $(p_0,p_1,f_0,f_1,\cdata,\cAI,\cH)$, error targets $(\alpha_k,\beta_k)$, and tuning parameters
$(f_{\fb,k},z_k, N_{\main, k} \, (\geq N_{\fixed,\Hum}(\alpha_{2,k}, \beta_{2,k})), \{\eta^{\Hum}_{h,k}\}_{h\in\{0, 1, *\}}, \{\eta^\esc_{h,k}\}_{h\in\{0, 1, *\}})$.}
Compute the split \eqref{eq:budget-split-alpha}--\eqref{eq:budget-split-beta} and 
boundaries \eqref{eq:boundaries}. \\
Pre-commit to $\mathsf F_k\in\{\Hum,\AI\},$ by \eqref{eq:precommitted-mode}.\;
Acquire $N_{\main,k}$ items and set $S\leftarrow0$.\;
\For{$i=1$ \KwTo $N_{\main,k}$}{
    Set $J\leftarrow1$ if $S>z_k$, $J\leftarrow0$ if $S<-z_k$, and
    $J\leftarrow *$ if $|S|\le z_k$.\;
    Draw $U_i\sim\mathrm{Uniform}[0,1]$.\;
    \If{$U_i\le \eta^{\Hum}_{J,k}$}{
        Query a human on item $i$, observe $X_i$, and set
        $S\leftarrow S+\ell_X(X_i)$.\;
        \If{$S\ge a_k$}{reject $H_0$ and stop.}
        \If{$S\le -b_k$}{accept $H_0$ and stop.}
    }
    \Else{
        Query the AI on item $i$, observe $R_i$, and set
        $S\leftarrow S+\ell_R(R_i)$.\;
        \If{$S\ge a_k$}{reject $H_0$ and stop.}
        \If{$S\le -b_k$}{accept $H_0$ and stop.}
        Draw $V_i\sim\mathrm{Uniform}[0,1]$.\;
        \If{$V_i\le\eta^{\esc}_{J,k}(R_i)$}{
            Query a human on item $i$, observe $X_i$, and set
            $S\leftarrow S+\ell_H(X_i,R_i)$.\;
            \If{$S\ge a_k$}{reject $H_0$ and stop.}
            \If{$S\le -b_k$}{accept $H_0$ and stop.}
        }
    }
}
If no boundary has been crossed, execute the pre-committed fallback mode
$\mathsf F_k$ and output the corresponding fallback test decision.\;
\end{algorithm}

\subsection{Discussion on Tuning Parameters with Asymptotics Analysis}\label{sec:main-theorem}
In the previous subsection we construct a feasible algorithm without giving details about the specific choice of the parameters. In this subsection, we suggest a particular set of parameter choices that guarantees the first order asymptotics of $\pi_k$. Recall our policy $\pi_k$ is indexed by $k=1,2,\ldots$, with
type-I and type-II error levels $\alpha_k\downarrow0$
and $\beta_k\downarrow0$ as $k\to\infty$. We tune the parameters to guarantee
\[
    \frac{\max\{\E^{\pi_k}_0[C^{\pi_k}],
                \E^{\pi_k}_1[C^{\pi_k}]\}}
         {\LB(\alpha_k,\beta_k)}
    \to 1
    \qquad\text{as }k\to\infty,
\]
and consequently
\[
    \frac{\max\{\E^{\pi_k}_0[C^{\pi_k}],
                \E^{\pi_k}_1[C^{\pi_k}]\}}
         {C^*(\alpha_k,\beta_k)}
    \to 1
    \qquad\text{as }k\to\infty.
\]
In particular, define:
\begin{equation*}\label{eq:Lk-def}
    L_k:=\max\!\left\{\log\frac{1}{\alpha_k},\,
                      \log\frac{1}{\beta_k}\right\}.
\end{equation*}

\vspace{1mm}
\noindent
We will show in Theorem~\ref{thm:first-order} that the
lower bound $\LB_k:=\LB(\alpha_k,\beta_k)$ scales
as $\LB_k=\Theta(L_k)$, so first-order optimality
is equivalent to matching $\LB_k$ up to an
$o(L_k)$ additive term. Throughout the asymptotic analysis that follows, we
impose the balanced small-error regime given by the following assumption.

\vspace{1mm}
\begin{assumption}[Balanced small-error regime]\label{ass:balanced-error-regime}
%    \begin{equation*}\label{eq:balanced-regime}
   $ \log\left(1/\alpha_k\right) = \Theta(L_k)$ and 
   $ \log\left(1/\beta_k\right)=\Theta(L_k)$.
%\end{equation*}
\end{assumption}
\vspace{1mm}

Assumption 3 requires the logarithmic type-I and type-II error levels, \(\log(1/\alpha_k)\) and \(\log(1/\beta_k)\), to be of the same order. In particular, it does not require \(\alpha_k\) and \(\beta_k\) to be of the same order. For example, \(\alpha_k=e^{-k}\) and \(\beta_k=e^{-2k}\) satisfy the assumption, even though \(\beta_k/\alpha_k\to0\). More generally, the assumption allows \(\beta_k=\alpha_k^q\) for any fixed \(q>0\). Thus, it accommodates a broad range of asymmetric sequences of vanishing type-I and type-II error levels.

\vspace{1mm}
Recall that the policy is equipped with the following set of tunable parameter sequences:
(i) a fixed pool size $N_{\main, k}$ as the total number of items to acquire, (ii) direction-specific sensing rules $\eta^\Hum_{h, k}\in[0,1]$, the probability to query human, and $\eta^\esc_{h,k}\in[0,1]$, the probability to escalate the item to the human given that the item has already been AI-queried before under direction $h$, 
(iii) a fallback-budget fraction $f_{\fb,k}$ that is bounded away from zero and one, and (iv) the hypothesis boundary $z_k$. In what follows, we pick specific values for these parameters.

\subsubsection*{Fixed pool size.} We choose $N_{\main, k}$ to be the minimum number of samples required to meet some information targets $(T_{1, k}, T_{0,k})$, and $N_{\main, k}$ is inspired by \eqref{eq:LB-def} with $(T_{1, k}, T_{0,k})$. Specifically, let %$T_{1,k}$ and $T_{0,k}$ be the buffered direction targets
%\begin{equation}\label{eq:buffered-targets}
    $T_{1,k}:=a_k+\Delta_k$ and 
    $T_{0,k}:=b_k+\Delta_k$ be the buffered direction targets
%\end{equation}
where $\Delta_k$ is the buffer with which we inflate the boundary on the sequential main stage to calculate the pool size. By inflating the boundary by $\Delta_k$, we ensure that, failure to cross the boundary
is a large-deviation event that occurs with
vanishing probability. The fallback test therefore contributes only $o(L_k)$ to
the expected cost and does not affect first-order
optimality.

For integer $N$, define the buffered design value
\begin{equation}\label{eq:F-Delta-def}
    F_k^\Delta(N):=N\cdata+
    \max\{\Gamma_1(T_{1,k},N),\Gamma_0(T_{0,k},N)\}.
\end{equation}

Choose
\begin{equation}\label{eq:Nbar-main-choice}
    \overline N_{\main,k}\in
    \argmin_{N\in\mathbb{Z}_+, N\ge N_{\fixed,\Hum}(\alpha_{2,k},\beta_{2,k})}F_k^\Delta(N).
\end{equation}
We then take the fixed pool size as
\begin{equation}\label{eq:Nmain-choice}
    N_{\main,k}:=\overline N_{\main,k}+1.
\end{equation}
In the above, we add an additive one-item cushion. We explain the reason of doing so subsequently.

\subsubsection*{Direction-specific sensing rules.}
Given $N_{\main, k}$, we now choose the human-query and escalation probabilities $(\eta^{\Hum}_{h,k}, \eta^{\esc}_{h,k})$ to match the minimal cost plan identified by the cost minimization problem \eqref{eq:Gamma-def}. For each direction $h\in\{0,1\}$, choose an optimizer
\begin{equation}\label{eq:Gamma-optimizer}
    \begin{split}(n^*_{\Hum,h,k},n^*_{\AI,h,k},n^*_{\esc,h,k})\in\argmin_{n_{\Hum},\,n_{\AI},\,n_{\esc}\ge0}\quad
        &\cH n_{\Hum}+\cAI n_{\AI}+\cH n_{\esc}\\
    \text{subject to}\quad
        &n_{\Hum}+n_{\AI}\le \overline N_{\main,k},\\
        &0\le n_{\esc}\le n_{\AI},\\
        &n_{\Hum}J_X^{(h)}+n_{\AI}I_R^{(h)}
          +n_{\AI}\Psi_h(n_{\esc}/n_{\AI})\ge T_{h,k}.
    \end{split}
\end{equation}

\vspace{1mm}
\noindent
The variables are named to match the interpretation of the program
$\Gamma_h$: $n^*_{\Hum,h,k}$ is the direct-human count, $n^*_{\AI,h,k}$ is the
AI-scored count, and $n^*_{\esc,h,k}$ is the escalation count among AI-scored
items.  
Because $n^*_{\Hum,h,k}+n^*_{\AI,h,k}\le \overline N_{\main,k}$ in the program
$\Gamma_h$, and since $\lceil x\rceil+\lceil y\rceil\le\lceil x+y\rceil+1$ for
any reals $x,y\ge0$, the cushion in \eqref{eq:Nmain-choice} gives
\begin{equation}\label{eq:B-fits-pool}
    \lceil n^*_{\Hum,h,k}\rceil+\lceil n^*_{\AI,h,k}\rceil
    \le\lceil n^*_{\Hum,h,k}+n^*_{\AI,h,k}\rceil+1
    \le \lceil \overline N_{\main,k}\rceil+1=N_{\main,k},
\end{equation}

\vspace{1mm}
\noindent
which ensures we have acquired enough items to make $ \lceil n^*_{\Hum,h,k}\rceil$ human queries and $\lceil n^*_{\AI,h,k}\rceil$ AI queries for the sequential main stage. 

The direct-human fraction of the direction-$h$ rule is
\begin{equation}\label{eq:phi-def}
    \eta^{\Hum}_{h,k}:=
    \begin{cases}
        \lceil n^*_{\Hum,h,k}\rceil/(\lceil n^*_{\Hum,h,k}\rceil+\lceil n^*_{\AI,h,k}\rceil),&\lceil n^*_{\Hum,h,k}\rceil+\lceil n^*_{\AI,h,k}\rceil>0,\\
        0,&\lceil n^*_{\Hum,h,k}\rceil+\lceil n^*_{\AI,h,k}\rceil=0.
    \end{cases}
\end{equation}
For all large $k$, $\lceil n^*_{\Hum,h,k}\rceil+\lceil n^*_{\AI,h,k}\rceil>0$ because $T_{h,k}>0$ and no zero-item rule can
supply positive information.

If $n^*_{\AI,h,k}>0$, define the planned escalation rate $s_{h,k}:=n^*_{\esc,h,k}/n^*_{\AI,h,k}\in[0,1]$. If $n^*_{\AI,h,k}=0$, set $s_{h,k}:=0$. To construct the frontier optimizer, following Proposition~\ref{prop:escalation-solution} we define 
$\eta^{\esc}_{h,k}:\calR\to[0,1]$ by
\begin{equation}\label{eq:eta-closed-form}
    \eta^{\esc}_{h,k}(r^{(h)}_j):=
    \begin{cases}
        1, & W^{(h)}_j\le s_{h,k},\\[1ex]
        \displaystyle\frac{s_{h,k}-W^{(h)}_{j-1}}{g_h(r^{(h)}_j)},
            & W^{(h)}_{j-1}<s_{h,k}<W^{(h)}_j,\\[2ex]
        0, & W^{(h)}_{j-1}\ge s_{h,k}.
    \end{cases}
\end{equation}
Then $\eta^{\esc}_{h,k}$ is the optimal solution of \eqref{eq:Psi-dual} with $s = s_{h,k}$. 

In the dead-zone $[-z_k,z_k]$ the sign of the likelihood-ratio statistic is
ambiguous. The policy therefore uses the averaged dead-zone rule
\begin{equation}\label{eq:deadzone-rule}
    \eta^{\Hum}_{*,k}:=\frac{\eta^{\Hum}_{0,k}+\eta^{\Hum}_{1,k}}2,
    \qquad
    \eta^{\esc}_{*,k}(r):=\frac{\eta^{\esc}_{0,k}(r)+\eta^{\esc}_{1,k}(r)}2,
    \quad r\in\calR.
\end{equation}
The average of two numbers in $[0,1]$ again lies in $[0,1]$, so the rule is a valid randomized sensing rule. The specific values of $\eta^{\Hum}_{*,k}$ and $\eta^{\esc}_{*,k}(r), r\in \calR$ will not affect the asymptotics of $\pi_k$; therefore, we choose the averaged rule here for simplicity. 

\subsubsection*{Fallback-budget fraction $f_{\fb, k}$.}
According to Theorem~\ref{thm:first-order} below, $f_{\fb, k}$ only affects the cost of $\pi_k$ up to some constant, so we can choose $f_{\fb, k}$ to be any constant that is bounded away from $0$ and $1$; for example, we can let $f_{\fb, k} = 1/2$. 

\subsubsection*{Buffer $\Delta_{k}$ and hypothesis boundary $z_k$.}
Additionally, to ensure that policy $\pi_k$ has asymptotically converging cost to $\LB_k$, we require that the remaining parameters $(z_k, \Delta_k)$ to satisfy the following conditions: 
\begin{enumerate}[leftmargin=2em,label=(\roman*)]
    \item 
    \begin{equation}\label{eq:delta-admissible}
        \Delta_k\to\infty, \, \Delta_k=o(L_k), \text{ and }\frac{\Delta_k^2}{L_k}\to\infty;
    \end{equation}
    \item 
    \begin{equation}\label{eq:z-admissible}
       z_k\ge0\text{ and } \frac{z_k+1}{\Delta_k}\to0.
    \end{equation}
\end{enumerate}

\vspace{1mm}
\noindent
Intuitively, the accumulated likelihood-ratio
statistic has variance of order $L_k$, and hence
fluctuations of order $\sqrt{L_k}$. Thus,
$\Delta_k^2/L_k\to\infty$ ensures that the buffer
dominates these fluctuations and makes the
probability of reaching fallback vanish, while
$\Delta_k=o(L_k)$ keeps the cost of this buffer
lower order. The condition $(z_k+1)/\Delta_k\to0$
similarly ensures that the additional cost incurred
while the statistic lies in the ambiguous region
$[-z_k,z_k]$ is negligible. We summarize the tuning parameters and their choices in Table~\ref{tab:tuning-parameters}.

\begin{table}[t]
\centering\small
\caption{Tuning parameters of the policy $\pi_k$ in Algorithm~\ref{algo:seq-policy}.}
\label{tab:tuning-parameters}
\begin{tabularx}{\textwidth}{@{}l >{\raggedright\arraybackslash}p{0.27\textwidth} >{\raggedright\arraybackslash}X@{}}
\hline\hline
Tuning parameter & Meaning & Choice\\
\hline
$f_{\fb,k}$
  & Fallback-budget fraction
  & Any constant in $(0,1)$ bounded away from $0$ and $1$\\[4pt]
$z_k$
  & Hypothesis boundary
  & Any $z_k\ge0$ with $(z_k+1)/\Delta_k\to0$ (see \eqref{eq:z-admissible}) where $\Delta_k$ satisfies $\Delta_k\to\infty$, $\Delta_k=o(L_k)$, and $\Delta_k^2/L_k\to\infty$ (see \eqref{eq:delta-admissible})\\[4pt]
$N_{\main,k}$
  & Fixed pool size
  & $N_{\main,k}=\overline N_{\main,k}+1$, where $\overline N_{\main,k}\in\argmin\bigl\{F^\Delta_k(N):N\in\Z_+,\,N\ge N_{\fixed,\Hum}(\alpha_{2,k},\beta_{2,k})\bigr\}$ and $F^\Delta_k(N)=N\cdata+\max\{\Gamma_1(T_{1,k},N),\Gamma_0(T_{0,k},N)\}$ (see \eqref{eq:F-Delta-def}--\eqref{eq:Nmain-choice})\\[4pt]
$\eta^{\Hum}_{h,k}$, $h\in\{0,1\}$
  & The probability to directly query human
  & $\eta^{\Hum}_{h,k}=\lceil n^*_{\Hum,h,k}\rceil/(\lceil n^*_{\Hum,h,k}\rceil+\lceil n^*_{\AI,h,k}\rceil)$, and $0$ if the denominator vanishes\\[4pt]
$\eta^{\esc}_{h,k}(\cdot)$, $h\in\{0,1\}$
  & The probability to escalate
to human given that the item has already been AI-queried before
  & Optimal solution of \eqref{eq:Psi-dual} with $s = s_{h,k}$ (see \eqref{eq:eta-closed-form})\\[4pt]
$\eta^{\Hum}_{*,k}$, $\eta^{\esc}_{*,k}(\cdot)$
  & Dead-zone sensing rule
  & $\eta^{\Hum}_{*,k}=(\eta^{\Hum}_{0,k}+\eta^{\Hum}_{1,k})/2$ and $\eta^{\esc}_{*,k}(r)=(\eta^{\esc}_{0,k}(r)+\eta^{\esc}_{1,k}(r))/2$, $r\in\calR$ (see \eqref{eq:deadzone-rule})\\
\hline\hline
\end{tabularx}
\end{table}

%\vspace{2mm}
Following the above choices of $N_{\main, k}$, 
$(\eta^{\Hum}_{h,k}, \eta^{\esc}_{h,k})$, $f_{\fb, k}, z_k$ and $\Delta_k$,
we now establish first-order
asymptotic optimality of $\pi_k$ in Theorem~\ref{thm:first-order}. Before stating the theorem, we remark that all choices of tuning parameters introduced in this section are neither unique nor finite-sample optimal. They should be viewed as reasonable sufficient conditions to ensure theoretical first order optimality as shown by the subsequent theorem. For practical implementation, we recommend practitioners to computationally tune these parameters for their own applications. 

\vspace{1mm}
\begin{theorem}[Parameterized ratio upper bound]\label{thm:first-order}
Suppose Assumptions~\ref{ass:score-model}-\ref{ass:balanced-error-regime} hold. The following hold:
\begin{enumerate}
\item[(i)] $\LB_k = \Theta(L_k)$;

\item[(ii)]
Suppose $N_{\main,k}$ is given by \eqref{eq:Nmain-choice}, $(\eta^{\Hum}_{h,k}, \eta^{\esc}_{h,k})$ is given by \eqref{eq:phi-def} and \eqref{eq:eta-closed-form}, and  $(\eta^{\Hum}_{*,k}, \eta^{\esc}_{*,k})$ is given by \eqref{eq:deadzone-rule}. Suppose $\Delta_k, z_k$ satisfy \eqref{eq:delta-admissible} and \eqref{eq:z-admissible} and $f_{\fb, k}$ is any constant bounded away from 0 and 1. 
Then for all large $k$,
\begin{equation}\label{eq:param-ratio-bound-1}
\begin{split}
    \frac{\max\{\E^{\pi_k}_0[C^{\pi_k}],\E^{\pi_k}_1[C^{\pi_k}]\}}{\LB_k} 
    = 1 +o(1).
\end{split}
\end{equation}
\end{enumerate}
\end{theorem} 

\vspace{1mm}
%For full details of the proof, we refer the readers to Appendix~\ref{appendix-sec:first-order-proof}.

As specified in \eqref{eq:delta-admissible}--\eqref{eq:z-admissible},
we impose only rate conditions on these tuning parameters.
To state explicit convergence rates, let $\nu_h^*$ denote the least
active-item consumption among the cost-minimizing unit-information
plans in direction $h$:
\begin{equation}\label{eq:nu-star-def}
    \nu_h^*:=\min_{(n_{\Hum},n_{\AI},n_{\esc})\in\mathcal M_h}(n_{\Hum}+n_{\AI}).
\end{equation}
where $\mathcal M_h$ denotes the set of minimizers of $\min_{n_{\Hum},n_{\AI},n_{\esc}\geq 0}\{\cH n_{\Hum}+\cAI n_{\AI}+\cH n_{\esc}:n_{\Hum}J_X^{(h)}+n_{\AI}I_R^{(h)}+n_{\AI}\Psi_h(n_{\esc}/n_{\AI})\ge1, 0\le n_{\esc}\le n_{\AI}\}$, which computes the minimum sensing cost required to generate one unit of information in direction \(h\). Define 
\begin{equation}\label{eq:floor-gap}
    G_k
:=
N_{\fixed,\Hum}(\alpha_k,\beta_k)
-\max\{\nu_1^*a_k,\nu_0^*b_k\},
\end{equation}
and write $G^+_k := [G_k]_+$. 
Thus, $G_k^+$ measures the unused item reserve supplied by the
mandatory full-label fallback floor. Corollary~\ref{cor:rates-1} provides explicit convergence rates with exact parameter choices.
\begin{corollary}[Parameter choices and rates]\label{cor:rates-1}
Suppose Assumptions~\ref{ass:score-model}-\ref{ass:balanced-error-regime} hold.
Let $f_{\fb,k}=1/2$ and $z_k=1$, let the optimizer in
\eqref{eq:Gamma-optimizer} be selected with least active-item consumption
$n^*_{\Hum,h,k}+n^*_{\AI,h,k}$ among the cost minimizers (cf.\
\eqref{eq:nu-star-def}). Then the following hold for all large $k$:

% and let $\Delta_k\to\infty$ with $\Delta_k=o(L_k)$.
% Then there are primitive
% constants $C_0<\infty$, $c_0>0$, and $c_G>0$ such that,  for all large $k$, the fallback
% probability satisfies
% \begin{equation}\label{eq:reserve-fallback-prob}
%     \Pp^{\pi_k}_h(E_{\fb})
%     \le C_0\exp\left\{-c_0\,\frac{(\Delta_k+c_GG_k^+)^2}{L_k}\right\},
% \end{equation}
% and consequently
% \begin{equation}\label{eq:reserve-refined-rate}
%         \frac{\max_h\E^{\pi_k}_h[C^{\pi_k}]}{\LB_k}
%         \le 1+O\left(
%         \frac{\Delta_k}{L_k}
%         +\exp\left\{-c_0\frac{(\Delta_k+c_GG_k^+)^2}{L_k}\right\}
%         +\frac{\log L_k}{L_k}
%         \right).
%     \end{equation}
%     In particular,
\begin{enumerate}[leftmargin=2em,label=(\roman*)]
    \item \textbf{(Baseline, no condition on the reserve.)} Letting $\Delta_k=\kappa\sqrt{L_k\log L_k}$ for a sufficiently large constant $\kappa$ depending only on the primitives, then regardless of the value of $G_k$, 
    \[
        \frac{\max_h\E^{\pi_k}_h[C^{\pi_k}]}{\LB_k}
        \le 1+\widetilde O(L_k^{-1/2}).
    \]
    \item \textbf{(Improvement when the reserve is generous.)} If
    $G_k^+\ge C\sqrt{L_k\log L_k}$ for a sufficiently large primitive constant
    $C$, 
    taking $\Delta_k$ polylogarithmic sharpens part (i) to
    \[
        \frac{\max_h\E^{\pi_k}_h[C^{\pi_k}]}{\LB_k}
        \le 1+\widetilde O(1/L_k).
    \]
\end{enumerate}
\end{corollary}

\subsection{Numerical Simulations}
\label{sec:scale-simulations}

We evaluate the finite-sample cost performance of SCALE under the known AI-output model.
The experiments examine how the benefit of combining AI queries with
selective human verification changes with the human query cost and the
required testing accuracy. We compare SCALE with implementable Human-only
and AI-only policies under the same error constraints and cost accounting. We consider independent labels $X_i\sim\Bern(p_h)$ under $H_h$, with
$p_0=0.10$ and $p_1=0.20$. Human queries reveal $X_i$ exactly, whereas
the AI returns a binary report $R_i\in\{0,1\}$ with known sensitivity
and specificity
$
    \Pp(R_i=1\mid X_i=1)
    =\Pp(R_i=0\mid X_i=0)=0.80.
$
We set $\cdata=\cAI=1$. Each policy pays for its entire sample pool
upfront, including items that remain unqueried when the sequential test
stops. Reported costs include data acquisition and all AI and human
queries, including those required by fallback.

Our preliminary numerical exploration suggested that total cost was
relatively insensitive to the parameters $z$ and $f_{\fb}$ over the ranges considered. We therefore use a coarse tuning scheme for these two parameters: at each setting, we fix $z=1$ and search over $f_{\fb}\in\{0.2,0.5,0.8\}$.
 We tune the integer pool size
$N_{\main}$ in unit increments over
\[
N_{\fixed,\Hum}(f_{\fb}\alpha,f_{\fb}\beta)
\le N_{\main}\le N_{\mathrm{max}},
\]
where $N_{\mathrm{max}}$ is the upper bound derived
from the single-source policy. We also tune the three
direct-human probabilities $\eta^{\Hum}_j\in[0,1]$ and the six
escalation probabilities $\eta^{\esc}_j(r)\in[0,1]$,
$j\in\{0,1,*\}$ and $r\in\{0,1\}$, using a grid search
with steps of $0.1$. We retain SCALE's pre-committed fallback rule.
No separate buffer $\Delta$ is introduced, since
$N_{\main}$ and the escalation probabilities are tuned directly.
Parameters are selected using a simulation-based cost
criterion, without a claim of global optimality.

The Human-only and AI-only baselines use a truncated sequential
probability ratio test (SPRT), followed, if no boundary is crossed, by
a same-source Neyman--Pearson (NP) test. Each baseline uses its designated
source in both stages. For a fixed $f_{\fb}$, its pool size is set to
$
    N_s=N_{\fixed,s}(f_{\fb}\alpha,f_{\fb}\beta)$,
    $s\in\{\Hum,\AI\},
$
computed using the exact randomized binomial NP test. This pool size is
analytically optimal within the baseline class for the given
$f_{\fb}$: a larger pool increases the upfront acquisition cost without
reducing the expected number of sequential queries. If fallback is
reached, all observations required by the terminal NP test have already
been queried, so the terminal decision incurs no additional acquisition
or query cost. For each baseline, $f_{\fb}$ is selected over the same
three-point grid by exact expected-cost evaluation.

After all parameters are fixed, we evaluate each policy using
$M=10^6$ independent Monte Carlo trajectories under each hypothesis,
independently of the tuning simulations. The plotted absolute cost is
the estimated worst-case expected cost
\[
    \widehat{\mathcal C}_{\pi}
    :=\max_{h\in\{0,1\}}
      \frac{1}{M}\sum_{m=1}^{M}C^{\pi}_{h,m},
\]
where $C^{\pi}_{h,m}$ is the total cost in replication $m$ under $H_h$.
To quantify relative performance, we report
\begin{equation}
\label{eq:scale-simulation-gap}
    \operatorname{Gap}(\%)
    :=100\left[
      \frac{\widehat{\mathcal C}_{\mathrm{SCALE}}}
      {\min\{\widehat{\mathcal C}_{\Hum},
             \widehat{\mathcal C}_{\AI}\}}-1
    \right].
\end{equation}
A negative gap indicates a cost saving relative to the cheaper of the
two evaluated single-source baselines.

We conduct two parameter sweeps. The first varies
$\cH\in\{2,5,10,20,50\}$ with $\alpha=\beta=0.01$.
The second varies $\alpha=\beta\in\{0.05,0.01,0.0001\}$ with
$\cH=10$. Figures~\ref{fig:scale-absolute-cost}
and~\ref{fig:scale-percentage-gap} report the absolute costs and
percentage gaps, respectively. Both horizontal axes use logarithmic
spacing, and the error targets decrease from left to right.

\begin{figure}[tbp]
    \centering
    \renewcommand{\thesubfigure}{\Alph{subfigure}}
    \begin{subfigure}[t]{0.5\textwidth}
        \centering
        \includegraphics[width=\linewidth]{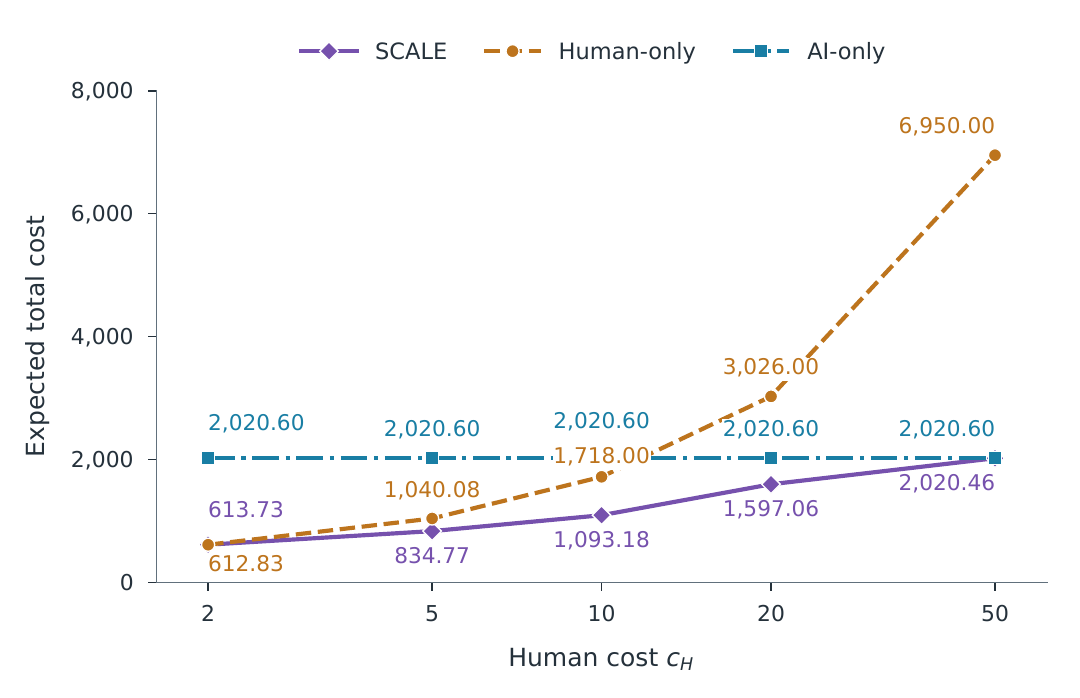}
        \caption{Varying $\cH$; $\alpha=\beta=0.01$.}
        \label{fig:scale-absolute-cost-human}
    \end{subfigure}\hfill
    \begin{subfigure}[t]{0.5\textwidth}
        \centering
        \includegraphics[width=\linewidth]{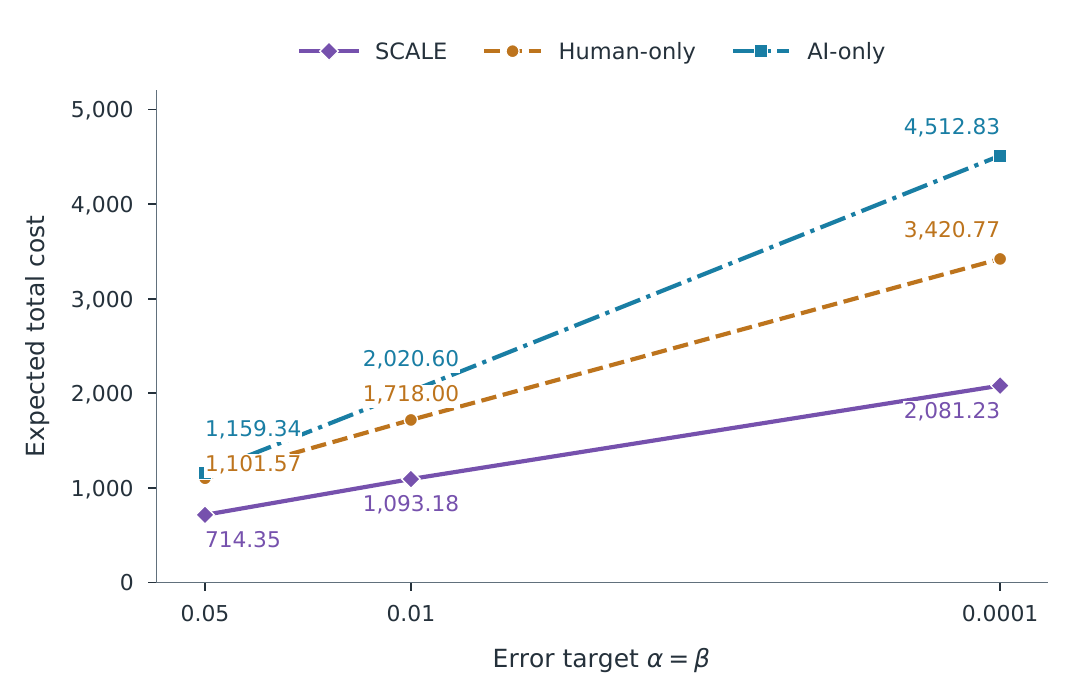}
        \caption{Varying $\alpha=\beta$; $\cH=10$.}
        \label{fig:scale-absolute-cost-error}
    \end{subfigure}
    \caption{Worst-case expected total costs of SCALE, Human-only,
    and AI-only. Each point is the maximum of the two estimated
    hypothesis-specific mean costs, based on $10^6$ Monte Carlo
    trajectories per hypothesis. Costs include the full upfront
    acquisition cost and all subsequent queries.}
    \label{fig:scale-absolute-cost}
\end{figure}

As the human query cost increases, both Human-only and SCALE become
more expensive, while AI-only remains constant at $2020.60$
(Figure~\ref{fig:scale-absolute-cost-human}). The percentage gap exhibits
a U-shaped pattern across the evaluated human costs
(Figure~\ref{fig:scale-percentage-gap-human}). At $\cH=2$, SCALE costs
$613.73$, compared with $612.83$ for Human-only, giving a small positive
gap of $0.15\%$. At $\cH=5$, $10$, and $20$, SCALE reduces cost relative
to the cheaper baseline by $19.74\%$, $36.37\%$, and $20.96\%$,
respectively. The largest observed saving occurs at $\cH=10$:
SCALE costs $1,093.18$, compared with $1,718.00$ for Human-only and
$2,020.60$ for AI-only, saving $624.82$ cost units relative to the
cheaper baseline. At $\cH=50$, SCALE costs $2,020.46$ and essentially
matches AI-only; the estimated gap of $-0.01\%$ is smaller than the
Monte Carlo uncertainty.

This pattern is consistent with the selected routing policies.
SCALE relies predominantly on direct human queries when human queries
are inexpensive, uses AI queries with selective human escalation at
intermediate costs, and becomes nearly AI-only when human queries are
expensive. At $\cH=50$, the selected pool contains $1324$ items and
supports an AI-only terminal NP test. Thus, the largest observed gains
arise at intermediate human costs, where selective verification offers
the greatest advantage over committing to a single source.

\begin{figure}[tbp]
    \centering
    \renewcommand{\thesubfigure}{\Alph{subfigure}}
    \begin{subfigure}[t]{0.5\textwidth}
        \centering
        \includegraphics[width=\linewidth]{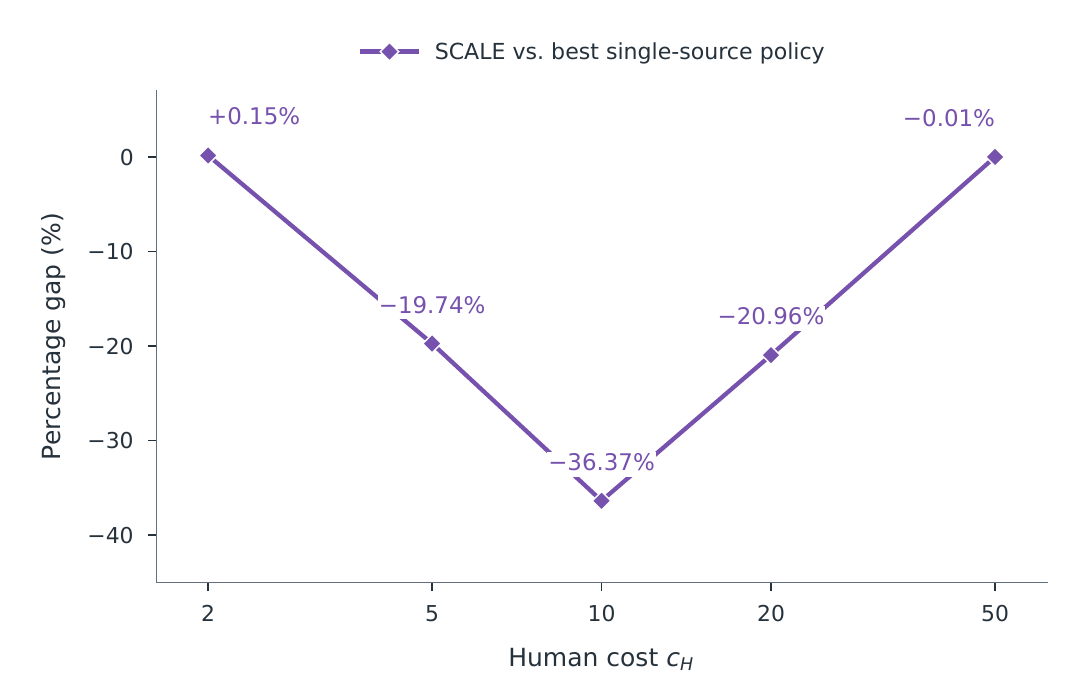}
        \caption{Varying $\cH$; $\alpha=\beta=0.01$.}
        \label{fig:scale-percentage-gap-human}
    \end{subfigure}\hfill
    \begin{subfigure}[t]{0.5\textwidth}
        \centering
        \includegraphics[width=\linewidth]{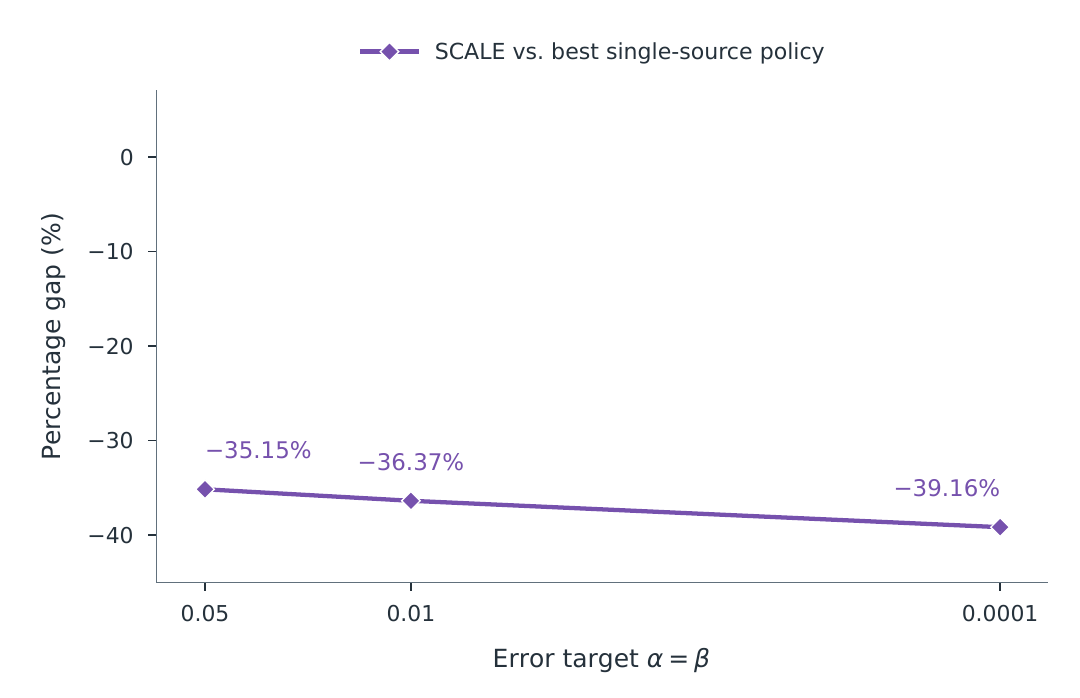}
        \caption{Varying $\alpha=\beta$; $\cH=10$.}
        \label{fig:scale-percentage-gap-error}
    \end{subfigure}
    \caption{SCALE's percentage cost gap relative to the cheaper
    evaluated Human-only or AI-only policy, as defined in
    \eqref{eq:scale-simulation-gap}. Negative values indicate cost
    savings. Panel~A shows a U-shaped pattern across human query
    costs; Panel~B shows increasing relative savings as the error
    targets become more stringent.}
    \label{fig:scale-percentage-gap}
\end{figure}

Tighter error targets increase the absolute costs of all three policies
(Figure~\ref{fig:scale-absolute-cost-error}). As $\alpha=\beta$
decreases from $0.05$ to $0.01$ and then to $0.0001$, SCALE's cost
increases from $714.35$ to $1,093.18$ and $2,081.23$, while its selected
pool size increases from $269$ to $410$ and $835$.
Human-only remains the cheaper single-source baseline throughout this
sweep, with corresponding costs of $1,101.57$, $1,718.00$, and $3,420.77$.
SCALE's relative savings increase from $35.15\%$ to $36.37\%$ and
$39.16\%$ (Figure~\ref{fig:scale-percentage-gap-error}). At the most
stringent target, SCALE saves $1,339.54$ cost units relative to
Human-only. Across these three error targets, stronger error control
therefore requires greater expenditure while increasing the relative
benefit of combining AI queries with selective human verification.
%We give proof of Corollary~\ref{cor:rates-1} in Appendix~\ref{appendix-sec:corollary-proof}.

%==============================================================================
% Section_6_v7.tex
%
% APPEND INSTRUCTIONS
% -------------------
% Insert this file immediately after Section 5
% ("A Sequential Selector-Free Cost-Aware Policy") and immediately before
% the current Conclusion in paper_draftv2.
%
% Version 7 keeps the complete algorithm and both headline guarantees, and
% adds two measurability/transfer clarifications requested in the second
% rigor review: predictability of the pilot-conditioned action rule and the
% abstract basis for transferring the Section 5 main-cost argument.
%==============================================================================

% New notation used only in Section 6.

\section{Pilot-Calibrated Design with Unknown AI Accuracy}
\label{sec:unknown-ai}

So far, our results have assumed that the conditional
AI-output laws $f_0$ and $f_1$ are known, as stated in
Assumption~\ref{ass:score-model}. This assumption can
be reasonable in settings where the same AI system
has been repeatedly evaluated on a stable population
and its performance has been estimated from a large
historical labeled data set. In such cases, the
AI-output model may be treated as a known
characteristic of the deployed system. Moreover,
when $f_0$ and $f_1$ are known, an analyst can in
principle test $H_0$ versus $H_1$ using AI reports
alone, because the systematic error in the AI output
can be accounted for statistically, although doing
so need not be cost optimal.

In other settings, however, the AI system may be new,
the target population may differ from the population
on which it was previously evaluated, or sufficiently
reliable labeled calibration data may simply be
unavailable. In these cases, $f_0$ and $f_1$ must be
estimated from observations containing both the AI
report and the human-verified label. Accordingly, in
this section we estimate $f_0$ and $f_1$ from an
independent paired pilot sample and construct a
plug-in policy based on $\hat f_0$ and $\hat f_1$.
We add guardrail terms that account for estimation
error and preserve finite-sample type-I and type-II
error control. We then show that, when the pilot
sample is sufficiently large, the resulting policy
retains the first-order asymptotic optimality
established in Theorem~\ref{thm:first-order}.

% The proof is deliberately modular. A single perturbation lemma controls the
% running likelihood ratio, the information frontier, and the sensing program.
% After these controls are established, the occupation, fallback, and cost
% arguments of Section~\ref{sec:upper-bound} apply with the original buffer
% $\Delta_k$ replaced by a slightly smaller effective buffer.

\subsection{Pilot Sample and Smoothed Plug-In Model}
\label{sec:unknown-ai-pilot}

In order to estimate $f_0, f_1$, we first need to obtain both pilot
samples with $X^{\pilot}_i = 0$ labels and $X^{\pilot}_i = 1$ labels.
With a slight buase of notation, let $m$ be a positive integer, the minimal number of samples for both
$X^{\pilot}_i = 0$ and $X^{\pilot}_i = 1$ labels we target to obtain.
A sufficiently large $m$ ensures we have accurate estimates for $f_0$
and $f_1$. In order to obtain enough samples for both labels, we
consider the following pilot sampling procedure: we sample each pilot
item and reveal its true label using a human query sequentially; we
stop sampling after we have collected $m$ samples for both labels.
Writing $C_x(n):=\sum^{n}_{i=1}\ind\{X^{\pilot}_i = x\}$ for the number
of label-$x$ samples among the first $n$, the total number of random
samples we get is the stopping time
$
    M_{\tot}:= \min\{n\geq 1:
       \min_{x\in\{0,1\}}C_x(n)\geq m\}.
$
The random numbers of $X^{\pilot}_i = 0$ and
$X^{\pilot}_i = 1$ samples are respectively $M_0:=C_0(M_{\tot})$ and
$M_1:=C_1(M_{\tot})$, and we have $M_0+M_1=M_{\tot}\geq 2m$ and
$\min\{M_0, M_1\} = m$. 

For each pilot sample $X^{\pilot}_i$, we also obtain its AI report
$R^{\pilot}_i$, generated by the same AI system on the same population
as in the main stage; this is what makes the pilot informative about
$f_0$ and $f_1$. We assume that, conditional on
the pilot labels, the AI reports are independent across samples and are
drawn from $f_x$ within the label-$x$ stratum. We summarize this
formally below:

\vspace{1mm}
\begin{assumption}[Independent pilot]
\label{ass:pilot-sample}
Under $H_h$, the pilot pairs
$(X_i^{\pilot},R_i^{\pilot})_{i\ge1}$ are i.i.d. with
$
\Pp_h(X_i^{\pilot}=x) = p_h^x(1-p_h)^{1-x}$  and 
$\Pp_h(R_i^{\pilot}=r|X_i^{\pilot}=x)
=f_x(r), \text{ for } x\in\{0,1\},r\in\calR.
$
The entire pilot stream is independent of the main-stage items and
policy randomization.
\end{assumption}
\vspace{1mm}

We denote the full pilot sample as:
\begin{equation*}\label{eq:pilot-data}
    \calD_m
    :=\{(0,R^{\pilot}_{0,j}):1\le j\le M_0\}
      \cup\{(1,R^{\pilot}_{1,j}):1\le j\le M_1\}.
\end{equation*}
Next, we use a smoothed empirical distribution to estimate $f_0, f_1$.
For $x\in\{0,1\}$ and $r\in\calR$, define
\begin{equation}\label{eq:pilot-estimator}
    M_x(r):=\sum_{j=1}^{M_x}\ind\{R^{\pilot}_{x,j}=r\},
    \qquad
    \hat f_x(r):=\frac{M_x(r)+\lambda_m}{M_x+|\calR|\lambda_m},
    \qquad \lambda_m:=m^{-2}.
\end{equation}
Here $M_x(r)$ counts how many pilot reports in the label-$x$ stratum
equal $r$, so the natural estimate of $f_x(r)$ is the empirical
frequency $M_x(r)/M_x$, which is well defined because $M_x\ge m\ge1$.
We add the smoothing term $\lambda_m$ to avoid the edge case
$M_x(r)=0$, which
would make the plug-in increments (see \eqref{eq:pilot-hat-increments})
infinite. The resulting estimate satisfies
$
    \hat f_x(r) > 0$ 
    and
    $\sum_{r\in\calR}\hat f_x(r)=1.
$
The choice
$\lambda_m=m^{-2}$ makes the smoothing displace the empirical frequency
by at most $|\calR|\lambda_m/M_x\le|\calR|m^{-3}$, which is negligible
relative to the sampling error of order $m^{-1/2}$ recorded in $r_m$
below; the correction therefore does not affect first-order asymptotic
optimality.

Next, we aim to control the error of our estimates $\hat f_x(r)$.
Let $\delta_m:=m^{-2}$ be a target failure probability, and set
\[
    t_m:=\sqrt{\frac{1}{2m}
       \log\!\left(\frac{4|\calR|}{\delta_m}\right)},
    \qquad
    r_m:=\min\left\{1,\;
       t_m+\frac{|\calR|\lambda_m}{m}\right\}.
\] 

\vspace{1mm}
\noindent
We then define the ``good pilot event''
\begin{equation}\label{eq:pilot-good-event}
\calE_{\pilot,m}
:=\left\{\max_{x\in\{0,1\},\,r\in\calR}
|\hat f_x(r)-f_x(r)|\le r_m\right\},
\end{equation}

\vspace{1mm}
\noindent
on which every plug-in report probability is within $r_m$ of its true
value, uniformly in $x$ and $r$. In
Appendix~\ref{appendix-sec:unknown-ai-pilot} we show that
$\Pp_h(\calE_{\pilot,m})\ge1-\delta_m$ for $h\in\{0,1\}$. 

Once we have $\hat f_0,\hat f_1$, we form the plug-in analogue of every
object in Sections~\ref{sec:model}--\ref{sec:upper-bound} that depends
on the report channel, by substituting $\hat f_x$ for $f_x$ wherever
$f_x$ appears and leaving all quantities that depend only on the known
primitives $(p_0,p_1,\cdata,\cAI,\cH)$ unchanged:
\begin{equation}\label{eq:pilot-hat-g-q}
\begin{split}
    \hat g_h(r):=p_h\hat f_1(r)+(1-p_h)\hat f_0(r),\qquad
    \hat q_h(r):=\frac{p_h\hat f_1(r)}{\hat g_h(r)},\qquad
    \hat\rho_h(x\mid r):=\hat q_h(r)^x(1-\hat q_h(r))^{1-x},\\
    \hat d^{(h)}(r):=
      \klbin(\hat q_h(r)\Vert\hat q_{1-h}(r)),\qquad
    % \label{eq:pilot-hat-g-q}\\
    \hat I_R^{(h)}:=\sum_{r\in\calR}\hat g_h(r)
       \log\frac{\hat g_h(r)}{\hat g_{1-h}(r)}.
\end{split}
\end{equation}
The full-label information $J_X^{(h)}$ is unchanged because $p_0,p_1$ remain
known. For $s\in[0,1]$, let
\begin{equation}\label{eq:pilot-hat-Psi}
\begin{aligned}
    \hat\Psi_h(s):=\max_{\eta:\calR\to[0,1]}\quad
       &\sum_{r\in\calR}\hat g_h(r)\eta(r)\hat d^{(h)}(r)\\
    \text{s.t.}\quad
       &\sum_{r\in\calR}\hat g_h(r)\eta(r)\le s,
\end{aligned}
\end{equation}
and define the plug-in sensing program
\begin{equation}\label{eq:pilot-hat-Gamma}
\begin{aligned}
    \hat\Gamma_h(T,N):=\min_{n_{\Hum},n_{\AI},n_{\esc}\ge0}\quad
       &\cH n_{\Hum}+\cAI n_{\AI}+\cH n_{\esc}\\
    \text{s.t.}\quad
       &n_{\Hum}+n_{\AI}\le N,\qquad 0\le n_{\esc}\le n_{\AI},\\
       &n_{\Hum}J_X^{(h)}+n_{\AI}\hat I_R^{(h)}
        +n_{\AI}\hat\Psi_h(n_{\esc}/n_{\AI})\ge T,
\end{aligned}
\end{equation}
with the same the convention as in \eqref{eq:Gamma-def} (i.e. $\hat\Psi_h(n_{\esc}/n_{\AI}) = 0$ whenever $n_{\AI} = n_{\esc} = 0$).
The plug-in increments are
\begin{equation}\label{eq:pilot-hat-increments}
    \hat\ell_R(r):=\log\frac{\hat g_1(r)}{\hat g_0(r)},
    \qquad
    \hat\ell_H(x,r):=\log
        \frac{\hat\rho_1(x\mid r)}{\hat\rho_0(x\mid r)}.
\end{equation}

\subsection{The Guarded Plug-In Policy}
\label{sec:unknown-ai-guarded-policy}

Section~\ref{sec:unknown-ai-pilot} estimates $f_0, f_1$ and replaces all
relevant known-$f$ objects with their plug-in versions, e.g.,
$\hat \Psi_h(s), \hat \Gamma_h(T, N)$, etc. Running
Algorithm~\ref{algo:seq-policy} on these objects alone, however, would not be
valid, since the statistic built from the plug-in increments
\eqref{eq:pilot-hat-increments} is not the exact log-likelihood ratio,
and the plug-in sensing program may overstate the information an
allocation delivers. To finish defining our policy, we therefore add
two guardrail terms, calibrated to the estimation error $r_m$ and hence
valid on $\calE_{\pilot,m}$: one inflates the information targets, the
other widens the stopping boundaries.

The first guard is an information-target guardrail, where instead of solving $\Gamma_h(T, N)$ at $T = T_{h, k}$ we add an extra buffer because the plug-in expected drift of the log-likelihood statistic may be overstated by the randomness in our estimates. Specifically, we use the following guarded information-target:
\begin{equation}\label{eq:pilot-guarded-targets}
    \widehat T^{\pilot}_{h,m,k}(N)
    :=T_{h,k}+(N+1)\varepsilon_m^{\mathrm{dr}},
\end{equation}
where $\varepsilon_m^{\mathrm{dr}}:=3|\calR| B_{\ell}r_m$, with  $B_{\ell}:=\max\{\log(p_1/p_0),\;
\log((1-p_0)/(1-p_1))\}$, is the upper bound on the deviation between the true-channel and plug-in-channel expected one-item drift of the plug-in log-likelihood statistic conditional on the pilot good event $\calE_{\pilot,m}$ (see details in Appendix~\ref{appendix-sec:uniform-pilot-perturbation}).
       
The guarded analogue of the buffered design value \eqref{eq:F-Delta-def}
is then
\begin{equation}\label{eq:pilot-guarded-design-value}
    \widehat F^{\pilot}_{m,k}(N)
    :=N\cdata+
      \max_{h\in\{0,1\}}
      \hat\Gamma_h(\widehat T^{\pilot}_{h,m,k}(N),N).
\end{equation}

\vspace{1mm}
\noindent
Call an integer
$N\ge N_{\fixed,\Hum}(\alpha_{2,k},\beta_{2,k})$ admissible if
$\hat\Gamma_h(\widehat T^{\pilot}_{h,m,k}(N),N)<\infty$ for both
$h\in\{0,1\}$, that is, if some allocation within a pool of size $N$
meets both guarded targets. If at least one admissible $N$ exists,
choose
\begin{equation}\label{eq:pilot-Nbar-choice}
    \widehat{\overline N}_{\main,m,k}
    \in\argmin_{N\in\mathbb Z_+, N\ge N_{\fixed,\Hum}(\alpha_{2,k},\beta_{2,k})}
       \widehat F^{\pilot}_{m,k}(N).
\end{equation}
and let $\widehat N_{\main,m,k}
    := \widehat{\overline N}_{\main,m,k} +1$. 
The remainder of the procedure is structurally identical to
Algorithm~\ref{algo:seq-policy}, with plug-in quantities and the
guardrails. Formally, for each $h$, choose an optimizer $(\hat n^*_{\Hum,h,m,k},\hat n^*_{\AI,h,m,k},
      \hat n^*_{\esc,h,m,k})$ of problem~\eqref{eq:pilot-hat-Gamma} at $N = \widehat{\overline N}_{\main,m,k}$ and $T = \widehat T^{\pilot}_{h,m,k}
          (\widehat{\overline N}_{\main,m,k})$. 
% \begin{equation}\label{eq:pilot-Gamma-optimizer}
% \begin{split}
%     (\hat n^*_{\Hum,h,m,k},\hat n^*_{\AI,h,m,k},
%       \hat n^*_{\esc,h,m,k})
% \in\argmin_{n_{\Hum},n_{\AI},n_{\esc}\ge0}\quad
%        &\cH n_{\Hum}+\cAI n_{\AI}+\cH n_{\esc}\\
%     \text{s.t.}\quad
%        &n_{\Hum}+n_{\AI}\le \widehat{\overline N}_{\main,m,k},\qquad 0\le n_{\esc}\le n_{\AI},\\
%        &n_{\Hum}J_X^{(h)}+n_{\AI}\hat I_R^{(h)}
%         +n_{\AI}\hat\Psi_h(n_{\esc}/n_{\AI})\ge \widehat T^{\pilot}_{h,m,k}
%           (\widehat{\overline N}_{\main,m,k}).
% \end{split}
% \end{equation}
Set 
\begin{equation}\label{eq:pilot-eta-human}
    \hat\eta^{\Hum}_{h,m,k}
    :=\begin{cases}
    \lceil\hat n^*_{\Hum,h,m,k}\rceil/(\lceil\hat n^*_{\Hum,h,m,k}\rceil
      +\lceil\hat n^*_{\AI,h,m,k}\rceil),
       &\lceil\hat n^*_{\Hum,h,m,k}\rceil
      +\lceil\hat n^*_{\AI,h,m,k}\rceil>0,\\
    0,&\lceil\hat n^*_{\Hum,h,m,k}\rceil
      +\lceil\hat n^*_{\AI,h,m,k}\rceil=0,
    \end{cases}
\end{equation}
and
\begin{equation*}\label{eq:pilot-shk}
    \hat s_{h,m,k}:=
    \begin{cases}
      \hat n^*_{\esc,h,m,k}/\hat n^*_{\AI,h,m,k},
         &\hat n^*_{\AI,h,m,k}>0,\\
      0,&\hat n^*_{\AI,h,m,k}=0.
    \end{cases}
\end{equation*}

\vspace{1mm}
\noindent
Choose $\hat\eta^{\esc}_{h,m,k}$ to be any optimizer of
\eqref{eq:pilot-hat-Psi} at $s=\hat s_{h,m,k}$; equivalently, sort the reports
by $\hat d^{(h)}(r)$ and use the fractional-knapsack rule of
Proposition~\ref{prop:escalation-solution}. In the dead zone, use
\begin{equation*}\label{eq:pilot-deadzone-rule}
    \hat\eta^{\Hum}_{*,m,k}
      :=\frac{\hat\eta^{\Hum}_{0,m,k}
              +\hat\eta^{\Hum}_{1,m,k}}2,
    \qquad
    \hat\eta^{\esc}_{*,m,k}(r)
      :=\frac{\hat\eta^{\esc}_{0,m,k}(r)
              +\hat\eta^{\esc}_{1,m,k}(r)}2.
\end{equation*}

If no admissible $N$ exists, so that the guarded outer problem
\eqref{eq:pilot-Nbar-choice} is infeasible, we fall back on the safe
default $\widehat N_{\main,m,k}:=
       N_{\fixed,\Hum}(\alpha_{2,k},\beta_{2,k})$
% \begin{equation}\label{eq:pilot-safe-default-N}
%     \widehat N_{\main,m,k}:=
%        N_{\fixed,\Hum}(\alpha_{2,k},\beta_{2,k}),
% \end{equation}
to query a human on every item.

Next, we define the second guardrail, on the crossing boundaries $a_k$
and $-b_k$. The policy now accumulates $\hat S$, built from the
plug-in increments \eqref{eq:pilot-hat-increments}, rather than the
exact log-likelihood ratio $S_t$, so crossing $a_k$ no longer certifies
the level $\alpha_{1,k}$; we therefore widen the boundaries by enough to
cover the gap $|\hat S-S|$.  To that end, let $\varepsilon_m^{\mathrm{LR}}$ defined in Appendix~\ref{sec-appendix:likelihood-ratio-bound} \eqref{Appendix-eq:pilot-eps-LR} be an upper bound on $|\hat\ell_R(r)-\ell_R(r)|$ on $\calE_{\pilot,m}$ (see details in Appendix~\ref{sec-appendix:likelihood-ratio-bound}). 
% \begin{equation}\label{eq:pilot-phi}
%     \phi_{\mathrm{rep}}(w):=
%     \log\frac{(1-p_1)(1-w)+p_1w}
%                    {(1-p_0)(1-w)+p_0w},\qquad 0\le w\le1,
% \end{equation}
% and
% \begin{equation}\label{eq:pilot-weight-intervals}
%     \underline w_m(r):=
%        \frac{\underline f_{1,m}(r)}
%             {\underline f_{1,m}(r)+\overline f_{0,m}(r)},
%     \qquad
%     \overline w_m(r):=
%        \frac{\overline f_{1,m}(r)}
%             {\overline f_{1,m}(r)+\underline f_{0,m}(r)},
% \end{equation}
% \vspace{1mm}
% \noindent
% where $\underline f_{x,m}(r):=[\hat f_x(r)-r_m]_+$ and
% $\overline f_{x,m}(r):=\min\{1,\hat f_x(r)+r_m\}$ are the lower and upper
% bounds on $f_x(r)$ on the pilot good event $\calE_{\pilot,m}$. We show in Appendix~\ref{appendix-sec:pilot-feasibility-proof} that $\varepsilon_m^{\mathrm{LR}}$ is the upper bound on $|\hat\ell_R(r)-\ell_R(r)|$ on $\calE_{\pilot,m}$, where 
% \begin{equation}\label{eq:pilot-eps-LR}
%     \varepsilon_m^{\mathrm{LR}}
%     :=\max_{r\in\calR}\max\left\{
%        \hat\ell_R(r)-\phi_{\mathrm{rep}}(\underline w_m(r)),
%        \phi_{\mathrm{rep}}(\overline w_m(r))-\hat\ell_R(r)\right\}.
% \end{equation}
If the guarded plug-in programs are feasible then we set the guardrail $\omega_{m,k}:=
      \widehat N_{\main,m,k}\varepsilon_m^{\mathrm{LR}}$, 
and the two stopping
boundaries are given by 
$$
a_k+\omega_{m,k} \text{ and } -(b_k+\omega_{m,k}).
$$
If the guarded plug-in programs are infeasible we set the guardrail
$\omega_{m,k}=0$, since the human-only safe default accumulates the exact
increments $\ell_X$, which depend only on the known $p_0,p_1$.
If neither boundary is reached, we use a fallback stage and reveal
every still-unknown label in the fixed set 
$$
\hat \calJ_{\Hum,k}:= \{1,\ldots,N_{\fixed,\Hum}(\alpha_{2,k},\beta_{2,k})\}
$$
and apply the exact randomized full-label Neyman--Pearson test at levels
$(\alpha_{2,k},\beta_{2,k})$. Denote the resulting policy by
$\hat\pi_{m,k}$ corresponding to $m$. 
\begin{table}[!t]
\centering\small
\caption{Parameters of the guarded plug-in policy $\hat\pi_{m,k}$ in Algorithm~\ref{algo:pilot-policy}.}
\label{tab:pilot-tuning-parameters}
\begin{tabularx}{\textwidth}{@{}l >{\raggedright\arraybackslash}p{0.23\textwidth} >{\raggedright\arraybackslash}X@{}}
\hline\hline
Parameter & Meaning & Value\\
\hline
$f_{\fb,k}$, $z_k$
  & Fallback-budget fraction and hypothesis boundary
  & Same as in Table~\ref{tab:tuning-parameters}\\[4pt]
% $\lambda_m$
%   & $\lambda_m=m^{-2}$, the smoothing mass added to the empirical report frequencies in $\hat f_x$ \eqref{eq:pilot-estimator}\\[4pt]
% $r_m$
%   & $r_m=\min\{1,\,t_m+|\calR|\lambda_m/m\}$ with $t_m=\sqrt{(2m)^{-1}\log(4|\calR|m^{2})}$, the estimation radius the two guardrails are calibrated to\\[4pt]
% $\varepsilon_m^{\mathrm{dr}}$
%   & $\varepsilon_m^{\mathrm{dr}}=3|\calR|B_{\ell}r_m$ with $B_{\ell}=\max\{\log(p_1/p_0),\log((1-p_0)/(1-p_1))\}$, the drift guardrail\\[4pt]
% $\varepsilon_m^{\mathrm{LR}}$
%   & Any upper bound on $|\hat\ell_R(r)-\ell_R(r)|$ valid on $\calE_{\pilot,m}$; we use \eqref{Appendix-eq:pilot-eps-LR}\\[4pt]
% $\widehat T^{\pilot}_{h,m,k}(N)$
%   & $\widehat T^{\pilot}_{h,m,k}(N)=T_{h,k}+(N+1)\varepsilon_m^{\mathrm{dr}}$ (see \eqref{eq:pilot-guarded-targets})\\[4pt]
$\widehat N_{\main,m,k}$
  & Fixed pool size
  & $\widehat N_{\main,m,k}=\widehat{\overline N}_{\main,m,k}+1$, where $\widehat{\overline N}_{\main,m,k}\in\argmin\bigl\{\widehat F^{\pilot}_{m,k}(N):N\in\Z_+,\,N\ge N_{\fixed,\Hum}(\alpha_{2,k},\beta_{2,k})\bigr\}$ and $\widehat F^{\pilot}_{m,k}(N)=N\cdata+\max_h\hat\Gamma_h(\widehat T^{\pilot}_{h,m,k}(N),N)$ (see \eqref{eq:pilot-guarded-targets}--\eqref{eq:pilot-Nbar-choice}); if problem~\eqref{eq:pilot-Nbar-choice} is not feasible, let $\widehat N_{\main,m,k}=N_{\fixed,\Hum}(\alpha_{2,k},\beta_{2,k})$\\[4pt]
$a_k + \omega_{m,k}, b_k + \omega_{m,k}$
  & Guarded stopping boundaries
  & See $a_k, b_k$ in \eqref{eq:boundaries}; 
  $\omega_{m,k}=\widehat N_{\main,m,k}\varepsilon_m^{\mathrm{LR}}$ when problem~\eqref{eq:pilot-Nbar-choice} is feasible, and $\omega_{m,k}=0$ otherwise\\[4pt]
$\hat\eta^{\Hum}_{h,m,k}, h\in\{0,1\}$
  & The probability to directly query human
  & See \eqref{eq:pilot-eta-human} if problem~\eqref{eq:pilot-Nbar-choice} is feasible; otherwise let $\hat\eta^{\Hum}_{h,m,k} = 1$\\[4pt]
$\hat\eta^{\esc}_{h,m,k}(\cdot), h\in\{0,1\}$
  & The probability to escalate
to human given that the item has already been AI-queried before
  & Optimal solution of \eqref{eq:pilot-hat-Psi} with $\hat s_{h,m,k}=\hat n^*_{\esc,h,m,k}/\hat n^*_{\AI,h,m,k}$ if problem~\eqref{eq:pilot-Nbar-choice} is feasible; otherwise let $\hat\eta^{\esc}_{h,m,k}(\cdot) = 0$\\[4pt]
$\hat\eta^{\Hum}_{*,m,k}$, $\hat\eta^{\esc}_{*,m,k}(\cdot)$
  & Dead-zone rule
  & Averages of the $h=0$ and $h=1$ rules, as in \eqref{eq:deadzone-rule}\\
\hline\hline
\end{tabularx}
\end{table}

We summarize the calculations of parameters needed for policy $\hat\pi_{m,k}$ in Table~\ref{tab:pilot-tuning-parameters}, and give the policy for estimated $f_0, f_1$ in
Algorithm~\ref{algo:pilot-policy}. We then show the feasibility and first-order optimality, both conditional on
the good pilot event $\calE_{\pilot,m}$.

\begin{algorithm}[h]
\small
\caption{Pilot-Calibrated Guarded Plug-In Policy\label{algo:pilot-policy} ($\hat\pi_{m,k}$)}
\DontPrintSemicolon
\KwIn{Primitives $(p_0,p_1,\cdata,\cAI,\cH)$, error targets
$(\alpha_k,\beta_k)$, admissible parameters
$(\Delta_k,f_{\fb,k},z_k)$, and pilot data $\calD_m$.}
Compute the error-budget split \eqref{eq:budget-split-alpha}–\eqref{eq:budget-split-beta}, boundaries $(a_k,b_k)$ in \eqref{eq:boundaries}, and
buffered targets $(T_{1,k},T_{0,k}) = (a_k + \Delta_k, b_k + \Delta_k)$. \; % in \eqref{eq:buffered-targets}.\;
Compute $\hat f_0,\hat f_1$, all plug-in quantities in
\eqref{eq:pilot-hat-g-q}--\eqref{eq:pilot-hat-increments}, and
$r_m,\varepsilon_m^{\mathrm{LR}},
\varepsilon_m^{\mathrm{dr}}$.\;
Form the guarded targets and outer objective in
\eqref{eq:pilot-guarded-targets}--\eqref{eq:pilot-guarded-design-value}.\;
\eIf{the guarded outer problem \eqref{eq:pilot-Nbar-choice} is feasible}{
  Solve \eqref{eq:pilot-Nbar-choice} for $\widehat{\overline N}_{\main,m,k}$ and set $\widehat N_{\main,m,k}
    \leftarrow \widehat{\overline N}_{\main,m,k} +1$; compute
  $\hat\eta^{\Hum}_{h,m,k}$, $\hat\eta^{\esc}_{h,m,k}$, and their dead-zone
  averages; set
  $\omega_{m,k}\leftarrow
     \widehat N_{\main,m,k}\varepsilon_m^{\mathrm{LR}}$.\;
}{
  Set $\widehat N_{\main,m,k}\leftarrow
       N_{\fixed,\Hum}(\alpha_{2,k},\beta_{2,k})$ and $\omega_{m,k}\leftarrow0$.\;
  For every $j\in\{0,1,*\}$ and $r\in\calR$, set
  $\hat\eta^{\Hum}_{j,m,k}\leftarrow1$ and
  $\hat\eta^{\esc}_{j,m,k}(r)\leftarrow0$.\;
}
Set the guarded boundaries to $a_k+\omega_{m,k}$ and
$-(b_k+\omega_{m,k})$.\;
Acquire $\widehat N_{\main,m,k}$ items and initialize $\hat S\leftarrow0$.\;
\For{$i=1$ \KwTo $\widehat N_{\main,m,k}$}{
  Set $J\leftarrow1$ if $\hat S>z_k$, $J\leftarrow0$ if
  $\hat S<-z_k$, and $J\leftarrow *$ otherwise.\;
  With probability $\hat\eta^{\Hum}_{J,m,k}$, query a human and add
  $\ell_X(X_i)$ to $\hat S$.\;
  Otherwise query the AI and add $\hat\ell_R(R_i)$; then, with probability
  $\hat\eta^{\esc}_{J,m,k}(R_i)$, query a human and add
  $\hat\ell_H(X_i,R_i)$.\;
  After every update, reject if $\hat S\ge a_k+\omega_{m,k}$ and accept if
  $\hat S\le-(b_k+\omega_{m,k})$.\;
}
If neither boundary is crossed, execute the fixed-set full-label fallback on $\hat \calJ_{\Hum,k}$.\;
\end{algorithm}

\vspace{1mm}
\begin{theorem}[Pilot-conditional exact validity]
\label{thm:pilot-validity}
Suppose Assumptions~\ref{ass:iid} and \ref{ass:pilot-sample} hold. For every $m\ge2$, every $k\ge1$, and every choice of the admissible parameters
$(\Delta_k,f_{\fb,k},z_k)$, the following hold almost surely on
$\calE_{\pilot,m}$:
\begin{equation*}\label{eq:pilot-conditional-validity}
    \Pp^{\hat\pi_{m,k}}_0
       (\delta^{\hat\pi_{m,k}}=1\mid\calD_m)
       \le\alpha_k,
    \qquad
    \Pp^{\hat\pi_{m,k}}_1
       (\delta^{\hat\pi_{m,k}}=0\mid\calD_m)
       \le\beta_k.
\end{equation*}
\end{theorem}
\vspace{1mm}

Theorem~\ref{thm:pilot-validity} is the plug-in counterpart of
Theorem~\ref{thm:feasibility}: the error targets are met exactly, at every
pilot size and without any rate condition, provided the pilot realization is
good. Averaging over pilot realizations removes the conditioning at the cost
of the failure probability of $\calE_{\pilot,m}$.

\vspace{1mm}
\begin{corollary}[Unconditional validity]
\label{cor:pilot-unconditional-validity}
Suppose Assumptions~\ref{ass:iid} and \ref{ass:pilot-sample} hold. For every $m\ge2$ and $k\ge1$, and every choice of the admissible parameters
$(\Delta_k,f_{\fb,k},z_k)$,
\begin{equation*}\label{eq:pilot-unconditional-validity}
    \Pp^{\hat\pi_{m,k}}_0(\delta^{\hat\pi_{m,k}}=1)
       \le\alpha_k+\delta_m,
    \qquad
    \Pp^{\hat\pi_{m,k}}_1(\delta^{\hat\pi_{m,k}}=0)
       \le\beta_k+\delta_m.
\end{equation*}
\end{corollary}

For first-order cost optimality of the policy, we focus on the balanced logarithmic regime under Assumption~\ref{ass:balanced-error-regime}. Under this regime, $\min\{T_{0,k},T_{1,k}\}=\Theta(L_k)$, so factors of
$\min\{T_{0,k},T_{1,k}\}/L_k$ are absorbed into primitive constants below. Theorem~\ref{thm:pilot-first-order} then states the convergence rate of $\hat\pi_{m,k}$. 
\vspace{1mm}
\begin{theorem}[First-order optimality with pilot-estimated AI accuracy]
\label{thm:pilot-first-order}
Suppose Assumptions~\ref{ass:iid}-\ref{ass:pilot-sample} hold, $f_{\fb,k}$ is bounded away
from zero and one, and $(\Delta_k,z_k)$ satisfy
\eqref{eq:delta-admissible} and \eqref{eq:z-admissible}. Let
$m=m_{\pilot,k}\to\infty$ satisfy
\begin{equation}\label{eq:pilot-main-rate-condition}
    L_kr_{m_{\pilot,k}}=o(\Delta_k).
\end{equation}
Then, on $\calE_{\pilot,m_{\pilot,k}}$ and for all sufficiently large $k$, 
\begin{align*}
&\frac{\max_h
    \E_h^{\hat\pi_{m_{\pilot,k},k}}
       [C^{\hat\pi_{m_{\pilot,k},k}}\mid\calD_{m_{\pilot,k}}]}
   {\LB_k}
   = 1+o(1).                       
\end{align*}
\end{theorem}

\vspace{1mm}
%We refer readers to Appendix~\ref{appendix-sec:pilot-first-order-proof} for theproof of Theorem~\ref{thm:pilot-first-order}. 
In
words, not knowing $f_0,f_1$ does not impact our first-order asymptotic optimality once the pilot is
large enough relative to $L_k$, though the cost of collecting that pilot is not
yet charged; we return to it in Section~\ref{sec:unknown-ai-rates}.

% {\color{red} ST: Is it necessary to define $\Delta^{\eff}_{m,k}$ in the main paper? Similarly, we probably don't need to show the long summation in (\ref{eq:pilot-ratio-bound}) in the main paper and just jump directly to 1 + o(1)?}

% {\color{red} ST: this subsection is short. Maybe we can combine with the validity result in the previous subsection? Alternatively, we move the validity result here to this subsection? So 6.2 covers the policy and 6.3 covers both the validity and asympt optimality.}

% {\color{red} ST: as much as possible, notations that do not receive much attention/discussions in the main paper can perhaps be deferred to appendix}

\subsection{Pilot-Size Tradeoffs and Calibration Cost}
\label{sec:unknown-ai-rates}
In this subsection, we further unpack Theorem~\ref{thm:pilot-first-order} and condition~\eqref{eq:pilot-main-rate-condition}. By construction, $r_m=O(\sqrt{\log{(m)}/m})$, and thus condition
\eqref{eq:pilot-main-rate-condition} is implied by
\begin{equation}\label{eq:pilot-rate-equivalent}
    \frac{m_{\pilot,k}}{\log m_{\pilot,k}}
       \left(\frac{\Delta_k}{L_k}\right)^2\longrightarrow\infty.
\end{equation}
Therefore, there is a tradeoff between pilot size and a larger statistical buffer. We summarize this and give specific conditions on tuning parameters to maintain the first-order optimality in Theorem~\ref{thm:pilot-first-order} in the following corollary:

\vspace{1mm}
\begin{corollary}[Concrete pilot and buffer rates]
\label{cor:pilot-rates}
Suppose Assumptions~\ref{ass:iid}-\ref{ass:pilot-sample} hold. Assume $f_{\fb,k}=1/2$ and $z_k=1$. Let $\Delta_k=L_k^{2/3}$ and 
$m_{\pilot,k}=\lceil L_k^{2/3}(\log L_k)^2\rceil$. Then we have that $m_{\pilot,k}/\log m_{\pilot,k}\gg L_k^{2/3}$, and on \(\calE_{\pilot,m_{\pilot,k}}\) and for all sufficiently large \(k\),
$$
\frac{\max_h
    \E_h^{\hat\pi_{m_{\pilot,k},k}}
       [C^{\hat\pi_{m_{\pilot,k},k}}\mid\calD_{m_{\pilot,k}}]}
   {\LB_k} \leq 1 + O(L^{-1/3}_k).
$$

\end{corollary}

\vspace{1mm}
%Proof of Corollary~\ref{cor:pilot-rates} follows directly by substituting $\Delta_k$ and $m_{\pilot,k}$ into \eqref{eq:pilot-rate-equivalent} and the rate of $\max_h \E_h^{\hat\pi_{m_{\pilot,k},k}}[C^{\hat\pi_{m_{\pilot,k},k}}\mid\calD_{m_{\pilot,k}}]$ given in Appendix~\ref{appendix-sec:pilot-first-order-1}, \eqref{eq:pilot-ratio-bound}; we therefore omit the details. 

Corollary~\ref{cor:pilot-rates} does not account for the cost for the pilot items.  If, instead, a fresh paired pilot is
collected solely for the $k$th test, its acquisition cost must also be charged
to that test.  Since each pilot item is acquired and labeled by both the AI and
the human, it has a total cost of $\cdata+\cAI+\cH$. To obtain a minimum of $m$ pilot samples for both labels, recall we need $M_{\tot}$ total pilot samples. Thus the total cost for pilot samples is $C_{m}^{\pilot}:=M_{\tot}(\cdata+\cAI+\cH)$. 
The maximum expected cost including the pilot data collection cost over two hypotheses is 
$\max_h \E_h^{\hat\pi_{m,k}}
  [C_m^{\pilot} + C^{\hat\pi_{m,k}}]$. Corollary~\ref{cor:pilot-charged} verifies this total cost under the choice of $\Delta_k$ given by Corollary~\ref{cor:pilot-rates}. 

  \vspace{1mm}
\begin{corollary}[Charging the one-time pilot]
\label{cor:pilot-charged}
Suppose Assumptions~\ref{ass:iid}-\ref{ass:pilot-sample} hold, and let $f_{\fb,k}=1/2$ and $z_k=1$. Then the choice in
Corollary~\ref{cor:pilot-rates} satisfies $L_kr_{m_{\pilot,k}}=o(\Delta_k)$ and $m_{\pilot,k}=o(L_k)$, and gives that for all sufficiently large \(k\),
\[
 \frac{\max_h \E_h^{\hat\pi_{m_{\pilot, k},k}}
  [C_{m_{\pilot, k}}^{\pilot} + C^{\hat\pi_{m_{\pilot, k},k}}]}{\LB_k}
 \leq 1+\widetilde O(L_k^{-1/3}).
\]

\end{corollary}
%We leave the proof of Corollary~\ref{cor:pilot-charged} to Appendix~\ref{appendix-sec:unknown-ai-rates}.

\vspace{1mm}
Corollary~\ref{cor:pilot-charged} closes the gap between the statistical
and operational costs of learning the AI-output model.  It tells us that even when the paired pilot must be collected specifically for
the current test and its full data, AI, and human costs are charged to the
procedure, the resulting total cost remains first-order optimal.  Thus the
known-$f$ benchmark is asymptotically attainable without assuming that
calibration data are available for free.

\section{Conclusion}
\label{sec:conclusion}

This paper studies how to combine inexpensive but imperfect AI information
with costly human verification when the goal is to conduct a statistically
valid hypothesis test at minimum cost. The key operational feature is
selectivity: after acquiring a fixed pool of items, the decision maker can
choose whether to query the AI, query a human directly, escalate an
AI-scored item to a human, or stop once sufficient evidence has accumulated.
This creates a joint statistical and operational design problem in which the
value of a query depends not only on its cost and information content, but
also on the information already collected.

We derive an information-theoretic lower bound that captures the minimum
cost required to satisfy the testing errors while accounting for data
acquisition, AI scoring, direct human review, and selective escalation. We
then develop SCALE, a sequential cost-aware policy that dynamically combines
these actions as evidence accumulates. SCALE is finite-sample valid and
matches the lower bound to first order as the target errors vanish. We also extend the analysis to the practically important case in which the
AI-output model is unknown and must be estimated from paired AI--human pilot
data. A guarded plug-in version of SCALE remains first-order optimal when
the pilot is sufficiently accurate. 

Several directions remain open. We have focused on a binary label, two simple
hypotheses, a single AI source, and a finite AI-report alphabet. Extending the
framework to composite hypotheses, multiple AI systems with heterogeneous
costs and accuracies, and richer or continuous report spaces would broaden
its applicability. Another natural direction is to learn the AI-output model
during the main experiment rather than through a separate pilot, thereby
jointly deciding when information should be used for calibration and when it
should be used for the hypothesis test itself. More broadly, the analysis
suggests that the relevant question in human--AI inference is not simply
whether AI should replace human judgment. Rather, the operational value of
AI comes from deciding when inexpensive machine information is sufficient
and when the remaining uncertainty is valuable enough to justify human
verification.

\label{sec:conclusion}

\bibliography{ref}

\begin{thebibliography}{46}
\expandafter\ifx\csname natexlab\endcsname\relax\def\natexlab#1{#1}\fi
\expandafter\ifx\csname url\endcsname\relax
  \def\url#1{{\tt #1}}\fi
\expandafter\ifx\csname urlprefix\endcsname\relax\def\urlprefix{URL }\fi
\expandafter\ifx\csname urlstyle\endcsname\relax
  \expandafter\ifx\csname doi\endcsname\relax
  \def\doi#1{doi:\discretionary{}{}{}#1}\fi \else
  \expandafter\ifx\csname doi\endcsname\relax
  \def\doi{doi:\discretionary{}{}{}\begingroup \urlstyle{rm}\Url}\fi \fi

\bibitem[{Alonzo et~al.(2003)Alonzo, Pepe, and Lumley}]{alonzo2003estimating}
Alonzo, Todd~A., Margaret~Sullivan Pepe, Thomas Lumley. 2003.
\newblock Estimating disease prevalence in two-phase studies.
\newblock {\it Biostatistics\/} {\bf 4}(2) 313--326.
\newblock \doi{10.1093/biostatistics/4.2.313}.

\bibitem[{Angelopoulos et~al.(2023)Angelopoulos, Bates, Fannjiang, Jordan, and
  Zrnic}]{angelopoulos2023prediction}
Angelopoulos, Anastasios~N., Stephen Bates, Clara Fannjiang, Michael~I. Jordan,
  Tijana Zrnic. 2023.
\newblock Prediction-powered inference.
\newblock {\it Science\/} {\bf 382}(6671) 669--674.
\newblock \doi{10.1126/science.adi6000}.

\bibitem[{Angelopoulos et~al.(2025)Angelopoulos, Eisenstein, Berant, Agarwal,
  and Fisch}]{angelopoulos2025costoptimal}
Angelopoulos, Anastasios~N., Jacob Eisenstein, Jonathan Berant, Alekh Agarwal,
  Adam Fisch. 2025.
\newblock Cost-optimal active {AI} model evaluation.
\newblock \doi{10.48550/arXiv.2506.07949}.

\bibitem[{Apakama et~al.(2025)Apakama, Nguyen, Hyppolite, Soffer, Mudrik, Ling,
  Moses, Temnycky, Glasser, Anderson, Parchure, Woullard, Edalati, Chan, Kronk,
  Freeman, Kia, Timsina, Levin, Khera, Kovatch, Charney, Carr, Richardson,
  Horowitz, Klang, and Nadkarni}]{apakama2025bias}
Apakama, Donald~U., Kim-Anh-Nhi Nguyen, Daphnee Hyppolite, Shelly Soffer, Aya
  Mudrik, Emilia Ling, Akini Moses, Ivanka Temnycky, Allison Glasser, Rebecca
  Anderson, Prathamesh Parchure, Evajoyce Woullard, Masoud Edalati, Lili Chan,
  Clair Kronk, Robert Freeman, Arash Kia, Prem Timsina, Matthew~A. Levin, Rohan
  Khera, Patricia Kovatch, Alexander~W. Charney, Brendan~G. Carr, Lynne~D.
  Richardson, Carol~R. Horowitz, Eyal Klang, Girish~N. Nadkarni. 2025.
\newblock Identifying bias at scale in clinical notes using large language
  models.
\newblock {\it Mayo Clinic Proceedings: Digital Health\/} {\bf 3}(4) 100296.
\newblock \doi{10.1016/j.mcpdig.2025.100296}.

\bibitem[{Baraud(2002)}]{baraud2002nonasymptotic}
Baraud, Yannick. 2002.
\newblock Non-asymptotic minimax rates of testing in signal detection.
\newblock {\it Bernoulli\/} {\bf 8}(5) 577--606.

\bibitem[{Begg and Greenes(1983)}]{begg1983assessment}
Begg, Colin~B., Robert~A. Greenes. 1983.
\newblock Assessment of diagnostic tests when disease verification is subject
  to selection bias.
\newblock {\it Biometrics\/} {\bf 39}(1) 207--215.
\newblock \doi{10.2307/2530820}.

\bibitem[{Bertsimas and Tsitsiklis(1997)}]{bertsimas1997introduction}
Bertsimas, Dimitris, John~N Tsitsiklis. 1997.
\newblock {\it Introduction to linear optimization\/}, vol.~6.
\newblock Athena scientific Belmont, MA.

\bibitem[{Boyd and Vandenberghe(2004)}]{boyd2004convex}
Boyd, Stephen, Lieven Vandenberghe. 2004.
\newblock {\it Convex optimization\/}.
\newblock Cambridge university press.

\bibitem[{Casella and Berger(2024)}]{casella2024statistical}
Casella, George, Roger Berger. 2024.
\newblock {\it Statistical inference\/}.
\newblock Chapman and Hall/CRC.

\bibitem[{Chen and Goldfarb-Tarrant(2025)}]{chen2025safer}
Chen, Hongyu, Seraphina Goldfarb-Tarrant. 2025.
\newblock Safer or luckier? {LLM}s as safety evaluators are not robust to
  artifacts.
\newblock {\it Proceedings of the 63rd Annual Meeting of the Association for
  Computational Linguistics (Volume 1: Long Papers)\/}. Association for
  Computational Linguistics, Vienna, Austria, 19750--19766.
\newblock \doi{10.18653/v1/2025.acl-long.970}.

\bibitem[{Chernoff(1959)}]{chernoff1959sequential}
Chernoff, Herman. 1959.
\newblock Sequential design of experiments.
\newblock {\it The Annals of Mathematical Statistics\/} {\bf 30}(3) 755--770.
\newblock \doi{10.1214/aoms/1177706205}.

\bibitem[{Cover and Thomas(2006)}]{cover2006elements}
Cover, Thomas~M., Joy~A. Thomas. 2006.
\newblock {\it Elements of Information Theory\/}.
\newblock 2nd ed. John Wiley \& Sons, Hoboken, NJ.
\newblock \doi{10.1002/047174882X}.

\bibitem[{Csillag et~al.(2025)Csillag, Struchiner, and
  Goedert}]{csillag2025prediction}
Csillag, Daniel, Claudio~Jose Struchiner, Guilherme~Tegoni Goedert. 2025.
\newblock Prediction-powered e-values.
\newblock {\it Proceedings of the 42nd International Conference on Machine
  Learning\/}, {\it Proceedings of Machine Learning Research\/}, vol. 267.
  PMLR, 11493--11514.

\bibitem[{Dekoninck et~al.(2026)Dekoninck, Petrov, Minchev, Marinov, Drencheva,
  Konova, Shumanov, Tsvetkov, Drenchev, Todorov et~al.}]{dekoninck2026open}
Dekoninck, Jasper, Ivo Petrov, Kristian Minchev, Miroslav Marinov, Maria
  Drencheva, Lyuba Konova, Milen Shumanov, Kaloyan Tsvetkov, Nikolay Drenchev,
  Lazar Todorov, et~al. 2026.
\newblock The open proof corpus: A large-scale study of llm-generated
  mathematical proofs.
\newblock {\it International Conference on Learning Representations\/}, vol.
  2026. 22214--22244.

\bibitem[{Freedman(1975)}]{freedman1975}
Freedman, David~A. 1975.
\newblock On tail probabilities for martingales.
\newblock {\it The Annals of Probability\/} {\bf 3}(1) 100--118.

\bibitem[{Imai et~al.(2023)Imai, Jiang, Greiner, Halen, and Shin}]{imai_judge}
Imai, Kosuke, Zhichao Jiang, D~James Greiner, Ryan Halen, Sooahn Shin. 2023.
\newblock Experimental evaluation of algorithm-assisted human decision-making:
  application to pretrial public safety assessment*.
\newblock {\it Journal of the Royal Statistical Society Series A: Statistics in
  Society\/} {\bf 186}(2) 167--189.
\newblock \doi{10.1093/jrsssa/qnad010}.
\newblock \urlprefix\url{https://doi.org/10.1093/jrsssa/qnad010}.

\bibitem[{Jiang et~al.(2026)Jiang, Chen, Cao, Lee, and
  Tan}]{jiang2026codejudgebench}
Jiang, Hongchao, Yiming Chen, Yushi Cao, Hung-yi Lee, Robby~T. Tan. 2026.
\newblock {CodeJudgeBench}: Benchmarking {LLM}-as-a-judge for coding tasks.
\newblock {\it Proceedings of the 64th Annual Meeting of the Association for
  Computational Linguistics (Volume 1: Long Papers)\/}. Association for
  Computational Linguistics, San Diego, California, United States,
  19416--19448.
\newblock \doi{10.18653/v1/2026.acl-long.888}.

\bibitem[{Kartik et~al.(2022)Kartik, Nayyar, and Mitra}]{kartik2022fixed}
Kartik, Dhruva, Ashutosh Nayyar, Urbashi Mitra. 2022.
\newblock Fixed-horizon active hypothesis testing.
\newblock {\it IEEE Transactions on Automatic Control\/} {\bf 67}(4)
  1882--1897.
\newblock \doi{10.1109/TAC.2021.3090742}.

\bibitem[{Kossen et~al.(2021)Kossen, Farquhar, Gal, and
  Rainforth}]{kossen2021active}
Kossen, Jannik, Sebastian Farquhar, Yarin Gal, Tom Rainforth. 2021.
\newblock Active testing: Sample-efficient model evaluation.
\newblock {\it Proceedings of the 38th International Conference on Machine
  Learning\/}, {\it Proceedings of Machine Learning Research\/}, vol. 139.
  PMLR, 5753--5763.

\bibitem[{Kry{\'s}ci{\'n}ski et~al.(2020)Kry{\'s}ci{\'n}ski, McCann, Xiong, and
  Socher}]{kryscinski2020evaluating}
Kry{\'s}ci{\'n}ski, Wojciech, Bryan McCann, Caiming Xiong, Richard Socher.
  2020.
\newblock Evaluating the factual consistency of abstractive text summarization.
\newblock {\it Proceedings of the 2020 conference on empirical methods in
  natural language processing (EMNLP)\/}. 9332--9346.

\bibitem[{Ma and Cand{\`e}s(2026)}]{ma2026opal}
Ma, Virginia~L., Emmanuel~J. Cand{\`e}s. 2026.
\newblock Optimized labeling resource allocation for prediction-assisted
  inference via {OPAL}.
\newblock \doi{10.48550/arXiv.2606.03211}.

\bibitem[{Madras et~al.(2018)Madras, Pitassi, and Zemel}]{madras2018predict}
Madras, David, Toniann Pitassi, Richard Zemel. 2018.
\newblock Predict responsibly: Improving fairness and accuracy by learning to
  defer.
\newblock {\it Advances in Neural Information Processing Systems\/}, vol.~31.
  Curran Associates, Inc.

\bibitem[{Marks-Anglin et~al.(2025)Marks-Anglin, Chen, Luo, Hubbard, and
  Chen}]{marksanglin2025optimal}
Marks-Anglin, Arielle, Jianmin Chen, Chongliang Luo, Rebecca Hubbard, Yong
  Chen. 2025.
\newblock Optimal surrogate-assisted sampling for cost-efficient validation of
  electronic health record outcomes.
\newblock {\it Statistics in Medicine\/} {\bf 44}(10--12) e70095.
\newblock \doi{10.1002/sim.70095}.

\bibitem[{Mazeika et~al.(2024)Mazeika, Phan, Yin, Zou, Wang, Mu, Sakhaee, Li,
  Basart, Li, Forsyth, and Hendrycks}]{mazeika2024harmbench}
Mazeika, Mantas, Long Phan, Xuwang Yin, Andy Zou, Zifan Wang, Norman Mu, Elham
  Sakhaee, Nathaniel Li, Steven Basart, Bo~Li, David Forsyth, Dan Hendrycks.
  2024.
\newblock {HarmBench}: A standardized evaluation framework for automated red
  teaming and robust refusal.
\newblock {\it Proceedings of the 41st International Conference on Machine
  Learning\/}, {\it Proceedings of Machine Learning Research\/}, vol. 235.
  PMLR, 35181--35224.

\bibitem[{McNamee(2003)}]{mcnamee2003efficiency}
McNamee, Roseanne. 2003.
\newblock Efficiency of two-phase designs for prevalence estimation.
\newblock {\it International Journal of Epidemiology\/} {\bf 32}(6) 1072--1078.
\newblock \doi{10.1093/ije/dyg230}.

\bibitem[{Movva et~al.(2024)Movva, Koh, and Pierson}]{movva2024annotation}
Movva, Rajiv, Pang~Wei Koh, Emma Pierson. 2024.
\newblock Annotation alignment: Comparing {LLM} and human annotations of
  conversational safety.
\newblock {\it Proceedings of the 2024 Conference on Empirical Methods in
  Natural Language Processing\/}. Association for Computational Linguistics,
  Miami, Florida, USA, 9048--9062.
\newblock \doi{10.18653/v1/2024.emnlp-main.511}.

\bibitem[{Mozannar and Sontag(2020)}]{mozannar2020defer}
Mozannar, Hussein, David Sontag. 2020.
\newblock Consistent estimators for learning to defer to an expert.
\newblock {\it Proceedings of the 37th International Conference on Machine
  Learning\/}, {\it Proceedings of Machine Learning Research\/}, vol. 119.
  PMLR, 7076--7087.

\bibitem[{Naghshvar and Javidi(2013{\natexlab{a}})}]{naghshvar2013active}
Naghshvar, Mohammad, Tara Javidi. 2013{\natexlab{a}}.
\newblock Active sequential hypothesis testing.
\newblock {\it The Annals of Statistics\/} {\bf 41}(6) 2703--2738.
\newblock \doi{10.1214/13-AOS1144}.

\bibitem[{Naghshvar and
  Javidi(2013{\natexlab{b}})}]{naghshvar2013sequentiality}
Naghshvar, Mohammad, Tara Javidi. 2013{\natexlab{b}}.
\newblock Sequentiality and adaptivity gains in active hypothesis testing.
\newblock {\it IEEE Journal of Selected Topics in Signal Processing\/} {\bf
  7}(5) 768--782.
\newblock \doi{10.1109/JSTSP.2013.2261279}.

\bibitem[{Neyman and Pearson(1933)}]{neyman1933ix}
Neyman, Jerzy, Egon~Sharpe Pearson. 1933.
\newblock Ix. on the problem of the most efficient tests of statistical
  hypotheses.
\newblock {\it Philosophical Transactions of the Royal Society of London.
  Series A, Containing Papers of a Mathematical or Physical Character\/} {\bf
  231}(694-706) 289--337.

\bibitem[{Nitinawarat et~al.(2013)Nitinawarat, Atia, and
  Veeravalli}]{nitinawarat2013controlled}
Nitinawarat, Sirin, George~K. Atia, Venugopal~V. Veeravalli. 2013.
\newblock Controlled sensing for multihypothesis testing.
\newblock {\it IEEE Transactions on Automatic Control\/} {\bf 58}(10)
  2451--2464.
\newblock \doi{10.1109/TAC.2013.2261188}.

\bibitem[{Nitinawarat and Veeravalli(2015)}]{nitinawarat2015controlled}
Nitinawarat, Sirin, Venugopal~V. Veeravalli. 2015.
\newblock Controlled sensing for sequential multihypothesis testing with
  controlled markovian observations and non-uniform control cost.
\newblock {\it Sequential Analysis\/} {\bf 34}(1) 1--24.
\newblock \doi{10.1080/07474946.2014.961864}.

\bibitem[{Pepe(1992)}]{pepe1992inference}
Pepe, Margaret~Sullivan. 1992.
\newblock Inference using surrogate outcome data and a validation sample.
\newblock {\it Biometrika\/} {\bf 79}(2) 355--365.
\newblock \doi{10.1093/biomet/79.2.355}.

\bibitem[{Reilly(1996)}]{reilly1996optimal}
Reilly, Marie. 1996.
\newblock Optimal sampling strategies for two-stage studies.
\newblock {\it American Journal of Epidemiology\/} {\bf 143}(1) 92--100.
\newblock \doi{10.1093/oxfordjournals.aje.a008662}.

\bibitem[{Sethi et~al.(2026)Sethi, Caskey, Gao, Churpek, Miller, Mayampurath,
  Salisbury-Afshar, Afshar, and Dligach}]{sethi2026stigma}
Sethi, Rohan, John Caskey, Yanjun Gao, Matthew~M. Churpek, Timothy~A. Miller,
  Anoop Mayampurath, Elizabeth Salisbury-Afshar, Majid Afshar, Dmitry Dligach.
  2026.
\newblock Detecting stigmatizing language in clinical notes with large language
  models for addiction care.
\newblock {\it npj Health Systems\/} {\bf 3} 15.
\newblock \doi{10.1038/s44401-026-00069-0}.

\bibitem[{Shrout and Newman(1989)}]{shrout1989design}
Shrout, Patrick~E, Stephen~C Newman. 1989.
\newblock Design of two-phase prevalence surveys of rare disorders.
\newblock {\it Biometrics\/}  549--555.

\bibitem[{Tang et~al.(2012)Tang, Qiu, Poon, and Tang}]{tang2012test}
Tang, Man-Lai, Shi-Fang Qiu, Wai-Yin Poon, Nian-Sheng Tang. 2012.
\newblock Test procedures for disease prevalence with partially validated data.
\newblock {\it Journal of Biopharmaceutical Statistics\/} {\bf 22}(2) 368--386.
\newblock \doi{10.1080/10543406.2010.544527}.

\bibitem[{Tao et~al.(2020)Tao, Zeng, and Lin}]{tao2020optimal}
Tao, Ran, Donglin Zeng, Dan-Yu Lin. 2020.
\newblock Optimal designs of two-phase studies.
\newblock {\it Journal of the American Statistical Association\/} {\bf
  115}(532) 1946--1959.
\newblock \doi{10.1080/01621459.2019.1671200}.

\bibitem[{Tenenbein(1970)}]{tenenbein1970double}
Tenenbein, Aaron. 1970.
\newblock A double sampling scheme for estimating from binomial data with
  misclassifications.
\newblock {\it Journal of the American Statistical Association\/} {\bf 65}(331)
  1350--1361.
\newblock \doi{10.1080/01621459.1970.10481170}.

\bibitem[{Tenzer et~al.(2026)Tenzer, Tolochinsky, and
  Romano}]{tenzer2026betting}
Tenzer, Yaniv, Elad Tolochinsky, Yaniv Romano. 2026.
\newblock Semi-supervised hypothesis testing by betting on predictions.
\newblock \doi{10.48550/arXiv.2605.28533}.

\bibitem[{Tong and Zhang(2024)}]{tong2024codejudge}
Tong, Weixi, Tianyi Zhang. 2024.
\newblock {CodeJudge}: Evaluating code generation with large language models.
\newblock {\it Proceedings of the 2024 Conference on Empirical Methods in
  Natural Language Processing\/}. Association for Computational Linguistics,
  Miami, Florida, USA, 20032--20051.
\newblock \doi{10.18653/v1/2024.emnlp-main.1118}.

\bibitem[{Vershinin et~al.(2026)Vershinin, Cohen, and
  Gurewitz}]{vershinin2026active}
Vershinin, George, Asaf Cohen, Omer Gurewitz. 2026.
\newblock Active sequential hypothesis testing with non-homogeneous costs.
\newblock {\it 2026 IEEE International Conference on Acoustics, Speech and
  Signal Processing (ICASSP)\/}. IEEE.
\newblock \doi{10.1109/ICASSP55912.2026.11463839}.

\bibitem[{Wald(1945)}]{waldSPRT}
Wald, Abraham. 1945.
\newblock Sequential tests of statistical hypotheses.
\newblock {\it Annals of Mathematical Statistics\/} {\bf 16} 256--298.
\newblock \urlprefix\url{https://api.semanticscholar.org/CorpusID:222593486}.

\bibitem[{Zhang et~al.(2025)Zhang, Scroggins, Harkins, Hulchafo, Moen,
  Tadiello, Barcelona, and Topaz}]{zhang2025equitable}
Zhang, Zhihong, Jihye~Kim Scroggins, Sarah Harkins, Ismael~Ibrahim Hulchafo,
  Hans Moen, Michele Tadiello, Veronica Barcelona, Maxim Topaz. 2025.
\newblock Toward equitable documentation: Evaluating {ChatGPT}'s role in
  identifying and rephrasing stigmatizing language in electronic health
  records.
\newblock {\it Nursing Outlook\/} {\bf 73}(4) 102472.
\newblock \doi{10.1016/j.outlook.2025.102472}.

\bibitem[{Zhao et~al.(2025)Zhao, Luo, Tian, Lin, Yan, Li, and
  Ma}]{zhao2025codejudgeeval}
Zhao, Yuwei, Ziyang Luo, Yuchen Tian, Hongzhan Lin, Weixiang Yan, Annan Li,
  Jing Ma. 2025.
\newblock {CodeJudge-Eval}: Can large language models be good judges in code
  understanding?
\newblock {\it Proceedings of the 31st International Conference on
  Computational Linguistics\/}. Association for Computational Linguistics, Abu
  Dhabi, UAE, 73--95.

\bibitem[{Zrnic and Cand{\`e}s(2024)}]{zrnic2024active}
Zrnic, Tijana, Emmanuel Cand{\`e}s. 2024.
\newblock Active statistical inference.
\newblock {\it Proceedings of the 41st International Conference on Machine
  Learning\/}, {\it Proceedings of Machine Learning Research\/}, vol. 235.
  PMLR, 62993--63010.

\end{thebibliography}
\bibliographystyle{ormsv080}

%% Here starts the e-companion (EC)
%%%%%%%%%%%%%%%%%%%%%%%%%%%%%%%%%%%%%%%%%%%%%%%%%%%%%%%%%%
\ECSwitch
%% Make hyperref anchors unique in the e-companion.
\renewcommand{\theHsection}{EC.\arabic{section}}
\renewcommand{\theHsubsection}{EC.\arabic{section}.\arabic{subsection}}
\renewcommand{\theHsubsubsection}{EC.\arabic{section}.\arabic{subsection}.\arabic{subsubsection}}
\renewcommand{\theHequation}{EC.\arabic{equation}}
\renewcommand{\theHtheorem}{EC.\arabic{theorem}}
\renewcommand{\theHlemma}{EC.\arabic{lemma}}
\renewcommand{\theHcorollary}{EC.\arabic{corollary}}
\renewcommand{\theHproposition}{EC.\arabic{proposition}}
\renewcommand{\theHdefinition}{EC.\arabic{definition}}
\renewcommand{\theHassumption}{EC.\arabic{assumption}}
\renewcommand{\theHfigure}{EC.\arabic{figure}}
\renewcommand{\theHtable}{EC.\arabic{table}}

% \ECDisclaimer
%%%%%%%%%%%%%%%%%%%%%%%%%%%%%%%%%%%%%%%%%%%%%%%%%%%%%%%%%%

\newcommand{\mD}{\mathcal D}

\section{Additional Materials for Section~\ref{sec:lower-bound}}
\label{appendix-sec:lower-bound}

\subsection{Proof of Lemma~\ref{lem:data-pool}}
\begin{proof}{Proof of Lemma~\ref{lem:data-pool}}
Suppose, for a contradiction, that there exists a feasible policy $\pi\in\calF(\alpha,\beta)$ using $N<N_{\fixed,\Hum}(\alpha,\beta)$ acquired items. We will construct a randomized full-label test $\bar\phi:\{0,1\}^N\to[0,1]$ with the same type-I and type-II error probabilities as $\pi$. Here $\bar\phi(x)$ denotes the probability of rejecting $H_0$ given the full label vector $x$. This will contradict the definition of $N_{\fixed,\Hum}(\alpha,\beta)$ as the minimum sample size for a full-label randomized test satisfying the target errors.

Represent any policy randomization by a seed $U$ independent of the data and with the same law under both hypotheses. Conditional on $U=u$, the policy is deterministic. Given a full label vector $x\in\{0,1\}^N$, a potential report vector $r\in\calR^N$, and a seed $u$, simulate the policy as follows. Whenever the simulated policy queries the AI on item $i$, reveal $r_i$ to the simulation. Whenever it queries a human on item $i$, reveal $x_i$. Continue until the simulated policy stops, and let $\tilde\phi(x,r,u)\in\{0,1\}$ be the resulting decision. Thus each item is AI-queried at most once and human-queried at most once, so the simulation reveals each $r_i$ and each $x_i$ at most once.

Under the actual experiment, $\tilde\phi(X,R,U)=\delta^\pi$ almost surely, because the simulation reveals exactly the observations that the policy would receive. Hence
\begin{equation}\label{eq:psi-errors}
    \Pp^\pi_0(\tilde\phi(X,R,U)=1)\le\alpha,
    \qquad
    \Pp^\pi_1(\tilde\phi(X,R,U)=0)\le\beta.
\end{equation}

Conditional on $X=x$, the potential reports $R_1,\ldots,R_N$ have product law $\prod_i f_{x_i}(r_i)$ by Assumption~\ref{ass:iid}. This conditional law does not depend on whether $H_0$ or $H_1$ is true. The seed law also does not depend on the hypothesis. Define
\[
    \bar\phi(x):=\E[\tilde\phi(x,R,U)\mid X=x],
    \qquad x\in\{0,1\}^N,
\]
where the expectation is taken over the common conditional law of $(R,U)$ given $X=x$. Because $\tilde\phi\in\{0,1\}$, its conditional expectation satisfies $\bar\phi(x)\in[0,1]$; thus $\bar\phi$ is a valid randomized full-label test. By the tower property of conditional expectation,
\[
    \E_h[\bar\phi(X)]
    =\E_h[\E[\tilde\phi(X,R,U)\mid X]]
    =\E_h[\tilde\phi(X,R,U)],
    \qquad h\in\{0,1\}.
\]
Together with \eqref{eq:psi-errors}, this implies
\[
    \E_0[\bar\tilde\phi(X)]\le\alpha,
    \qquad
    \E_1[1-\bar\tilde\phi(X)]\le\beta.
\]
We have therefore constructed a full-label randomized test satisfying the target errors with only $N<N_{\fixed,\Hum}(\alpha,\beta)$ labels. This contradicts the minimality in Definition~\ref{def:Nfull}. Hence every feasible policy must satisfy $N\ge N_{\fixed,\Hum}(\alpha,\beta)$.
\end{proof}

\subsection{Transcript Distributions}
In this section we formally define the transcript and derive the transcript distribution, which will be used in the proof of subsequent results. 

Recall that a (full) transcript is the realization of a sequence of actions and observations up to stopping time, concatenated with the final decision $\delta$: $\calT=(A_1,O_1,A_2,O_2,\ldots,A_{\ltime-1},O_{\ltime-1},A_{\ltime},\delta)$. Also let $\tau=(a_1,o_1,a_2,o_2,\ldots,a_{\stime-1},o_{\stime-1},a_{\stime}, \tilde\delta)$ be the realization of $\calT$, with 
$\stime$, $\tilde\delta$, $a_t$, $o_t$ being the realizations of $\ltime$, $\delta$, $A_t$ and $O_t$ for each $t$, respectively.  
We also define the partial transcript \(\calT_t\) as the realization of a sequence of actions and observations at some $t < \ltime$: 
\begin{equation}\label{eq:partial-transcript-form}
    \calT_t=(A_1,O_1,A_2,O_2,\ldots,A_t,O_t).
\end{equation}
Its realization is denoted as $\tau_t=(a_1,o_1,a_2,o_2,\ldots,a_t,o_t)$. We let $\tau_0 = \emptyset$. 
In other words, $\tau_t$ is the realization of history $\mathcal{H}_t$. 

Let $\mathfrak T_N$ be the set of valid transcripts induced by an admissible policy $\pi$. Specifically, $\mathfrak T_N:= \{\tau: a_t\in\calA_t(\tau_{t-1})\text{ for all }t\in[\stime], o_t\in\calR \text{ if }a_t = \AI(i) \text{ for some }i\in[N^\pi], o_t\in\{0,1\} \text{ if }a_t = \Hum(i) \text{ for some }i\in[N^\pi], \text{ and }a_t= \text{STOP} \text{ if and only if }t = \stime\}$.  We note that $\mathfrak T_N$ has finite cardinality, as each $a_t$, $o_t$ and $\tilde\delta$ takes only finite values, and $\stime \leq 2N^\pi + 1$. We let $P^\pi_h$ be the distribution of transcript under $H_h$, and $P^\pi_h(\tau) = \Pp^\pi_h(\calT = \tau)$ is the transcript PMF (Probability Mass Function) at $\tau$. We have that for a feasible $\pi\in \calF(\alpha, \beta)$, the support of $P^\pi_h$ has: $\text{supp}(P^\pi_h)\subseteq \mathfrak T_N$ for $h\in\{0,1\}$. 

We let $\varphi_h(o_t|a_t, \tau_{t-1}) = \Pp^\pi_h(O_t = o_t|A_t = a_t,\calT_{t-1} = \tau_{t-1})$ be the conditional probability mass of observing $o_t$ at epoch $t$, given the partial transcript $\tau_{t-1}$ and the action $a_t$. Then we have that
\begin{equation}\label{appendix-equ:transcript-mass}
\begin{split}
    &\,\varphi_h(o_t|a_t, \tau_{t-1}) \\
    =& 
    \begin{cases}
        g_h(o_t), &o_t\in\calR, a_t = \AI(i), i\text{ has not been  human-queried before  }t,\\
        f_x(o_t), &o_t\in\calR, a_t = \AI(i), i\text{ has been human-queried before }t \text{ with observation }x,\\
        p^{o_t}_h(1-p_h)^{1-o_t}, &o_t\in\{0,1\}, a_t = \Hum(i), i\text{ has not been  AI-queried before  }t,\\
        \rho_h(o_t|r), &o_t\in\{0,1\}, a_t = \Hum(i), i\text{ has been AI-queried before  }t\text{ with observation }r,\\
        0, &\text{otherwise}.
    \end{cases}
\end{split}
\end{equation}

We let $\chi^\pi(\tau)$ represent the common policy-kernel factor of transcript $\tau$. Specifically, let
$\chi^{\pi}(\tau) = \prod^{\stime}_{t=1}\Pp^\pi_h(A_t = a_t|\calT_{t-1} = \tau_{t-1})\Pp^\pi_h(\delta^\pi = \tilde\delta|\calT_{\stime-1} = \tau_{\stime-1}, A_{\stime} = \text{STOP})$. Because the policy and its randomization law are hypothesis-independent, the conditional distributions of the action and of the terminal decision, given the same realized history, are identical under \(H_0\) and \(H_1\). That is, $\Pp^\pi_0(A_t = a_t|\calT_{t-1} = \tau_{t-1}) = \Pp^\pi_1(A_t = a_t|\calT_{t-1} = \tau_{t-1})$ for all $t$, and that $\chi^{\pi}(\tau)$ is the same across $H_0$ and $H_1$. We are now ready to characterize the transcript distribution.

\begin{lemma}[Transcript probability mass]\label{lem:transcript-pmf}
For every admissible policy $\pi$, every
$h\in\{0,1\}$, and every $\tau\in\mathfrak T_N$,
\begin{equation}\label{eq:transcript-pmf}
    P_h^\pi(\tau)
    =\chi^\pi(\tau)
      \prod_{t=1}^{\stime-1}
      \varphi_h(o_t\mid a_t,\tau_{t-1}).
\end{equation}
Additionally, $\text{supp}(P^\pi_1) = \text{supp}(P^\pi_0)$, and that for $\tau\in\text{supp}(P^\pi_h)$,
\begin{equation}\label{eq:transcript-llr}
    \log\frac{P_1^\pi(\tau)}{P_0^\pi(\tau)}
    =\sum_{t=1}^{\stime-1}
      \log\frac{\varphi_1(o_t\mid a_t,\tau_{t-1})}
               {\varphi_0(o_t\mid a_t,\tau_{t-1})}.
\end{equation}
\end{lemma}

\begin{proof}{Proof of Lemma~\ref{lem:transcript-pmf}}
\eqref{eq:transcript-pmf} follows from the chain rule for joint probability mass functions:
\begin{align*}
    P_h^\pi(\tau) &= 
\left(\prod^{\stime-1}_{t=1}\Pp^\pi_h(O_t = o_t|A_t = a_t, \calT_{t-1} = \tau_{t-1})\right)
\left(\prod^{\stime}_{t=1}\Pp^\pi_h(A_t = a_t|\calT_{t-1} = \tau_{t-1})\right)\\
&\qquad\qquad\qquad\qquad\qquad\cdot\Pp^\pi_h(\delta^\pi = \tilde\delta|\calT_{\stime - 1} = \tau_{\stime - 1}, A_{\stime} = \text{STOP})\\
&= \chi^\pi(\tau)\prod^{\stime-1}_{t=1}\varphi_h(o_t|a_t, \tau_{t-1}), \forall \tau\in \mathfrak T_N,
\end{align*}
and $P_h^\pi(\tau)= 0$ if $\tau\notin \mathfrak T_N$.

Now we verify that for any $\tau$, $P^\pi_0(\tau) >0$ if and only if $P^\pi_1(\tau) >0$. In particular, $\chi^\pi(\tau)$ is the same across $H_0$ and $H_1$, and by definition of $\varphi_h(o_t|a_t, \tau_{t-1})$, $\varphi_0(o_t|a_t, \tau_{t-1}) > 0$ if and only if $\varphi_1(o_t|a_t, \tau_{t-1}) > 0$. Then \eqref{eq:transcript-llr} is well defined and follows straightforwardly from \eqref{eq:transcript-pmf}. 
\end{proof}

\subsection{Proof of Lemma~\ref{lem:error-to-kl}}
To prove Lemma~\ref{lem:error-to-kl}, we need auxiliary Lemmas~\ref{lem:kl-chain-rule}-\ref{lem:binary-monotonicity}.

\begin{lemma}[Theorem 2.5.3 of 
\citet{cover2006elements}]\label{lem:kl-chain-rule}
Let $(Y,Z)$ take values in a finite or countably infinite set. Under probability laws $P$ and $Q$, let $P_{Y,Z}$ and $Q_{Y,Z}$ denote the joint laws of $(Y,Z)$, let $P_Y$ and $Q_Y$ denote the marginal laws of $Y$, and let $P_{Z\mid Y=y}$ and $Q_{Z\mid Y=y}$ denote the conditional laws of $Z$ given $Y=y$, respectively. Then
\begin{equation}\label{eq:kl-chain-rule}
    D(P_{Y,Z}\Vert Q_{Y,Z})
    =D(P_Y\Vert Q_Y)
      +\sum_{y\in\mathcal Y}P_Y(y)D(P_{Z\mid Y=y}\Vert Q_{Z\mid Y=y}),
\end{equation}
with the usual extended-value convention. In particular, forgetting $Z$ cannot increase KL divergence:
\begin{equation}\label{eq:forgetting-kl}
    D(P_Y\Vert Q_Y)\le D(P_{Y,Z}\Vert Q_{Y,Z}).
\end{equation}
\end{lemma}

\begin{lemma}[Binary coarsening]\label{lem:binary-coarsening}
Let $P$ and $Q$ be probability mass functions on the same finite or countably infinite set $S$, and let $E\subseteq S$. Then
\begin{equation}\label{eq:binary-coarsening}
    D(P\Vert Q)\ge \klbin(P(E)\Vert Q(E)).
\end{equation}
\end{lemma}

\begin{proof}{Proof of Lemma~\ref{lem:binary-coarsening}}
Let \(Z\) denote the original \(S\)-valued random outcome and define \(Y:=\ind\{Z\in E\}\). Since $Y$ is a deterministic function of $Z$, the joint variable $(Y,Z)$ contains exactly the same information as $Z$, and therefore $D(P_{Y,Z}\Vert Q_{Y,Z})=D(P\Vert Q)$. The marginal law of $Y$ under $P$ is Bernoulli with success probability $P(E)$, and the marginal law of $Y$ under $Q$ is Bernoulli with success probability $Q(E)$. Applying Lemma~\ref{lem:kl-chain-rule} and then forgetting $Z$ gives
\[
    D(P\Vert Q)
    =D(P_{Y,Z}\Vert Q_{Y,Z})
    \ge D(P_Y\Vert Q_Y)
    =\klbin(P(E)\Vert Q(E)),
\]
which proves \eqref{eq:binary-coarsening}.
\end{proof}

\begin{lemma}[Monotonicity of Bernoulli KL in the testing region]\label{lem:binary-monotonicity}
If $0<b<a<1$, then $\klbin(a\Vert b)$ is increasing in $a$ for fixed $b$ and decreasing in $b$ for fixed $a$.
\end{lemma}

\begin{proof}{Proof of Lemma~\ref{lem:binary-monotonicity}}
Expanding the logarithms in
the definition of $\klbin(a\Vert b)$,
\[
    \klbin(a\Vert b)
    =a\log a-a\log b
    +(1-a)\log(1-a)-(1-a)\log(1-b).
\]

\emph{Derivative in $a$.} Differentiating term by
term with respect to $a$, with $b$ held fixed:
\begin{align*}
    \tfrac{\partial}{\partial a}[a\log a]
    &=\log a+1
    &&\text{(product rule)},\\
    \tfrac{\partial}{\partial a}[-a\log b]
    &=-\log b,\\
    \tfrac{\partial}{\partial a}[(1-a)\log(1-a)]
    &=-\log(1-a)-1
    &&\text{(product and chain rules)},\\
    \tfrac{\partial}{\partial a}[-(1-a)\log(1-b)]
    &=\log(1-b).
\end{align*}
Summing the four lines, the constants $+1$ and
$-1$ cancel, leaving
\[
    \tfrac{\partial}{\partial a}\klbin(a\Vert b)
    =\log a-\log b-\log(1-a)+\log(1-b)
    =\log\frac{a(1-b)}{b(1-a)}.
\]
Since $a>b$, we have $a(1-b)>b(1-a)$ (equivalent to
$a>b$ after expanding), so the argument of the
logarithm exceeds one and the derivative is strictly
positive. Hence $\klbin(a\Vert b)$ is strictly
increasing in $a$ on the region $a>b$.

\emph{Derivative in $b$.} Differentiating term by
term with respect to $b$, with $a$ held fixed:
\begin{align*}
    \tfrac{\partial}{\partial b}[a\log a]
    &=0,\\
    \tfrac{\partial}{\partial b}[-a\log b]
    &=-\frac{a}{b},\\
    \tfrac{\partial}{\partial b}[(1-a)\log(1-a)]
    &=0,\\
    \tfrac{\partial}{\partial b}[-(1-a)\log(1-b)]
    &=\frac{1-a}{1-b}
    &&\text{(chain rule)}.
\end{align*}
Summing,
\[
    \tfrac{\partial}{\partial b}\klbin(a\Vert b)
    =-\frac{a}{b}+\frac{1-a}{1-b}
    =\frac{-a(1-b)+b(1-a)}{b(1-b)}
    =\frac{b-a}{b(1-b)}.
\]
Since $a>b$ and $b(1-b)>0$, the numerator is negative
and the denominator is positive, so the derivative is
strictly negative. Hence $\klbin(a\Vert b)$ is strictly
decreasing in $b$ on the region $a>b$.
\end{proof}

\begin{proof}{Proof of Lemma~\ref{lem:error-to-kl}}
Let $E=\{\tau: \tilde\delta=1\}$ be the set of rejection transcripts. By Lemma~\ref{lem:binary-coarsening},
\[
    D(P_1^\pi\Vert P_0^\pi)
    \ge \klbin(P_1^\pi(E)\Vert P_0^\pi(E)).
\]
The type-II constraint gives $P_1^\pi(E)=\Pp^{\pi}_1(\delta^\pi=1)\ge 1-\beta$, and the type-I constraint gives $P_0^\pi(E)=\Pp^{\pi}_0(\delta^\pi=1)\le\alpha$. Since $\alpha+\beta<1$, we have $1-\beta>\alpha$. By Lemma~\ref{lem:transcript-pmf}, \(P_0^\pi\) and \(P_1^\pi\) have common support; together with the error constraints, this implies \(0<P_0^\pi(E)<P_1^\pi(E)<1\). Lemma~\ref{lem:binary-monotonicity} gives
\[
    \klbin(P_1^\pi(E)\Vert P_0^\pi(E))
    \ge \klbin(1-\beta\Vert\alpha)=A.
\]
This proves the forward KL requirement.

For the reverse direction, let $F=\{\tau: \tilde\delta=0\}$ be the set of acceptance transcripts. Lemma~\ref{lem:binary-coarsening} gives
\[
    D(P_0^\pi\Vert P_1^\pi)
    \ge \klbin(P_0^\pi(F)\Vert P_1^\pi(F)).
\]
The type-I constraint gives $P_0^\pi(F)\ge 1-\alpha$, and the type-II constraint gives $P_1^\pi(F)\le\beta$. Since $1-\alpha>\beta$, Lemma~\ref{lem:binary-monotonicity} gives
\[
    \klbin(P_0^\pi(F)\Vert P_1^\pi(F))
    \ge \klbin(1-\alpha\Vert\beta)=B.
\]
\end{proof}

\subsection{Proof of Theorem~\ref{thm:kl-decomposition}}
\begin{proof}{Proof of Theorem~\ref{thm:kl-decomposition}}

We prove \eqref{eq:forward-kl-decomp}. The proof of
\eqref{eq:reverse-kl-decomp} is obtained by
interchanging the roles of $H_0$ and $H_1$.

% The proof computes the conditional expected
% log-likelihood ratio in each of the four cases of the
% one-step observation rule and then sums the
% contributions. The four cases are: an AI query not
% preceded by a human query on that item, an AI query
% after that item's human label has been revealed, a
% human query not preceded by an AI query on that item,
% and a human query after that item's AI report has been
% observed. These cases produce, respectively,
% $I_R^{(1)}$, zero, $J_X^{(1)}$, and $d^{(1)}(R_i)$.

According to Lemma~\ref{lem:transcript-pmf}, $P^\pi_0$ and $P^\pi_1$ have common support, and
\[
    \log\frac{P_1^\pi(\tau)}{P_0^\pi(\tau)}
    =\sum_{t=1}^{\stime-1}
      \log\frac{\varphi_1(o_t\mid a_t,\tau_{t-1})}
               {\varphi_0(o_t\mid a_t,\tau_{t-1})}.
\]

Define the one-step log-likelihood-ratio increment
\[
    L_t:=
    \begin{cases}
        \log\dfrac{\varphi_1(O_t\mid A_t,\calT_{t-1})}
                  {\varphi_0(O_t\mid A_t,\calT_{t-1})},
        & t< \ltime,\\[2ex]
        0, & t\geq \ltime,
    \end{cases}
\]
Then, since $\ltime\leq 2N^\pi + 1$ almost surely, we have that
\begin{equation}\label{eq:kl-as-sum-of-Lt}
    D(P_1^\pi\Vert P_0^\pi)
    =\E^\pi_1\!\left[\log\frac{P_1^\pi(\calT)}{P_0^\pi(\calT)}\right]
    = \E^\pi_1\left[\sum_{t=1}^{\ltime-1}
      \log\frac{\varphi_1(O_t\mid A_t,\calT_{t-1})}
               {\varphi_0(O_t\mid A_t,\calT_{t-1})}
               \right]
    = \sum_{t=1}^{2N^\pi}\E^\pi_1[L_t].
\end{equation}

It remains to evaluate $\E^\pi_1[L_t]$. Let 
\[
    \mathcal{F}_t = \sigma(A_1, O_1,\cdots, A_{t-1}, O_{t-1}, A_t)
\]
be the filtration generated by the history $\calH_{t-1} = \sigma(A_1, O_1,\cdots, A_{t-1}, O_{t-1})$ and action $A_t$. Then
$\{t< \ltime\}\in\mathcal F_{t}$. By law of iterated expectation,
$$
D(P_1^\pi\Vert P_0^\pi)
    =\sum_{t=1}^{2N^\pi}\E^\pi_1[L_t] = \sum_{t=1}^{2N^\pi}\E^\pi_1[\E^\pi_1[L_t|\calF_t]] = \sum_{t=1}^{2N^\pi}\E^\pi_1[\ind\{t< \ltime\}\E^\pi_1[L_t|\calF_t]],
$$
where the third equality follows from definition of $L_t$ and that $\{t< \ltime\}\in\mathcal F_{t}$. 

% Define two
% availability indicators for item $i$:
% \[
%     A_{i,t}^{\AI}=1
%     \quad\text{if item }i\text{ has already been
%     AI-queried before }t,
% \]
% \[
%     A_{i,t}^{\Hum}=1
%     \quad\text{if item }i\text{ has already been
%     human-queried before }t.
% \]
% If $A_{i,t}^{\AI}=1$, let $R_{i,t}$ denote the observed
% AI report. If $A_{i,t}^{\Hum}=1$, let $X_{i,t}$ denote
% the observed human label. 
We now compute $\E^\pi_1[L_t\mid\calF_t]$ for $t< \ltime$. 
By definition of an admissible policy, for all $t< \ltime$ we have that $\sum^t_{s=1}\ind\{A_s = \AI(i)\}\leq 1$ and $\sum^t_{s=1}\ind\{A_s = \Hum(i)\}\leq 1$ for all $i$. Thus we have four cases of the filtration $\calF_t$: 

\emph{Case 1: AI query not preceded by a human query.}
On the event
$\{t< \ltime\}\cap\{A_t=\AI(i),\,\sum^{t-1}_{s=1}\ind\{A_s = \Hum(i)\}=0\}$,
\[
    \E^\pi_1[L_t\mid\mathcal F_t]
    =\sum_{r\in\calR}g_1(r)\log\frac{g_1(r)}{g_0(r)}
    =D(G_1\Vert G_0)
    =I_R^{(1)},
\]
by the definition of KL
divergence and the definition \eqref{eq:IR1} of
$I_R^{(1)}$.

\emph{Case 2: direct human query.} On the event
$\{t< \ltime\}\cap\{A_t=\Hum(i),\,\sum^{t-1}_{s=1}\ind\{A_s = \AI(i)\}=0\}$, 
\[
    \E^\pi_1[L_t\mid\mathcal F_t]
    =p_1\log\frac{p_1}{p_0}
      +(1-p_1)\log\frac{1-p_1}{1-p_0}
    =\klbin(p_1\Vert p_0)
    =J_X^{(1)}.
\]

\emph{Case 3: human query after an AI report.} On the
event $\{t< \ltime\}\cap\{A_t=\Hum(i),\,
\sum^{t-1}_{s=1}\ind\{A_s = \AI(i)\}=1\}$. We have for each $o_t\in\{0,1\}$, 
\[
    \log\frac{\rho_1(o_t\mid r)}{\rho_0(o_t\mid r)}
    =o_t\log\frac{q_1(r)}{q_0(r)}
     +(1-o_t)\log\frac{1-q_1(r)}{1-q_0(r)},
\]
and $\E^\pi_1[O_t\mid\mathcal F_t]=q_1(R_i)$ on this event, so
\[
    \E^\pi_1[L_t\mid\mathcal F_t]
    =q_1(R_i)\log\frac{q_1(R_i)}{q_0(R_i)}
      +(1-q_1(R_i))\log\frac{1-q_1(R_i)}{1-q_0(R_i)}
    =\klbin(q_1(R_i)\Vert q_0(R_i))
    =d^{(1)}(R_i).
\]

\emph{Case 4: AI query after the item's human label has
been revealed.} On the event
$\{t< \ltime\}\cap\{A_t=\AI(i),\,\sum^{t-1}_{s=1}\ind\{A_s = \Hum(i)\}=1\}$, for $O_t= r\in\calR$ and $X_i = x$, 
$\varphi_1(O_t|A_t, \calT_{t-1}) = \varphi_0(O_t|A_t, \calT_{t-1}) = f_x(r)$ almost surely, and thus
$\E^\pi_1[L_t\mid\mathcal F_t]=0$.

Combining the above four cases, we have that 
\begin{align*}
D(P_1^\pi\Vert P_0^\pi)
   &= \sum_{t=1}^{2N^\pi}\E^\pi_1[\ind\{t< \ltime\}\E^\pi_1[L_t|\calF_t]]\\
   &= \sum_{t=1}^{2N^\pi}\E^\pi_1\Big[\sum^{N^\pi}_{i=1}\ind\{t< \ltime, A_t = \AI(i), \sum^{t-1}_{s=1}\ind\{A_s = \Hum(i)\} = 0\}
   \E^\pi_1[L_t|\calF_t]\Big] \\
   &\qquad\qquad+ \sum_{t=1}^{2N^\pi}\E^\pi_1\Big[\sum^{N^\pi}_{i=1}\ind\{t< \ltime, A_t = \Hum(i), \sum^{t-1}_{s=1}\ind\{A_s = \AI(i)\} = 0\}\E^\pi_1[L_t|\calF_t]\Big]\\
   &\qquad\qquad
   +\sum_{t=1}^{2N^\pi}\E^\pi_1\Big[\sum^{N^\pi}_{i=1}\ind\{t< \ltime, A_t = \Hum(i), \sum^{t-1}_{s=1}\ind\{A_s = \AI(i)\} = 1\}\E^\pi_1[L_t|\calF_t]\Big]\\
   &\qquad\qquad+ \sum_{t=1}^{2N^\pi}\E^\pi_1\Big[\sum^{N^\pi}_{i=1}\ind\{t< \ltime, A_t = \AI(i), \sum^{t-1}_{s=1}\ind\{A_s = \Hum(i)\} = 1\}\E^\pi_1[L_t|\calF_t]\Big]\\
   & = \sum_{t=1}^{2N^\pi}\E^\pi_1\Big[\sum^{N^\pi}_{i=1}\ind\{t< \ltime, A_t = \AI(i), \sum^{t-1}_{s=1}\ind\{A_s = \Hum(i)\} = 0\}
   I^{(1)}_R\Big] \\
   &\qquad\qquad+ \sum_{t=1}^{2N^\pi}\E^\pi_1\Big[\sum^{N^\pi}_{i=1}\ind\{t< \ltime, A_t = \Hum(i), \sum^{t-1}_{s=1}\ind\{A_s = \AI(i)\} = 0\}J^{(1)}_X\Big]\\
   &\qquad\qquad+
   \sum_{t=1}^{2N^\pi}\E^\pi_1\Big[\sum^{N^\pi}_{i=1}\ind\{t< \ltime, A_t = \Hum(i), \sum^{t-1}_{s=1}\ind\{A_s = \AI(i)\} = 1\}d^{(1)}(R_i)\Big]\\
   & = \E^\pi_1\left[
        N^{\dir, \pi}_{\AI} I_R^{(1)}
        +N^{\dir,\pi}_{\Hum} J_X^{(1)}
        +\sum_{i\in\calI_{\AI\Hum}^\pi}d^{(1)}(R_i)
      \right],
\end{align*}
where the last equality follows by the definition of $N^{\dir, \pi}_{\AI}, N^{\dir,\pi}_{\Hum}$ and $\calI^\pi_{\AI\Hum}$.
Then \eqref{eq:reverse-kl-decomp} follows by
interchanging the roles of $H_0$ and $H_1$.
\end{proof}

\subsection{Proof of Lemma~\ref{lem:capacity-frontier-properties}}
\begin{proof}{Proof of Lemma~\ref{lem:capacity-frontier-properties}}
We first prove the dual representation \eqref{eq:Psi-dual}. We start by deriving the dual program of \eqref{eq:Psi-dual}, and show that the dual of \eqref{eq:Psi-dual} is \eqref{eq:Psi-frontier-def}. Then, according to Theorem 4.1 of \citet{bertsimas1997introduction}, which states that for linear programming problem, the dual of dual is the primal, we conclude that \eqref{eq:Psi-dual} is the dual of \eqref{eq:Psi-frontier-def}.

Let
\(\lambda\ge0\) be the dual variable for the constraint
\(\sum_r g_h(r)\eta(r)\le s\), and for each $r\in\calR$, let \(\mu(r)\ge0\) be the dual
variable for the constraint \(\eta(r)\le1\). Following the recipe in Section~4.2 of \citet{bertsimas1997introduction}, the linear programming (LP) dual problem for \eqref{eq:Psi-dual} is:
\begin{equation}\label{appendix-equ:dual}
\begin{split}
    \min_{\lambda\geq 0,\mu}\quad
        &\lambda s+\sum_{r\in\calR}\mu(r)\\
    \text{s.t.}\quad
        &\lambda g_h(r)+\mu(r)\ge g_h(r)d^{(h)}(r),\qquad \forall r\in\calR,\\
        &\mu(r)\ge0,\qquad \forall r\in\calR. 
\end{split}
\end{equation}
For each fixed \(\lambda\ge0\), the constraint
\(\mu(r)\ge g_h(r)(d^{(h)}(r)-\lambda)\) together with \(\mu(r)\ge0\)
implies that 
\[
    \mu(r)=g_h(r)\pos{d^{(h)}(r)-\lambda}
\]
minimizes the objective in \eqref{appendix-equ:dual}. 
Substituting $\mu(r)=g_h(r)\pos{d^{(h)}(r)-\lambda}$ into the dual objective gives \(\lambda s+\sum_r
g_h(r)\pos{d^{(h)}(r)-\lambda}\). Additionally, since the LP problem \eqref{eq:Psi-dual} is feasible and
bounded above, strong LP duality
(Theorem~4.4 of \citet{bertsimas1997introduction}) gives
\[
    \Psi_h(s)
    =\inf_{\lambda\ge0}
        \left\{\lambda s+
        \sum_{r\in\calR}g_h(r)\pos{d^{(h)}(r)-\lambda}\right\},
\]
which is \eqref{eq:Psi-frontier-def}.

Fix \(h\in\{0,1\}\). The maximization in \eqref{eq:Psi-dual} is
a finite-dimensional linear program in the variables
\(\{\eta(r)\}_{r\in\calR}\), parameterized by \(s\).
Its feasible region is nonempty (\(\eta\equiv0\) is feasible) and
bounded (\(0\le\eta(r)\le1\)), so the maximum is attained for every
\(s\in[0,1]\). Monotonicity of $\Psi_h(s)$ in $s$ follows because increasing \(s\) enlarges
the feasible set. According to Section~5.2 of \citet{bertsimas1997introduction}, $\Psi_h(s)$ is concave and piecewise linear in \(s\), hence
continuous on \([0,1]\). 

It remains to prove \eqref{eq:Psi-one-chain-rule}. At \(s=1\), \(\eta(r)=1\) for every \(r\in\calR\) is feasible (because
\(\sum_{r\in\calR}g_h(r)=1\)) and optimal (because \(d^{(h)}(r)\geq 0\) for all $r\in\calR$) for \eqref{eq:Psi-dual}. Thus
\[
    \Psi_h(1)=\sum_{r\in\calR}g_h(r)d^{(h)}(r).
\]

For $h \in \{0,1\}$, let $P_h^{X,R}$ denote the joint pmf of
$(X_i,R_i)$ under $H_h$, with
$P_h^{X,R}(x,r)= p_h^x(1-p_h)^{1-x}f_x(r)$. For \(h=1\), we compute \(D(P_1^{X,R}\Vert P_0^{X,R})\) in two different
ways and equate the two expressions.

On one hand, we have that 
\begin{align*}
    D(P_1^{X,R}\Vert P_0^{X,R}) &= \sum_{x\in\{0,1\}}\sum_{r\in\calR}P_1^{X,R}(x,r)\log\frac{P_1^{X,R}(x,r)}{P_0^{X,R}(x,r)}\\
    &=\sum_{x\in\{0,1\}}\sum_{r\in\calR}p_1^x(1-p_1)^{1-x}f_x(r)\log\frac{p_1^x(1-p_1)^{1-x}}{p_0^x(1-p_0)^{1-x}}\\
    &=\sum_{x\in\{0,1\}}p_1^x(1-p_1)^{1-x}\log\frac{p_1^x(1-p_1)^{1-x}}{p_0^x(1-p_0)^{1-x}}\\
    &= \klbin(p_1||p_0)\\
    &= J^{(1)}_X. 
\end{align*}
In the above equations, the
second equality follows from the definition of $P_h^{X,R}(x,r)$, and the third follows since \(\sum_r P_1^{X,R}(x,r)=p_1^x(1-p_1)^{1-x}\). Moreover, Lemma \ref{lem:kl-chain-rule} gives
\[
    D(P_1^{X,R}\Vert P_0^{X,R})
    =D(G_1\Vert G_0)
    +\sum_{r\in\calR}g_1(r)D(\rho_1(\cdot\mid r)\Vert\rho_0(\cdot\mid r))
    =I_R^{(1)}+\sum_{r\in\calR}g_1(r)d^{(1)}(r),
\]
where the second equality uses \eqref{eq:IR1} and
\eqref{eq:d1-def}.

Equating the two expressions for
\(D(P_1^{X,R}\Vert P_0^{X,R})\) gives
\[
    J_X^{(1)}=I_R^{(1)}+\sum_{r\in\calR}g_1(r)d^{(1)}(r) = I_R^{(1)} + \Psi_1(1).
\]
The proof for
\(h=0\) is identical.
\end{proof}

\subsection{Proof of Proposition~\ref{prop:escalation-solution}}

\begin{proof}{Proof of Proposition~\ref{prop:escalation-solution}}
Fix \(h\in\{0,1\}\) and $s\in[0,1]$, and abbreviate
\[
    g_j:=g_h(r^{(h)}_j),\qquad
    d_j:=d^{(h)}(r^{(h)}_j),\qquad
    W_j:=W^{(h)}_j .
\]
By Assumption~\ref{ass:score-model}, \(g_j>0\) for every \(j\), and by
construction,
\[
    d_1\ge d_2\ge\cdots\ge d_{|\calR|}\ge0,
    \qquad
    W_j=\sum_{\ell=1}^j g_\ell,
    \qquad
    W_{|\calR|}=1.
\]

We first verify feasibility. If \(s=0\), the rule in
\eqref{eq:eta-frontier-closed-form} sets every component equal to zero.
If \(s\in(0,1]\), let $q:=\min\{j\in[|\calR|]:W_j\ge s\}$. 
Then the same rule sets
\[
    \eta^*_h(r^{(h)}_j;s)
    =
    \begin{cases}
        1, & j<q,\\[1ex]
        (s-W_{q-1})/g_q, & j=q,\\[1ex]
        0, & j>q.
    \end{cases}
\]
This representation also covers \(s=W_q\), in which case the middle
component equals one. Since
\(W_{q-1}<s\le W_q=W_{q-1}+g_q\), every component belongs to
\([0,1]\), and
\[
    \sum_{j=1}^{|\calR|} g_j\eta^*_h(r^{(h)}_j;s)
    =W_{q-1}+g_q\frac{s-W_{q-1}}{g_q}
    =s.
\]
Thus \(\eta^*_h:= \{\eta^*_h(r^{(h)}_j; s)\}_{j\in[|\calR|]}\) is feasible for \eqref{eq:Psi-dual}.

It remains to prove optimality. The case \(s=0\) is immediate because
strict positivity of the \(g_j\)'s forces every feasible component to be
zero. Suppose \(s\in(0,1]\), let \(q\) be as above, and consider any
feasible rule \(\eta(\cdot; s):\calR\to[0,1]\). Using
\(\sum_jg_j\eta(r^{(h)}_j; s)\le s
=\sum_jg_j\eta^*_h(r^{(h)}_j; s)\), we obtain
\begin{align*}
&\sum_{j=1}^{|\calR|} g_jd_j
   \bigl(\eta(r^{(h)}_j; s)-\eta^*_h(r^{(h)}_j;s)\bigr)\\
&\quad=
  \sum_{j=1}^{|\calR|} g_j(d_j-d_q)
   \bigl(\eta(r^{(h)}_j;s)-\eta^*_h(r^{(h)}_j;s)\bigr)
  +d_q\sum_{j=1}^{|\calR|} g_j
   \bigl(\eta(r^{(h)}_j; s)-\eta^*_h(r^{(h)}_j; s)\bigr)
  \le0.
\end{align*}
Indeed, the last sum is nonpositive and \(d_q\ge0\). For \(j<q\),
\(d_j-d_q\ge0\) while
\(\eta(r^{(h)}_j;s)-\eta^*_h(r^{(h)}_j; s)
=\eta(r^{(h)}_j;s)-1\le0\). For \(j>q\),
\(d_j-d_q\le0\) while
\(\eta(r^{(h)}_j;s)-\eta^*_h(r^{(h)}_j;s)
=\eta(r^{(h)}_j;s)\ge0\). The \(j=q\) term vanishes because
\(d_j-d_q=0\). Hence every feasible \(\eta\) has objective value no
larger than that of \(\eta^*_h\), proving that
\(\eta^*_h\) is optimal for \eqref{eq:Psi-dual}.
\end{proof}

\subsection{Proof of Lemma~\ref{lem:selected-info-bound}}

\begin{proof}{Proof of Lemma~\ref{lem:selected-info-bound}}
By definition, $\calI_{\AI\Hum}^\pi\subseteq
\calI_{\AI}^\pi$, and in particular $N^{\pi}_{\AI\Hum}\le N^{\dir, \pi}_{\AI}$ on every
sample path. If $\E^\pi_h[N^{\dir, \pi}_{\AI}]=0$, then $N^{\dir, \pi}_{\AI}=0$ almost
surely, since $N^{\dir, \pi}_{\AI}$ is a nonnegative integer with zero mean.
By $N^{\pi}_{\AI\Hum}\le N^{\dir, \pi}_{\AI}$, also $N^{\pi}_{\AI\Hum}=0$ almost surely, so both sides
of \eqref{eq:selected-info-bound} are zero. Assume $\E^\pi_h[N^{\dir, \pi}_{\AI}]>0$
for the remainder.

For each item $i\in[N^\pi]$, let
\[
    B_i:=\ind\{\text{item }i\text{ is AI-queried before any human
    query on it}\},
\]
equivalently, that item $i$'s first paid action is an AI query,
and let
\[
    E_i:=\ind\{\text{item }i\text{ is human-queried after its AI
    report is observed}\}.
\]
Then $N^{\dir, \pi}_{\AI}=\sum_{i=1}^{N^\pi}B_i$ and $N^{\pi}_{\AI\Hum}=\sum_{i=1}^{N^\pi}E_i$, so
$\E^\pi_h[N^{\dir, \pi}_{\AI}]=\sum_i\E^\pi_h[B_i]$ and $\E^\pi_h[N^{\pi}_{\AI\Hum}]=\sum_i\E^\pi_h[E_i]$. Since $\calI_{\AI\Hum}^\pi\subseteq\calI_{\AI}^\pi$, $E_i=1$ implies $B_i=1$, so
$0\le E_i\le B_i\le1$. Since
$\calI_{\AI\Hum}^\pi=\{i:E_i=1\}$,
\[
    \sum_{i\in\calI_{\AI\Hum}^\pi}d^{(h)}(R_i)
    =\sum_{i=1}^{N^\pi}E_i\,d^{(h)}(R_i).
\]

Fix $\lambda\ge0$. For
every item $i$,
\begin{equation}\label{eq:selected-pathwise}
    E_i\,d^{(h)}(R_i) \le \lambda E_i+E_i\pos{d^{(h)}(R_i)-\lambda}
    \le \lambda E_i+B_i\pos{d^{(h)}(R_i)-\lambda}.
\end{equation}
Summing
\eqref{eq:selected-pathwise} over $i$, taking expectation under
$H_h$, and using $\sum_i\E^\pi_h[E_i]=\E^\pi_h[N^{\pi}_{\AI\Hum}]$,
\begin{equation}\label{eq:selected-dual-step1}
    \E^\pi_h\left[\sum_{i=1}^{N^\pi}E_i\,d^{(h)}(R_i)\right]
    \le \lambda\E^\pi_h[N^{\pi}_{\AI\Hum}]+
       \sum_{i=1}^{N^\pi}\E^\pi_h\left[B_i\pos{d^{(h)}(R_i)-\lambda}\right].
\end{equation}

We claim $B_i$ is independent of
$R_i$ under $H_h$. The indicator $B_i$ records whether item
$i$'s first paid action is an AI query. That decision is made
from the history available just before item $i$ is first
queried, and at that moment neither $R_i$ (revealed only by an
AI query on $i$) nor $X_i$ (revealed only by a human query on
$i$) has been observed. Hence $B_i$ is a function of the
randomization seed and the other items' variables
$\{(X_j,R_j):j\ne i\}$ alone. By Assumption~\ref{ass:iid} the
family $\{(X_j,R_j):j\ne i\}$ is independent of $R_i$, and the
seed is independent of all data; therefore $B_i\perp R_i$ under
$H_h$. The product rule for independent random variables gives
the first equality below, and the marginal law $R_i\sim g_h$
under $H_h$ (by \eqref{eq:gh-def}) gives the second:
\[
    \E^\pi_h\left[B_i\pos{d^{(h)}(R_i)-\lambda}\right]
    =\E^\pi_h[B_i]\,\E^\pi_h\left[\pos{d^{(h)}(R_i)-\lambda}\right]
    =\E^\pi_h[B_i]\sum_{r\in\calR}g_h(r)\pos{d^{(h)}(r)-\lambda}.
\]
Summing over $i$ and using $\sum_i\E^\pi_h[B_i]=\E^\pi_h[N^{\dir, \pi}_{\AI}]$,
\[
    \sum_{i=1}^{N^\pi}\E^\pi_h\left[B_i\pos{d^{(h)}(R_i)-\lambda}\right]
    =\E^\pi_h[N^{\dir, \pi}_{\AI}]\sum_{r\in\calR}g_h(r)\pos{d^{(h)}(r)-\lambda}.
\]

Substituting into
\eqref{eq:selected-dual-step1} and factoring out $\E^\pi_h[N^{\dir, \pi}_{\AI}]>0$,
\[
    \E^\pi_h\left[\sum_{i=1}^{N^\pi}E_i\,d^{(h)}(R_i)\right]
    \le
    \E^\pi_h[N^{\dir, \pi}_{\AI}]\left[
    \lambda\,\frac{\E^\pi_h[N^{\pi}_{\AI\Hum}]}{\E^\pi_h[N^{\dir, \pi}_{\AI}]}
    +\sum_{r\in\calR}g_h(r)\pos{d^{(h)}(r)-\lambda}
    \right].
\]
The left side does not depend on $\lambda$, so the inequality
holds with the right side replaced by its infimum over
$\lambda\ge0$. Write $s:=\E^\pi_h[N^{\pi}_{\AI\Hum}]/\E^\pi_h[N^{\dir, \pi}_{\AI}]$; then $s\in[0,1]$
because $0\le\E^\pi_h[N^{\pi}_{\AI\Hum}]\le\E^\pi_h[N^{\dir, \pi}_{\AI}]$. By \eqref{eq:Psi-frontier-def},
\[
    \inf_{\lambda\ge0}\left\{
    \lambda s+\sum_{r\in\calR}g_h(r)\pos{d^{(h)}(r)-\lambda}
    \right\}
    =\Psi_h(s).
\]
Combining the last two displays with the identity
$\sum_{i\in\calI_{\AI\Hum}^\pi}d^{(h)}(R_i)=\sum_iE_i\,d^{(h)}(R_i)$
yields
\[
    \E^\pi_h\left[\sum_{i\in\calI_{\AI\Hum}^\pi}d^{(h)}(R_i)\right]
    \le\E^\pi_h[N^{\dir, \pi}_{\AI}]\,\Psi_h\!\left(\frac{\E^\pi_h[N^{\pi}_{\AI\Hum}]}{\E^\pi_h[N^{\dir, \pi}_{\AI}]}\right),
\]
which is \eqref{eq:selected-info-bound}.
\end{proof}

\subsection{Proof of Proposition~\ref{prop:lower-bound-convexity}}
\begin{proof}{Proof of Proposition~\ref{prop:lower-bound-convexity}}
We show that Problem~\eqref{eq:Gamma-def} is a convex
optimization problem. We first notice that the first and second constraints of \eqref{eq:Gamma-def} are linear constraints in the decision variables $(n_{\Hum}, n_{\AI}, n_\esc)$. The third constraint involves nonlinear function 
$n_{\AI}\Psi_h(n_{\esc}/n_{\AI})$; we show that it is jointly concave in $(n_{\AI}, n_\esc)$. Notice that $\Psi_h(s)$ is concave in $s$ (see Theorem 5.1 of \citet{bertsimas1997introduction}). Additionally, $\Psi_h^{\mathrm{persp}}(n_{\AI},n_{\esc}):= n_{\AI}\Psi_h(n_{\esc}/n_{\AI})$ is the perspective function of $\Psi_h$ with domain $\{(n_\AI,n_{\esc}): 0\leq n_\esc \leq n_{\AI}\}$ and therefore is also concave (Section 3.2.6 of \citet{boyd2004convex}). Finally the feasible set of \eqref{eq:Gamma-def} is a superlevel set of a concave function intersected with 
half-spaces formed by
linear inequalities. Thus the feasible set is convex. The objective is linear. Therefore \eqref{eq:Gamma-def} is a convex optimization problem. 
\end{proof}

\subsection{Proof of Theorem~\ref{thm:lower-bound}}
\begin{proof}{Proof of Theorem~\ref{thm:lower-bound}}
    Fix any feasible policy $\pi$ using $N^\pi$ items. 
We first show that $\E^\pi_1[C^\pi]\geq N^\pi\cdata + \Gamma_1(A, N^\pi)$, and that  $\E^\pi_0[C^\pi]\geq N^\pi\cdata + \Gamma_0(B, N^\pi)$ can be derived analogously. Recall 
\begin{align*}
    \E^\pi_1[C^\pi] = \cdata N^\pi + \E^\pi_1[\cAI N^{\mathrm{tot}}_{\AI}+\cH N^{\mathrm{tot}}_{\Hum}]
    &\ge \cdata N^\pi +\cAI \E^\pi_1[N^{\dir, \pi}_{\AI}]+\cH(\E^\pi_1[N^{\dir,\pi}_{\Hum}]+\E^\pi_1[N^{\pi}_{\AI\Hum}]),
\end{align*}
where the inequality follows since each item counted by $N^{\dir, \pi}_{\AI}$ is AI-queried once, so $N^{\dir, \pi}_{\AI}\le N^{\mathrm{tot}}_\AI$, and every direct human query and every AI-after-report human escalation is a human query, so $N^{\dir,\pi}_{\Hum}+N^{\pi}_{\AI\Hum}\le N^{\mathrm{tot}}_{\Hum}$. 
Thus it is sufficient to show that $\cAI\E^\pi_1[N^{\dir, \pi}_{\AI}]+\cH(\E^\pi_1[N^{\dir,\pi}_{\Hum}]+\E^\pi_1[N^{\pi}_{\AI\Hum}])\geq \Gamma_1(A, N^\pi)$. 

We show $\cAI\E^\pi_1[N^{\dir, \pi}_{\AI}]+\cH(\E^\pi_1[N^{\dir,\pi}_{\Hum}]+\E^\pi_1[N^{\pi}_{\AI\Hum}])\geq \Gamma_1(A, N^\pi)$ by noting that $(\E^\pi_1[N^{\dir,\pi}_{\Hum}],\E^\pi_1[N^{\dir, \pi}_{\AI}],\E^\pi_1[N^{\pi}_{\AI\Hum}])$ is feasible for $\Gamma_1(A,N^\pi)$. Specifically, 
by Lemma~\ref{lem:error-to-kl} and Theorem~\ref{thm:kl-decomposition},
\[
    A\le
    \E^\pi_1\left[
        N^{\dir, \pi}_{\AI} I_R^{(1)}+N^{\dir,\pi}_{\Hum} J_X^{(1)}+
        \sum_{i\in\calI_{\AI\Hum}^\pi}d^{(1)}(R_i)
    \right].
\]
Since an item cannot be both direct-human-before-AI and AI-before-human, we have $N^{\dir,\pi}_{\Hum}+N^{\dir, \pi}_{\AI}\le N^\pi$ on every sample path. Thus $\E^\pi_1[N^{\dir,\pi}_{\Hum}]+\E^\pi_1[N^{\dir, \pi}_{\AI}]\le N^\pi$. Also $0\le \E^\pi_1[N^{\pi}_{\AI\Hum}]\le \E^\pi_1[N^{\dir, \pi}_{\AI}]$. By Lemma~\ref{lem:selected-info-bound},
\[
    \E^\pi_1\left[\sum_{i\in\calI_{\AI\Hum}^\pi}d^{(1)}(R_i)\right]
    \le \E^\pi_1[N^{\dir, \pi}_{\AI}]\Psi_1(\E^\pi_1[N^{\pi}_{\AI\Hum}]/\E^\pi_1[N^{\dir, \pi}_{\AI}]),
\]
where $\E^\pi_1[N^{\dir, \pi}_{\AI}]\Psi_1(\E^\pi_1[N^{\pi}_{\AI\Hum}]/\E^\pi_1[N^{\dir, \pi}_{\AI}])=0$ when $\E^\pi_1[N^{\dir, \pi}_{\AI}]=0$, as in \eqref{eq:Gamma-def}. Therefore
\[
    \E^\pi_1[N^{\dir,\pi}_{\Hum}]J_X^{(1)}+\E^\pi_1[N^{\dir, \pi}_{\AI}]I_R^{(1)}+\E^\pi_1[N^{\dir, \pi}_{\AI}]\Psi_1(\E^\pi_1[N^{\pi}_{\AI\Hum}]/\E^\pi_1[N^{\dir, \pi}_{\AI}])\ge A.
\]
So $(\E^\pi_1[N^{\dir,\pi}_{\Hum}],\E^\pi_1[N^{\dir, \pi}_{\AI}],\E^\pi_1[N^{\pi}_{\AI\Hum}])$ is feasible for $\Gamma_1(A,N^\pi)$, and we have that 
$$
\cAI\E^\pi_1[N^{\dir,\pi}_{\AI}]+\cH(\E^\pi_1[N^{\dir,\pi}_{\Hum}]+\E^\pi_1[N^{\pi}_{\AI\Hum}])\geq \Gamma_1(A, N^\pi),
$$
and therefore
\begin{equation}\label{eq:cost-h1-lb}
    \E^\pi_1[C^\pi]\ge N^\pi\cdata+\Gamma_1(A,N^\pi).
\end{equation}

By similar arguments, we have that
\begin{equation}\label{eq:cost-h0-lb}
    \E^\pi_0[C^\pi]\ge N^\pi\cdata+\Gamma_0(B,N^\pi)
\end{equation}
Combining \eqref{eq:cost-h1-lb} and \eqref{eq:cost-h0-lb},
\[
    \max\{\E^\pi_0[C^\pi],\E^\pi_1[C^\pi]\}
    \ge
    N^\pi\cdata+\max\{\Gamma_1(A,N^\pi),\Gamma_0(B,N^\pi)\}.
\]
Since $N^\pi\ge N_{\fixed,\Hum}(\alpha,\beta)$ by Lemma~\ref{lem:data-pool}, the right side is at least its minimum over integers $N\ge N_{\fixed,\Hum}(\alpha,\beta)$, which is $\LB(\alpha,\beta)$. Taking the infimum over all feasible policies proves
\[
    C^*(\alpha,\beta)\ge\LB(\alpha,\beta).
\]
\end{proof}

\section{Additional Materials for Section~\ref{sec:policy-construction}}

\subsection{Computation of $N_{\fixed, \AI}(\alpha,\beta)$}\label{appendix-sec:computation-NAI}

We now describe how to compute $N_{\fixed, \AI}(\alpha,\beta)$, together with the cutoff
and randomization of the associated report test, mirroring the computation of
$N_{\fixed,\Hum}(\alpha,\beta)$ in Section~\ref{sec:lower-bound-samples}.  As in
Definition~\ref{def:NAI}, all expectations and probabilities below refer to
the i.i.d.\ report model: under $H_h$ the reports $R_1,\ldots,R_N$ are
i.i.d.\ with mass function $g_h$.  By the Neyman--Pearson
lemma~\citep{neyman1933ix}, for any fixed sample size $N$ the most powerful
report test at type-I level $\alpha$ is the likelihood-ratio test.  In the
full-label case the likelihood ratio is strictly increasing in the label sum,
which reduces the test to a count threshold; for a general report alphabet
$\calR$ no scalar count is sufficient, so the test thresholds the
log-likelihood-ratio statistic itself.

Fix $N\in\N$ and suppose the report vector $R=(R_1,\ldots,R_N)$ is observed.
For $t=1,\ldots,N$, the running log-likelihood ratio under $H_1$ versus
$H_0$ is, by
\eqref{eq:AI-report-increment},
\begin{equation}\label{eq:report-llr}
    \mathcal L_t^{\mathrm{seq}}(R_{1:t})
    :=\log\prod_{i=1}^t\frac{g_1(R_i)}{g_0(R_i)}
    =\sum_{i=1}^t\ell_R(R_i).
\end{equation}
Since $\calR$ is finite, the terminal statistic
$\mathcal L_N^{\mathrm{seq}}$ takes finitely many values:
writing $n_r:=|\{i\le N:R_i=r\}|$ for the report counts,
$\mathcal L_N^{\mathrm{seq}}=\sum_{r\in\calR}n_r\,\ell_R(r)$, and
$(n_r)_{r\in\calR}$ is multinomial with parameters $(N,g_h)$ under $H_h$.
The distribution of $\mathcal L_N^{\mathrm{seq}}$ under each hypothesis is therefore
computable exactly.  Let $\mathcal{S}_N$ denote the finite set of values of
$\mathcal L_N^{\mathrm{seq}}$; since $g_h(r)>0$ for every $r\in\calR$ and both
$h\in\{0,1\}$, every point of $\mathcal{S}_N$ has positive probability under both
hypotheses.

For a target type-I level $\alpha$, define the cutoff $c^{\AI}_N$ to be the
smallest point of $\mathcal{S}_N$ such that
$\Pp_0(\mathcal L_N^{\mathrm{seq}}>c^{\AI}_N)\le\alpha$, and set
\begin{equation}\label{eq:gammaAI-def}
    \gamma^{\AI}_N:=
    \frac{\alpha-\Pp_0(\mathcal L_N^{\mathrm{seq}}>c^{\AI}_N)}
         {\Pp_0(\mathcal L_N^{\mathrm{seq}}=c^{\AI}_N)},
\end{equation}
with $\gamma^{\AI}_N=0$ if
$\Pp_0(\mathcal L_N^{\mathrm{seq}}>c^{\AI}_N)=\alpha$. The
denominator is positive because $c^{\AI}_N\in\mathcal{S}_N$, and
$\gamma^{\AI}_N\in[0,1]$ by the choice of $c^{\AI}_N$.  The Neyman--Pearson
report test is the randomized threshold test
\begin{equation}\label{eq:psi-star}
    \psi_N^*(R)=
    \begin{cases}
        1, & \mathcal L_N^{\mathrm{seq}}(R_{1:N})>c^{\AI}_N,\\
        \gamma^{\AI}_N,
           & \mathcal L_N^{\mathrm{seq}}(R_{1:N})=c^{\AI}_N,\\
        0, & \mathcal L_N^{\mathrm{seq}}(R_{1:N})<c^{\AI}_N.
    \end{cases}
\end{equation}
By construction, $\psi_N^*$ meets the type-I constraint exactly:
$\E_0[\psi_N^*(R)]=\alpha$.  Its type-II error is
\begin{equation}\label{eq:typeII-AI}
    \E_1[1-\psi_N^*(R)]
    =\Pp_1(\mathcal L_N^{\mathrm{seq}}<c^{\AI}_N)
    +(1-\gamma^{\AI}_N)\,
      \Pp_1(\mathcal L_N^{\mathrm{seq}}=c^{\AI}_N).
\end{equation}
It follows, exactly as for $N_{\fixed,\Hum}(\alpha,\beta)$ in
Section~\ref{sec:lower-bound-samples}, that $N_{\fixed, \AI}(\alpha,\beta)$ can be
computed by increasing $N$ from $1$ upward and checking whether
\eqref{eq:typeII-AI} is at most $\beta$: if $\psi_N^*$ fails the type-II
target, no report test on $N$ reports can meet it; if it succeeds, it is
itself a valid test meeting both error targets.  The first $N$ at which
\eqref{eq:typeII-AI} is at most $\beta$ is therefore $N_{\fixed, \AI}(\alpha,\beta)$.

% When $G_0\ne G_1$ the search terminates, because
% $\mathcal L_N^{\mathrm{seq}}$ has
% strictly positive drift $I_R^{(1)}=D(G_1\Vert G_0)>0$ under $H_1$ and
% strictly negative drift $-I_R^{(0)}$ under $H_0$, so both error
% probabilities of the threshold test vanish as $N\to\infty$; when $G_0=G_1$
% the statistic is identically zero and the search never terminates,
% consistent with $N_{\fixed, \AI}(\alpha,\beta)=+\infty$ above.  Finally, if
% $|\calR|=2$ and $\ell_R$ is nonconstant, then
% $\mathcal L_N^{\mathrm{seq}}$ is a strictly
% monotone function of the count of the report value with the larger
% $\ell_R$, and \eqref{eq:psi-star} reduces to a binomial threshold test
% exactly parallel to Section~\ref{sec:lower-bound-samples}, with $p_h$
% replaced by the corresponding report probability under $g_h$.

\subsection{Proof of Theorem~\ref{thm:feasibility}}\label{appendix-sec:feasibility-proof}
To prove Theorem~\ref{thm:feasibility}, we need auxiliary Lemmas~\ref{lem:boundary-error} and \ref{lem:fallback-error}. We present the proof of Theorem~\ref{thm:feasibility} in the end of this subsection. We let $E_+$ be the event that Algorithm~\ref{algo:seq-policy} stops before
fallback by crossing the upper boundary $a_k$, and let $E_-$ be the event that
it stops before fallback by crossing the lower boundary $-b_k$.  

\begin{lemma}[Boundary-route error control]\label{lem:boundary-error}
$\Pp^{\pi_k}_0(E_+)\le \alpha_{1,k}$ and $\Pp^{\pi_k}_1(E_-)\le \beta_{1,k}$.
\end{lemma}

\begin{proof}{Proof of Lemma~\ref{lem:boundary-error}}
Index the likelihood-ratio updates by $t=1,2,\ldots$, where each update
corresponds to one paid observation revealed by
Algorithm~\ref{algo:seq-policy}: a direct-human label, an AI report, or a
post-report escalation label. Recall $\calH_t:=\sigma\bigl(A_1,O_1,\ldots,A_t,O_t\bigr)$ (with $\mathcal H_0$ the trivial $\sigma$-algebra $\{\emptyset,\Omega\}$) is the filtration generated by the actions and 
observations up to and including update $t$.  Let
$\xi_t$ be the exact $H_1$-versus-$H_0$ log-likelihood increment contributed
by the observation at update $t$:
\begin{equation}\label{eq:xi-increment}
    \xi_t:=
    \begin{cases}
        \ell_X(X_i),
            & A_t=\Hum(i)\text{ with item $i$ not previously AI-queried},\\
        \ell_R(R_i),
            & A_t=\AI(i),\\
        \ell_H(X_i,R_i),
            & A_t=\Hum(i)\text{ with report $R_i$ already observed},
    \end{cases}
\end{equation}
where $\ell_X,\ell_R,\ell_H$ are the increments
\eqref{eq:direct-human-increment}--\eqref{eq:escalation-increment}.  The
running statistic and its exponential are
\begin{equation}\label{eq:Lambda-def}
    S_t:=\sum_{s=1}^{t}\xi_s,
    \qquad
    \Lambda_t:=\exp(S_t)=\prod_{s=1}^{t}e^{\xi_s},
    \qquad
    \Lambda_0:=1.
\end{equation}
The primitives $p_0,p_1,f_0,f_1$ do not depend on $k$: $p_0,p_1\in(0,1)$ are
fixed constants by \eqref{eq:hypotheses}, and $f_0,f_1$ are strictly positive
on the finite set $\calR$ by Assumption~\ref{ass:score-model}.  Hence all
four primitives are bounded away from $0$ and $1$, so the increments
$\xi_t$ are uniformly bounded; since the number of updates is at most
$2N_{\main,k}$, each $\Lambda_t$ is bounded and hence integrable. Next we show that $(\Lambda_t)_{t\ge0}$ is a nonnegative $(\calH_t)$-martingale under
$H_0$ with $\E^{\pi_k}_0[\Lambda_t]=\Lambda_0=1$.

\emph{Martingale property under $H_0$.}  Fix an update time $t$ and condition
on the pair $(\calH_{t-1},A_t)$.  The action $A_t$ is a measurable
function of $\calH_{t-1}$ and of a fresh randomization seed whose law does
not depend on the hypothesis and which is independent of the not-yet-revealed
observation $O_t$; conditioning on $A_t$ therefore leaves the $H_0$-law of
$O_t$ unchanged.  We evaluate $\E^{\pi_k}_0[e^{\xi_t}\mid\calH_{t-1},A_t]$ in each
case.

If $A_t=\Hum(i)$ with item $i$ not previously queried, then under $H_0$ the
label $X_i\sim\Bern(p_0)$ is independent of $\calH_{t-1}$, so by
\eqref{eq:direct-human-increment}
\begin{equation}\label{eq:mart-direct}
    \E^{\pi_k}_0\!\left[e^{\xi_t}\mid\calH_{t-1},A_t\right]
    =\sum_{x\in\{0,1\}}p_0^{x}(1-p_0)^{1-x}\,
        \frac{p_1^{x}(1-p_1)^{1-x}}{p_0^{x}(1-p_0)^{1-x}}
    =\sum_{x\in\{0,1\}}p_1^{x}(1-p_1)^{1-x}=1.
\end{equation}
If $A_t=\AI(i)$, then under $H_0$ the report $R_i\sim G_0$ is independent of
$\calH_{t-1}$, so by \eqref{eq:AI-report-increment}
\begin{equation}\label{eq:mart-report}
    \E^{\pi_k}_0\!\left[e^{\xi_t}\mid\calH_{t-1},A_t\right]
    =\sum_{r\in\calR}g_0(r)\,\frac{g_1(r)}{g_0(r)}
    =\sum_{r\in\calR}g_1(r)=1.
\end{equation}
If $A_t=\Hum(i)$ with report $R_i=r$ already observed (so $R_i$ is
$\calH_{t-1}$-measurable), then under $H_0$ the escalated label has
conditional mass $\rho_0(\cdot\mid r)$ by \eqref{eq:rho-def}, so by
\eqref{eq:escalation-increment}
\begin{equation}\label{eq:mart-escalation}
    \E^{\pi_k}_0\!\left[e^{\xi_t}\mid\calH_{t-1},A_t\right]
    =\sum_{x\in\{0,1\}}\rho_0(x\mid r)\,
        \frac{\rho_1(x\mid r)}{\rho_0(x\mid r)}
    =\sum_{x\in\{0,1\}}\rho_1(x\mid r)=1.
\end{equation}
In every case $\E^{\pi_k}_0[e^{\xi_t}\mid\calH_{t-1},A_t]=1$.  Averaging over the
$\calH_{t-1}$-conditional law of $A_t$ gives
$\E^{\pi_k}_0[e^{\xi_t}\mid\calH_{t-1}]=1$, and since $\Lambda_{t-1}$ is
$\calH_{t-1}$-measurable,
\begin{equation}\label{eq:mart-property}
    \E^{\pi_k}_0\!\left[\Lambda_t\mid\calH_{t-1}\right]
    =\Lambda_{t-1}\,
     \E^{\pi_k}_0\!\left[e^{\xi_t}\mid\calH_{t-1}\right]
    =\Lambda_{t-1}.
\end{equation}
Hence $(\Lambda_t)_{t\ge0}$ is a nonnegative $(\calH_t)$-martingale under
$H_0$ with $\E^{\pi_k}_0[\Lambda_t]=\Lambda_0=1$.

Let $\tau_{\mathrm b}$ be the index of the last paid update in the
main stage: the update at which a boundary is first crossed, or, if no
crossing occurs, the last paid update before entering fallback.
Then $\tau_{\mathrm b}\le 2N_{\main,k}$.  The optional stopping
theorem gives
\[
    \E^{\pi_k}_0[\Lambda_{\tau_{\mathrm{b}}}]=\E^{\pi_k}_0[\Lambda_0]=1.
\]
On $E_+$, $S_{\tau_{\mathrm{b}}}\ge a_k$, hence $\Lambda_{\tau_{\mathrm{b}}}\ge e^{a_k}=1/\alpha_{1,k}$.
Thus
\[
    \Pp^{\pi_k}_0(E_+) = \E^{\pi_k}_0[\ind_{E_+}]
    \le \alpha_{1,k}\E^{\pi_k}_0[\Lambda_{\tau_{\mathrm{b}}}\ind_{E_+}]
    \le \alpha_{1,k}\E^{\pi_k}_0[\Lambda_{\tau_{\mathrm{b}}}] 
    = \alpha_{1,k}\E^{\pi_k}_0[\Lambda_0]
    \le \alpha_{1,k}.
\]

The lower-boundary statement follows from the identical argument applied
under $H_1$ to $\Lambda_t^{-1}:=\exp(-S_t)$, with the roles of $(p_0,p_1)$,
$(g_0,g_1)$, and $(\rho_0(\cdot\mid r),\rho_1(\cdot\mid r))$ interchanged in
each of the three cases above: the same computation shows
$\E^{\pi_k}_1[e^{-\xi_t}\mid\calH_{t-1},A_t]=1$ in every case, so
$(\Lambda_t^{-1})_{t\ge0}$ is a nonnegative $(\calH_t)$-martingale under
$H_1$ with $\E^{\pi_k}_1[\Lambda_t^{-1}]=\Lambda_0^{-1}=1$.  Bounded optional
stopping applied to $\tau_{\mathrm{b}}$ under $H_1$ gives
$\E^{\pi_k}_1[\Lambda_{\tau_{\mathrm{b}}}^{-1}]=1$.  On $E_-$, $S_{\tau_{\mathrm{b}}}\le-b_k$, hence
$\Lambda_{\tau_{\mathrm{b}}}^{-1}\ge e^{b_k}=1/\beta_{1,k}$, and
\[
    \Pp^{\pi_k}_1(E_-)
    \le \beta_{1,k}\E^{\pi_k}_1[\Lambda_{\tau_{\mathrm{b}}}^{-1}\ind_{E_-}]
    \le \beta_{1,k}\E^{\pi_k}_1[\Lambda_{\tau_{\mathrm{b}}}^{-1}]
    =\beta_{1,k}.
\]
\end{proof}

Let $E_{\fb}$ be the event that no boundary is crossed before the fixed pool is
exhausted. Lemma~\ref{lem:fallback-error} bounds the fallback probabilities.
\begin{lemma}[Fallback-route error control]\label{lem:fallback-error}
\[
    \Pp^{\pi_k}_0(E_{\fb}\cap\{\text{fallback rejects }H_0\})\le\alpha_{2,k},\qquad
    \Pp^{\pi_k}_1(E_{\fb}\cap\{\text{fallback accepts }H_0\})\le\beta_{2,k}.
\]
\end{lemma}
\begin{proof}{Proof of Lemma~\ref{lem:fallback-error}}
There are two cases, according to the
deterministic value of $\mathsf F_k$.

If $\mathsf F_k=\Hum$, write $N_H:=N_{\fixed,\Hum}(\alpha_{2,k},\beta_{2,k})$.  The
fallback test is the full-label test $\phi^*_{N_H}$ on the
fixed index set $\calJ_{\Hum,k}$.  By \eqref{eq:human-fallback-level}, its
unconditional type-I error is at most $\alpha_{2,k}$. 
Then
\[
    \Pp^{\pi_k}_0(E_{\fb}\cap\{\text{fallback rejects}\})
    \leq
    \Pp^{\pi_k}_0(\{\text{fallback rejects}\})
    =\E^{\pi_k}_0[\phi^*_{N_H}]
    \le \alpha_{2,k}.
\]
The type-II statement follows the same way:
\[
    \Pp^{\pi_k}_1(E_{\fb}\cap\{\text{fallback accepts}\})
    \leq
    \Pp^{\pi_k}_1(\{\text{fallback accepts}\})
    =\E^{\pi_k}_1[1-\phi^*_{N_H}]
    \le \beta_{2,k}.
\]

If $\mathsf F_k=\AI$, write
$\calJ_{\AI,k}:=\{1,\ldots,N_{\fixed, \AI}(\alpha_{2,k},\beta_{2,k})\}$, so the fallback test is the report test
$\psi^*_{N_{\fixed, \AI}(\alpha_{2,k},\beta_{2,k})}$ applied to the fixed reports $\{R_i:i\in\calJ_{\AI,k}\}$.  Then
\eqref{eq:AI-fallback-level} and $\ind_{E_{\fb}}\le1$, gives
\[
    \Pp^{\pi_k}_0(E_{\fb}\cap\{\text{fallback rejects}\})\le\E^{\pi_k}_0[\psi^*_{N_{\fixed, \AI}(\alpha_{2,k},\beta_{2,k})}]\le\alpha_{2,k},
\]
and
\[
    \Pp^{\pi_k}_1(E_{\fb}\cap\{\text{fallback accepts}\})\le\E^{\pi_k}_1[1-\psi^*_{N_{\fixed, \AI}(\alpha_{2,k},\beta_{2,k})}]\le\beta_{2,k}.
\]
\end{proof}

\begin{proof}{Proof of Theorem~\ref{thm:feasibility}}
By construction, Algorithm~\ref{algo:seq-policy} terminates through exactly
one of three mutually exclusive routes: it crosses the upper boundary
(event $E_+$), crosses the lower boundary (event $E_-$), or reaches the end
of the fixed pool without crossing either boundary and executes the
pre-committed fallback (event $E_{\fb}$); in particular $E_+,E_-,E_{\fb}$ are
pairwise disjoint.

Consequently the event that the policy rejects $H_0$ is the disjoint union
of the upper boundary event $E_+$ and the fallback-rejection event
$E_{\fb}\cap\{\text{fallback rejects}\}\subseteq E_{\fb}$, so by countable
additivity of probability over disjoint events,
\[
    \Pp^{\pi_k}_0(\delta^{\pi_k}=1)
    =\Pp^{\pi_k}_0(E_+)
     +\Pp^{\pi_k}_0\bigl(E_{\fb}\cap\{\text{fallback rejects}\}\bigr).
\]
By Lemmas~\ref{lem:boundary-error} and~\ref{lem:fallback-error}, the two
terms on the right are at most $\alpha_{1,k}$ and $\alpha_{2,k}$
respectively, so
\[
    \Pp^{\pi_k}_0(\delta^{\pi_k}=1)
    \le \alpha_{1,k}+\alpha_{2,k}
    =\alpha_k,
\]
where the equality is the budget identity following
\eqref{eq:budget-split-alpha}.

Likewise, the event that the policy accepts $H_0$ under $H_1$ is the
disjoint union of the lower boundary event $E_-$ and the fallback-acceptance
event $E_{\fb}\cap\{\text{fallback accepts}\}$, so
\[
    \Pp^{\pi_k}_1(\delta^{\pi_k}=0)
    =\Pp^{\pi_k}_1(E_-)
     +\Pp^{\pi_k}_1\bigl(E_{\fb}\cap\{\text{fallback accepts}\}\bigr).
\]
The same two lemmas and \eqref{eq:budget-split-beta} give
\[
    \Pp^{\pi_k}_1(\delta^{\pi_k}=0)
    \le \beta_{1,k}+\beta_{2,k}
    =\beta_k.
\]
\end{proof}

\section{Additional Materials for Sections~\ref{sec:main-theorem}}\label{appendix-sec:first-order-proof}
Throughout, we let $T'_{1,k}:=\klbin(1-\beta_k||\alpha_k)$, $T'_{0,k}:=\klbin(1-\alpha_k||\beta_k)$
for the information thresholds. Recall that $T_{1,k} = a_k + \Delta_k$ and $T_{0,k} = b_k + \Delta_k$ where $a_k = \log(1/((1-f_{\fb, k})\alpha_k))$ and $b_k = \log(1/((1-f_{\fb, k})\beta_k))$.

\subsection{Proof of Theorem~\ref{thm:first-order}}\label{sec-appendix:LB-growth}
To prove Theorem~\ref{thm:first-order}(i), we first state an auxiliary Lemma~\ref{lem:binary-kl-asymptotic} together with its proof.

\begin{lemma}[Binary-KL thresholds are logarithmic]\label{lem:binary-kl-asymptotic}
Under Assumption~\ref{ass:balanced-error-regime}, as \(k\to\infty\),
\[
    T'_{1,k}=\log\frac1{\alpha_k}+o(1),
    \qquad
    T'_{0,k}=\log\frac1{\beta_k}+o(1).
\]
Consequently, since
\(a_k=\log(1/\alpha_{1,k})=\log(1/\alpha_k)+\log(1/(1-f_{\fb, k}))\)
and
\(b_k=\log(1/\beta_{1,k})=\log(1/\beta_k)+\log(1/(1-f_{\fb, k}))\)
by \eqref{eq:budget-split-alpha}--\eqref{eq:budget-split-beta},
\[
    a_k=T'_{1,k}+\log(1/(1-f_{\fb, k}))+o(1),
    \qquad
    b_k=T'_{0,k}+\log(1/(1-f_{\fb, k}))+o(1).
\]
\end{lemma}
\begin{proof}{Proof of Lemma~\ref{lem:binary-kl-asymptotic}}
We prove the expansion for \(T'_{1,k}\); the proof for \(T'_{0,k}\) is symmetric.
By the definition of Bernoulli KL divergence,
\begin{align}
    T'_{1,k}
    &=(1-\beta_k)\log\frac{1-\beta_k}{\alpha_k}
      +\beta_k\log\frac{\beta_k}{1-\alpha_k} \notag\\
    &=\log\frac1{\alpha_k}
      -\beta_k\log\frac1{\alpha_k}
      +(1-\beta_k)\log(1-\beta_k)
      +\beta_k\log\beta_k
      -\beta_k\log(1-\alpha_k).
      \label{eq:Ak-expanded-detailed}
\end{align}
Every correction term is bounded in absolute value by a primitive constant
multiple of \(\beta_kL_k\): the first directly; the second because
\((1-\beta_k)\log(1-\beta_k)=O(\beta_k)\); the third because
\(|\beta_k\log\beta_k|\le\beta_kL_k\); and the fourth because
\(-\beta_k\log(1-\alpha_k)=O(\alpha_k\beta_k)\).
Under Assumption~\ref{ass:balanced-error-regime},
\[
    \min\!\left\{\log\frac1{\alpha_k},\log\frac1{\beta_k}\right\}
    =\Theta(L_k),
\]
so \(\beta_k\le e^{-\Theta(L_k)}\) and therefore
\(\beta_kL_k\to0\). It follows that
\(T'_{1,k}=\log(1/\alpha_k)+o(1)\). Interchanging \(\alpha_k\) and \(\beta_k\)
gives \(T'_{0,k}=\log(1/\beta_k)+o(1)\). The final two relations follow straightforwardly.
\end{proof}

\begin{proof}{Proof of Theorem~\ref{thm:first-order}}

Part (i). We first show that $\LB_k=O(L_k)$ and then $\LB_k=\Omega(L_k)$.

\vspace{1mm}
\emph{Step 1: $\LB_k=O(L_k)$.}
By definition of $\LB_k$, we have that 
\begin{align*}
\LB_k &= \min_{N\geq N_{\fixed,\Hum}(\alpha_k, \beta_k)}N\cdata + \max\{\Gamma_1(T'_{1,k}, N), \Gamma_0(T'_{0,k}, N)\}.   
\end{align*}
To show $\LB_k=O(L_k)$, it suffices to find an integer solution $N^1_k$ to \eqref{eq:LB-def} that satisfies $N^1_k\ge N_{\fixed,\Hum}(\alpha_k,\beta_k)$ and that $N^1_k\cdata + \max\{\Gamma_1(T'_{1,k}, N^1_k), \Gamma_0(T'_{0,k}, N^1_k)\} = O(L_k)$.

\emph{Bounding the floor.} We first show $N_{\fixed,\Hum}(\alpha_k,\beta_k)$ is
$O(L_k)$ by exhibiting a cheap full-label test. Consider the
sample-mean test that rejects $H_0$ iff $\bar X_N\ge\bar p$, where
$\bar X_N:=N^{-1}\sum_{i=1}^N X_i$ and $\bar p:=(p_0+p_1)/2$. Under
$H_h$ the labels are i.i.d.\ $\Bern(p_h)$ by Assumption~\ref{ass:iid}, with mean $p_0<\bar p$ under $H_0$ and
$p_1>\bar p$ under $H_1$. Writing
$c_{\mathrm{Hfd}}:=((p_1-p_0)/2)^2>0$,
Hoeffding's inequality gives
\[
    \Pp^{\pi_k}_0(\bar X_N\ge\bar p)\le e^{-2Nc_{\mathrm{Hfd}}},
    \qquad
    \Pp^{\pi_k}_1(\bar X_N<\bar p)\le e^{-2Nc_{\mathrm{Hfd}}}.
\]
At $N_k^0:=\lceil L_k/(2c_{\mathrm{Hfd}})\rceil$, we have
$\Pp^{\pi_k}_0(\bar X_{N^0_k}\geq \bar p)\leq \alpha_k$ and
$\Pp^{\pi_k}_1(\bar X_{N^0_k}< \bar p)\leq \beta_k$. It follows by
Definition~\ref{def:Nfull} that 
$N_{\fixed,\Hum}(\alpha_k,\beta_k)\le N_k^0=O(L_k)$.

\emph{Choosing the evaluation point.} We now choose the sample size $N^1_k$ as:
\[
    N_k^1:=\max\!\left\{N^0_k, \lceil T'_{1,k}/J_X^{(1)}\rceil, \lceil T'_{0,k}/J_X^{(0)}\rceil
    \right\}.
\]
In particular, $N_k^1\geq N^0_k$ ensures that $N^1_k\geq N_{\fixed,\Hum}(\alpha_k, \beta_k)$; $N_k^1\geq \lceil T'_{1,k}/J_X^{(1)}\rceil$ and $N_k^1\geq \lceil T'_{0,k}/J_X^{(0)}\rceil$ ensure that $N^1_k$ direct-human queries supply $T'_{1,k}$ and $T'_{0,k}$ units of information under direction $h = 1$ and $h = 0$, respectively, since each direct-human query supply $J_X^{(h)}$ information. 

\emph{Bounding the objective at $N_k^1$.}
We now show that $N^1_k\cdata + \max\{\Gamma_1(T'_{1,k}, N^1_k), \Gamma_0(T'_{0,k}, N^1_k)\} = O(L_k)$.
First notice $N_k^1=O(L_k)$: $N_k^0 = O(L_k)$, and
$T'_{1,k},T'_{0,k}=O(L_k)$ by Lemma~\ref{lem:binary-kl-asymptotic} with
$J_X^{(0)},J_X^{(1)}$ constant. 

Additionally, for each direction $h$, the
all-direct-human triple
$(n_{\Hum},n_{\AI},n_{\esc}):=(\lceil T'_{h,k}/J_X^{(h)}\rceil,0,0)$
is feasible for \eqref{eq:Gamma-def} at $T = T'_{h,k}, N = N_k^1$: its
item count $\lceil T'_{h,k}/J_X^{(h)}\rceil\le N_k^1$ satisfies the
capacity constraint by the choice of $N_k^1$; its escalation term is
zero by the $n_{\AI}=0$ convention; and its information
$\lceil T'_{h,k}/J_X^{(h)}\rceil\,J_X^{(h)}\ge T'_{h,k}$ meets the
information constraint. Its objective value is
$\cH\lceil T'_{h,k}/J_X^{(h)}\rceil=O(L_k)$, so
$\Gamma_h(T'_{h,k},N_k^1)=O(L_k)$ for each $h$. Therefore
\[
    \LB_k\le N_k^1\cdata
        +\max\{\Gamma_1(T'_{1,k},N_k^1),\Gamma_0(T'_{0,k},N_k^1)\}
        =O(L_k).
\]

\vspace{1mm}
\emph{Step 2: $\LB_k=\Omega(L_k)$.}
It suffices to show
$N_{\fixed,\Hum}(\alpha_k,\beta_k)=\Omega(L_k)$, and the argument follows by noting that any feasible solution $N$ to \eqref{eq:LB-def} must satisfy $N\geq N_{\fixed,\Hum}(\alpha_k,\beta_k)$ and therefore must have cost $\LB_k\geq N_{\fixed,\Hum}(\alpha_k,\beta_k)\cdata = \Omega(L_k)$. Let $\pi_f$ be a feasible
full-label test using $N_f:=N_{\fixed,\Hum}(\alpha_k,\beta_k)$ labels (for example we can let $\pi_f$ be the Neyman-Pearson test $\phi^*_{N_f}$; see the details in Section~\ref{sec:lower-bound-samples}); its
transcript is the label vector $(X_1,\ldots,X_{N_f})$ with no AI
queries, and under $H_h$ these labels are i.i.d.\ $\Bern(p_h)$ by
Assumption~\ref{ass:iid}. Applying the finite
chain rule (Lemma~\ref{lem:kl-chain-rule}) inductively across the
$N_f$ independent coordinates, each conditional term reduces to the
per-label divergence $\klbin(p_1\Vert p_0)$, so
\[
    D(P_1^{\pi_f}\Vert P_0^{\pi_f})
    =N_f\,\klbin(p_1\Vert p_0)
    =N_f J_X^{(1)},
\]
the last equality by the definition \eqref{eq:JX-def} of $J_X^{(1)}$.
Since $\pi_f$ is feasible, Lemma~\ref{lem:error-to-kl} gives
$D(P_1^{\pi_f}\Vert P_0^{\pi_f})\ge T'_{1,k}$, so $N_f\ge T'_{1,k}/J_X^{(1)}$.
By Lemma~\ref{lem:binary-kl-asymptotic}, $T'_{1,k}=\log(1/\alpha_k)+o(1)$,
hence $N_f=\Omega(\log(1/\alpha_k))$. The same argument in the reverse
direction, with $D(P_0^{\pi_f}\Vert P_1^{\pi_f})=N_f J_X^{(0)}\ge T'_{0,k}$
and $T'_{0,k}=\log(1/\beta_k)+o(1)$, and thus
$N_f=\Omega(\log(1/\beta_k))$. Taking the larger,
$N_f=\Omega(\max\{\log(1/\alpha_k),\log(1/\beta_k)\})=\Omega(L_k)$.
Combining Step~1 and Step~2 gives
$\LB_k=\Theta(L_k)$.

Part (ii). Fix a
sufficiently large $k$. Under $H_h$, the cost decomposes as
\[
    \E_h^{\pi_k}[C^{\pi_k}]
    =\cdata N_{\main,k}
     +\E_h^{\pi_k}[C_k^{\main}]
     +\E_h^{\pi_k}[C_k^{\fb}],
\]
where $C_k^{\main}$ is the sensing cost incurred during the sequential
main stage and $C_k^{\fb}$ is the additional cost of the pre-committed
fallback, equal to zero when the main stage stops at a boundary.

We first control how often the policy uses a rule other than the one
corresponding to the true hypothesis. The following lemma bounds the
expected number of items processed with the wrong-direction or dead-zone
rule.
\begin{lemma}[Direction tracking]\label{lem:direction-tracking}
Let $W_{1,k}$ be the number of sensed items processed while
$S_{i-1}\le z_k$ under $H_1$, before the policy stops or exhausts the pool.
Let $W_{0,k}$ be the number of sensed items processed while
$S_{i-1}\ge -z_k$ under $H_0$. Then there is a constant $C_W<\infty$ such that
\begin{equation}\label{eq:direction-tracking-bound}
    \E^{\pi_k}_1[W_{1,k}]+\E^{\pi_k}_0[W_{0,k}]
    \le C_W(z_k+1).
\end{equation}
\end{lemma}
The proof of Lemma~\ref{lem:direction-tracking} is given in
Appendix~\ref{appendix-sec:direction-tracking}.
Since the sensing cost of each item is bounded by a primitive constant,
this result limits the expected cost of wrong-direction and dead-zone
items to $O(z_k+1)$.

We next bound the main-stage cost. For items processed with the correct
direction rule, the planned sensing cost and information yield are linked
by the optimizer of $\Gamma_h$. The bounded likelihood-ratio overshoot
then bounds their expected total cost by
$\Gamma_h(T_{h,k},\overline N_{\main,k})+O(1)$.
Combining this bound with Lemma~\ref{lem:direction-tracking} gives the
following result.
\begin{lemma}[Main-stage sensing cost]\label{lem:main-cost}
There is a constant $C_M<\infty$ such that, for $h=0,1$,
\begin{equation}\label{eq:main-cost-bound}
    \E^{\pi_k}_h[C^{\main}_k]
    \le \Gamma_h(T_{h,k},\overline N_{\main,k})
       +C_M(z_k+1).
\end{equation}
\end{lemma}
The proof of Lemma~\ref{lem:main-cost} is given in
Appendix~\ref{appendix-sec:main-cost}. 

To control the expected fallback cost, we next bound the probability
that the main stage exhausts the pool without crossing a boundary.
\begin{lemma}[Fallback probability]\label{lem:fallback-low}
Let $E_{\fb}$ be the event that the fixed pool is exhausted before either
boundary is crossed. There are constants $C_{\fb}<\infty$ and $c_{\fb}>0$
such that, for $h=0,1$ and all large $k$,
\begin{equation}\label{eq:fallback-prob-bound}
    \Pp^{\pi_k}_h(E_{\fb})
    \le C_{\fb}\frac{z_k+1}{\Delta_k}
      +C_{\fb}\exp\left\{-c_{\fb}\frac{\Delta_k^2}{L_k}\right\}.
\end{equation}
\end{lemma}
The proof of Lemma~\ref{lem:fallback-low} is given in
Appendix~\ref{appendix-sec:fallback-low}.

The pre-committed fallback requires at most $O(L_k)$ additional cost
whenever it is invoked, since $N_{\main,k}=O(L_k)$ by
Lemma~\ref{lemma-appendix:N-main} and either completion mode queries at
most $N_{\main,k}$ items. Multiplying this deterministic cost bound by
the probability in Lemma~\ref{lem:fallback-low} yields the next lemma.
\begin{lemma}[Fallback expected cost]\label{lem:fallback-cost}
There is a constant $C_F<\infty$ such that, for $h=0,1$,
\begin{equation}\label{eq:fallback-cost-bound}
    \E^{\pi_k}_h[C^{\fb}_k]
    \le C_F L_k\left[
        \frac{z_k+1}{\Delta_k}
        +\exp\left\{-c_{\fb}\frac{\Delta_k^2}{L_k}\right\}
    \right].
\end{equation}
\end{lemma}
The proof of Lemma~\ref{lem:fallback-cost} is given in
Appendix~\ref{appendix-sec:fallback-cost}.

It remains to compare the acquisition cost and the buffered sensing
plan with the lower bound $\LB_k$. The following perturbation result
bounds the cost of increasing the information targets and imposing the
fallback sample-size floor.
\begin{lemma}[Perturbation from the buffered design to the lower bound]\label{lem:LB-perturb}
There is a constant $C_P<\infty$ such that
\begin{equation}\label{eq:LB-perturb}
    N_{\main,k}\cdata+
    \max_h \Gamma_h(T_{h,k},\overline N_{\main,k})
    \le \LB_k+C_P\Delta_k
        +O\!\left(\log L_k+\log\left(\frac{1}{f_{\fb,k}}\right)
        +\log\left(\frac{1}{1-f_{\fb,k}}\right)\right).
\end{equation}
\end{lemma}
The proof of Lemma~\ref{lem:LB-perturb} is given in
Appendix~\ref{appendix-sec:LB-perturb}.

Combining the cost decomposition with
Lemmas~\ref{lem:main-cost}, \ref{lem:fallback-cost}, and
\ref{lem:LB-perturb}, we obtain
\[
\begin{aligned}
    \max_h\E_h^{\pi_k}[C^{\pi_k}]
    \le{}&\LB_k+C_P\Delta_k+C_M(z_k+1)
    +C_F L_k\left[
       \frac{z_k+1}{\Delta_k}
       +\exp\left\{-c_{\fb}\frac{\Delta_k^2}{L_k}\right\}
       \right]\\
    &+O\!\left(\log L_k+\log\frac{1}{f_{\fb,k}}
                        +\log\frac{1}{1-f_{\fb,k}}\right).
\end{aligned}
\]
Since $\Delta_k=o(L_k)$, the term $C_M(z_k+1)$ can be absorbed into a
constant multiple of $L_k(z_k+1)/\Delta_k$ for all large $k$.
Consequently, there are primitive constants $c>0$ and
$C_1,C_2,C_3,C_4<\infty$ such that
\begin{equation}\label{eq:param-ratio-bound}
\begin{aligned}
    \max_h\E_h^{\pi_k}[C^{\pi_k}]
    \le{}&\LB_k+C_1\Delta_k
      +C_2\frac{(z_k+1)L_k}{\Delta_k}
      +C_3L_k\exp\left\{-c\frac{\Delta_k^2}{L_k}\right\}\\
      &+C_4\left(\log L_k+\log\frac{1}{f_{\fb,k}}
                         +\log\frac{1}{1-f_{\fb,k}}\right).
\end{aligned}
\end{equation}
By part (i), $\LB_k=\Theta(L_k)$. Dividing
\eqref{eq:param-ratio-bound} by $\LB_k$, the conditions
\eqref{eq:delta-admissible} and \eqref{eq:z-admissible} give
\[
    \frac{\Delta_k}{L_k}\to0,
    \qquad \frac{z_k+1}{\Delta_k}\to0,
    \qquad \exp\left\{-c\frac{\Delta_k^2}{L_k}\right\}\to0.
\]
The remaining logarithmic contribution also vanishes because
$f_{\fb,k}$ is bounded away from zero and one. Hence
\[
    \frac{\max_h\E_h^{\pi_k}[C^{\pi_k}]}{\LB_k}\le1+o(1).
\]
Finally, feasibility of $\pi_k$ and Theorem~\ref{thm:lower-bound} imply $\max_h\E_h^{\pi_k}[C^{\pi_k}]
    \ge C^*(\alpha_k,\beta_k)\ge\LB_k$. The two bounds together prove part (ii).
\end{proof}

\subsection{Proof of Lemma~\ref{lem:LB-perturb}}\label{appendix-sec:LB-perturb}
We advance the proof of  Lemma~\ref{lem:LB-perturb} as the proofs of Lemmas~\ref{lem:direction-tracking}-\ref{lem:fallback-cost} rely on the conclusions from Lemma~\ref{lem:LB-perturb}, while the proof of Lemma~\ref{lem:LB-perturb} does not rely on the conclusions from Lemmas~\ref{lem:direction-tracking}-\ref{lem:fallback-cost}. 

To prove Lemma~\ref{lem:LB-perturb}, we need auxiliary Lemma~\ref{lem:safe-floor-first-order}. We relegate the proof of Lemma~\ref{lem:safe-floor-first-order} to Appendix~\ref{appendix-sec:safe-floor-first-order}.

\begin{lemma}[Safe-floor stability at first order]\label{lem:safe-floor-first-order}
(i) $N_{\fixed,\Hum}(\alpha_k, \beta_k) = O(L_k)$; 

(ii)
\begin{equation}\label{eq:safe-floor-result}
    0\le N_{\fixed,\Hum}(\alpha_{2,k},\beta_{2,k})-N_{\fixed,\Hum}(\alpha_k,\beta_k)
     = O\!\left(\log L_k+\log\left(\frac{1}{f_{\fb, k}}\right)\right).
\end{equation}
\end{lemma}

We are now ready to prove Lemma~\ref{lem:LB-perturb}.
\begin{proof}{Proof of Lemma~\ref{lem:LB-perturb}}
To show \eqref{eq:LB-perturb} holds, we find an integer $\tilde N_k$ with $\tilde N_k\geq N_{\fixed,\Hum}(\alpha_{2,k},\beta_{2,k})$ and $F^{\Delta}_k(\tilde N_k)\leq\LB_k + C_P\Delta_k + O(\log L_k+\log(1/f_{\fb, k}) + \log(1/(1-f_{\fb, k})))$, and thus by construction of $\overline N_{\main,k}$, we have that $F^{\Delta}_k(\overline N_{\main,k}) \leq F^{\Delta}_k(\tilde N_k)\leq \LB_k + C_P\Delta_k + O(\log L_k+\log(1/f_{\fb, k}) + \log(1/(1-f_{\fb, k})))$.

\vspace{1mm}
\emph{Step 1: Construction of $\tilde N_k$.}
Let $N_k^*$ be an optimizer in \eqref{eq:LB-def} for
$(\alpha_k,\beta_k)$. That is,
$$
N^*_k\in\arg\min_{N\geq N_{\fixed,\Hum}(\alpha_k, \beta_k)} N\cdata + \max\{\Gamma_1(T'_{1,k}, N), \Gamma_0(T'_{0,k}, N)\}. 
$$
Thus $\LB_k = N_k^*\cdata+\max_h\Gamma_h(T'_{h,k},N_k^*)$.
Also define the
\emph{reserve}
\begin{equation}\label{eq:Rk-def}
    R_k:=\left\lceil
        \frac{\pos{a_k+\Delta_k-T'_{1,k}}}{J_X^{(1)}}
        +\frac{\pos{b_k+\Delta_k-T'_{0,k}}}{J_X^{(0)}}
        +N_{\fixed,\Hum}(\alpha_{2,k},\beta_{2,k})
              -N_{\fixed,\Hum}(\alpha_k,\beta_k)
    \right\rceil.
\end{equation}
Let 
$$
\tilde N_k:= \lceil N^*_k + R_k\rceil.
$$

\vspace{1mm}
\emph{Step 2: show that $F^\Delta_k(\tilde N_k)\leq \LB_k + C_P\Delta_k + O(\log L_k+\log(1/f_{\fb, k}) + \log(1/(1-f_{\fb, k})))$.}
We first notice that $\tilde N_k \geq N^*_k + R_k\geq N_{\fixed,\Hum}(\alpha_k, \beta_k) + N_{\fixed,\Hum}(\alpha_{2,k},\beta_{2,k})-N_{\fixed,\Hum}(\alpha_k,\beta_k)= N_{\fixed,\Hum}(\alpha_{2,k}, \beta_{2,k})$, using $N_k^*\ge N_{\fixed,\Hum}(\alpha_k,\beta_k)$ (Step~1). We now bound $F_k^\Delta(N_k^*+R_k)$ from above.

Fix $h\in\{0,1\}$. Since $\LB_k<\infty$, we also have that $\Gamma_h(T'_{h,k},N_k^*)<\infty$ for $h\in\{0,1\}$, so there is an
optimal triple $(n_{\Hum}^*,n_{\AI}^*,n_{\esc}^*)$ for
$\Gamma_h(T'_{h,k},N_k^*)$; being feasible, $(n_{\Hum}^*,n_{\AI}^*,n_{\esc}^*)$ for
$\Gamma_h(T'_{h,k},N_k^*)$ satisfies: $n_{\Hum}^*+n_{\AI}^*\le N_k^*$, $0\leq n_{\esc}^*\leq n_{\AI}^*$ and that $n_{\Hum}^* J^{(h)}_X +n_{\AI}^* I^{(h)}_R + n_{\AI}^*\Psi_h(n_{\esc}^*/n_{\AI}^*)\geq T'_{h,k}$. We construct the perturbed triple $(\hat n_{\Hum}, \hat n_{\AI}, \hat n_{\esc}) := (n_{\Hum}^*+R_k,\,n_{\AI}^*,\,n_{\esc}^*)$, which
adds $R_k$ direct-human items. We check it is feasible for \eqref{eq:Gamma-def} with $T = T_{h,k}$ and $N = \tilde N_k$:
\begin{itemize}
\item $\hat n_{\Hum} + \hat n_{\AI} = (n_{\Hum}^*+R_k)+n_{\AI}^*\le N_k^*+R_k = \tilde N_k$,
since $n_{\Hum}^*+n_{\AI}^*\le N_k^*$.
\item $\hat n_{\AI} = n_{\AI}^*,\hat n_{\esc} = n_{\esc}^*$, so
$0\le\hat n_{\esc}\le \hat n_{\AI}$ holds since $0 \leq n^*_{\esc}\leq n^*_{\AI}$.
\item $\hat n_{\Hum} J^{(h)}_X +\hat n_{\AI} I^{(h)}_R + \hat n_{\AI}\Psi_h(\hat n_{\esc}/\hat n_{\AI}) = (n_{\Hum}^*+R_k) J^{(h)}_X +n_{\AI}^* I^{(h)}_R + n_{\AI}^*\Psi_h(n_{\esc}^*/n_{\AI}^*)\geq T'_{h,k} + R_k J^{(h)}_X$. By
\eqref{eq:Rk-def}, $R_k\geq [T_{h,k} - T'_{h,k}]_+/J_X^{(h)}$ for all $h\in\{0,1\}$, so
$R_kJ_X^{(h)}\ge\pos{T_{h,k}-T'_{h,k}}$. Thus $\hat n_{\Hum} J^{(h)}_X +\hat n_{\AI} I^{(h)}_R + \hat n_{\AI}\Psi_h(\hat n_{\esc}/\hat n_{\AI})\geq T'_{h,k}+\pos{T_{h,k}-T'_{h,k}}\ge T_{h,k}$.
\end{itemize}
Then, since $(\hat n_{\Hum}, \hat n_{\AI}, \hat n_{\esc})$ is feasible for \eqref{eq:Gamma-def} with $T = T_{h,k}$ and $N = \tilde N_k$, we have that
\begin{align*}
    \Gamma_h(T_{h,k}, \tilde N_k)
    &\leq c_{\Hum}\hat n_{\Hum} + c_{\AI}\hat n_{\AI} + c_{\Hum}\hat n_{\esc}\\
    &= c_{\Hum}(n^*_{\Hum}+R_k) + c_{\AI}n^*_{\AI} + c_{\Hum}n^*_{\esc}\\
    &\leq \Gamma_h(T'_{h,k},N_k^*)+\cH R_k.
\end{align*}
Thus
\begin{align*}
F_k^\Delta(\tilde N_k) 
=& \lceil N^*_k + R_k\rceil\cdata + \max\{\Gamma_1(T_{1,k}, \tilde N_k), \Gamma_0(T_{0,k}, \tilde N_k)\}\\
    \le&\Bigl( N_k^*\cdata+\max_h\Gamma_h(T'_{h,k},N_k^*)\Bigr)
       +R_k(\cdata+\cH) + \cdata\\
    =&\LB_k+R_k(\cdata+\cH) + \cdata,
\end{align*}
where the last equality uses $\LB_k = N_k^*\cdata+\max_h\Gamma_h(T'_{h,k},N_k^*)$
from Step~1. It remains to bound $R_k(\cdata+\cH)$. By
Lemma~\ref{lem:binary-kl-asymptotic},
$\pos{a_k-T'_{1,k}}=\log(1/(1-f_{\fb, k}))+o(1)$ and
$\pos{b_k-T'_{0,k}}=\log(1/(1-f_{\fb, k}))+o(1)$, hence
$\pos{a_k+\Delta_k-T'_{1,k}}\le\Delta_k+\log(1/(1-f_{\fb, k}))+o(1)$ and likewise
for the $b$-term; and by Lemma~\ref{lem:safe-floor-first-order},
$N_{\fixed,\Hum}(\alpha_{2,k},\beta_{2,k})-N_{\fixed,\Hum}(\alpha_k,\beta_k)
=O\!\left(\log L_k+\log(1/f_{\fb, k})\right)$.
Substituting these estimates into \eqref{eq:Rk-def} (the $+1$ from the ceiling
is $O(1)$) gives
\[
    R_k\le(1/J_X^{(1)}+1/{J_X^{(0)}})\Delta_k
        +O\!\left(\log L_k+\log(1/f_{\fb, k})+\log(1/(1-f_{\fb, k}))\right).
\]
Therefore, setting
$C_P:=(\cdata+\cH)\bigl(1/J_X^{(1)}+1/J_X^{(0)}\bigr)$,
\[
    F_k^\Delta(\tilde N_k)\le\LB_k+R_k(\cdata+\cH) + \cdata
    \le\LB_k+C_P\Delta_k
        +O\!\left(\log L_k+\log(1/f_{\fb, k}) + \log(1/(1-f_{\fb, k}))\right).
\]
Finally, since $\overline N_{\main,k}$ minimizes $F_k^\Delta$ over
$N\ge N_{\fixed,\Hum}(\alpha_{2,k},\beta_{2,k})$ and $\tilde N_k\ge N_{\fixed,\Hum}(\alpha_{2,k}, \beta_{2,k})$
lies in its feasible range,
$F^{\Delta}_k(\overline N_{\main,k})\le F^{\Delta}_k(\tilde N_k)
\le\LB_k+C_P\Delta_k+O\!\left(\log L_k+\log(1/f_{\fb, k}) + \log(1/(1-f_{\fb, k}))\right)$.
\end{proof}

Lemma~\ref{lemma-appendix:N-main} is an immediate consequence of Lemma~\ref{lem:LB-perturb}. 

\begin{lemma}\label{lemma-appendix:N-main}
    $\overline N_{\main,k}=O(L_k)$ and $N_{\main, k} = O(L_k)$.
\end{lemma}
\begin{proof}{Proof of Lemma~\ref{lemma-appendix:N-main}}
From the proof of Lemma~\ref{lem:LB-perturb} we have shown that $F^{\Delta}_k(\overline N_{\main,k})\le F^{\Delta}_k(\tilde N_k)
\le\LB_k+C_P\Delta_k+O\!\left(\log L_k+\log(1/f_{\fb, k}) + \log(1/(1-f_{\fb, k}))\right) = O(L_k)$. 
In the meanwhile, since $F_k^\Delta(\overline N_{\main,k})\ge\overline N_{\main,k}\cdata$,
we get $\overline N_{\main,k}=O(L_k)$, and
$N_{\mathrm{main},k}=\overline N_{\main,k}+1=O(L_k)$.
\end{proof}

\subsubsection{Proof of Lemma~\ref{lem:safe-floor-first-order}}\label{appendix-sec:safe-floor-first-order}
In this subsection, we prove Lemma~\ref{lem:safe-floor-first-order}. We first present an auxiliary Lemma~\ref{lem:Nfull-V-identity}, and then present the proof of Lemma~\ref{lem:safe-floor-first-order} at the end of this subsection. 

To introduce Lemma~\ref{lem:Nfull-V-identity},
we motivate the quantity $V(u, v)$ that will appear in Lemma~\ref{lem:Nfull-V-identity}. 
Consider the family of threshold tests that reject $H_0$ when 
$\bar{X}_n := n^{-1}\sum_{i=1}^n X_i \geq \vartheta$ for some
$\vartheta \in (p_0, p_1)$. For \(p\in(0,1)\), let \(\Pp_p\) denote the probability law under which \(X_1,X_2,\ldots\) are i.i.d.\ \(\Bern(p)\), and let \(\E_p\) denote expectation with respect to \(\Pp_p\).
By the Chernoff bound for Bernoulli sums,
\[
\Pp_{p_0}\!\left(\bar{X}_n \geq \vartheta\right)
\leq e^{-n D_0(\vartheta)}, 
\qquad 
\Pp_{p_1}\!\left(\bar{X}_n \leq \vartheta\right)
\leq e^{-n D_1(\vartheta)},
\]
where $D_0(\vartheta) := \text{kl}(\vartheta \,\|\, p_0)$ and
$D_1(\vartheta) := \text{kl}(\vartheta \,\|\, p_1)$. Let
\(u_k:=\log(1/\alpha_k)\), \(v_k:=\log(1/\beta_k)\). 
To satisfy the type-I error target $e^{-u_k} = \alpha_k$, it suffices that 
$n D_0(\vartheta) \geq u_k$, i.e.,
$n \geq u_k / D_0(\vartheta)$. 
Symmetrically, the type-II error target $e^{-v_k} = \beta_k$ requires 
$n \geq v_k / D_1(\vartheta)$. 
For a fixed threshold $\vartheta$, the smallest sample size at which both 
Chernoff bounds simultaneously deliver the target errors is therefore
\[
\max\!\left\{\frac{u_k}{D_0(\vartheta)},
\frac{v_k}{D_1(\vartheta)}\right\},
\]
and optimizing over $\vartheta \in (p_0, p_1)$ yields the quantity
\begin{equation}\label{eq:V-def-identity}
    V(u,v)
    :=\inf_{\vartheta\in(p_0,p_1)}
    \max\!\left\{\frac{u}{D_0(\vartheta)},
    \frac{v}{D_1(\vartheta)}\right\}.
\end{equation}
This construction immediately gives the upper bound
$N_{\fixed,\Hum}(\alpha_k, \beta_k) \leq V(u_k, v_k)+1$
by exhibiting an explicit feasible test. 
The next lemma asserts that this Chernoff-style bound is in fact 
first-order tight: no test on i.i.d.\ Bernoulli labels can achieve 
the target errors with substantially fewer samples. Recall that  \(L_k:=\max\{u_k,v_k\}\).

\begin{lemma}[First-order sample-size identity for the full-label benchmark]\label{lem:Nfull-V-identity}
% Fix \(0<p_0<p_1<1\). For \(\vartheta\in(p_0,p_1)\), define
% \(D_0(\vartheta):=\klbin(\vartheta\Vert p_0)\) and
% \(D_1(\vartheta):=\klbin(\vartheta\Vert p_1)\),
% both strictly positive. For \(u,v>0\), define
% \begin{equation}\label{eq:V-def-identity}
%     V(u,v)
%     :=\inf_{\vartheta\in(p_0,p_1)}
%     \max\!\left\{\frac{u}{D_0(\vartheta)},
%     \frac{v}{D_1(\vartheta)}\right\}.
% \end{equation}
Let \(\alpha_k,\beta_k\downarrow 0\) satisfy
Assumption~\ref{ass:balanced-error-regime}.
% and write
% \(u_k:=\log(1/\alpha_k)\), \(v_k:=\log(1/\beta_k)\),
% \(L_k:=\max\{u_k,v_k\}\), and \(m_k:=\min\{u_k,v_k\}=\Theta(L_k)\). 
Then
$V(u_k,v_k)=\Theta(L_k)$ and
\begin{equation}\label{eq:Nfull-V-asymp-balanced}
    V(u_k,v_k)-O(\log L_k)\le N_{\fixed,\Hum}(e^{-u_k},e^{-v_k})\le V(u_k,v_k)+1 .
\end{equation}
\end{lemma}

\begin{proof}{Proof of Lemma~\ref{lem:Nfull-V-identity}}
The proof proceeds in five steps. Step~1 establishes the monotonicity of
$D_0$ and $D_1$ and shows that $V(u_k,v_k)=\Theta(L_k)$. Step~2 applies
Chernoff bounds to an optimally chosen sample-mean threshold test,
yielding $N_{\fixed,\Hum}\le V(u_k,v_k)+1$. Step~3 derives a lower bound on
Bernoulli type-class probabilities $\Pp_p(\sum_{i=1}^n X_i=j)$. Step~4 shows that, for any feasible
test, no type class can have probability exceeding the respective error
levels under both hypotheses. Step~5 combines Steps~3--4 to show that a
sample size below $V(u_k-\delta_k,v_k-\delta_k)$, where
$\delta_k=O(\log L_k)$, would violate this incompatibility. Assumption~\ref{ass:balanced-error-regime} then gives
$N_{\fixed,\Hum}\ge V(u_k,v_k)-O(\log L_k)$.

\paragraph{Step 1: properties of \(D_0,D_1,V\).}
Differentiating gives, for \(\vartheta\in(0,1)\),
\[
    D_0'(\vartheta)
    =\log\frac{\vartheta(1-p_0)}{p_0(1-\vartheta)},
    \qquad
    D_1'(\vartheta)
    =\log\frac{\vartheta(1-p_1)}{p_1(1-\vartheta)}.
\]
Hence on \((p_0,p_1)\), \(D_0\) is strictly increasing and \(D_1\)
is strictly decreasing, so
\[
    0<D_0(\vartheta)<D_0(p_1)=J^{(1)}_X,
    \qquad
    0<D_1(\vartheta)<D_1(p_0)=J^{(0)}_X,
\]
where \(J^{(1)}_X=\klbin(p_1\Vert p_0)>0\) and \(J^{(0)}_X=\klbin(p_0\Vert p_1)>0\).

Taking \(\bar\vartheta:=(p_0+p_1)/2\in(p_0,p_1)\), both
\(D_0(\bar\vartheta),D_1(\bar\vartheta)\) are positive constants, so
\begin{equation}\label{eq:V-upper-Theta-L}
    V(u,v)
    \le\max\!\left\{\frac{u}{D_0(\bar\vartheta)},
       \frac{v}{D_1(\bar\vartheta)}\right\}
    \le C_V\max\{u,v\}
\end{equation}
for a finite constant \(C_V\) depending only on \(p_0,p_1\).
Also,
\begin{equation}\label{eq:V-lower-J}
    V(u,v)\ge\max\!\left\{\frac{u}{J^{(1)}_X},\frac{v}{J^{(0)}_X}\right\},
\end{equation}
because \(D_0(\vartheta)\le J^{(1)}_X\) and
\(D_1(\vartheta)\le J^{(0)}_X\) on \((p_0,p_1)\).
Together, \eqref{eq:V-upper-Theta-L} and \eqref{eq:V-lower-J}
imply \(V(u_k,v_k)=\Theta(L_k)\) since \(L_k=\max\{u_k,v_k\}\to\infty\).

\paragraph{Step 2: Chernoff upper bound.}
We prove the bound
\begin{equation}\label{eq:Nfull-upper}
    N_{\fixed,\Hum}(e^{-u_k},e^{-v_k})\le V(u_k,v_k)+1.
\end{equation}

For
\(\vartheta\in(p_0,p_1)\), Chernoff's
bound for Bernoulli sums gives
\begin{equation}\label{eq:chernoff-twosided}
    \Pp_{p_0}\!\left(\bar X_n\ge \vartheta\right)
    \le e^{-nD_0(\vartheta)},
    \qquad
    \Pp_{p_1}\!\left(\bar X_n\le \vartheta\right)
    \le e^{-nD_1(\vartheta)}.
\end{equation}

The infimum defining \(V(u_k,v_k)\) in \eqref{eq:V-def-identity} is
attained. Indeed,
\(f(\vartheta):=\max\{u_k/D_0(\vartheta),v_k/D_1(\vartheta)\}\) is
continuous on \((p_0,p_1)\); as \(\vartheta\downarrow p_0\) we have
\(D_0(\vartheta)\downarrow0\), and as \(\vartheta\uparrow p_1\) we have
\(D_1(\vartheta)\downarrow0\), so \(f(\vartheta)\to+\infty\) at both
endpoints. Extending \(f\) by \(+\infty\) to
the endpoints makes it lower semicontinuous and coercive on the
compact interval \([p_0,p_1]\); it therefore attains its minimum,
necessarily at an interior point
\(\vartheta^\star\in(p_0,p_1)\), and that
minimum value is \(V(u_k,v_k)\) by \eqref{eq:V-def-identity}.

Define \(n_k:=\lceil V(u_k,v_k)\rceil\). Then
\(n_k\ge V(u_k,v_k)\ge u_k/D_0(\vartheta^\star)\) and
\(n_k\ge V(u_k,v_k)\ge v_k/D_1(\vartheta^\star)\), so
\(n_kD_0(\vartheta^\star)\ge u_k\) and
\(n_kD_1(\vartheta^\star)\ge v_k\).

By the definition of \(N_{\fixed,\Hum}\) (the smallest sample size at
which \emph{some} test on i.i.d.\ Bernoulli labels achieves the
target errors), to prove
\(N_{\fixed,\Hum}(\alpha_k,\beta_k)\le n_k\) it suffices to exhibit one
feasible test on \(n_k\) labels. We use the sample-mean test that
rejects \(H_0\) when \(\bar X_{n_k}\ge \vartheta^\star\). By
\eqref{eq:chernoff-twosided},
\[
    \Pp_{p_0}(\bar X_{n_k}\ge \vartheta^\star)
    \le e^{-n_kD_0(\vartheta^\star)}\le e^{-u_k}=\alpha_k,
\]
\[
    \Pp_{p_1}(\bar X_{n_k}\le \vartheta^\star)
    \le e^{-n_kD_1(\vartheta^\star)}\le e^{-v_k}=\beta_k.
\]
This test is feasible at \((\alpha_k,\beta_k)\), so
\[
    N_{\fixed,\Hum}(\alpha_k,\beta_k)\le n_k=\lceil V(u_k,v_k)\rceil
    \le V(u_k,v_k)+1 .
\]

\paragraph{Step 3: type-class point-mass lower bound.}
We prove the elementary inequality: For \(j\in\{0,\ldots,n\}\),
\begin{equation}\label{eq:point-mass}
    \Pp_p(\sum_{i=1}^n X_i=j)
    \ge\frac1{n+1}\exp\{-n\klbin(j/n\Vert p)\}.
\end{equation}
Let \(h(\vartheta):=-\vartheta\log\vartheta
-(1-\vartheta)\log(1-\vartheta)\) (binary entropy, with
\(0\log 0=0\)). It suffices to prove
\begin{equation}\label{eq:binom-coeff-lower}
    \binom{n}{j}\ge\frac1{n+1}e^{nh(j/n)}.
\end{equation}
Indeed, multiplying \eqref{eq:binom-coeff-lower} by
\(p^j(1-p)^{n-j}\) gives
\[
    \Pp_p(\sum_{i=1}^n X_i=j)
    =\binom{n}{j}p^j(1-p)^{n-j}
    \ge\frac1{n+1}\exp\bigl\{n h(j/n)+j\log p+(n-j)\log(1-p)\bigr\}.
\]
Writing \(\vartheta:=j/n\), the exponent inside the braces becomes
\(n[h(\vartheta)+\vartheta\log p+(1-\vartheta)\log(1-p)]\), and a
direct expansion of
\(\klbin(\vartheta\Vert p)
=\vartheta\log(\vartheta/p)
+(1-\vartheta)\log((1-\vartheta)/(1-p))\) gives the identity
\(h(\vartheta)+\vartheta\log p+(1-\vartheta)\log(1-p)
=-\klbin(\vartheta\Vert p)\).
Substituting this yields \eqref{eq:point-mass}.

The proof of \eqref{eq:binom-coeff-lower} uses an auxiliary
Bernoulli computation with parameter \(\vartheta:=j/n\). Let
\(Y_1,\ldots,Y_n\) be i.i.d.\ \(\Bern(\vartheta)\), and let
\(T_n:=\sum_i Y_i\), so
\(T_n\sim\operatorname{Binomial}(n,\vartheta)\).
We will show that \(\Pp(T_n=j)\ge 1/(n+1)\), and then rearrange
this to get \eqref{eq:binom-coeff-lower}.

\emph{Mode at \(j\).} The ratio of consecutive binomial masses is
\[
    \frac{\Pp(T_n=r+1)}{\Pp(T_n=r)}
    =\frac{n-r}{r+1}\cdot\frac{\vartheta}{1-\vartheta},
    \qquad r\in\{0,1,\ldots,n-1\}.
\]
Solving \((n-r)\vartheta\ge(r+1)(1-\vartheta)\) gives
\(r\le\vartheta(n+1)-1\), and with \(\vartheta=j/n\), the bound
becomes \(r\le j-1+j/n\). For integer \(r\)
and \(j\in\{1,\ldots,n-1\}\), this is equivalent to \(r\le j-1\).
Hence the masses are nondecreasing as \(r\) goes from \(0\) to
\(j\), and strictly decreasing as \(r\) goes from \(j\) to \(n\).
So \(\Pp(T_n=j)\) is the maximum of the \(n+1\) masses
\(\Pp(T_n=0),\ldots,\Pp(T_n=n)\).

\emph{Lower bound on the mode.} The \(n+1\) binomial masses are
nonnegative and sum to \(1\). The maximum of \(n+1\) nonnegative
numbers summing to \(1\) is at least \(1/(n+1)\), so
\[
    \Pp(T_n=j)
    =\binom{n}{j}\vartheta^j(1-\vartheta)^{n-j}
    \ge\frac1{n+1}.
\]

\emph{Rearrangement.} Using the identity
\(\vartheta^j(1-\vartheta)^{n-j}
=\exp\{j\log\vartheta+(n-j)\log(1-\vartheta)\}
=\exp\{-nh(\vartheta)\}\)
(valid for \(\vartheta=j/n\) by direct expansion of \(h\)), the previous
display gives
\[
    \binom{n}{j}\ge\frac{1}{n+1}\cdot\exp\{nh(\vartheta)\},
\]
which is \eqref{eq:binom-coeff-lower}.

\paragraph{Step 4: type-class incompatibility for feasible tests.}
We prove the following fact, used in Step 5 to derive the lower
bound on \(N_{\fixed,\Hum}\): if there exists a test on \(n\) Bernoulli
labels with type-I error at most \(e^{-u}\) and type-II error at
most \(e^{-v}\), then for every \(j\in\{0,\ldots,n\}\),
\begin{equation}\label{eq:type-incompat}
    \Pp_{p_0}(\sum_{i=1}^n X_i=j)\le 2e^{-u}
    \quad\text{or}\quad
    \Pp_{p_1}(\sum_{i=1}^n X_i=j)\le 2e^{-v}.
\end{equation}

Fix \(n,u,v\) and suppose there is a (possibly randomized) test
\(\phi:\{0,1\}^n\to[0,1]\), where \(\phi(x)\) denotes the
conditional probability of rejecting \(H_0\) given the observation
\(x\), with type-I error at most \(e^{-u}\) under \(\Bern(p_0)\)
and type-II error at most \(e^{-v}\) under \(\Bern(p_1)\). For each
\(j\in\{0,\ldots,n\}\), define
\[
    \bar\phi(j)
    :=\frac{1}{\binom{n}{j}}\sum_{x:\sum_i x_i=j}\phi(x),
\]
the average of \(\phi(x)\) over the type class
\(\{x\in\{0,1\}^n:\sum_i x_i=j\}\).

Conditional on \(\sum_{i=1}^n X_i=j\), the label vector \((X_1,\ldots,X_n)\) is
uniform on the type class under either Bernoulli law. Indeed,
\(\Pp_p(X_1=x_1,\ldots,X_n=x_n)=p^j(1-p)^{n-j}\) for every
\(x\in\{0,1\}^n\) with \(\sum_i x_i=j\), so all such \(x\) are
equiprobable. Since \(\bar\phi(j)\) is the average of \(\phi(x)\)
over the type class,
\[
    \E_p[\phi(X)\mid \sum_{i=1}^n X_i=j]
    =\bar\phi(j)
    \qquad\text{under either }p\in\{p_0,p_1\}.
\]
By the tower property,
\[
    \E_{p_0}[\phi(X)]
    =\E_{p_0}\!\left[\E_{p_0}[\phi(X)\mid \sum_{i=1}^n X_i]\right]
    =\E_{p_0}[\bar\phi(\sum_{i=1}^n X_i)]
    =\sum_{j=0}^n\Pp_{p_0}(\sum_{i=1}^n X_i=j)\bar\phi(j),
\]
and similarly
\[
    \E_{p_1}[1-\phi(X)]
    =\sum_{j=0}^n\Pp_{p_1}(\sum_{i=1}^n X_i=j)(1-\bar\phi(j)).
\]
Since \(\bar\phi(j)\in[0,1]\), each summand in these two sums is
nonnegative. The type-I error bound \(\E_{p_0}[\phi(X)]\le e^{-u}\)
then forces each individual summand to be at most \(e^{-u}\):
\[
    \Pp_{p_0}(\sum_{i=1}^n X_i=j)\bar\phi(j)\le e^{-u}
    \qquad\text{for every }j\in\{0,\ldots,n\}.
\]
Similarly, \(\Pp_{p_1}(\sum_{i=1}^n X_i=j)(1-\bar\phi(j))\le e^{-v}\) for every
\(j\). Now suppose, for some \(j\), both
\(\Pp_{p_0}(\sum_{i=1}^n X_i=j)>2e^{-u}\) and \(\Pp_{p_1}(\sum_{i=1}^n X_i=j)>2e^{-v}\). Then
the displayed bounds force
\[
    \bar\phi(j)<\frac{e^{-u}}{2e^{-u}}=\frac12,
    \qquad
    1-\bar\phi(j)<\frac{e^{-v}}{2e^{-v}}=\frac12,
\]
which together give \(\bar\phi(j)+(1-\bar\phi(j))<1\), a
contradiction. Hence every feasible test satisfies, for every
\(j\),
\begin{equation}
    \Pp_{p_0}(\sum_{i=1}^n X_i=j)\le 2e^{-u}
    \quad\text{or}\quad
    \Pp_{p_1}(\sum_{i=1}^n X_i=j)\le 2e^{-v}.
\end{equation}

\paragraph{Step 5: lower bound via an \(\eta\)-free shifted estimate.}
Let \(n_k:=N_{\fixed,\Hum}(e^{-u_k},e^{-v_k})\). By Step 2 and
\eqref{eq:V-upper-Theta-L}, \(n_k=O(L_k)\), so there is a
constant \(C_N\) with \(n_k\le C_N L_k\) for all large \(k\).
We prove
\begin{equation}\label{eq:Nfull-lower-balanced}
    n_k\ge V(u_k,v_k)-O(\log L_k)
    \quad\text{as }k\to\infty.
\end{equation}
Let \(m_k:=\min\{u_k,v_k\}\). Then $m_k = \Theta(L_k)$ by
Assumption~\ref{ass:balanced-error-regime}.

\vspace{1mm}
\emph{Setup.} Let
\(I:=[p_0/2,(1+p_1)/2]\subset(0,1)\). Since \(D_0,D_1\) are
continuously differentiable on the interior of \([0,1]\), their
derivatives are bounded on \(I\); let \(B_D<\infty\) satisfy
\(|D_0'(\vartheta)|\le B_D\) and
\(|D_1'(\vartheta)|\le B_D\) for all \(\vartheta\in I\).
For each \(k\) set
\[
    \delta_k:=B_D+1+\log\{2(n_k+1)\}.
\]
Since \(n_k\le C_N L_k\), \(\delta_k=O(\log L_k)\). Since
\(m_k=\Theta(L_k)\),
\(\delta_k/m_k=O((\log L_k)/L_k)\to 0\), so
\(\delta_k<m_k=\min\{u_k,v_k\}\) for all large \(k\); in
particular \(u_k-\delta_k>0\) and \(v_k-\delta_k>0\). We also note
that \(n_k\to\infty\): by Lemma~\ref{lem:error-to-kl} applied to a
feasible test on \(n_k\) labels,
\(n_k J^{(1)}_X\ge\klbin(1-\beta_k\Vert\alpha_k)\), and the right-hand side
tends to infinity as \(\alpha_k,\beta_k\to 0\); this makes the grid
spacing \(1/(2n_k)\to0\) used in the Discretizing paragraph below.

\vspace{1mm}
\emph{An \(\eta\)-free shifted lower bound.} The core estimate is
\begin{equation}\label{eq:nk-ge-Vk-minus-etaLk}
    n_k\ge V(u_k-\delta_k,v_k-\delta_k)
    \qquad\text{for all sufficiently large }k.
\end{equation}
We prove \eqref{eq:nk-ge-Vk-minus-etaLk} by contradiction: suppose it
fails along an infinite subsequence \(\{k_\ell\}_{\ell\ge1}\), so that
\begin{equation}\label{eq:n-lt-V-minus}
    n_{k_\ell}<V(u_{k_\ell}-\delta_{k_\ell},v_{k_\ell}-\delta_{k_\ell}).
\end{equation}
We will derive a contradiction by exhibiting, for each \(\ell\)
large, a type class \(j_{k_\ell}\in\{0,1,\ldots,n_{k_\ell}\}\)
for which
\(\Pp_{p_0}(\sum^{n_{k_\ell}}_{i=1} X_i=j_{k_\ell})>2e^{-u_{k_\ell}}\) and
\(\Pp_{p_1}(\sum^{n_{k_\ell}}_{i=1} X_i=j_{k_\ell})>2e^{-v_{k_\ell}}\),
violating the type-class incompatibility
\eqref{eq:type-incompat}.

\vspace{1mm}
\emph{Finding a good \(\vartheta_{k_\ell}\in(p_0,p_1)\).} We claim that
\eqref{eq:n-lt-V-minus} implies the existence of
\(\vartheta_{k_\ell}\in(p_0,p_1)\) with
\begin{equation}\label{eq:vartheta-small-divs}
    n_{k_\ell} D_0(\vartheta_{k_\ell})
       <u_{k_\ell}-\delta_{k_\ell},
    \qquad
    n_{k_\ell} D_1(\vartheta_{k_\ell})
       <v_{k_\ell}-\delta_{k_\ell}.
\end{equation}
The strategy is to define two open subsets of \((p_0,p_1)\), one
where the first inequality of \eqref{eq:vartheta-small-divs} holds and
one where the second holds, show that they cover \((p_0,p_1)\),
and use connectedness to conclude they must overlap.

By the definition of \(V\) as an infimum,
\(n_{k_\ell}<V(u_{k_\ell}-\delta_{k_\ell},v_{k_\ell}-\delta_{k_\ell})\)
means: for every \(\vartheta\in(p_0,p_1)\),
\[
    \max\!\left\{
      \frac{u_{k_\ell}-\delta_{k_\ell}}{D_0(\vartheta)},
      \frac{v_{k_\ell}-\delta_{k_\ell}}{D_1(\vartheta)}
    \right\}>n_{k_\ell}.
\]
This is equivalent to: for every \(\vartheta\in(p_0,p_1)\),
\(D_0(\vartheta)<(u_{k_\ell}-\delta_{k_\ell})/n_{k_\ell}\) or
\(D_1(\vartheta)<(v_{k_\ell}-\delta_{k_\ell})/n_{k_\ell}\). Define the
two sets
\[
    U:=\left\{\vartheta\in(p_0,p_1):
       D_0(\vartheta)<
       \frac{u_{k_\ell}-\delta_{k_\ell}}{n_{k_\ell}}\right\},
\]
\[
    V':=\left\{\vartheta\in(p_0,p_1):
       D_1(\vartheta)<
       \frac{v_{k_\ell}-\delta_{k_\ell}}{n_{k_\ell}}\right\}.
\]
The displayed equivalence above says exactly
\(U\cup V'=(p_0,p_1)\). Any point
\(\vartheta_{k_\ell}\in U\cap V'\) satisfies both inequalities in
\eqref{eq:vartheta-small-divs}; we now
show \(U\cap V'\neq\emptyset\).

The set \(U\) is open in \((p_0,p_1)\) because \(D_0\) is
continuous; it contains a right-neighborhood of \(p_0\) because
\(D_0(p_0)=0<(u_{k_\ell}-\delta_{k_\ell})/n_{k_\ell}\) and \(D_0\)
is continuous at \(p_0\), so
\(D_0(\vartheta)<(u_{k_\ell}-\delta_{k_\ell})/n_{k_\ell}\)
for \(\vartheta\) sufficiently close to \(p_0\). In particular, \(U\) is
nonempty. Similarly, \(V'\) is open and contains a
left-neighborhood of \(p_1\) (since \(D_1(p_1)=0\)), so \(V'\) is
nonempty.

If \(U\cap V'=\emptyset\), then \(U\) and \(V'\) would be two
disjoint nonempty open subsets of \((p_0,p_1)\) whose union is
\((p_0,p_1)\). This contradicts connectedness of the interval
\((p_0,p_1)\). Hence \(U\cap V'\neq\emptyset\); any
\(\vartheta_{k_\ell}\in U\cap V'\) satisfies
\eqref{eq:vartheta-small-divs}.

\vspace{1mm}
\emph{Discretizing.} The point-mass lower bound \eqref{eq:point-mass}
is stated for binomial point masses at integer values \(j\). To
apply it, we replace \(\vartheta_{k_\ell}\) by a nearby grid point of the
form \(j/n_{k_\ell}\).

Choose \(j_{k_\ell}\in\{0,1,\ldots,n_{k_\ell}\}\) with
\(|j_{k_\ell}/n_{k_\ell}-\vartheta_{k_\ell}|
\le 1/(2n_{k_\ell})\) (rounding
\(\vartheta_{k_\ell}\cdot n_{k_\ell}\) to the nearest integer), and write
\(\widetilde\vartheta_{k_\ell}:=j_{k_\ell}/n_{k_\ell}\). Since
\(\vartheta_{k_\ell}\in(p_0,p_1)\), its distance to the boundary of \(I\)
(at \(p_0/2\) and \((1+p_1)/2\)) is bounded below by the positive
constant \(\min\{p_0/2,(1-p_1)/2\}\). Since
\(1/(2n_{k_\ell})\to 0\) (because \(n_{k_\ell}\to\infty\)), for
all \(\ell\) large,
\(|\widetilde\vartheta_{k_\ell}-\vartheta_{k_\ell}|
<\min\{p_0/2,(1-p_1)/2\}\), so both
\(\vartheta_{k_\ell}\) and \(\widetilde\vartheta_{k_\ell}\) lie in \(I\).

By the mean value theorem applied to \(D_h\) on \(I\) and the
derivative bound \(|D_h'(\vartheta)|\le B_D\) for \(\vartheta\in I\),
\[
    |D_h(\widetilde\vartheta_{k_\ell})
      -D_h(\vartheta_{k_\ell})|
    \le B_D\cdot
       |\widetilde\vartheta_{k_\ell}-\vartheta_{k_\ell}|
    \le\frac{B_D}{2n_{k_\ell}},\qquad h\in\{0,1\}.
\]
Multiplying by \(n_{k_\ell}\) gives
\(n_{k_\ell}|D_h(\widetilde\vartheta_{k_\ell})
-D_h(\vartheta_{k_\ell})|\le B_D/2\), so
\(n_{k_\ell}D_h(\widetilde\vartheta_{k_\ell})
\le n_{k_\ell}D_h(\vartheta_{k_\ell})+B_D/2\).
Combining with the strict inequalities in
\eqref{eq:vartheta-small-divs} and substituting
\(\delta_{k_\ell}=B_D+1+\log\{2(n_{k_\ell}+1)\}\),
\begin{align*}
    n_{k_\ell}D_0(\widetilde\vartheta_{k_\ell})
    &<u_{k_\ell}-\delta_{k_\ell}+\frac{B_D}{2}\\
    &=u_{k_\ell}-\bigl(B_D+1+\log\{2(n_{k_\ell}+1)\}\bigr)+\frac{B_D}{2}\\
    &=u_{k_\ell}-\log\{2(n_{k_\ell}+1)\}-1-\frac{B_D}{2}\\
    &<u_{k_\ell}-\log\{2(n_{k_\ell}+1)\},
\end{align*}
where the last strict inequality uses \(-1-B_D/2<0\) (since
\(B_D\ge 0\)). Symmetrically,
\[
    n_{k_\ell}D_1(\widetilde\vartheta_{k_\ell})
    <v_{k_\ell}-\log\{2(n_{k_\ell}+1)\}.
\]

\vspace{1mm}
\emph{Contradiction.} By the point-mass bound
\eqref{eq:point-mass} applied with \(p=p_0\), \(n=n_{k_\ell}\),
\(j=j_{k_\ell}\),
\[
    \Pp_{p_0}(\sum^{n_{k_\ell}}_{i=1} X_i=j_{k_\ell})
    \ge\frac1{n_{k_\ell}+1}
       e^{-n_{k_\ell}D_0(\widetilde\vartheta_{k_\ell})}.
\]
The strict inequality
\(n_{k_\ell}D_0(\widetilde\vartheta_{k_\ell})
<u_{k_\ell}-\log\{2(n_{k_\ell}+1)\}\)
from the Discretizing paragraph gives, after negation,
\(-n_{k_\ell}D_0(\widetilde\vartheta_{k_\ell})
>-u_{k_\ell}+\log\{2(n_{k_\ell}+1)\}\).
Exponentiating,
\[
    e^{-n_{k_\ell}D_0(\widetilde\vartheta_{k_\ell})}
    >e^{-u_{k_\ell}+\log\{2(n_{k_\ell}+1)\}}
    =2(n_{k_\ell}+1)\,e^{-u_{k_\ell}}.
\]
Dividing by \(n_{k_\ell}+1\),
\[
    \frac1{n_{k_\ell}+1}
       e^{-n_{k_\ell}D_0(\widetilde\vartheta_{k_\ell})}
    >2e^{-u_{k_\ell}}.
\]
Combining with the point-mass bound,
\(\Pp_{p_0}(\sum^{n_{k_\ell}}_{i=1} X_i=j_{k_\ell})>2e^{-u_{k_\ell}}\).
Symmetrically,
\(\Pp_{p_1}(\sum^{n_{k_\ell}}_{i=1} X_i=j_{k_\ell})>2e^{-v_{k_\ell}}\).

By the definition of
\(n_{k_\ell}=N_{\fixed,\Hum}(e^{-u_{k_\ell}},e^{-v_{k_\ell}})\), some
test on \(n_{k_\ell}\) Bernoulli labels achieves errors at most
\((e^{-u_{k_\ell}},e^{-v_{k_\ell}})\). Applying the type-class
incompatibility \eqref{eq:type-incompat} to this test at
\(j=j_{k_\ell}\) yields either
\(\Pp_{p_0}(\sum^{n_{k_\ell}}_{i=1} X_i=j_{k_\ell})\le 2e^{-u_{k_\ell}}\) or
\(\Pp_{p_1}(\sum^{n_{k_\ell}}_{i=1} X_i=j_{k_\ell})\le 2e^{-v_{k_\ell}}\),
contradicting both strict inequalities we just derived. The
contradiction shows the contradiction hypothesis was false, so
\eqref{eq:nk-ge-Vk-minus-etaLk} holds.

\vspace{1mm}
\emph{From the shifted bound to \eqref{eq:Nfull-lower-balanced}.}
We have \(\delta_k<m_k\) and \(u_k,v_k\ge m_k\), so
\[
    u_k-\delta_k\ge\Bigl(1-\frac{\delta_k}{m_k}\Bigr)u_k,
    \qquad
    v_k-\delta_k\ge\Bigl(1-\frac{\delta_k}{m_k}\Bigr)v_k .
\]
By the positive homogeneity \(V(cu,cv)=cV(u,v)\) and coordinatewise
monotonicity of \(V\),
\[
    V(u_k-\delta_k,v_k-\delta_k)
    \ge\Bigl(1-\frac{\delta_k}{m_k}\Bigr)V(u_k,v_k),
\]
so, using \(V(u_k,v_k)\le C_V L_k\) from \eqref{eq:V-upper-Theta-L},
\[
    0\le V(u_k,v_k)-V(u_k-\delta_k,v_k-\delta_k)
    \le\frac{\delta_k}{m_k}V(u_k,v_k)
    \le\frac{\delta_k}{m_k}\,C_V L_k .
\]
Combining with \eqref{eq:nk-ge-Vk-minus-etaLk},
\[
    n_k\ge V(u_k-\delta_k,v_k-\delta_k)
    \ge V(u_k,v_k)-\frac{\delta_k}{m_k}\,C_V L_k .
\]
Since \(\delta_k=O(\log L_k)\) and \(m_k=\Theta(L_k)\), the last term is
\(O((\log L_k)L_k/L_k)=O(\log L_k)\). This proves
\eqref{eq:Nfull-lower-balanced}, which together with the upper bound
\eqref{eq:Nfull-upper} proves \eqref{eq:Nfull-V-asymp-balanced}.
\end{proof}

\begin{proof}{Proof of Lemma~\ref{lem:safe-floor-first-order}}

Part (i) directly follows from the proof of Theorem~\ref{thm:first-order}(i). We thus omit its proof. We prove Part (ii) now.
Let \(\alpha_k,\beta_k\downarrow 0\) with
\(u_k,v_k,L_k\) as in Lemma~\ref{lem:Nfull-V-identity}, and write
\(m_k:=\min\{u_k,v_k\}=\Theta(L_k)\).   Let \(c_k:=\log(1/f_{\fb,k})\ge0\), so \(\alpha_{2,k}=e^{-(u_k+c_k)}\)
and \(\beta_{2,k}=e^{-(v_k+c_k)}\). Adding the common shift \(c_k\) to
both coordinates preserves balance:
\(\min\{u_k+c_k,v_k+c_k\}=m_k+c_k=\Theta(L_k+c_k)\), so the strengthened
Lemma~\ref{lem:Nfull-V-identity} (the \(O(\log\cdot)\) form
\eqref{eq:Nfull-V-asymp-balanced}) applies at both levels:
\begin{equation}\label{eq:Nfull-at-original}
    N_{\fixed,\Hum}(\alpha_k,\beta_k)=V(u_k,v_k)+O(\log L_k),
\end{equation}
\begin{equation}\label{eq:Nfull-at-halved}
    N_{\fixed,\Hum}(\alpha_{2,k},\beta_{2,k})=V(u_k+c_k,v_k+c_k)+O(\log(L_k+c_k)).
\end{equation}
Subtracting,
\[
    N_{\fixed,\Hum}(\alpha_{2,k},\beta_{2,k})-N_{\fixed,\Hum}(\alpha_k,\beta_k)
    =\bigl[V(u_k+c_k,v_k+c_k)-V(u_k,v_k)\bigr]
     +O(\log(L_k+c_k)).
\]

\emph{Frontier shift.}
For every \(\vartheta\in(p_0,p_1)\), using \(u_k\ge m_k\) and
\(v_k\ge m_k\),
\[
    u_k+c_k\le\left(1+\frac{c_k}{m_k}\right)u_k,
    \qquad
    v_k+c_k\le\left(1+\frac{c_k}{m_k}\right)v_k.
\]
Dividing by \(D_0(\vartheta)>0\) and \(D_1(\vartheta)>0\), taking the
maximum, and then the infimum over \(\vartheta\in(p_0,p_1)\) (the
positive factor \(1+c_k/m_k\)
pulls out of the infimum),
\[
    V(u_k+c_k,v_k+c_k)\le\left(1+\frac{c_k}{m_k}\right)V(u_k,v_k).
\]
By coordinatewise monotonicity of \(V\) (immediate from
\eqref{eq:V-def-identity}, since \(u_k+c_k\ge u_k\) and
\(v_k+c_k\ge v_k\)), \(V(u_k+c_k,v_k+c_k)\ge V(u_k,v_k)\). Combining,
\begin{equation}\label{eq:V-perturb-final}
    0\le V(u_k+c_k,v_k+c_k)-V(u_k,v_k)\le\frac{c_k}{m_k}V(u_k,v_k).
\end{equation}

\vspace{1mm}
\emph{Conclusion.}
By Lemma~\ref{lem:Nfull-V-identity}, \(V(u_k,v_k)=\Theta(L_k)\); hence the
right-hand side of \eqref{eq:V-perturb-final} is
\(\frac{c_k}{m_k}V(u_k,v_k)=O(c_k)=O(\log(1/f_{\fb,k}))\). Combining with
the subtraction display and
\(O(\log(L_k+c_k))=O(\log L_k+\log(1/f_{\fb,k}))\),
\[
    0\le N_{\fixed,\Hum}(\alpha_{2,k},\beta_{2,k})-N_{\fixed,\Hum}(\alpha_k,\beta_k)
    \le\frac{\log(1/f_{\fb,k})}{m_k}V(u_k,v_k)+O(\log L_k)
    =O\!\left(\log L_k+\log\left(\frac{1}{f_{\fb, k}}\right)\right),
\]
which is \eqref{eq:safe-floor-result}. The lower bound holds because
\(N_{\fixed,\Hum}\) is nonincreasing in each error budget and
\(\alpha_{2,k}=f_{\fb,k}\alpha_k<\alpha_k\),
\(\beta_{2,k}=f_{\fb,k}\beta_k<\beta_k\).
\end{proof}

\subsection{Proof of Lemma~\ref{lem:direction-tracking}}\label{appendix-sec:direction-tracking}
We define some notations that will be used throughout Appendices~\ref{appendix-sec:direction-tracking} - \ref{appendix-sec:corollary-proof}. 
Let $Z_{j,k}$ denote the one-item
$H_1$-versus-$H_0$ log-likelihood increment produced by rule
$j\in\{0,1,*\}$ on a fresh item before boundary stopping. According to 
Algorithm~\ref{algo:seq-policy}, under rule $j$, the item is first sent to
a human with probability $\eta^{\Hum}_{j, k}$ and otherwise sent to the AI, after which
it is escalated to a human with probability $\eta^{\esc}_{j, k}(R)$.  Writing
$U,V\sim\mathrm{Uniform}[0,1]$ for the two independent randomization seeds, the
increment is
\begin{equation}\label{eq:Zjk-def}
    Z_{j,k}
    :=\ind\{U\le\eta^{\Hum}_{j, k}\}\,\ell_X(X)
    +\ind\{U>\eta^{\Hum}_{j, k}\}\Bigl(
        \ell_R(R)
        +\ind\{V\le\eta^{\esc}_{j, k}(R)\}\,\ell_H(X,R)
    \Bigr),
\end{equation}
where $\ell_X,\ell_R,\ell_H$ are the increments
\eqref{eq:direct-human-increment}--\eqref{eq:escalation-increment}.  
% The first
% term is the direct-human contribution, and the second is the AI-report
% contribution $\ell_R(R)$ together with the escalation contribution
% $\ell_H(X,R)$ that is added only when the item is escalated.  Under $H_h$ the
% report and label satisfy $R\sim G_h$, with $X\sim\Bern(p_h)$ on the
% direct-human branch and $X\mid\{R=r\}\sim\Bern(q_h(r))$ on the escalation
% branch, while $U,V$ have the same law under both hypotheses.  
Write
$\sigma_1:=+1$ and $\sigma_0:=-1$, so the signed drift under $H_h$ is
$\E^{\pi_k}_h[\sigma_hZ_{j,k}]$.

Throughout Appendices~\ref{appendix-sec:direction-tracking} - \ref{appendix-sec:fallback-cost}, we index the process at item-completion times.
Work with the completed-item continuation on which the fresh tuple
$(X_i,R_i,U_i,V_i)$ is generated for every acquired item and each AI-first
item is completed according to the escalation rule selected at its start; if
the actual policy has already stopped, this continuation is only
counterfactual. With some abuse of notations we let $\calH_0$ be the trivial sigma-field and let $\calH_i$ be
the sigma-field generated by the sensing decisions, randomization seeds, and
observations through the completion of item $i$.  Set $S_0:=0$ and recursively
choose $J_i$ from $S_{i-1}$ by \eqref{eq:rule-by-sign}, and define
\begin{equation*}
\begin{aligned}
    Y_i
    :={}&\ind\{U_i\le\eta^{\Hum}_{J_i,k}\}\,\ell_X(X_i)+\ind\{U_i>\eta^{\Hum}_{J_i,k}\}
      \Bigl(
        \ell_R(R_i)
        +\ind\{V_i\le\eta^{\esc}_{J_i,k}(R_i)\}\,\ell_H(X_i,R_i)
      \Bigr).
\end{aligned}
\end{equation*}
Note that the conditional law
of $Y_i$ given $\calH_{i-1}$ is exactly the unconditional law of the
single-item increment $Z_{J_i,k}$ under the (now fixed) rule $J_i$.
With some abuse of notations we let $S_i:=\sum_{\ell=1}^{i}Y_\ell$ denote the
cumulative log-likelihood statistic after completing the first i items. 
Because $J_i$ is $\calH_{i-1}$-measurable and the tuple for item $i$ is fresh,
the conditional law of $Y_i$ given $\calH_{i-1}$ is the law of the generic
increment $Z_{J_i,k}$.  In particular, $S_i$ is the cumulative log-likelihood
ratio at the end of item $i$, while $S_{i-1}$ and $J_i$ are
$\calH_{i-1}$-measurable.  Under $H_0$ we apply the same item-level
construction to the reflected increments $-Y_i$ and the reflected statistic
$-S_i$.

To prove Lemma~\ref{lem:direction-tracking}, we need auxiliary Lemma~\ref{lem:variance-drift}, and we present the proof of Lemma~\ref{lem:direction-tracking} in the end of this subsection.

\begin{lemma}[Finite-alphabet increment control and target-scale drift]
\label{lem:variance-drift}
There are primitive constants $B_{\mathrm{inc}}<\infty$, $c_\times>0$,
$\mu_0>0$, and
$C_V<\infty$ such that the following statements hold.

(i)
For every $k$, $j\in\{0,1,*\}$, and $h\in\{0,1\}$,
\begin{equation}\label{eq:increment-bound}
    |Z_{j,k}|\le B_{\mathrm{inc}}.
\end{equation}

(ii) For every $k$, $h\in\{0,1\}$, and $j\in\{0,1,*\}$,
\begin{equation}\label{eq:reverse-comparison}
    \E^{\pi_k}_h[\sigma_hZ_{j,k}]
    \ge c_\times\E^{\pi_k}_{1-h}[\sigma_{1-h}Z_{j,k}].
\end{equation}

(iii) For all sufficiently large $k$,
\begin{equation}\label{eq:mu0-def}
    \E^{\pi_k}_h[\sigma_hZ_{j,k}]\ge \mu_0
    \qquad
    (h\in\{0,1\},\ j\in\{0,1,*\}).
\end{equation}

(iv) For all sufficiently large $k$, $h\in\{0,1\}$, and
$j\in\{0,1,*\}$,
\begin{equation}\label{eq:variance-drift}
    \Var_h(Z_{j,k})\le C_V\E^{\pi_k}_h[\sigma_hZ_{j,k}].
\end{equation}
\end{lemma}

\begin{proof}{Proof of Lemma~\ref{lem:variance-drift}}
(i) We notice that $\ell_X(X), \ell_R(R), \ell_H(X, R)$ all have bounded absolute values since 
the support sets are finite and common under the two hypotheses, $p_0,p_1\in(0,1)$, and $g_1$ and $g_0$ are strictly positive on $\calR$.  This proves \eqref{eq:increment-bound}.

(ii) We establish the reverse comparison \eqref{eq:reverse-comparison}, starting
with a term-by-term derivation of the one-item signed expectation.
Taking
expectations in \eqref{eq:Zjk-def} termwise gives, for $h=1$,
\[
    \E^{\pi_k}_1[Z_{j,k}]
    =\eta^{\Hum}_{j,k}\,\E^{\pi_k}_1[\ell_X(X)]
     +(1-\eta^{\Hum}_{j,k})\,\E^{\pi_k}_1[\ell_R(R)]
     +(1-\eta^{\Hum}_{j,k})\sum_{r\in\calR}g_1(r)\,\eta^{\esc}_{j,k}(r)\,
        \E^{\pi_k}_1[\ell_H(X,r)\mid R=r],
\]
and, by the identical computation with $H_0$ in place of $H_1$,
\[
    \E^{\pi_k}_0[Z_{j,k}]
    =\eta^{\Hum}_{j,k}\,\E^{\pi_k}_0[\ell_X(X)]
     +(1-\eta^{\Hum}_{j,k})\,\E^{\pi_k}_0[\ell_R(R)]
     +(1-\eta^{\Hum}_{j,k})\sum_{r\in\calR}g_0(r)\,\eta^{\esc}_{j,k}(r)\,
        \E^{\pi_k}_0[\ell_H(X,r)\mid R=r].
\]
By \eqref{eq:direct-human-increment} and \eqref{eq:JX-def},
$\E^{\pi_k}_1[\ell_X(X)]=J_X^{(1)}$ and $\E^{\pi_k}_0[\ell_X(X)]=-J_X^{(0)}$; by
\eqref{eq:AI-report-increment} and \eqref{eq:IR1}--\eqref{eq:IR0},
$\E^{\pi_k}_1[\ell_R(R)]=I_R^{(1)}$ and $\E^{\pi_k}_0[\ell_R(R)]=-I_R^{(0)}$;
and by \eqref{eq:escalation-increment} and \eqref{eq:d1-def}--\eqref{eq:d0-def},
$\E^{\pi_k}_1[\ell_H(X,r)\mid R=r]=d^{(1)}(r)$ and
$\E^{\pi_k}_0[\ell_H(X,r)\mid R=r]=-d^{(0)}(r)$.  Substituting into the
two displays above and using $\sigma_1=1$ and $\sigma_0=-1$, the one-item
signed expectation of rule $j$ under $H_h$ is
\begin{equation}\label{eq:drift-expansion}
    \E^{\pi_k}_h[\sigma_hZ_{j,k}]
    =\eta^{\Hum}_{j, k}\,J_X^{(h)}
     +(1-\eta^{\Hum}_{j, k})\,I_R^{(h)}
     +(1-\eta^{\Hum}_{j, k})\sum_{r\in\calR}g_h(r)\,\eta^{\esc}_{j, k}(r)\,d^{(h)}(r),
    \qquad h\in\{0,1\},\ j\in\{0,1,*\}.
\end{equation}
The human probability $\eta^{\Hum}_{j, k}$ and the escalation rule $\eta^{\esc}_{j, k}$ are fixed
numbers that do not depend on the hypothesis, so
$\E^{\pi_k}_h[\sigma_hZ_{j,k}]$ and
$\E^{\pi_k}_{1-h}[\sigma_{1-h}Z_{j,k}]$ differ only
through the information coefficients and the report weights $g_h(r)$.

We compare \eqref{eq:drift-expansion} for $h$ against $1-h$ termwise.  Because
the primitives $(p_0,p_1,f_0,f_1)$ are fixed, there are only finitely many
component KL pairs: the direct-label pair $(J_X^{(1)},J_X^{(0)})$, the report
pair $(I_R^{(1)},I_R^{(0)})$, and the $|\calR|$ escalation pairs
$(d^{(1)}(r),d^{(0)}(r))$.  For any two laws $P,Q$ on a common finite support,
$D(P\Vert Q)=0$ if and only if $P=Q$, in which case $D(Q\Vert P)=0$ as well;
otherwise both are strictly positive.  Hence each component pair is either
zero in both directions or strictly positive in both, and
\[
    c_{\mathrm{cmp}}:=\min\Bigl\{
        \tfrac{J_X^{(h)}}{J_X^{(1-h)}},\,
        \tfrac{I_R^{(h)}}{I_R^{(1-h)}},\,
        \tfrac{d^{(h)}(r)}{d^{(1-h)}(r)}
        \ :\ h\in\{0,1\},\ r\in\calR,\ J_X^{(1-h)}, I_R^{(1-h)}, d^{(1-h)}(r)>0
    \Bigr\}
\]
is a strictly positive constant depending only on the primitives, and every
component obeys the two-sided bound
\begin{equation}\label{eq:component-comparison}
    J_X^{(h)}\ge c_{\mathrm{cmp}}\,J_X^{(1-h)},
    \quad
    I_R^{(h)}\ge c_{\mathrm{cmp}}\,I_R^{(1-h)},
    \quad
    d^{(h)}(r)\ge c_{\mathrm{cmp}}\,d^{(1-h)}(r),\quad 
    \forall h\in\{0,1\}, r\in \calR, 
\end{equation}
where each inequality is trivial
when its right-hand side vanishes.

The escalation term additionally carries the report weight $g_h(r)$, which
differs across hypotheses.  Since $g_0$ and $g_1$ are fixed and strictly
positive on the finite set $\calR$ (by \eqref{eq:gh-def} and
Assumption~\ref{ass:score-model}), their likelihood ratio is bounded: $m_g:=\min_{r\in\calR,\ h\in\{0,1\}}\frac{g_h(r)}{g_{1-h}(r)}\in(0,\infty)$.

Combining $g_h(r)\ge m_g\,g_{1-h}(r)\ge0$ with the escalation bound in
\eqref{eq:component-comparison}, and multiplying the two nonnegative-termed
inequalities together, for every $r\in\calR$
\[
    g_h(r)\,d^{(h)}(r)
    \ge m_g\,c_{\mathrm{cmp}}\;
      g_{1-h}(r)\,d^{(1-h)}(r);
\]
multiplying by $\eta^{\esc}_{j,k}(r)\in[0,1]$, summing over $r$ (both
operations preserve the inequality, since they act identically on both
sides), and then multiplying by the nonnegative weight $(1-\eta^{\Hum}_{j,k})$
gives
\[
    (1-\eta^{\Hum}_{j,k})\sum_{r\in\calR}g_h(r)\,\eta^{\esc}_{j,k}(r)\,d^{(h)}(r)
    \ge m_gc_{\mathrm{cmp}}\,(1-\eta^{\Hum}_{j,k})\sum_{r\in\calR}g_{1-h}(r)\,\eta^{\esc}_{j,k}(r)\,d^{(1-h)}(r).
\]
Additionally, multiplying \eqref{eq:component-comparison} $\eta^{\Hum}_{j, k}$ and $1-\eta^{\Hum}_{j, k}$ on both sides gives
\[
    \eta^{\Hum}_{j, k}\,J_X^{(h)}\ge c_{\mathrm{cmp}}\,\eta^{\Hum}_{j, k}\,J_X^{(1-h)},
    \qquad
    (1-\eta^{\Hum}_{j, k})\,I_R^{(h)}\ge c_{\mathrm{cmp}}\,(1-\eta^{\Hum}_{j, k})\,I_R^{(1-h)}.
\]

Set $c_\times:=c_{\mathrm{cmp}}\min\{1,m_g\}>0$.  Then we have 
\begin{align*}
    \E^{\pi_k}_h[\sigma_hZ_{j,k}]
    &\ge c_\times\Bigl[
        \eta^{\Hum}_{j,k}J_X^{(1-h)}
        +(1-\eta^{\Hum}_{j,k})I_R^{(1-h)}\\
    &\hspace{27mm}
        +(1-\eta^{\Hum}_{j,k})
        \sum_{r\in\calR}g_{1-h}(r)\eta^{\esc}_{j,k}(r)d^{(1-h)}(r)
        \Bigr]\\
    &=c_\times\,\E^{\pi_k}_{1-h}[\sigma_{1-h}Z_{j,k}].
\end{align*}
for $h\in\{0,1\},\ j\in\{0,1,*\},$ which is \eqref{eq:reverse-comparison}.  The argument applies verbatim to the
dead-zone rule $j=*$, since \eqref{eq:drift-expansion} holds for its parameters
$(\eta^{\Hum}_{*,k},\eta^{\esc}_{*,k})$ as well; the constant $c_\times$ depends only on the
primitives and not on the rule.

(iii)
We have that $\E^{\pi_k}_h[\sigma_hZ_{j,k}]
    =\eta^{\Hum}_{j, k}J_X^{(h)}
     +(1-\eta^{\Hum}_{j, k})I_R^{(h)}
     +(1-\eta^{\Hum}_{j, k})\sum_{r\in\calR}g_h(r)\eta^{\esc}_{j, k}(r)d^{(h)}(r)$.

\emph{Correct rule $j=h$.}  By construction \eqref{eq:eta-closed-form}, the
escalation rule $\eta^{\esc}_{h,k}$ is the fractional-knapsack maximizer in the dual
\eqref{eq:Psi-dual} at budget $s_{h,k}$, so it attains the frontier exactly:
\begin{equation}\label{eq:eta-attains}
    \sum_{r\in\calR}g_h(r)\eta^{\esc}_{h,k}(r)=s_{h,k},
    \qquad
    \sum_{r\in\calR}g_h(r)\eta^{\esc}_{h,k}(r)d^{(h)}(r)=\Psi_h(s_{h,k}).
\end{equation}
Multiplying the drift expansion by $\lceil n^*_{\Hum,h,k}\rceil+\lceil n^*_{\AI,h,k}\rceil$ and using
$\eta^{\Hum}_{h,k}=\lceil n^*_{\Hum,h,k}\rceil/(\lceil n^*_{\Hum,h,k}\rceil+\lceil n^*_{\AI,h,k}\rceil)$, $1-\eta^{\Hum}_{h,k}=\lceil n^*_{\AI,h,k}\rceil/(\lceil n^*_{\Hum,h,k}\rceil+\lceil n^*_{\AI,h,k}\rceil)$ from
\eqref{eq:phi-def}, together with \eqref{eq:eta-attains},
\[
    (\lceil n^*_{\Hum,h,k}\rceil+\lceil n^*_{\AI,h,k}\rceil)
    \E^{\pi_k}_h[\sigma_hZ_{h,k}]
    =\lceil n^*_{\Hum,h,k}\rceil J_X^{(h)}+\lceil n^*_{\AI,h,k}\rceil I_R^{(h)}+\lceil n^*_{\AI,h,k}\rceil \Psi_h(s_{h,k}).
\]
Given
$\lceil n^*_{\Hum,h,k}\rceil\ge n^*_{\Hum,h,k}$ and $\lceil n^*_{\AI,h,k}\rceil\ge n^*_{\AI,h,k}$, and $\Psi_h\ge0$, so
the right-hand side is at least
\[
    n^*_{\Hum,h,k}J_X^{(h)}+n^*_{\AI,h,k}I_R^{(h)}
    +n^*_{\AI,h,k}\Psi_h\!\Bigl(\tfrac{n^*_{\esc,h,k}}{n^*_{\AI,h,k}}\Bigr)
    \ge T_{h,k},
\]
the last inequality being the information constraint of $\Gamma_h$ in
\eqref{eq:Gamma-def}, met by the optimizer \eqref{eq:Gamma-optimizer} (if
$n^*_{\AI,h,k}=0$ then $s_{h,k}=0$, $\Psi_h(0)=0$, and the escalation term
vanishes on both sides).  Hence 
\begin{equation}\label{eq:correct-rule-target}
    (\lceil n^*_{\Hum,h,k}\rceil+\lceil n^*_{\AI,h,k}\rceil)
    \E^{\pi_k}_h[\sigma_hZ_{h,k}]\ge T_{h,k}.
\end{equation}
By construction, $\lceil n^*_{\Hum,h,k}\rceil+\lceil n^*_{\AI,h,k}\rceil\leq \lceil n^*_{\Hum,h,k} + n^*_{\AI,h,k}\rceil + 1 \leq N_{\main, k} = O(L_k)$. Additionally, $T_{h,k}=\Theta(L_k)$. It thus follows there is a
primitive constant $c_{\mathrm{drift}}>0$ such that, for all sufficiently large
$k$,
\begin{equation}\label{eq:correct-drift}
    \E^{\pi_k}_h[\sigma_hZ_{h,k}]
    \ge\frac{T_{h,k}}{\lceil n^*_{\Hum,h,k}\rceil+\lceil n^*_{\AI,h,k}\rceil}
    \ge c_{\mathrm{drift}}.
\end{equation}

\emph{Wrong rule $j=1-h$.}  The reverse comparison \eqref{eq:reverse-comparison}
gives
$\E^{\pi_k}_h[\sigma_hZ_{1-h,k}]
\ge c_\times\E^{\pi_k}_{1-h}[\sigma_{1-h}Z_{1-h,k}]$, and
\eqref{eq:correct-drift} applied to direction $1-h$ gives
$\E^{\pi_k}_{1-h}[\sigma_{1-h}Z_{1-h,k}]\ge c_{\mathrm{drift}}$.  Therefore
$\E^{\pi_k}_h[\sigma_hZ_{1-h,k}]\ge c_\times c_{\mathrm{drift}}$.

\emph{Dead-zone rule $j=*$.}  Its parameters \eqref{eq:deadzone-rule} are
$\eta^{\Hum}_{*,k}=\tfrac12(\eta^{\Hum}_{0,k}+\eta^{\Hum}_{1,k})$ and
$\eta^{\esc}_{*,k}=\tfrac12(\eta^{\esc}_{0,k}+\eta^{\esc}_{1,k})$, hence
\[
    \eta^{\Hum}_{*,k}\ge\tfrac12\eta^{\Hum}_{h,k},
    \qquad
    1-\eta^{\Hum}_{*,k}\ge\tfrac12(1-\eta^{\Hum}_{h,k}),
    \qquad
    \eta^{\esc}_{*,k}(r)\ge\tfrac12\eta^{\esc}_{h,k}(r).
\]
Substituting these inequalities into the
nonnegative terms of the drift expansion for
$\E^{\pi_k}_h[\sigma_hZ_{*,k}]$,
\[
    \E^{\pi_k}_h[\sigma_hZ_{*,k}]
    \ge\tfrac12\eta^{\Hum}_{h,k}J_X^{(h)}
      +\tfrac12(1-\eta^{\Hum}_{h,k})I_R^{(h)}
      +\tfrac14(1-\eta^{\Hum}_{h,k})\!\sum_{r\in\calR}g_h(r)\eta^{\esc}_{h,k}(r)d^{(h)}(r),
\]
where the escalation factor $\tfrac14$ uses
$(1-\eta^{\Hum}_{*,k})\eta^{\esc}_{*,k}(r)\ge\tfrac14(1-\eta^{\Hum}_{h,k})\eta^{\esc}_{h,k}(r)$
(the product of the two nonnegative-termed bounds above).  Since
$\tfrac12\ge\tfrac14$, the first two coefficients may be further weakened to
$\tfrac14$, so comparing with the expansion \eqref{eq:drift-expansion} for
$\E^{\pi_k}_h[\sigma_hZ_{h,k}]$,
\[
    \E^{\pi_k}_h[\sigma_hZ_{*,k}]
    \ge\tfrac14\Bigl[\eta^{\Hum}_{h,k}J_X^{(h)}+(1-\eta^{\Hum}_{h,k})I_R^{(h)}
        +(1-\eta^{\Hum}_{h,k})\sum_{r\in\calR}g_h(r)\eta^{\esc}_{h,k}(r)d^{(h)}(r)\Bigr]
    =\tfrac14\,\E^{\pi_k}_h[\sigma_hZ_{h,k}].
\]
With \eqref{eq:correct-drift} this gives
$\E^{\pi_k}_h[\sigma_hZ_{*,k}]\ge c_{\mathrm{drift}}/4$.

Taking $\mu_0:=\min\{c_{\mathrm{drift}},\,
c_\times c_{\mathrm{drift}},\,c_{\mathrm{drift}}/4\}>0$ proves
\eqref{eq:mu0-def}.

(iv) By parts (i) and (iii), for all sufficiently large $k$,
\[
    \Var_h(Z_{j,k})
    \le \E^{\pi_k}_h[Z_{j,k}^2]
    \le B_{\mathrm{inc}}^2
    \le \frac{B_{\mathrm{inc}}^2}{\mu_0}\,
       \E^{\pi_k}_h[\sigma_hZ_{j,k}].
\]
Thus \eqref{eq:variance-drift} holds with
$C_V:=B_{\mathrm{inc}}^2/\mu_0$, which depends
only on the primitives.
\end{proof}

\begin{proof}{Proof of Lemma~\ref{lem:direction-tracking}}
Fix a sufficiently large $k$. We prove the bound for $W_{1,k}$ under $H_1$;
the bound for $W_{0,k}$ under
$H_0$ follows by applying the same argument to the reflected statistic $-S_i$,
since by Lemma~\ref{lem:variance-drift}(iii) every deployed rule also has drift at least
$\mu_0$ toward the correct lower boundary under $H_0$.

Use the item-level process $(Y_i,\calH_i,S_i)$ defined at the beginning of
this subsection.  The rule $J_i$ is
$\calH_{i-1}$-measurable, being determined by the sign test
\eqref{eq:rule-by-sign} applied to $S_{i-1}$.  Conditionally on
$\calH_{i-1}$, item $i$ is a \emph{fresh} item: its report, label, and
randomization seeds are drawn independently of the past, so the conditional law
of $Y_i$ given $\calH_{i-1}$ is exactly the unconditional law of the
single-item increment $Z_{J_i,k}$ under the (now fixed) rule $J_i$.  This gives, for every $i$, that
the following inequalities hold almost surely:
\begin{equation}\label{eq:dt-hyp}
    \E^{\pi_k}_1[Y_i\mid\calH_{i-1}]\ge\mu_0,
    \qquad
    |Y_i|\le B_{\mathrm{inc}},
    \qquad
    \Var_1(Y_i\mid\calH_{i-1})
    \le C_V\E^{\pi_k}_1[Y_i\mid\calH_{i-1}] ,
\end{equation}
where the third bound is the conditional form of \eqref{eq:variance-drift}:
$\Var_1(Y_i\mid\calH_{i-1})=\Var_1(Z_{J_i,k})
\le C_V\E^{\pi_k}_1[Z_{J_i,k}]
=C_V\E^{\pi_k}_1[Y_i\mid\calH_{i-1}]$.
Since $Y_i\le B_{\mathrm{inc}}$ gives
$\E^{\pi_k}_1[Y_i\mid\calH_{i-1}]\le B_{\mathrm{inc}}$,
\eqref{eq:dt-hyp} yields the
drift-controlled second moment
\begin{equation}\label{eq:dt-secmom}
    \E^{\pi_k}_1[Y_i^2\mid\calH_{i-1}]
    =\Var_1(Y_i\mid\calH_{i-1})
      +\bigl(\E^{\pi_k}_1[Y_i\mid\calH_{i-1}]\bigr)^2
    \le (C_V+B_{\mathrm{inc}})
       \E^{\pi_k}_1[Y_i\mid\calH_{i-1}] .
\end{equation}

Since $S_0=0\le z_k$, let
$\tau_{\mathrm{trk},k}:=
\min\{i:S_i\ge a_k\}\wedge N_{\main,k}\le N_{\main,k}$. Then
$W_{1,k}
\le\sum_{i=1}^{\tau_{\mathrm{trk},k}}\ind\{S_{i-1}\le z_k\}$,
so it suffices to bound $\sum_{i=1}^{\tau_{\mathrm{trk},k}}\ind\{S_{i-1}\le z_k\}$.  

To obtain this bound, we introduce a smooth potential $\Phi$ whose conditional
one-step change is uniformly negative whenever $S_{i-1}\le z_k$ and
nonpositive otherwise.  Summing this drift inequality up to $\tau_{\mathrm{trk},k}$ and telescoping
the stopped process will then convert the total decrease of $\Phi(S_i)$ into an
upper bound on the expected number of indices for which $S_{i-1}\le z_k$.

\emph{A smooth potential.}  Set
$\theta:=\dfrac{1}{e(C_V+B_{\mathrm{inc}})}>0$, so that
$\theta(C_V+B_{\mathrm{inc}})=1/e$ and
$\theta B_{\mathrm{inc}}\le1/e$, and define
\begin{equation}\label{eq:dt-potential}
    \Phi(s):=
    \begin{cases}
        z_k-s, & s\le z_k,\\[1mm]
        -\dfrac1\theta\bigl(1-e^{-\theta(s-z_k)}\bigr), & s\ge z_k.
    \end{cases}
\end{equation}
$\Phi$ is convex, non-increasing, and $1$-Lipschitz.  On the linear branch
$s<z_k$ it has $\Phi'(s)=-1$ and $\Phi''(s)=0$; on the curved branch $s>z_k$,
\[
    \Phi'(s)=-e^{-\theta(s-z_k)},
    \qquad
    \Phi''(s)=\theta e^{-\theta(s-z_k)}\in(0,\theta] .
\]
The two one-sided values agree at $z_k$ ($\Phi(z_k^\pm)=0$ and
$\Phi'(z_k^\pm)=-1$), so $\Phi\in C^1$, and the curvature is uniformly bounded:
\begin{equation}\label{eq:dt-curvature}
    \sup_{s\in\R}\Phi''(s)=\theta
    =\frac1{e(C_V+B_{\mathrm{inc}})},
\end{equation}
the supremum being attained as $s\downarrow z_k$.  Finally $\Phi(s)=z_k-s\to+\infty$
as $s\to-\infty$, while $\Phi(s)\to-1/\theta$ as $s\to+\infty$, so $\Phi$ is
\emph{unbounded above but bounded below}, with
\begin{equation}\label{eq:dt-phi-bounds}
    \Phi(S_0)=\Phi(0)=z_k,
    \qquad
    \inf_{s\in\R}\Phi(s)=-\tfrac1\theta=-e(C_V+B_{\mathrm{inc}}).
\end{equation}
% Only the lower bound in \eqref{eq:dt-phi-bounds} is used below.  The large values
% of $\Phi$ deep in the region (e.g. $\Phi(-b_k)=z_k+b_k$, which grows at least
% as fast as $b_k=\Theta(L_k)$, at the lower boundary) are harmless: the
% stopping-time telescoping depends only on the endpoints
% $\Phi(S_0)$ and $\Phi(S_{\tau_{\mathrm{trk},k}})$, so any deep dip of
% $S$ is offset by the
% subsequent climb back.

\emph{One-step drift inequality.}  We claim that for every $i$,
\begin{equation}\label{eq:dt-drift}
    \E^{\pi_k}_1[\Phi(S_i)\mid\calH_{i-1}]-\Phi(S_{i-1})
    \le -\tfrac{\mu_0}2\,\ind\{S_{i-1}\le z_k\}.
\end{equation}
Since $\Phi\in C^1$ with $\Phi''\le M$ everywhere, the second-order Taylor
inequality
\begin{equation}\label{eq:dt-taylor}
    \Phi(y)\le\Phi(x)+\Phi'(x)(y-x)+\tfrac12\,M_{x,y}\,(y-x)^2,
    \qquad M_{x,y}:=\sup_{[x\wedge y,\,x\vee y]}\Phi''\le \theta,
\end{equation}
holds for all $x,y$ (this extends the usual Taylor bound to our piecewise-$C^2$
$\Phi$ because $\Phi'$ is absolutely continuous, being continuous with matching
one-sided derivatives at the single kink $z_k$).  Applying it with
$x=S_{i-1}$, $y=S_i$: note that $M_{x,y}$ as written here depends on the
random $S_i$; since $|Y_i|\le B_{\mathrm{inc}}$, the interval
$[x\wedge y,x\vee y]$ is always contained in
$[S_{i-1}-B_{\mathrm{inc}},S_{i-1}+B_{\mathrm{inc}}]$, so
$M_{x,y}\leq \theta$ almost surely,
for every possible realization of $Y_i$.  With $y-x=Y_i$, taking
$\E^{\pi_k}_1[\cdot\mid\calH_{i-1}]$, and using the second-moment bound
\eqref{eq:dt-secmom},
\begin{equation}\label{eq:dt-onestep}
\begin{split}
    \E^{\pi_k}_1[\Phi(S_i)\mid\calH_{i-1}]-\Phi(S_{i-1})
    &\le \Phi'(S_{i-1})\E^{\pi_k}_1[Y_i\mid\calH_{i-1}]
       +\tfrac12 \E^{\pi_k}_1[M_{x,y}Y_i^2\mid\calH_{i-1}].
\end{split}
\end{equation}

\emph{Low steps} ($S_{i-1}\le z_k$, so $\Phi'(S_{i-1})=-1$).
Using \eqref{eq:dt-secmom}, $M_{x,y}\le\theta$ in \eqref{eq:dt-onestep} and
$\theta(C_V+B_{\mathrm{inc}})=1/e$,
\begin{align*}
    &\E^{\pi_k}_1[\Phi(S_i)\mid\calH_{i-1}]-\Phi(S_{i-1})\\
    &\le \Bigl[-1+\tfrac12\theta(C_V+B_{\mathrm{inc}})\Bigr]
       \E^{\pi_k}_1[Y_i\mid\calH_{i-1}]
        && [\eqref{eq:dt-secmom},\,\eqref{eq:dt-onestep},\ \Phi'(S_{i-1})=-1,\ M_{x,y}\le\theta\,]\\
    &= \Bigl[-1+\tfrac1{2e}\Bigr]
       \E^{\pi_k}_1[Y_i\mid\calH_{i-1}]
        && [\,\theta(C_V+B_{\mathrm{inc}})=1/e\,]\\
    &\le -\tfrac12\E^{\pi_k}_1[Y_i\mid\calH_{i-1}]
        && [\,1-\tfrac1{2e}\ge\tfrac12\,]\\
    &\le -\tfrac{\mu_0}2
        && [\,\eqref{eq:dt-hyp}\,].
\end{align*}

\emph{High steps} ($S_{i-1}>z_k$, so $\Phi'(S_{i-1})=-\alpha$ with
$\alpha:=e^{-\theta(S_{i-1}-z_k)}\in(0,1)$).  The step interval
$[S_{i-1}\wedge S_i,\,S_{i-1}\vee S_i]$ lies in
$[S_{i-1}-B_{\mathrm{inc}},\infty)$; since
$\Phi''$ is non-increasing on the curved branch and vanishes below $z_k$,
\[
    M_{x,y}\le\sup_{s\ge S_{i-1}-B_{\mathrm{inc}}}\Phi''(s)
    \le\theta e^{-\theta(S_{i-1}-B_{\mathrm{inc}}-z_k)}
    =\theta\,\alpha\,e^{\theta B_{\mathrm{inc}}}.
\]
Substituting into \eqref{eq:dt-onestep},
\begin{align*}
    &\E^{\pi_k}_1[\Phi(S_i)\mid\calH_{i-1}]-\Phi(S_{i-1})\\
    &\le \alpha\Bigl[-1+\tfrac12\theta e^{\theta B_{\mathrm{inc}}}
       (C_V+B_{\mathrm{inc}})\Bigr]
       \E^{\pi_k}_1[Y_i\mid\calH_{i-1}]
        && [\eqref{eq:dt-secmom},\,\eqref{eq:dt-onestep},\ \Phi'(S_{i-1})=-\alpha\,]\\
    &= \alpha\Bigl[-1+\tfrac{e^{\theta B_{\mathrm{inc}}}}{2e}\Bigr]
       \E^{\pi_k}_1[Y_i\mid\calH_{i-1}]
        && [\,\theta(C_V+B_{\mathrm{inc}})=1/e\,]\\
    &\le -\tfrac\alpha2\E^{\pi_k}_1[Y_i\mid\calH_{i-1}]\ \le\ 0
        && [\,e^{\theta B_{\mathrm{inc}}}\le e^{1/e}\le e\,].
\end{align*}
The extra factor $1/e$ in the definition of $\theta$ is exactly what keeps this
coefficient negative: with the alternative choice
$\theta=1/(C_V+B_{\mathrm{inc}})$ the same computation gives
$-1+\tfrac12 e^{\theta B_{\mathrm{inc}}}$, which can be positive (up to
$-1+\tfrac e2>0$ when $\theta B_{\mathrm{inc}}$ is close to $1$), so the descent property would
fail on high steps.  Note also that the linear gain
$-\alpha\E^{\pi_k}_1[Y_i\mid\calH_{i-1}]$ and the
curvature cost both carry the factor $\alpha$, so the gain dominates uniformly in
the height $S_{i-1}-z_k\ge0$: a down-crossing of $z_k$ raises $\Phi$, but the
next-step positive drift lowers $\E^{\pi_k}_1[\Phi(S_i)\mid\calH_{i-1}]$ by at least
as much, leaving no recharge term.  The two cases together establish
\eqref{eq:dt-drift}.

\emph{Stopping-time telescoping.} Because
$\tau_{\mathrm{trk},k}\le N_{\main,k}$ is a bounded stopping time,
$\ind\{i\le\tau_{\mathrm{trk},k}\}$ is
$\calH_{i-1}$-measurable. Hence the tower
property and \eqref{eq:dt-drift} give
\begin{align*}
    \E^{\pi_k}_1[\Phi(S_{\tau_{\mathrm{trk},k}})]-\Phi(S_0)
    &=\sum_{i=1}^{N_{\main,k}}\E^{\pi_k}_1\!\left[
        \ind\{i\le\tau_{\mathrm{trk},k}\}
        \bigl(\Phi(S_i)-\Phi(S_{i-1})\bigr)\right]\\
    &=\sum_{i=1}^{N_{\main,k}}\E^{\pi_k}_1\!\left[
        \ind\{i\le\tau_{\mathrm{trk},k}\}
        \left(\E^{\pi_k}_1[\Phi(S_i)\mid\calH_{i-1}]-\Phi(S_{i-1})\right)\right]\\
    &\le-\tfrac{\mu_0}2\,\E^{\pi_k}_1\Bigl[
        \sum_{i=1}^{\tau_{\mathrm{trk},k}}
        \ind\{S_{i-1}\le z_k\}\Bigr]
    \le-\tfrac{\mu_0}2\,\E^{\pi_k}_1[W_{1,k}].
\end{align*}
Additionally, we have
$\Phi(S_{\tau_{\mathrm{trk},k}})
\ge\inf\Phi=-e(C_V+B_{\mathrm{inc}})$ from
\eqref{eq:dt-phi-bounds}. Combining with $\Phi(S_0)=z_k$ gives
\[
    -e(C_V+B_{\mathrm{inc}})-z_k
    \ \le\ \E^{\pi_k}_1[\Phi(S_{\tau_{\mathrm{trk},k}})]-\Phi(S_0)
    \ \le\ -\tfrac{\mu_0}2\,\E^{\pi_k}_1[W_{1,k}].
\]
Rearranging the outer inequality,
\[
    \E^{\pi_k}_1[W_{1,k}]
    \le\frac{2\bigl(z_k+e(C_V+B_{\mathrm{inc}})\bigr)}{\mu_0}
    \le \frac{2\bigl(e(C_V+B_{\mathrm{inc}})+1\bigr)(z_k+1)}
              {\mu_0},
\]
which is the $H_1$ part of \eqref{eq:direction-tracking-bound} with
$C_W:=4\bigl(e(C_V+B_{\mathrm{inc}})+1\bigr)/\mu_0$.
The $H_0$ part is identical with $S_i$
replaced by $-S_i$, giving the same constant by the symmetric application of
Lemma~\ref{lem:variance-drift}(iii) noted at the start of the proof, so this single
$C_W$ covers both terms in \eqref{eq:direction-tracking-bound}.
\end{proof}

\subsection{Proof of Lemma~\ref{lem:main-cost}}\label{appendix-sec:main-cost}

% \begin{lemma}[Main-stage sensing cost]\label{lem:main-cost}
% There is a constant $C_M<\infty$ such that, for $h=0,1$,
% \begin{equation}\label{eq:main-cost-bound}
%     \E^{\pi_k}_h[C^{\main}_k]
%     \le \Gamma_h(T_{h,k},\overline N_{\main,k})
%        +C_M(z_k+1)
%        +C_M L_k\left[\frac{z_k+1}{\Delta_k}
%          +\exp\left\{-c_{\fb}\frac{\Delta_k^2}{L_k}\right\}\right],
% \end{equation}
% where $C^{\main}_k$ is the sensing cost incurred before fallback, not including
% the data-acquisition cost $\cdata N_{\main,k}$ and not including fallback completion
% cost. {\color{red} ST: I didn't check this, please double check}
% \end{lemma}

\begin{proof}{Proof of Lemma~\ref{lem:main-cost}}
Fix a sufficiently large \(k\).  We prove the result under \(H_1\), for
which the correct rule is the direction-\(1\) rule and the relevant
boundary is \(a_k\).  The \(H_0\) result follows by applying the same
argument to \(-S\), with \(a_k,T_{1,k}\), and the direction-\(1\) rule
replaced by \(b_k,T_{0,k}\), and the direction-\(0\) rule.

\emph{Cost decomposition.}  For
\(i=1,\ldots,N_{\main,k}\), let \(I_i\) indicate that the main stage
starts item \(i\), and define
\[
 \tau_{\main,k}:=\sum_{i=1}^{N_{\main,k}}I_i .
\]
Thus \(I_i=\ind\{i\le\tau_{\main,k}\}\) is \(\calH_{i-1}\)-measurable.  On a
boundary-crossing path, item \(\tau_{\main,k}\) is the terminal item and may be
only partially processed; on \(E_{\fb}\), all items are completed and
\(\tau_{\main,k}=N_{\main,k}\).  Recall that \((Y_i,S_i)\) is the completed-item
continuation defined at the beginning of this subsection.  Classify the
started items by the predictable sign test
\[
 \mathcal C:=\{i:I_i=1,\ S_{i-1}>z_k\},
 \qquad
 \mathcal W:=\{i:I_i=1,\ S_{i-1}\le z_k\},
\]
and write \(N_c:=|\mathcal C|\), so that
\(|\mathcal W|=W_{1,k}\).  Define the sensing cost of fully completing
item \(i\) on this continuation by
\[
\begin{aligned}
 C_i^{\mathrm{item}}
 :={}&\ind\{U_i\le\eta^{\Hum}_{J_i,k}\}\,\cH
 +\ind\{U_i>\eta^{\Hum}_{J_i,k}\}
   \Bigl(
      \cAI
      +\ind\{V_i\le\eta^{\esc}_{J_i,k}(R_i)\}\,\cH
   \Bigr).
\end{aligned}
\]
Let $C^{\mathcal C}_k:=\sum_{i\in\mathcal C}
      C_i^{\mathrm{item}}$, $C^{\mathcal W}_k:=\sum_{i\in\mathcal W}
      C_i^{\mathrm{item}}$, and we have $C^{\main}_k\le C^{\mathcal C}_k+C^{\mathcal W}_k$. 

Set \(c_{\max}:=\cH+\cAI+\cH\), a primitive upper bound on the sensing
cost of a completed item. Taking expectations gives
\begin{equation}\label{eq:mc-split}
\begin{aligned}
 \E^{\pi_k}_1[C^{\main}_k]
 &\le
 \underbrace{\E^{\pi_k}_1[C^{\mathcal W}_k]}
    _{\textnormal{Cost 1: wrong/dead-zone}}
 +\underbrace{\E^{\pi_k}_1[C^{\mathcal C}_k]}
    _{\textnormal{Cost 2: correct}}.
\end{aligned}
\end{equation}
We bound these two terms in turn.

\emph{Cost 1: wrong and dead-zone items.}  Each completed-item sensing
cost is at most \(c_{\max}\).  Hence Lemma~\ref{lem:direction-tracking}
gives
\begin{equation}\label{eq:mc-overhead1}
 \E^{\pi_k}_1[C^{\mathcal W}_k]
 \le c_{\max}\E^{\pi_k}_1[W_{1,k}]
 \le c_{\max}C_W(z_k+1).
\end{equation}

\emph{Cost 2: correct items.} Let
$N^{\mathrm{plan}}_{1,k}:=
\lceil n^*_{\Hum,1,k}\rceil+\lceil n^*_{\AI,1,k}\rceil$.
If the first \(N^{\mathrm{plan}}_{1,k}\) sensed items were all processed with the
direction-\(1\) rule, their expected sensing cost would be
\(\lceil n^*_{\Hum,1,k}\rceil\cH
+\lceil n^*_{\AI,1,k}\rceil\cAI
+\lceil n^*_{\AI,1,k}\rceil s_{1,k}\cH\).
Rounding changes this by at most a primitive constant, so the planned
block cost is at most
\(\Gamma_1(T_{1,k},\overline N_{\main,k})+O(1)\).
A direction-\(1\) item has planned per-item cost
\begin{equation}\label{eq:mc-peritem}
\begin{aligned}
 c_{\mathrm{planned}}
 := \E^{\pi_k}_1[C^{\text{item}}_i]
 = \eta^{\Hum}_{1,k}\cH
 +(1-\eta^{\Hum}_{1,k})
     \bigl(\cAI+s_{1,k}\cH\bigr)
 \le
 \frac{\Gamma_1(T_{1,k},\overline N_{\main,k})+O(1)}
      {N^{\mathrm{plan}}_{1,k}} .
\end{aligned}
\end{equation}

Since all items in \(\mathcal C\) apply the direction-\(1\)
rule, the tower property gives
\[
\begin{aligned}
 \E^{\pi_k}_1[C^{\mathcal C}_k]
 &=\sum_{i=1}^{N_{\main,k}}\E^{\pi_k}_1[
    I_i\ind\{S_{i-1}>z_k\}C_i^{\mathrm{item}}]=c_{\mathrm{planned}}
   \sum_{i=1}^{N_{\main,k}}\E^{\pi_k}_1[
    I_i\ind\{S_{i-1}>z_k\}]
 =c_{\mathrm{planned}}\E^{\pi_k}_1[N_c].
\end{aligned}
\]
Likewise,
\[
 \E^{\pi_k}_1[S_{\tau_{\main,k}}]
 =\E^{\pi_k}_1\!\left[
   \sum_{i=1}^{N_{\main,k}}I_i
   \E^{\pi_k}_1[Y_i\mid\calH_{i-1}]
 \right],
\]
and splitting the sum over \(\mathcal C\) and \(\mathcal W\) yields
\[
 \E^{\pi_k}_1[Z_{1,k}]\,\E^{\pi_k}_1[N_c]
 =\E^{\pi_k}_1[S_{\tau_{\main,k}}]
 -\E^{\pi_k}_1\!\left[
    \sum_{i\in\mathcal W}
    \E^{\pi_k}_1[Y_i\mid\calH_{i-1}]
  \right].
\]
Every deployed rule has conditional drift at least \(\mu_0>0\) by
Lemma~\ref{lem:variance-drift}(iii), so the second term on the
right is nonnegative.  Moreover,
\(S_{\tau_{\main,k}}\le a_k+B_{\mathrm{inc}}\) pathwise: before a
terminal boundary-crossing item is started, \(S_{\tau_{\main,k}-1}<a_k\), and its
completed continuation satisfies \(Y_{\tau_{\main,k}}\le B_{\mathrm{inc}}\); on pool exhaustion,
\(S_{\tau_{\main,k}}<a_k\).  Thus
\[
 \E^{\pi_k}_1[Z_{1,k}]\,\E^{\pi_k}_1[N_c]
 \le a_k+B_{\mathrm{inc}} .
\]
Combining this bound with the cost identity, \eqref{eq:mc-peritem}, and
\(N^{\mathrm{plan}}_{1,k}\E^{\pi_k}_1[Z_{1,k}]
\ge a_k+\Delta_k\) from \eqref{eq:correct-rule-target}, we obtain
\[
\begin{aligned}
 \E^{\pi_k}_1[C^{\mathcal C}_k]
 &\le
 \bigl(\Gamma_1(T_{1,k},\overline N_{\main,k})+O(1)\bigr)
 \frac{a_k+B_{\mathrm{inc}}}{a_k+\Delta_k}\\
 &\le \Gamma_1(T_{1,k},\overline N_{\main,k})+O(1),
\end{aligned}
\]
where the last inequality holds for all large \(k\), since
\(\Delta_k\ge B_{\mathrm{inc}}\).

\emph{Combine.}  Substituting the bounds from Costs 1 and 2 into the
decomposition \eqref{eq:mc-split} gives
\begin{equation}\label{eq:mc-final}
\begin{aligned}
 \E^{\pi_k}_1[C^{\main}_k]
 &\le \Gamma_1(T_{1,k},\overline N_{\main,k})
 +C_M(z_k+1).
\end{aligned}
\end{equation}
This is \eqref{eq:main-cost-bound} under \(H_1\); the reflected argument
at the beginning gives the result under \(H_0\).
\end{proof}

\subsection{Proof of Lemma~\ref{lem:fallback-low}}\label{appendix-sec:fallback-low}

\begin{proof}{Proof of Lemma~\ref{lem:fallback-low}}
Fix a sufficiently large $k$. For \(h\in\{0,1\}\), let
\[
N^{\mathrm{plan}}_{h,k}:=
\lceil n^*_{\Hum,h,k}\rceil+\lceil n^*_{\AI,h,k}\rceil
\]
be the rounded direction-\(h\) count.
We prove the claim under $H_1$; the $H_0$ proof is identical with $S$ replaced
by $-S$, the lower boundary $-b_k$, and
$N^{\mathrm{plan}}_{0,k},b_k,T_{0,k}$ in place of
$N^{\mathrm{plan}}_{1,k},a_k,T_{1,k}$. Throughout, use the item-level process
$(Y_i,\calH_i,S_i)$ defined at the beginning of the proof of
Lemma~\ref{lem:direction-tracking}.  By the
proof of Lemma~\ref{lem:direction-tracking}, for every $i$ the following
inequalities hold almost surely:
\begin{equation}\label{eq:fb-hyp}
    \E^{\pi_k}_1[Y_i\mid\calH_{i-1}]\ge\mu_0,
    \qquad
    |Y_i|\le B_{\mathrm{inc}},
    \qquad
    \Var_1(Y_i\mid\calH_{i-1})
    \le C_V\E^{\pi_k}_1[Y_i\mid\calH_{i-1}]
    \le C_V B_{\mathrm{inc}} .
\end{equation}
By
\eqref{eq:B-fits-pool}, $N^{\mathrm{plan}}_{1,k}\le N_{\main,k}$,
and by the proof of
Lemma~\ref{lem:variance-drift}(iii) (the inequality
$N^{\mathrm{plan}}_{1,k}\E^{\pi_k}_1[Z_{1,k}]\ge
T_{1,k}$ together with $N^{\mathrm{plan}}_{1,k}=O(L_k)$),
\begin{equation}\label{eq:fb-planned}
    N^{\mathrm{plan}}_{1,k}\,\E^{\pi_k}_1[Z_{1,k}]
    \ge T_{1,k}=a_k+\Delta_k,
    \qquad
    N^{\mathrm{plan}}_{1,k}\le C_1 L_k
\end{equation}
for a primitive constant $C_1$, where
$\E^{\pi_k}_1[Z_{1,k}]=\E^{\pi_k}_1[Y_i\mid\calH_{i-1}]$ on a direction-$1$
item is the per-item direction-$1$ drift.

\emph{Step 1: on $E_{\fb}$ the statistic stays below $a_k$ at item
$N^{\mathrm{plan}}_{1,k}$.}
On $E_{\fb}$ no boundary is crossed in the whole pool, so in particular the upper
boundary is not crossed by item
$N^{\mathrm{plan}}_{1,k}\le N_{\main,k}$; hence
\begin{equation}\label{eq:fb-below}
    E_{\fb}\subseteq\{S_{N^{\mathrm{plan}}_{1,k}}<a_k\}.
\end{equation}

\emph{Step 2: drift accumulated by item $N^{\mathrm{plan}}_{1,k}$ on
the small-occupation event.}
On $E_{\fb}$, the policy processes at least the first
$N^{\mathrm{plan}}_{1,k}$ items. Let $\mathcal I_k
    :=\{1\le i\le N^{\mathrm{plan}}_{1,k}:S_{i-1}\le z_k\}$ be the
set of items in this planned block that do not use the direction-$1$
rule. Since $W_{1,k}$ counts all such items processed before pool exhaustion,
$|\mathcal I_k|\le W_{1,k}$ on $E_{\fb}$. For
$i\notin\mathcal I_k$, the policy uses the direction-$1$ rule and hence
$\E^{\pi_k}_1[Y_i\mid\calH_{i-1}]=\E^{\pi_k}_1[Z_{1,k}]$; for
$i\in\mathcal I_k$, we have
$\E^{\pi_k}_1[Y_i\mid\calH_{i-1}]\ge0$.
Therefore, on $E_{\fb}$,
\begin{align*}
    \sum_{i=1}^{N^{\mathrm{plan}}_{1,k}}
       \E^{\pi_k}_1[Y_i\mid\calH_{i-1}]
    =&
    \bigl(N^{\mathrm{plan}}_{1,k}-|\mathcal I_k|\bigr)
       \E^{\pi_k}_1[Z_{1,k}]
    +\sum_{i\in\mathcal I_k}
      \E^{\pi_k}_1[Y_i\mid\calH_{i-1}]\\
    \ge&
    N^{\mathrm{plan}}_{1,k}\E^{\pi_k}_1[Z_{1,k}]
    -|\mathcal I_k|\E^{\pi_k}_1[Z_{1,k}]\\
    \ge&
    N^{\mathrm{plan}}_{1,k}\E^{\pi_k}_1[Z_{1,k}]
       -B_{\mathrm{inc}}\,W_{1,k}\\
    \geq &
    (a_k+\Delta_k)-B_{\mathrm{inc}}\,W_{1,k},
\end{align*}
where the second inequality uses $|\mathcal I_k|\le W_{1,k}$ and
$\E^{\pi_k}_1[Z_{1,k}]\le B_{\mathrm{inc}}$, and the last uses
\eqref{eq:fb-planned}. Consequently, on
$E_{\fb}$ and the small-occupation event
\begin{equation}\label{eq:small-w-event}
    \Bigl\{W_{1,k}\le \tfrac{\Delta_k}{4B_{\mathrm{inc}}}\Bigr\},
\end{equation}
the accumulated drift satisfies
\begin{equation}\label{eq:fb-drift-lb}
    \sum_{i=1}^{N^{\mathrm{plan}}_{1,k}}
       \E^{\pi_k}_1[Y_i\mid\calH_{i-1}]
    \ge a_k+\Delta_k
       -B_{\mathrm{inc}}\cdot\tfrac{\Delta_k}{4B_{\mathrm{inc}}}
    =a_k+\tfrac{3\Delta_k}{4}.
\end{equation}

\emph{Step 3: martingale tail bound \eqref{eq:mart-tail}.}
Decompose the statistic at item $N^{\mathrm{plan}}_{1,k}$ into its
predictable (drift) part and a
fluctuation part,
\[
    S_{N^{\mathrm{plan}}_{1,k}}
    =\sum_{i=1}^{N^{\mathrm{plan}}_{1,k}}
       \E^{\pi_k}_1[Y_i\mid\calH_{i-1}]
     +M_{N^{\mathrm{plan}}_{1,k}},
    \qquad
    M_n:=\sum_{i=1}^{n}
      \Bigl(Y_i-\E^{\pi_k}_1[Y_i\mid\calH_{i-1}]\Bigr) .
\]
% Because the policy acquires the whole pool of $N_{\main,k}$ items in advance, the
% increment $Y_i$ and its conditional expectation
% $\E^{\pi_k}_1[Y_i\mid\calH_{i-1}]$ are defined for every
% $i\le N_{\main,k}$; we take $M_n$ for $n\le N_{\main,k}$ and, when needed,
% its stopped version $M_{n\wedge\tau_{\mathrm{trk},k}}$.

\emph{(i) $M_n$ is a martingale.}  Each
$\E^{\pi_k}_1[Y_i\mid\calH_{i-1}]$ is $\calH_{i-1}$-measurable, and the
increments are bounded by \eqref{eq:fb-hyp}, hence integrable.  The centred
increment has zero conditional mean,
\[
    \E^{\pi_k}_1\!\left[
      Y_i-\E^{\pi_k}_1[Y_i\mid\calH_{i-1}]
      \mathrel{\Big|}\calH_{i-1}\right]=0,
\]
so $\E^{\pi_k}_1[M_n\mid\calH_{n-1}]=M_{n-1}$; thus $(M_n)_{n\ge0}$ is a
martingale with $M_0=0$ with respect to $(\calH_n)$.  Its increments are
bounded:
\begin{equation}\label{eq:fb-incr}
    \left|Y_i-\E^{\pi_k}_1[Y_i\mid\calH_{i-1}]\right|
    \le |Y_i|+\E^{\pi_k}_1[Y_i\mid\calH_{i-1}]
    \le 2B_{\mathrm{inc}} .
\end{equation}

\emph{(ii) Accumulated predictable variance.}  Since the increments of $M_n$
are martingale differences, their conditional second moment equals the
conditional variance of $Y_i$.  We write the accumulated predictable variance
as
\[
    V_n:=\sum_{i=1}^{n}\E^{\pi_k}_1\!\left[
       \left(Y_i-\E^{\pi_k}_1[Y_i\mid\calH_{i-1}]\right)^2
       \mathrel{\Big|}\calH_{i-1}\right]
        =\sum_{i=1}^{n}\Var_1(Y_i\mid\calH_{i-1}).
\]
By the variance bound in \eqref{eq:fb-hyp} and then
\eqref{eq:fb-planned},
$\Var_1(Y_i\mid\calH_{i-1})\le C_VB_{\mathrm{inc}}$ holds a.s., and thus:
\begin{equation}\label{eq:fb-pqv}
    V_{N^{\mathrm{plan}}_{1,k}}
    =\sum_{i=1}^{N^{\mathrm{plan}}_{1,k}}
       \Var_1(Y_i\mid\calH_{i-1})
    \le \sum_{i=1}^{N^{\mathrm{plan}}_{1,k}}C_V B_{\mathrm{inc}}
    = C_V B_{\mathrm{inc}}\,N^{\mathrm{plan}}_{1,k}
    \le C_1 C_V B_{\mathrm{inc}}\,L_k .
\end{equation}

\emph{(iii) Reduction to a martingale lower tail.}  On the event
$E_{\fb}\cap\left\{W_{1,k}\le
\tfrac{\Delta_k}{4B_{\mathrm{inc}}}\right\}$ we have
$S_{N^{\mathrm{plan}}_{1,k}}<a_k$ by
\eqref{eq:fb-below} and
$\sum_{i\le N^{\mathrm{plan}}_{1,k}}
  \E^{\pi_k}_1[Y_i\mid\calH_{i-1}]
\ge a_k+\tfrac{3\Delta_k}4$ by
\eqref{eq:fb-drift-lb}, hence
\[
    M_{N^{\mathrm{plan}}_{1,k}}
    =S_{N^{\mathrm{plan}}_{1,k}}
     -\sum_{i=1}^{N^{\mathrm{plan}}_{1,k}}
        \E^{\pi_k}_1[Y_i\mid\calH_{i-1}]
    < a_k-\Bigl(a_k+\tfrac{3\Delta_k}4\Bigr)=-\tfrac{3\Delta_k}4 .
\]
(On $E_{\fb}$ the pool is exhausted without stopping, so
$N^{\mathrm{plan}}_{1,k}\le N_{\main,k}$ and the martingale is defined
through item $N^{\mathrm{plan}}_{1,k}$.) Therefore
$E_{\fb}\cap\left\{W_{1,k}\le
\tfrac{\Delta_k}{4B_{\mathrm{inc}}}\right\}
\subseteq\{M_{N^{\mathrm{plan}}_{1,k}}\le-\tfrac{3\Delta_k}4\}$.

\emph{(iv) Freedman's inequality.}  We use the following form
\citep{freedman1975}: if $(M_n)$ is a martingale with $M_0=0$ whose increments
satisfy $M_i-M_{i-1}\le R$ for all $i$, and
$V_n:=\sum_{i=1}^{n}\E[(M_i-M_{i-1})^2\mid\calH_{i-1}]$ is its accumulated
predictable variance, then for all $\lambda>0$ and $v>0$,
\[
    \Pp\bigl(\exists\,n:\ M_n\ge\lambda\ \text{and}\ V_n\le v\bigr)
    \le\exp\!\left\{-\frac{\lambda^2}{2(v+R\lambda)}\right\}.
\]
Apply this to the martingale $-M_n$, whose increments
$\E^{\pi_k}_1[Y_i\mid\calH_{i-1}]-Y_i\le2B_{\mathrm{inc}}$
are bounded above by $2B_{\mathrm{inc}}$
(by \eqref{eq:fb-hyp}).
Its accumulated predictable variance also equals $V_n$; since
$V_{N^{\mathrm{plan}}_{1,k}}\le C_1 C_V B_{\mathrm{inc}}\,L_k$ holds surely by
\eqref{eq:fb-pqv} (not merely with high probability), the event
$\{V_{N^{\mathrm{plan}}_{1,k}}\le
  C_1 C_V B_{\mathrm{inc}}\,L_k\}$ has probability
$1$, so intersecting with it does not change any probability below, and the
single time point $n=N^{\mathrm{plan}}_{1,k}$ is in particular one
instance of the "$\exists
n$" event.  With $\lambda=\tfrac{3\Delta_k}4$ and $v=C_VB_{\mathrm{inc}}N_{\main,k}$,
\begin{equation}\label{eq:mart-tail}
    \Pp^{\pi_k}_1\bigl(E_{\fb}\cap
       \{W_{1,k}\le \tfrac{\Delta_k}{4B_{\mathrm{inc}}}\}\bigr)
    \le\Pp^{\pi_k}_1\Bigl(
       M_{N^{\mathrm{plan}}_{1,k}}\le-\tfrac{3\Delta_k}4\Bigr)
    \le\exp\!\left\{-\frac{(3\Delta_k/4)^2}
        {2\bigl(C_1 C_V B_{\mathrm{inc}}\,L_k
        +2B_{\mathrm{inc}}\cdot\tfrac{3\Delta_k}4\bigr)}\right\}.
\end{equation}
The denominator equals
$2C_1 C_V B_{\mathrm{inc}}L_k+3B_{\mathrm{inc}}\Delta_k
\le2C_1C_VB_{\mathrm{inc}}L_k+3B_{\mathrm{inc}}L_k$
(using $\Delta_k\le L_k$), which is at most
$C_2(C_VB_{\mathrm{inc}}+B_{\mathrm{inc}}^2)L_k$ for a
primitive constant $C_2$; the numerator is $\tfrac9{16}\Delta_k^2$.  Hence
\eqref{eq:mart-tail} has the stated form $C\exp\{-c_{\fb}\Delta_k^2/L_k\}$ with
$c_{\fb}=c/(B_{\mathrm{inc}}^2+C_VB_{\mathrm{inc}})$.  

\emph{Step 4: the complementary event.}  By Markov's inequality and the
occupation bound of Lemma~\ref{lem:direction-tracking},
\[
    \Pp^{\pi_k}_1\Bigl(W_{1,k}>
       \tfrac{\Delta_k}{4B_{\mathrm{inc}}}\Bigr)
    \le \frac{4B_{\mathrm{inc}}\,\E^{\pi_k}_1[W_{1,k}]}{\Delta_k}
    \le \frac{4B_{\mathrm{inc}}}{\Delta_k}\cdot C_W(z_k+1)
    = C\,\frac{z_k+1}{\Delta_k}.
\]

\emph{Conclusion.}  Splitting $E_{\fb}$ according to
\eqref{eq:small-w-event},
\[
    \Pp^{\pi_k}_1(E_{\fb})
    \le\Pp^{\pi_k}_1\bigl(E_{\fb}\cap
       \{W_{1,k}\le\tfrac{\Delta_k}{4B_{\mathrm{inc}}}\}\bigr)
      +\Pp^{\pi_k}_1\Bigl(W_{1,k}>
        \tfrac{\Delta_k}{4B_{\mathrm{inc}}}\Bigr),
\]
and Steps~3--4 bound the two terms by the two terms of
\eqref{eq:fallback-prob-bound}.  This proves the bound under $H_1$, and the
$H_0$ case follows symmetrically.
\end{proof}

\subsection{Proof of Lemma~\ref{lem:fallback-cost}}\label{appendix-sec:fallback-cost}
% \begin{lemma}[Fallback expected cost]\label{lem:fallback-cost}
% There is a constant $C_F<\infty$ such that, for $h=0,1$,
% \begin{equation}\label{eq:fallback-cost-bound}
%     \E^{\pi_k}_h[C^{\fb}_k]
%     \le C_F L_k\left[
%         \frac{z_k+1}{\Delta_k}
%         +\exp\left\{-c_{\fb}\frac{\Delta_k^2}{L_k}\right\}
%     \right],
% \end{equation}
% where $C^{\fb}_k$ is the fallback completion cost: $C^{\fb}_k=0$ on
% $E_{\fb}^c$ (a boundary is crossed, so no fallback is executed), and on
% $E_{\fb}$, $C^{\fb}_k$ is the realized cost of applying the pre-committed
% fallback test (as in \eqref{eq:precommit-human-cost}--\eqref{eq:precommit-AI-cost}).
% \end{lemma}
\begin{proof}{Proof of Lemma~\ref{lem:fallback-cost}}
Fix a sufficiently large $k$. The human fallback completion cost is at most $\cH N_{\fixed,\Hum}(\alpha_{2,k}, \beta_{2,k})=O(L_k)$ from
Lemma~\ref{lem:safe-floor-first-order}. If
the AI fallback is unavailable, it is never selected.  If it is selected, then
$N_{\fixed, \AI}(\alpha_{2,k},\beta_{2,k})\le N_{\main,k}=O(L_k)$ by the definition of availability in the
pre-commitment rule, so its completion cost is also $O(L_k)$.  Hence the
pre-committed fallback cost is deterministically at most $C L_k$ for a
primitive constant $C$, on both the human and AI branches.  Since
$C^{\fb}_k=0$ on $E_{\fb}^c$ and $C^{\fb}_k\le CL_k$ on $E_{\fb}$, we have the
pointwise bound $C^{\fb}_k\le CL_k\,\ind_{E_{\fb}}$; taking
$\E^{\pi_k}_h[\cdot]$ of both sides,
\[
    \E^{\pi_k}_h[C^{\fb}_k]\le CL_k\,\Pp^{\pi_k}_h(E_{\fb}),
\]
and substituting the probability estimate of Lemma~\ref{lem:fallback-low}
proves \eqref{eq:fallback-cost-bound} with $C_F:=C\,C_{\fb}$.
\end{proof}

\subsection{Proof of Corollary~\ref{cor:rates-1}}\label{appendix-sec:corollary-proof}
\begin{proof}{Proof of Corollary~\ref{cor:rates-1}}
To prove Corollary~\ref{cor:rates-1} we first show that, for every
$\Delta_k$ with $\Delta_k\to\infty$ and $\Delta_k=o(L_k)$, there are
primitive constants $C_0<\infty$, $c_0>0$, and $c_G>0$ such that,  for all large $k$, the fallback probability satisfies
\begin{equation}\label{eq:reserve-fallback-prob}
    \Pp^{\pi_k}_h(E_{\fb})
    \le C_0\exp\left\{-c_0\,\frac{(\Delta_k+c_GG_k^+)^2}{L_k}\right\},
\end{equation}
and consequently
\begin{equation}\label{eq:reserve-refined-rate}
        \frac{\max_h\E^{\pi_k}_h[C^{\pi_k}]}{\LB_k}
        \le 1+O\left(
        \frac{\Delta_k}{L_k}
        +\exp\left\{-c_0\frac{(\Delta_k+c_GG_k^+)^2}{L_k}\right\}
        +\frac{\log L_k}{L_k}
        \right).
    \end{equation}
Then, parts (i) and (ii) can be obtained by directly plugging in the values $\Delta_k$. 

For
$h\in\{0,1\}$, let $\rho_h^*$ be the value of
\begin{equation}\label{eq:rho-star-def}
\begin{aligned}
    \rho_h^*:=\min_{a,b,e\ge0}\ \ &\cH a+\cAI b+\cH e\\
    \text{s.t. }&aJ_X^{(h)}+bI_R^{(h)}+b\Psi_h(e/b)\ge1,\\
        &0\le e\le b,
\end{aligned}
\end{equation}
where $\Psi_h(e/b = 0)$ when $b = e=0$. Recall
$\mathcal M_h$ is the set of minimizers of \eqref{eq:rho-star-def}. Define the unbuffered sensing item scale
\begin{equation}\label{eq:Nsense-def}
    N_{\sense,k}:=\max\{\nu_1^* a_k,\nu_0^* b_k\}.
\end{equation}

We also use the notations assembled in
the proof of Lemma~\ref{lem:fallback-low}: the item-level increments $Y_i$
with conditional expectations
$\E^{\pi_k}_1[Y_i\mid\calH_{i-1}]$ for every $i\le N_{\main,k}$; thus
\eqref{eq:fb-hyp} holds. The rounded direction-$1$ block size
\[
N^{\mathrm{plan}}_{1,k}:=
\lceil n^*_{\Hum,1,k}\rceil+\lceil n^*_{\AI,1,k}\rceil
\]
satisfies $N^{\mathrm{plan}}_{1,k}\le N_{\main,k}$ and
$N^{\mathrm{plan}}_{1,k}\E^{\pi_k}_1[Z_{1,k}]
\ge T_{1,k}=a_k+\Delta_k$
(\eqref{eq:fb-planned} and
\eqref{eq:B-fits-pool}); the occupation count $W_{1,k}$ of
Lemma~\ref{lem:direction-tracking}; and the martingale
$M_n:=\sum_{i=1}^{n}
\bigl(Y_i-\E^{\pi_k}_1[Y_i\mid\calH_{i-1}]\bigr)$, whose increments are
bounded by $2B_{\mathrm{inc}}$ and whose
accumulated predictable variance satisfies
$V_n\le C_VB_{\mathrm{inc}}\,n
\le C_VB_{\mathrm{inc}}\,N_{\main,k}$ surely. Finally,
Lemma~\ref{lem:variance-drift}(iii) gives a
primitive constant $\mu_0>0$ such that
\begin{equation}\label{eq:reserve-mu0}
    \E^{\pi_k}_1[Y_i\mid\calH_{i-1}]\ \ge\ \mu_0
    \qquad\text{for every }i\le N_{\main,k}\text{ and all large }k.
\end{equation}

\emph{Step 1: reserve inventory.}  We claim that, beyond the planned block,
the pool contains at least
\begin{equation}\label{eq:reserve-inventory}
    R_k:=N_{\main,k}-N^{\mathrm{plan}}_{1,k}
    \ \ge\ \pos{G_k^+-\nu_1^*\Delta_k-2}
\end{equation}
items.  First, the program $\Gamma_1(T,N)$ is positively homogeneous in
$(T;n_{\Hum},n_{\AI},n_{\esc})$.  Consequently, when
$\overline N_{\main,k}\ge\nu_1^*T_{1,k}$, the minimizers in $\mathcal M_h$ scaled by $T_{1,k}$ is feasible for
$\Gamma_1(T_{1,k},\overline N_{\main,k})$, attains the unconstrained optimal
cost $\rho_1^*T_{1,k}$, and consumes $\nu_1^*T_{1,k}$ active items, so the
least-consumption selection satisfies
$n^*_{\Hum,1,k}+n^*_{\AI,1,k}\le\nu_1^*T_{1,k}$; when
$\overline N_{\main,k}<\nu_1^*T_{1,k}$, the same inequality holds trivially
because the capacity constraint forces
$n^*_{\Hum,1,k}+n^*_{\AI,1,k}\le\overline N_{\main,k}$.  Hence
\[
    N^{\mathrm{plan}}_{1,k}
    \le n^*_{\Hum,1,k}+n^*_{\AI,1,k}+2
    \le\nu_1^*(a_k+\Delta_k)+2 .
\]
Second, by \eqref{eq:Nbar-main-choice}--\eqref{eq:Nmain-choice} the pool size
dominates the fallback floor, which dominates the unsplit floor because
$N_{\fixed,\Hum}$ is nonincreasing in both budgets
(Lemma~\ref{lem:safe-floor-first-order}); by \eqref{eq:floor-gap} and
\eqref{eq:Nsense-def},
\[
    N_{\main,k}\ \ge\ N_{\fixed,\Hum}(\alpha_{2,k},\beta_{2,k})
    \ \ge\ N_{\fixed,\Hum}(\alpha_k,\beta_k)
    \ =\ N_{\sense,k}+G_k
    \ \ge\ \nu_1^*a_k+G_k .
\]
Subtracting the two displays gives $R_k\ge G_k-\nu_1^*\Delta_k-2$; since also
$R_k\ge0$ and $G_k\le G_k^+$, the claim \eqref{eq:reserve-inventory} follows.

\emph{Step 2: drift accumulated over the whole pool.} On $E_{\fb}$ no boundary is
crossed, so all $N_{\main,k}$ items are sensed and $S_{N_{\main,k}}<a_k$.
Split the predictable sum at $N^{\mathrm{plan}}_{1,k}$. Among the first
$N^{\mathrm{plan}}_{1,k}$ items,
those with $S_{i-1}>z_k$ use the direction-$1$ rule and contribute
$\E^{\pi_k}_1[Y_i\mid\calH_{i-1}]=\E^{\pi_k}_1[Z_{1,k}]$; the remaining
ones number at most $W_{1,k}$ and contribute
$\E^{\pi_k}_1[Y_i\mid\calH_{i-1}]\ge0$, a loss of at most
$\E^{\pi_k}_1[Z_{1,k}]\le B_{\mathrm{inc}}$ each.  By
\eqref{eq:fb-planned},
\[
    \sum_{i=1}^{N^{\mathrm{plan}}_{1,k}}
       \E^{\pi_k}_1[Y_i\mid\calH_{i-1}]
    \ \ge\ N^{\mathrm{plan}}_{1,k}\E^{\pi_k}_1[Z_{1,k}]
       -B_{\mathrm{inc}}\,W_{1,k}
    \ \ge\ a_k+\Delta_k-B_{\mathrm{inc}}\,W_{1,k}.
\]
Every item beyond the block contributes
$\E^{\pi_k}_1[Y_i\mid\calH_{i-1}]\ge\mu_0$ by
\eqref{eq:reserve-mu0}, regardless of which rule the sign test selects. Define the enlarged
buffer $\widetilde\Delta_k:=\Delta_k+\mu_0R_k$.  
Therefore
\begin{equation}\label{eq:reserve-drift}
    \sum_{i=1}^{N_{\main,k}}\E^{\pi_k}_1[Y_i\mid\calH_{i-1}]
    \ \ge\ a_k+\Delta_k+\mu_0R_k-B_{\mathrm{inc}}\,W_{1,k}
    \ =\ a_k+\widetilde\Delta_k-B_{\mathrm{inc}}\,W_{1,k}.
\end{equation}

\emph{Step 3: two martingale tails.}  Set
$t_k:=\widetilde\Delta_k/(4B_{\mathrm{inc}})$ and
split $E_{\fb}$ along $\{W_{1,k}\le t_k\}$.

(a) On $E_{\fb}\cap\{W_{1,k}\le t_k\}$, combining $S_{N_{\main,k}}<a_k$ with
\eqref{eq:reserve-drift},
\[
    M_{N_{\main,k}}
    =S_{N_{\main,k}}
     -\sum_{i=1}^{N_{\main,k}}\E^{\pi_k}_1[Y_i\mid\calH_{i-1}]
    \ <\ a_k-\bigl(a_k+\widetilde\Delta_k-B_{\mathrm{inc}}\,t_k\bigr)
    \ =\ -\tfrac{3}{4}\widetilde\Delta_k .
\]
Freedman's inequality \citep{freedman1975}, in the form quoted before
\eqref{eq:mart-tail} and applied to $-M_{N_{\main, k}}$ with increment
bound $2B_{\mathrm{inc}}$ and
predictable variance at most $C_VB_{\mathrm{inc}}N_{\main,k}$ surely, gives
\[
    \Pp^{\pi_k}_1\bigl(E_{\fb}\cap\{W_{1,k}\le t_k\}\bigr)
    \le\exp\left\{-\frac{(3\widetilde\Delta_k/4)^2}
        {2\bigl(C_VB_{\mathrm{inc}}\,N_{\main,k}
        +2B_{\mathrm{inc}}\cdot\tfrac34\widetilde\Delta_k\bigr)}\right\}
    \le\exp\left\{-c_a\frac{\widetilde\Delta_k^2}{L_k}\right\},
\]
where the last step uses 
$\widetilde\Delta_k\le\Delta_k+\mu_0N_{\main,k}=O(L_k)$
(recall $\mu_0\le B_{\mathrm{inc}}$),
so that the denominator is at most a primitive multiple of $L_k$.

(b) For the occupation event we replace the Markov step of
Lemma~\ref{lem:fallback-low} by a second application of the same tail bound.
Suppose $W_{1,k}>t_k$ and let $n$ be the index of the
$\lceil t_k\rceil$-th sensed item with $S_{i-1}\le z_k$; then
$n\ge\lceil t_k\rceil$ and,
\[
    M_{n-1}
    =S_{n-1}
     -\sum_{i=1}^{n-1}\E^{\pi_k}_1[Y_i\mid\calH_{i-1}]
    \ \le\ z_k-\mu_0(\lceil t_k\rceil-1)
    \ \le\ -\tfrac{\mu_0}{2}\,t_k
\]
for all large $k$, because
$t_k\ge\Delta_k/(4B_{\mathrm{inc}})\to\infty$ and $z_k=1$. Hence
writing $n':=n-1$ gives
\[
    \{W_{1,k}>t_k\}\subseteq
    \{\exists\,n'\le N_{\main,k}-1:\,-M_{n'}\ge\mu_0t_k/2\}.
\]
Since $V_{n'}\le C_VB_{\mathrm{inc}}\,N_{\main,k}$ surely for every
$n'\le N_{\main,k}$, the same maximal
Freedman inequality yields
\[
    \Pp^{\pi_k}_1\bigl(W_{1,k}>t_k\bigr)
    \le\exp\left\{-\frac{(\mu_0t_k/2)^2}
        {2\bigl(C_VB_{\mathrm{inc}}N_{\main,k}+2B_{\mathrm{inc}}\cdot\mu_0t_k/2\bigr)}\right\}
    \le\exp\left\{-c_b\frac{\widetilde\Delta_k^2}{L_k}\right\},
\]
using $t_k=\widetilde\Delta_k/(4B_{\mathrm{inc}})$ and
$C_VB_{\mathrm{inc}}N_{\main,k}\le C_1L_k$ again. Adding
(a) and (b),
$\Pp^{\pi_k}_1(E_{\fb})\le2\exp\{-c'\widetilde\Delta_k^2/L_k\}$ with
$c':=\min\{c_a,c_b\}$.

It remains to replace $\widetilde\Delta_k =\Delta_k + \mu_0 R_k$ by $\Delta_k+c_GG_k^+$.  Set
$\bar\nu:=\max\{\nu_1^*,\nu_0^*\}$ and
$c_G:=\min\{\mu_0/2,\,1/(2\bar\nu),\,1/2\}$.  If $G_k^+\ge2(\bar\nu\Delta_k+2)$,
then \eqref{eq:reserve-inventory} gives $R_k\ge G_k^+/2$, so
$\widetilde\Delta_k\ge\Delta_k+(\mu_0/2)G_k^+\ge\Delta_k+c_GG_k^+$.
Otherwise $G_k^+<2\bar\nu\Delta_k+4$, so
$\Delta_k+c_GG_k^+\le(1+2c_G\bar\nu)\Delta_k+4c_G\le2\Delta_k+2
\le2\widetilde\Delta_k+2$.  In both cases
$\widetilde\Delta_k\ge\tfrac12(\Delta_k+c_GG_k^+)-1$, and since
$\Delta_k\to\infty$, for all large $k$ we have
$\widetilde\Delta_k^2\ge\tfrac18(\Delta_k+c_GG_k^+)^2$.  This proves
\eqref{eq:reserve-fallback-prob} with $c_0:=c'/8$ and $C_0:=2$.

\emph{Step 4: assembly, and parts (i)--(ii).}  Write
$C^{\pi_k}=\cdata N_{\main,k}+C^{\main}_k+C^{\fb}_k$ as in the proof of
Theorem~\ref{thm:first-order}.  Costs 1 and 2 in the proof of
Lemma~\ref{lem:main-cost} bound the boundary-crossing runs by
$\Gamma_h(T_{h,k},\overline N_{\main,k})+O(1)
+c_{\max}C_W(z_k+1)$, where the occupation term is $O(1)$ because
$z_k=1$.  
% Cost 3 in that proof bounds the fallback-scan cost by $c_{\max}N_{\main,k}\Pp^{\pi_k}_h(E_{\fb})$. 
The proof of
Lemma~\ref{lem:fallback-cost} gives
$\E^{\pi_k}_h[C^{\fb}_k]\le CL_k\Pp^{\pi_k}_h(E_{\fb})$.  Since
$N_{\main,k}=O(L_k)$, collecting terms, taking the maximum over $h$, and
applying Lemma~\ref{lem:LB-perturb} with $f_{\fb,k}=1/2$ and
\eqref{eq:reserve-fallback-prob},
\[
    \max_h\E^{\pi_k}_h[C^{\pi_k}]
    \le\LB_k+C_P\Delta_k+O(\log L_k)
    +CL_k\exp\left\{-c_0\frac{(\Delta_k+c_GG_k^+)^2}{L_k}\right\}.
\]
Dividing by $\LB_k=\Theta(L_k)$ yields \eqref{eq:reserve-refined-rate}.

For part (i), take $\Delta_k=\kappa\sqrt{L_k\log L_k}$.  Using only
$(\Delta_k+c_GG_k^+)^2\ge\Delta_k^2=\kappa^2L_k\log L_k$, the exponential term in
\eqref{eq:reserve-refined-rate} is at most $L_k^{-c_0\kappa^2}=o(L_k^{-1/2})$,
because $c_0\kappa^2>1/2$ by the choice $\kappa>1/\sqrt{2c_0}$, while
$\Delta_k/L_k=\kappa\sqrt{\log L_k/L_k}=\widetilde O(L_k^{-1/2})$ dominates
$\log L_k/L_k$.  The ratio is therefore $1+\widetilde O(L_k^{-1/2})$, with no
condition on $G_k$.

For part (ii), suppose $G_k^+\ge C\sqrt{L_k\log L_k}$ and take $\Delta_k$
polylogarithmic with $\Delta_k\to\infty$.  Then
$(\Delta_k+c_GG_k^+)^2\ge c_G^2C^2L_k\log L_k$, so the exponential term in
\eqref{eq:reserve-refined-rate} is at most $L_k^{-c_0c_G^2C^2}$, which is
$O(L_k^{-x})$ for any preassigned fixed $x$ once the primitive constant $C$
is large enough, while $\Delta_k/L_k$ and $\log L_k/L_k$ are
$\widetilde O(1/L_k)$.  The ratio is therefore $1+\widetilde O(1/L_k)$,
sharpening part (i).
\end{proof}

\section{Additional Materials for Section~\ref{sec:unknown-ai-pilot}}
\subsection{Pilot Concentration}\label{appendix-sec:unknown-ai-pilot}

\begin{lemma}[Pilot concentration]
\label{lem:pilot-concentration}
Under Assumption~\ref{ass:pilot-sample}, $\Pp_h(\calE_{\pilot,m})\ge1-\delta_m$ for $h\in\{0,1\}$.
\end{lemma}

\begin{proof}{Proof of Lemma~\ref{lem:pilot-concentration}}
The proof proceeds in three steps: a Hoeffding bound on the empirical
frequencies $M_x(r)/M_x$, a deterministic bound on
$|\hat f_x(r)-M_x(r)/M_x|$, and a combination of the two that also
yields the rate.

\emph{Step 1: concentration of the empirical frequencies.}
Fix $h\in\{0, 1\}$, $x\in\{0,1\}$ and $r\in\calR$. Under
Assumption~\ref{ass:pilot-sample}, conditional on $M_x$ the indicators
$\ind\{R^{\pilot}_{x,j}=r\}$, $1\le j\le M_x$, are i.i.d.\
$\Bern(f_x(r))$, and $M_x(r)$ is their sum. Then
$\E[M_x(r)/M_x|M_x] = f_x(r)$. Hoeffding's inequality for sums of
independent $[0,1]$-valued random variables gives, for every $t>0$,
\[
    \Pp_h\!\left(\left|\frac{M_x(r)}{M_x}-f_x(r)\right|>t\Bigg|M_x\right)
    \le 2e^{-2M_xt^2}\leq 2e^{-2mt^2} \text{ almost surely}.
\]
Then using the law of iterated expectation we have
\[
    \Pp_h\!\left(\left|\frac{M_x(r)}{M_x}-f_x(r)\right|>t\right)
    \le 2e^{-2mt^2}.
\]
Evaluating the bound at $t=t_m$, we obtain
\[
    \Pp_h\!\left(\left|\frac{M_x(r)}{M_x}-f_x(r)\right|>t_m\right)
    \le 2e^{-2mt_m^2}
    =2\exp\!\left(-\log\frac{4|\calR|}{\delta_m}\right)
    =\frac{\delta_m}{2|\calR|}.
\]
Since there are exactly $2|\calR|$ pairs $(x,r)$, a union bound gives
\[
    \Pp_h(\calE_{\mathrm{emp}})\ge 1-2|\calR|\cdot\frac{\delta_m}{2|\calR|}
    =1-\delta_m,
    \qquad\text{where}\quad
    \calE_{\mathrm{emp}}
    :=\left\{\max_{x\in\{0,1\},\,r\in\calR}
      \left|\frac{M_x(r)}{M_x}-f_x(r)\right|\le t_m\right\}.
\]

\emph{Step 2: upper bound on $|\hat f_x(r)-M_x(r)/M_x|$.}
Notice that
\[
    \hat f_x(r)
    =\frac{M_x(r)+\lambda_m}{M_x+|\calR|\lambda_m}
    =(1-w_x)\,\frac{M_x(r)}{M_x}+w_x\cdot\frac{1}{|\calR|},
    \qquad
    w_x:=\frac{|\calR|\lambda_m}{M_x+|\calR|\lambda_m}\in(0,1),
\]
so the smoothed estimate \eqref{eq:pilot-estimator} is a convex
combination of the empirical frequency and the uniform mass function.
Consequently, for every $x$ and $r$,
\[
    \left|\hat f_x(r)-\frac{M_x(r)}{M_x}\right|
    =w_x\left|\frac{1}{|\calR|}-\frac{M_x(r)}{M_x}\right|
    \le w_x
    \le\frac{|\calR|\lambda_m}{M_x}
    \le\frac{|\calR|\lambda_m}{m},
\]
where the first inequality holds because $1/|\calR|$ and $M_x(r)/M_x$
both lie in $[0,1]$, and the last two inequalities use
$M_x+|\calR|\lambda_m\ge M_x\ge m$.

\emph{Step 3: combination and rate.}
On $\calE_{\mathrm{emp}}$ the triangle inequality yields, for all $x$
and $r$,
\[
    |\hat f_x(r)-f_x(r)|
    \le\left|\hat f_x(r)-\frac{M_x(r)}{M_x}\right|
      +\left|\frac{M_x(r)}{M_x}-f_x(r)\right|
    \le t_m+\frac{|\calR|\lambda_m}{m}.
\]
Because $\hat f_x(r)$ and $f_x(r)$ are both probabilities in $[0,1]$,
we always have $|\hat f_x(r)-f_x(r)|\le1$, so on
$\calE_{\mathrm{emp}}$,
\[
    \max_{x\in\{0,1\},\,r\in\calR}|\hat f_x(r)-f_x(r)|
    \le\min\left\{1,\;t_m+\frac{|\calR|\lambda_m}{m}\right\}=r_m.
\]
Hence $\calE_{\mathrm{emp}}\subseteq\calE_{\pilot,m}$ and
$\Pp_h(\calE_{\pilot,m})\ge\Pp_h(\calE_{\mathrm{emp}})\ge1-\delta_m$.
\end{proof}

\section{Additional Materials for Section~\ref{sec:unknown-ai-guarded-policy} on Feasibility}\label{appendix-sec:pilot-feasibility-proof}
Throughout this section, channel-derived quantities without hats are
evaluated at the true channel $f=(f_0,f_1)$, whereas hatted quantities
are evaluated at the plug-in channel $\hat f=(\hat f_0,\hat f_1)$.
Conditional on $\calD_m$, let $\E^{\hat\pi_{m,k}}_h$ denote expectation
under $H_h$ for the process induced by the plug-in policy
$\hat\pi_{m,k}$ and the true report channel $f$, and let
$\E^{\hat\pi_{m,k}}_{h,\hat f}$ denote the corresponding expectation
when $f$ is replaced by $\hat f$, with the label law
$X\sim\Bern(p_h)$ unchanged.

We index the main-stage
likelihood-ratio updates by $t=1,2,\ldots$ and recall $\calH_t:=\sigma\bigl(A_1,O_1,\ldots,A_t,O_t\bigr)$ (with $\mathcal H_0$ the trivial $\sigma$-algebra $\{\emptyset,\Omega\}$) is the filtration generated by the actions and 
observations up to and including update $t$. Let
$\calH^m_t:=\sigma(\calD_m)\vee\calH_t$ be its pilot-augmented
version. Also let
$\xi_t$ be the exact log-likelihood increment \eqref{eq:xi-increment}
of the observation revealed by $\hat\pi_{m,k}$ at update $t$, and
$S_t:=\sum_{s\le t}\xi_s$. Since $f_0, f_1$ are unknown, $S_t$ is also unobservable by the policy $\hat\pi_{m,k}$. Let 
\begin{equation}\label{eq:pilot-xi-increment}
    \hat\xi_t:=
    \begin{cases}
        \ell_X(X_i),
            & A_t=\Hum(i)\text{ with item $i$ not previously AI-queried},\\
        \hat\ell_R(R_i),
            & A_t=\AI(i),\\
        \hat\ell_H(X_i,R_i),
            & A_t=\Hum(i)\text{ with report $R_i$ already observed}.
    \end{cases}
\end{equation}
Then $\hat S_t := \sum_{s\leq t}\hat\xi_s$ is the statistic used in $\hat\pi_{m,k}$. 

We first present an auxiliary Lemma~\ref{lem:pilot-perturbation-1} in Appendix~\ref{sec-appendix:likelihood-ratio-bound}, and then in Appendix~\ref{appendix-sec:pilot-validity-proof} we present the proofs of Theorem~\ref{thm:pilot-validity} and Corollary~\ref{cor:pilot-unconditional-validity}.

\subsection{Likelihood Ratio Bound}\label{sec-appendix:likelihood-ratio-bound}

Define 
\begin{equation}\label{Appendix-eq:pilot-weight-intervals}
    \underline w_m(r):=
       \frac{\underline f_{1,m}(r)}
            {\underline f_{1,m}(r)+\overline f_{0,m}(r)},
    \qquad
    \overline w_m(r):=
       \frac{\overline f_{1,m}(r)}
            {\overline f_{1,m}(r)+\underline f_{0,m}(r)},
\end{equation}
where $\underline f_{x,m}(r):=[\hat f_x(r)-r_m]_+$ and $\overline f_{x,m}(r):=\min\{1,\hat f_x(r)+r_m\}$ are the lower and upper bound on $f_x(r)$ on the pilot good event $\calE_{\pilot,m}$. Let 
\begin{equation}\label{Appendix-eq:pilot-phi}
    \phi_{\mathrm{rep}}(w):=
    \log\frac{(1-p_1)(1-w)+p_1w}
                   {(1-p_0)(1-w)+p_0w},\qquad 0\le w\le1.
\end{equation}
Then we have that $\phi_{\mathrm{rep}}(w(r)) = \ell_R(r)$ with $w(r):=\frac{f_1(r)}{f_0(r)+f_1(r)}$, and $\phi_{\mathrm{rep}}(\hat w(r)) = \hat\ell_R(r)$ with $\hat w(r):=\frac{\hat f_1(r)}{\hat f_0(r)+\hat f_1(r)}$ for all $r\in \calR$. Finally we let $\varepsilon_m^{\mathrm{LR}}$ be an upper bound on $|\hat\ell_R(r)-\ell_R(r)|$ on $\calE_{\pilot,m}$:
\begin{equation}\label{Appendix-eq:pilot-eps-LR}
    \varepsilon_m^{\mathrm{LR}}
    :=\max_{r\in\calR}\max\left\{
       \hat\ell_R(r)-\phi_{\mathrm{rep}}(\underline w_m(r)),
       \phi_{\mathrm{rep}}(\overline w_m(r))-\hat\ell_R(r)\right\}.
\end{equation}

\begin{lemma}[Likelihood ratio bound]
\label{lem:pilot-perturbation-1}
\begin{enumerate}[leftmargin=2em,label=(\roman*)]
\item For every $x\in\{0,1\}$ and $r\in\calR$,
\begin{equation}\label{eq:pilot-cancellation}
    \ell_R(r)+\ell_H(x,r)=\ell_X(x),
    \qquad
    \hat\ell_R(r)+\hat\ell_H(x,r)=\ell_X(x).
\end{equation}
In particular, 
$|\ell_X|,|\ell_R|,|\hat\ell_R|\le B_{\ell}$ and
$|\ell_H|,|\hat\ell_H|\le2B_{\ell}$.

\item 
On $\calE_{\pilot,m}$,
\begin{equation}\label{eq:pilot-report-perturbation}
    |\hat\ell_R(r)-\ell_R(r)|\le
       \varepsilon_m^{\mathrm{LR}}\quad(r\in\calR),
\end{equation}
and, on every realized transcript involving at most $N$ items, $|\hat S_t-S_t|\le N\varepsilon_m^{\mathrm{LR}}$
at every update time.

\end{enumerate}
\end{lemma}

\begin{proof}{Proof of Lemma~\ref{lem:pilot-perturbation-1}}
Throughout, recall that $f_x(r)>0$ for all $x,r$
(Assumption~\ref{ass:score-model}; Section~\ref{sec:unknown-ai} relaxes
only the knowledge of $f_0,f_1$, not their positivity), and that
$\hat f_x(r)>0$ and $\sum_r\hat f_x(r)=1$ by the smoothing in
\eqref{eq:pilot-estimator}. Hence all objects below are well-defined
for both the true channel $f=(f_0,f_1)$ and the plug-in channel
$\hat f=(\hat f_0,\hat f_1)$.

\emph{Part (i).}
We prove the two identities in \eqref{eq:pilot-cancellation} in turn.

For the first identity, fix $x\in\{0,1\}$ and $r\in\calR$. By
\eqref{eq:q-h-def}, $q_h(r)=p_hf_1(r)/g_h(r)$, hence
$1-q_h(r)=(g_h(r)-p_hf_1(r))/g_h(r)=(1-p_h)f_0(r)/g_h(r)$, and the
conditional-label mass function \eqref{eq:rho-def} can be written as
\[
 \rho_h(x\mid r)
 =\frac{p_h^x(1-p_h)^{1-x}f_x(r)}{g_h(r)},
 \qquad h\in\{0,1\}.
\]
In the ratio of the two hypotheses the factor $f_x(r)$ is common to
the numerator and the denominator and cancels:
\[
 \frac{\rho_1(x\mid r)}{\rho_0(x\mid r)}
 =\frac{p_1^x(1-p_1)^{1-x}}{p_0^x(1-p_0)^{1-x}}
  \cdot\frac{g_0(r)}{g_1(r)}
 =e^{\ell_X(x)}\,\frac{g_0(r)}{g_1(r)}.
\]
Taking logarithms and adding $\ell_R(r)=\log(g_1(r)/g_0(r))$ on both
sides gives $\ell_R(r)+\ell_H(x,r)=\ell_X(x)$.

For the second identity, the plug-in objects obey the same algebra:
by \eqref{eq:pilot-hat-g-q}, $\hat q_h(r)=p_h\hat f_1(r)/\hat g_h(r)$,
hence
$1-\hat q_h(r)=(\hat g_h(r)-p_h\hat f_1(r))/\hat g_h(r)
=(1-p_h)\hat f_0(r)/\hat g_h(r)$, and
\eqref{eq:pilot-hat-g-q} gives
\[
 \hat\rho_h(x\mid r)
 =\frac{p_h^x(1-p_h)^{1-x}\hat f_x(r)}{\hat g_h(r)}.
\]
Repeating the computation above with $\hat f_x,\hat g_h,\hat\rho_h$ in
place of $f_x,g_h,\rho_h$ (the factor $\hat f_x(r)$ now cancels in the
ratio) yields $\hat\ell_R(r)+\hat\ell_H(x,r)=\ell_X(x)$. The
right-hand side is the \emph{true} $\ell_X$ because $p_0,p_1$ are
known and enter $\hat g_h,\hat q_h$ unchanged.

For the increment bounds, use the normalized weights defined above. Since
\[
 (1-p_h)(1-w(r))+p_hw(r)
 =\frac{(1-p_h)f_0(r)+p_hf_1(r)}{f_0(r)+f_1(r)}
 =\frac{g_h(r)}{f_0(r)+f_1(r)},
\]
the normalizer $f_0(r)+f_1(r)$ cancels in the ratio defining
$\phi_{\mathrm{rep}}$ in \eqref{Appendix-eq:pilot-phi}, and
$\ell_R(r)=\phi_{\mathrm{rep}}(w(r))$; the same computation with
$\hat f$ in place of $f$ gives
$\hat\ell_R(r)=\phi_{\mathrm{rep}}(\hat w(r))$. Because
$\phi_{\mathrm{rep}}$ is increasing with
$\phi_{\mathrm{rep}}(0)=\log\frac{1-p_1}{1-p_0}=\ell_X(0)$ and
$\phi_{\mathrm{rep}}(1)=\log\frac{p_1}{p_0}=\ell_X(1)$, both $\ell_R$ and
$\hat\ell_R$ take values in
$[\ell_X(0),\ell_X(1)]\subseteq[-B_{\ell},B_{\ell}]$.
Finally, \eqref{eq:pilot-cancellation} and the triangle inequality
give $|\ell_H|=|\ell_X-\ell_R|\le2B_{\ell}$ and likewise
$|\hat\ell_H|\le2B_{\ell}$.

\emph{Part (ii).}
First, direct differentiation gives
\[
 \phi_{\mathrm{rep}}'(w)=
 \frac{p_1-p_0}{\{(1-p_1)(1-w)+p_1w\}
 \{(1-p_0)(1-w)+p_0w\}}>0,
\]
so $\phi_{\mathrm{rep}}$ is increasing.
On $\calE_{\pilot,m}$ we have $|\hat f_x(r)-f_x(r)|\le r_m$, so
$f_x(r)\ge\max\{\hat f_x(r)-r_m,0\}=\underline f_{x,m}(r)$ and
$f_x(r)\le\min\{1,\hat f_x(r)+r_m\}=\overline f_{x,m}(r)$;
deterministically also
$\hat f_x(r)\in[\underline f_{x,m}(r),\overline f_{x,m}(r)]$.
The denominators in
\eqref{Appendix-eq:pilot-weight-intervals} are positive because
$\underline f_{1,m}(r)+\overline f_{0,m}(r)
\ge\overline f_{0,m}(r)\ge\hat f_0(r)>0$ and
$\overline f_{1,m}(r)+\underline f_{0,m}(r)
\ge\overline f_{1,m}(r)\ge\hat f_1(r)>0$. The map
$(a,b)\mapsto a/(a+b)$ is nondecreasing in $a\ge0$ and nonincreasing in
$b\ge0$ on $\{a+b>0\}$, so both
$w(r)$ and $\hat w(r)$ defined above lie in
$[\underline w_m(r),\overline w_m(r)]$. By the representation
established in part (i), $\ell_R(r)=\phi_{\mathrm{rep}}(w(r))$ and
$\hat\ell_R(r)=\phi_{\mathrm{rep}}(\hat w(r))$, and
$\phi_{\mathrm{rep}}$ is increasing, so both
lie in the interval
$[\phi_{\mathrm{rep}}(\underline w_m(r)),
  \phi_{\mathrm{rep}}(\overline w_m(r))]$. Two numbers in
a common interval differ by at most the distance from either of them to
the farther endpoint; hence
\[
 |\hat\ell_R(r)-\ell_R(r)|
 \le\max\left\{
     \hat\ell_R(r)-\phi_{\mathrm{rep}}(\underline w_m(r)),\;
     \phi_{\mathrm{rep}}(\overline w_m(r))-\hat\ell_R(r)\right\}
 \le\varepsilon_m^{\mathrm{LR}},
\]
proving \eqref{eq:pilot-report-perturbation}. 

Moreover, $\hat S_t-S_t
 = \sum_{i:\;\text{AI-scored, label not yet revealed}}
   \bigl(\hat\ell_R(R_i)-\ell_R(R_i)\bigr)$ according to part (i). Therefore we have $|\hat S_t-S_t|\leq N\varepsilon_m^{\mathrm{LR}}$.
\end{proof}

\subsection{Proofs of Theorem~\ref{thm:pilot-validity} and Corollary~\ref{cor:pilot-unconditional-validity}}\label{appendix-sec:pilot-validity-proof}
\begin{proof}{Proof of Theorem~\ref{thm:pilot-validity}}
Fix a realized pilot in $\calE_{\pilot,m}$ and condition on $\calD_m$
throughout; all plug-in quantities and design coefficients
($\hat f_x$, $\hat\ell_R$, $\hat\ell_H$, $\widehat N_{\main,m,k}$,
$\hat\eta^{\Hum}_{J,m,k}$, $\hat\eta^{\esc}_{J,m,k}$, and
$\omega_{m,k}$) are then deterministic.

If the guarded outer problem \eqref{eq:pilot-Nbar-choice} is infeasible, the safe-default policy queries a human on every item. Hence every observed increment equals the exact direct-label increment \(\ell_X\), so \(\hat\xi_t=\xi_t\) and, pathwise, $\hat S_t=S_t$ for every $t$.
Therefore, the boundary-route argument below applies directly with \(\omega_{m,k}=0\). It remains to consider the guarded branch, in which $\omega_{m,k}
=\widehat N_{\main,m,k}\varepsilon_m^{\mathrm{LR}}$.

\emph{Step 1: conditional martingale property of the true likelihood
process.}  The policy selects each sensing action using the plug-in
statistic $\hat S_{t-1}$, and that action depends only on the fixed pilot, the
revealed history, and a fresh hypothesis-independent randomization seed; it
does not depend on the not-yet-revealed observation $O_t$.  Moreover, the main
sample and the policy randomization are independent of $\calD_m$ by
Assumption~\ref{ass:pilot-sample}.  Thus, conditionally on
$(\calH^m_{t-1},A_t)$, the observation $O_t$ retains its true model law, and
the one-step likelihood-ratio identities
\eqref{eq:mart-direct}--\eqref{eq:mart-escalation}, followed by averaging over
$A_t$, give
\[
    \E^{\hat\pi_{m,k}}_0[e^{\xi_t}\mid\calH^m_{t-1}]=1,
    \qquad
    \E^{\hat\pi_{m,k}}_1[e^{-\xi_t}\mid\calH^m_{t-1}]=1.
\]
Consequently, $e^{S_t}$ is a nonnegative $(\calH^m_t)$-martingale under
$\Pp^{\hat\pi_{m,k}}_0(\cdot\mid\calD_m)$, and $e^{-S_t}$ is one under
$\Pp^{\hat\pi_{m,k}}_1(\cdot\mid\calD_m)$, both with initial value $1$.

\emph{Step 2: boundary routes.} Let $\tau_{\mathrm b}$ be the index of the last paid update in the
main stage: the update at which a boundary is first crossed, or, if no
crossing occurs, the last paid update before entering fallback.
Let $\hat E_+:=\bigl\{\hat S_{\tau_{\mathrm{b}}}\ge a_k+\omega_{m,k}\text{ and policy stops before fallback}\bigr\}$ and $\hat E_-:=\bigl\{\hat S_{\tau_{\mathrm{b}}}\le-(b_k+\omega_{m,k})\text{ and policy stops before fallback}\bigr\}$.

Given $\calD_m$, the main-stage horizon is
fixed and each item contributes at most two updates; hence $\tau_{\mathrm{b}}$ is an
$(\calH^m_t)$-stopping time bounded by $2\widehat N_{\main,m,k}$.  Step~1 and
bounded optional stopping therefore give
\[
 \E^{\hat\pi_{m,k}}_0[e^{S_{\tau_{\mathrm{b}}}}\mid\calD_m]=1,
 \qquad
 \E^{\hat\pi_{m,k}}_1[e^{-S_{\tau_{\mathrm{b}}}}\mid\calD_m]=1.
\]
On $\hat E_+$, Lemma~\ref{lem:pilot-perturbation-1}(ii) and the guard definition of $\omega_{m,k}$ imply $S_{\tau_{\mathrm{b}}}\ \ge\ \hat S_{\tau_{\mathrm{b}}}
   -\widehat N_{\main,m,k}\varepsilon_m^{\mathrm{LR}}
 \ =\ \hat S_{\tau_{\mathrm{b}}}-\omega_{m,k}
 \ \ge\ a_k$,
and hence
\[
 \Pp^{\hat\pi_{m,k}}_0(\hat E_+\mid\calD_m)
 \le\alpha_{1,k}\,
    \E^{\hat\pi_{m,k}}_0\bigl[e^{S_{\tau_{\mathrm{b}}}}\ind_{\hat E_+}\mid\calD_m\bigr]
 \le\alpha_{1,k}\,\E^{\hat\pi_{m,k}}_0
    \bigl[e^{S_{\tau_{\mathrm{b}}}}\mid\calD_m\bigr]
 =\alpha_{1,k}.
\]
Similarly, on $\hat E_-$,
$S_{\tau_{\mathrm{b}}}\le\hat S_{\tau_{\mathrm{b}}}+\omega_{m,k}\le-b_k$, so
$\Pp^{\hat\pi_{m,k}}_1(\hat E_-\mid\calD_m)\le\beta_{1,k}$.  

\emph{Step 3: fallback route.} On the complement
$\hat E_{\fb}:=(\hat E_+\cup\hat E_-)^c$, the policy reveals all labels in the fixed set $\hat \calJ_{\Hum,k}$ and applies the exact randomized full-label
Neyman--Pearson test at levels $(\alpha_{2,k},\beta_{2,k})$.  The set is
available because $\widehat N_{\main,m,k}\ge
N_{\fixed,\Hum}(\alpha_{2,k},\beta_{2,k})$ in both branches of the policy.  Its labels
and the test randomization are independent of $\calD_m$, so the test retains
these levels conditionally on $\calD_m$.  Finally, intersecting a test-error
event with $\hat E_{\fb}$ can only reduce its probability (equivalently,
$\ind_{\hat E_{\fb}}\le1$).  Thus Lemma~\ref{lem:fallback-error} gives
\[
 \Pp^{\hat\pi_{m,k}}_0
 \bigl(\hat E_{\fb}\cap\{\text{fallback rejects}\}
    \mid\calD_m\bigr)\le\alpha_{2,k},
 \qquad
 \Pp^{\hat\pi_{m,k}}_1
 \bigl(\hat E_{\fb}\cap\{\text{fallback accepts}\}
    \mid\calD_m\bigr)\le\beta_{2,k}.
\]

\emph{Step 4: route decomposition.} The events
$\hat E_+,\hat E_-,\hat E_{\fb}$ are disjoint and exhaustive (at the
first exit time $\hat S_{\tau_{\mathrm{b}}}$ cannot be simultaneously
$\ge a_k+\omega_{m,k}>0$ and $\le-(b_k+\omega_{m,k})<0$), and the
policy rejects through exactly one of two disjoint routes:
\[
 \{\delta^{\hat\pi_{m,k}}=1\}
 =\hat E_+\ \cup\
  \bigl(\hat E_{\fb}\cap\{\text{fallback rejects}\}\bigr),
\]
a disjoint union. Steps 2 and 3 and the budget split
$\alpha_{1,k}+\alpha_{2,k}=\alpha_k$ of \eqref{eq:budget-split-alpha}
give
\[
 \Pp^{\hat\pi_{m,k}}_0
 (\delta^{\hat\pi_{m,k}}=1\mid\calD_m)
 \le\alpha_{1,k}+\alpha_{2,k}=\alpha_k.
\]
Similarly we have $\Pp^{\hat\pi_{m,k}}_1
 (\delta^{\hat\pi_{m,k}}=0\mid\calD_m)
 \le\beta_{1,k}+\beta_{2,k}=\beta_k$.
This proves Theorem~\ref{thm:pilot-validity}.
\end{proof}

\begin{proof}{Proof of Corollary~\ref{cor:pilot-unconditional-validity}}
The event $\calE_{\pilot,m}$ is $\sigma(\calD_m)$-measurable: in
\eqref{eq:pilot-good-event}, $\hat f_x$ is computed from $\calD_m$,
while $f_x$ and $r_m$ are deterministic. By the tower property,
\[
 \Pp^{\hat\pi_{m,k}}_0(\delta^{\hat\pi_{m,k}}=1)
 =\E_0\bigl[\Pp^{\hat\pi_{m,k}}_0
    (\delta^{\hat\pi_{m,k}}=1\mid\calD_m)\bigr]
 \le\alpha_k\,\Pp_0(\calE_{\pilot,m})
  +\Pp_0(\calE_{\pilot,m}^{\,c})
 \le\alpha_k+\delta_m,
\]
where the expectation $\E_0$ is with respect to the pilot data $\calD_m$ under $H_0$, 
the first inequality applies Theorem~\ref{thm:pilot-validity}
on $\calE_{\pilot,m}$ and bounds the conditional probability by $1$
on the complement, and the second uses
$\Pp_0(\calE_{\pilot,m}^{\,c})\le \delta_m$ from
Lemma~\ref{lem:pilot-concentration}. The type-II bound follows in the
same way under $\Pp^{\hat\pi_{m,k}}_1$ applied to
$\{\delta^{\hat\pi_{m,k}}=0\}$ with $\beta_k$ in place of $\alpha_k$,
proving Corollary~\ref{cor:pilot-unconditional-validity}.
\end{proof}

\section{Additional Materials for Section~\ref{sec:unknown-ai-guarded-policy} on First-order Cost Optimality}\label{appendix-sec:pilot-first-order-proof}

Conditional on $\calD_m$, let $\E^{\hat\pi_{m,k}}_h$ denote expectation
under $H_h$ for the process induced by the plug-in policy
$\hat\pi_{m,k}$ and the true report channel $f$, and let
$\E^{\hat\pi_{m,k}}_{h,\hat f}$ denote the corresponding expectation
when $f$ is replaced by $\hat f$, with the label law
$X\sim\Bern(p_h)$ unchanged.

Throughout this section, we index the
main-stage items by $i=1,\ldots,\widehat N_{\main,m,k}$, let
$\calH^m_{i-1}$ be the pilot-augmented filtration generated by
$\calD_m$ and the first $i-1$ processed items, let
$J_i\in\{0,1,*\}$ be the regime of item $i$ (determined from
$\hat S_{i-1}$ and the dead zone $\pm z_k$ in
Algorithm~\ref{algo:pilot-policy}, hence
$\calH^m_{i-1}$-measurable). Write $U_i,V_i$ for the fresh
randomization seeds of item $i$, and define its plug-in-statistic increment and sensing cost by
\begin{equation}\label{eq:hat-Z-C-def}
\begin{aligned}
    \hat Z^{\mathrm{item}}_i(\eta)
    &:=\ind\{U_i\le\eta^{\Hum}\}\,\ell_X(X_i)
    +\ind\{U_i>\eta^{\Hum}\}\Bigl(
        \hat\ell_R(R_i)
        +\ind\{V_i\le\eta^{\esc}(R_i)\}\,
          \hat\ell_H(X_i,R_i)
    \Bigr),\\
    C^{\mathrm{item}}_i(\eta)
    &:=\ind\{U_i\le\eta^{\Hum}\}\,\cH
    +\ind\{U_i>\eta^{\Hum}\}\Bigl(
        \cAI+\ind\{V_i\le\eta^{\esc}(R_i)\}\,\cH
    \Bigr).
\end{aligned}
\end{equation}
Then, for $j\in\{0,1,*\}$, $\hat Z^{\mathrm{item}}_i(\hat\eta_{j,m,k})$ is the one-item
plug-in-statistic increment produced by the rule
$\hat\eta_{j,m,k}:=(\hat\eta^{\Hum}_{j,m,k},\hat\eta^{\esc}_{j,m,k})$ on a fresh item. Let $\hat S_i=\sum_{\ell\le i}\hat Z^{\mathrm{item}}_\ell(\hat\eta_{J_\ell,m,k})$ denote the cumulative plug-in log-likelihood statistic after completing the first \(i\) items, where $J_\ell$ is the sensing rule used by item $\ell$. For the item-level analysis below, we use a completed-item continuation. We generate \((X_i,R_i,U_i,V_i)\) for every acquired item and recursively define \(J_i\) from \(\hat S_{i-1}\) and $\hat S_i
=\hat S_{i-1}
+\hat Z_i^{\mathrm{item}}(\hat\eta_{J_i,m,k})$.
If the actual policy stops during or before item \(i\), we counterfactually complete that item and continue the recursion for the remaining items. This continuation is used only for the proof and does not represent additional queries or costs incurred by the actual policy.

In Appendix~\ref{appendix-sec:uniform-pilot-perturbation}, we present auxiliary Lemma~\ref{lem:pilot-perturbation-2} to prove Theorem~\ref{thm:pilot-first-order}. Then in Appendices~\ref{appendix-sec:pilot-first-order-1}--\ref{appendix-sec:pilot-main-cost}, we present the proof of Theorem~\ref{thm:pilot-first-order} and other lemmas.

\subsection{Uniform Pilot Perturbation}\label{appendix-sec:uniform-pilot-perturbation}

Let
\begin{equation}\label{eq:pilot-fstar}
    f_{\star}:=\min_{x\in\{0,1\},\,r\in\calR}f_x(r)
\end{equation}
denote the smallest true channel mass, which is strictly positive by
Assumption~\ref{ass:score-model}. Recall $B_{\ell}=\max\left\{
       \log\frac{p_1}{p_0},\;
       \log\frac{1-p_0}{1-p_1}\right\}$.
We also let $\varepsilon_m^{\mathrm{cost}}:=\cH|\calR|r_m$ be an upper bound (conditional on $\calE_{\pilot,m}$) on the deviation between the true-channel and plug-in-channel expected one-item sensing cost.

For a count triple $(n_{\Hum},n_{\AI},n_{\esc})$ with
$n_{\Hum},n_{\AI}\ge0$ and $0\le n_{\esc}\le n_{\AI}$, define its true
and plug-in information values by
\begin{equation}\label{eq:pilot-plan-information-values}
\begin{aligned}
    I^{(h)}_{\mathrm{plan}}
    &:=n_{\Hum}J_X^{(h)}+n_{\AI}I_R^{(h)}
      +n_{\AI}\Psi_h(n_{\esc}/n_{\AI}),\\
    \hat I^{(h)}_{\mathrm{plan}}
    &:=n_{\Hum}J_X^{(h)}+n_{\AI}\hat I_R^{(h)}
      +n_{\AI}\hat\Psi_h(n_{\esc}/n_{\AI}),
\end{aligned}
\end{equation}
where $\hat\Psi_h(n_{\esc}/n_{\AI})$ when
$n_{\AI}=n_{\esc} =0$. 

\begin{lemma}[Uniform pilot perturbation]
\label{lem:pilot-perturbation-2}
\begin{enumerate}[leftmargin=2em,label=(\roman*)]

\item 
On $\calE_{\pilot,m}$, for every $h\in\{0,1\}$ and every one-item
randomized sensing rule,
\begin{equation}\label{eq:pilot-drift-cost-perturbation}
\begin{aligned}
 &\left|
 \E_h^{\hat\pi_{m,k}}
   [\hat Z^{\mathrm{item}}_i(\eta)\mid\calD_m]
 -\E_{h,\hat f}^{\hat\pi_{m,k}}
   [\hat Z^{\mathrm{item}}_i(\eta)\mid\calD_m]
 \right|
 \le\varepsilon_m^{\mathrm{dr}},\\
 &\left|
 \E_h^{\hat\pi_{m,k}}
   [C^{\mathrm{item}}_i(\eta)\mid\calD_m]
 -\E_{h,\hat f}^{\hat\pi_{m,k}}
   [C^{\mathrm{item}}_i(\eta)\mid\calD_m]
 \right|
 \le\varepsilon_m^{\mathrm{cost}}.
\end{aligned}
\end{equation}

\item There is a primitive constant $L_{\mathrm{loc}}<\infty$ such that, on
$\calE_{\pilot,m}$ and whenever $r_m\le f_{\star}/4$, $\hat q_h$, $1-\hat q_h$ and $\hat g_h$ are lower bounded by a constant that only depends on the primitives, and moreover,
\begin{align}
  \varepsilon_m^{\mathrm{LR}}
  &+\max_h|\hat I_R^{(h)}-I_R^{(h)}|
  +\max_{h,r}|\hat d^{(h)}(r)-d^{(h)}(r)|
  +\max_h\sup_{s\in[0,1]}|\hat\Psi_h(s)-\Psi_h(s)|
  \le L_{\mathrm{loc}}r_m.                   \label{eq:pilot-local-bound}
\end{align}

\item On $\calE_{\pilot,m}$ and whenever
$r_m\le f_{\star}/4$, for every $N\ge1$ and every count triple
$(n_{\Hum},n_{\AI},n_{\esc})$ as in
\eqref{eq:pilot-plan-information-values} with
$n_{\Hum}+n_{\AI}\le N$,
\begin{equation}\label{eq:pilot-plan-information-bound}
    \max_{h\in\{0,1\}}
    \bigl|I^{(h)}_{\mathrm{plan}}-\hat I^{(h)}_{\mathrm{plan}}\bigr|
    \le L_{\mathrm{loc}}Nr_m.
\end{equation}
\end{enumerate}
\end{lemma}

\begin{proof}{Proof of Lemma~\ref{lem:pilot-perturbation-2}}
Throughout, recall that $f_x(r)>0$ for all $x,r$
(Assumption~\ref{ass:score-model}), and that
$\hat f_x(r)>0$ and $\sum_r\hat f_x(r)=1$ by the smoothing in
\eqref{eq:pilot-estimator}. Hence all objects below are well-defined
for both the true channel $f=(f_0,f_1)$ and the plug-in channel
$\hat f=(\hat f_0,\hat f_1)$.

\emph{Part (i).}
Fix $h\in\{0,1\}$, fix a one-item randomized sensing rule
$(\eta^{\Hum},\eta^{\esc})$, and work on $\calE_{\pilot,m}$; the
bounds obtained below are uniform in $h$ and in the rule. Write
$P_h(x,r):=p_h^x(1-p_h)^{1-x}f_x(r)$ and
$\hat P_h(x,r):=p_h^x(1-p_h)^{1-x}\hat f_x(r)$ for the joint mass
functions of $(X,R)$ under the true and plug-in report channels,
respectively. For every $r$,
\[
 |\hat g_h(r)-g_h(r)|
 \le p_h|\hat f_1(r)-f_1(r)|+(1-p_h)|\hat f_0(r)-f_0(r)|
 \le r_m,
\]
so $\|g_h-\hat g_h\|_1\le|\calR| r_m$. The same computation for the
joint laws $P_h$ and $\hat P_h$ gives
\[
 \sum_{x,r}|P_h-\hat P_h|(x,r)
 =\sum_{r}\bigl[p_h|\hat f_1-f_1|(r)+(1-p_h)|\hat f_0-f_0|(r)\bigr]
 \le|\calR| r_m.
\]

\emph{Drift.} Write
$F(x,r):=\hat\ell_R(r)+\eta^{\esc}(r)\hat\ell_H(x,r)$. The
direct-human contributions to the two expectations are identical, so
\[
 \E_h^{\hat\pi_{m,k}}
   [\hat Z^{\mathrm{item}}_i(\eta)\mid\calD_m]
 -\E_{h,\hat f}^{\hat\pi_{m,k}}
   [\hat Z^{\mathrm{item}}_i(\eta)\mid\calD_m]
 =(1-\eta^{\Hum})\sum_{x,r}
   \bigl(P_h-\hat P_h\bigr)(x,r)\,F(x,r).
\]
Lemma~\ref{lem:pilot-perturbation-1}(i) gives $|F|\le B_{\ell}+2B_{\ell}=3B_{\ell}$, so
\[
 \left|
 \E_h^{\hat\pi_{m,k}}
   [\hat Z^{\mathrm{item}}_i(\eta)\mid\calD_m]
 -\E_{h,\hat f}^{\hat\pi_{m,k}}
   [\hat Z^{\mathrm{item}}_i(\eta)\mid\calD_m]
 \right|
 \le(1-\eta^{\Hum})\cdot3B_{\ell}\,
    \sum_{x,r}|P_h-\hat P_h|(x,r)
 \le3|\calR| B_{\ell}r_m
 =\varepsilon_m^{\mathrm{dr}}.
\]

\emph{Cost.} 
The terms $\eta^{\Hum}\cH+(1-\eta^{\Hum})\cAI$ are common to the two
expected costs, so they cancel in the difference, leaving
\begin{align*}
 &\left|
 \E_h^{\hat\pi_{m,k}}
   [C^{\mathrm{item}}_i(\eta)\mid\calD_m]
 -\E_{h,\hat f}^{\hat\pi_{m,k}}
   [C^{\mathrm{item}}_i(\eta)\mid\calD_m]
 \right|\\
 &\quad=(1-\eta^{\Hum})\,\cH\left|\sum_{r}
     \bigl(g_h(r)-\hat g_h(r)\bigr)\eta^{\esc}(r)\right|
 \le\cH\|g_h-\hat g_h\|_1
 \le\cH|\calR| r_m
 =\varepsilon_m^{\mathrm{cost}}.
\end{align*}
This and the drift bound prove
\eqref{eq:pilot-drift-cost-perturbation}.

\emph{Part (ii).}
Work on $\calE_{\pilot,m}$ and assume $r_m\le f_{\star}/4$. For all $x,r$,
$\hat f_x(r)\ge f_x(r)-r_m\ge f_{\star}-f_{\star}/4=3f_{\star}/4$ and
$\underline f_{x,m}(r)\ge\hat f_x(r)-r_m\ge f_{\star}/2$, while
$\hat f_x(r)\le\overline f_{x,m}(r)\le1$. Thus all of
$f_x(r),\hat f_x(r),\underline f_{x,m}(r),\overline f_{x,m}(r)$
lie in $[f_{\star}/2,1]$.

\emph{(a) Show that $\max_{h,r}|\hat d^{(h)}(r)-d^{(h)}(r)|
 +\max_h|\hat I_R^{(h)}-I_R^{(h)}|
 \le C_1r_m = O(r_m)$.} Since $g_h(r)$ and $\hat g_h(r)$ are
convex combinations of $f_0(r),f_1(r)$ and of
$\hat f_0(r),\hat f_1(r)$ respectively, 
\[
 1\ge g_h(r)\ge\min\{f_0(r),f_1(r)\}\ge f_{\star},
 \qquad
 1\ge\hat g_h(r)\ge\min\{\hat f_0(r),\hat f_1(r)\}\ge\frac{3f_{\star}}4,
\]
and the posteriors satisfy
$q_h(r)=p_hf_1(r)/g_h(r)\ge p_0f_{\star}$,
$1-q_h(r)=(1-p_h)f_0(r)/g_h(r)\ge(1-p_1)f_{\star}$, and likewise
$\hat q_h(r)\ge p_0\cdot3f_{\star}/4$ and
$1-\hat q_h(r)\ge(1-p_1)\cdot3f_{\star}/4$. Hence all of
$q_h(r),\hat q_h(r)$ lie in $[c_{\mathrm{post}},1-c_{\mathrm{post}}]$, where
$c_{\mathrm{post}}:=\tfrac12\min\{p_0,1-p_1\}f_{\star}$. Next, writing
\[
 \frac{\hat f_1(r)}{\hat g_h(r)}-\frac{f_1(r)}{g_h(r)}
 =\frac{(\hat f_1(r)-f_1(r))\,g_h(r)
   +f_1(r)\,(g_h(r)-\hat g_h(r))}{g_h(r)\hat g_h(r)}
\]
and using $|\hat f_x(r)-f_x(r)|\le r_m$ together with
$|\hat g_h(r)-g_h(r)|\le r_m$ (proof of part (i)),
\[
 |\hat q_h(r)-q_h(r)|
 \le p_h\,\frac{r_m\cdot1+1\cdot r_m}{f_{\star}\cdot(3f_{\star}/4)}
 =\frac{8\,p_h\,r_m}{3f_{\star}^2}
 \le\frac{8r_m}{3f_{\star}^2}.
\]
The map $\klbin(\cdot\Vert\cdot)$ is continuously differentiable,
hence Lipschitz, on the compact square
$[c_{\mathrm{post}},1-c_{\mathrm{post}}]^2$, with constant
depending only on $c_{\mathrm{post}}$; since both $(q_h(r),q_{1-h}(r))$ and
$(\hat q_h(r),\hat q_{1-h}(r))$ lie in this square,
\[
 |\hat d^{(h)}(r)-d^{(h)}(r)|
 \le C\bigl(|\hat q_h(r)-q_h(r)|+|\hat q_{1-h}(r)-q_{1-h}(r)|\bigr)
 =O(r_m).
\]
Similarly, $\log$ is Lipschitz on $[3f_{\star}/4,1]$ with constant
$4/(3f_{\star})$, and $|\log(\hat g_h(r)/\hat g_{1-h}(r))|
\le\log(4/(3f_{\star}))$, so
\begin{align*}
 |\hat I_R^{(h)}-I_R^{(h)}|
 &\le\sum_r|\hat g_h(r)-g_h(r)|\,
    \left|\log\frac{\hat g_h(r)}{\hat g_{1-h}(r)}\right|
  +\sum_rg_h(r)\,\bigl|\log\hat g_h(r)-\log g_h(r)\bigr|\\
 &\qquad+\sum_rg_h(r)\,
    \bigl|\log\hat g_{1-h}(r)-\log g_{1-h}(r)\bigr|
 \;\le\;|\calR| r_m\log\frac{4}{3f_{\star}}+\frac{8r_m}{3f_{\star}}.
\end{align*}
Collecting terms yields a primitive constant $C_1$, depending only on
$p_0,p_1,f_{\star},|\calR|$, with
\[
 \max_{h,r}|\hat d^{(h)}(r)-d^{(h)}(r)|
 +\max_h|\hat I_R^{(h)}-I_R^{(h)}|
 \le C_1r_m.
\]

\emph{(b) Show that $\varepsilon_m^{\mathrm{LR}} = O(r_m)$.}
Fix $r\in\calR$. Recall by construction \eqref{Appendix-eq:pilot-phi} of $\phi_{\mathrm{rep}}$,
$\hat\ell_R(r)=\phi_{\mathrm{rep}}(\hat w(r))$ with
$\hat w(r) = \frac{\hat f_1(r)}{\hat f_0(r)+\hat f_1(r)}\in[\underline w_m(r),\overline w_m(r)]$. 
$\phi_{\mathrm{rep}}$ is increasing, so
\[
 \phi_{\mathrm{rep}}(\underline w_m(r))\le\hat\ell_R(r)
 \le\phi_{\mathrm{rep}}(\overline w_m(r)).
\]
Consequently both terms inside the inner maximum in the definition
\eqref{Appendix-eq:pilot-eps-LR} of $\varepsilon_m^{\mathrm{LR}}$ are at most
$\phi_{\mathrm{rep}}(\overline w_m(r))
 -\phi_{\mathrm{rep}}(\underline w_m(r))$, whence
\[
 \varepsilon_m^{\mathrm{LR}}
 \le\max_{r\in\calR}\bigl\{
    \phi_{\mathrm{rep}}(\overline w_m(r))
    -\phi_{\mathrm{rep}}(\underline w_m(r))\bigr\}.
\]
We bound the right-hand side in two steps: a uniform bound on
$\phi_{\mathrm{rep}}'$, and a bound on the width
$\overline w_m(r)-\underline w_m(r)$.

First, we have that
\[
 \sup_{w\in[0,1]}\phi_{\mathrm{rep}}'(w)\le
 C_{\phi}:=\frac{p_1-p_0}
   {\min\{p_0,1-p_0\}\,\min\{p_1,1-p_1\}},
\]
and the mean-value theorem gives
$\phi_{\mathrm{rep}}(\overline w_m(r))
-\phi_{\mathrm{rep}}(\underline w_m(r))
\le C_{\phi}\bigl(\overline w_m(r)-\underline w_m(r)\bigr)$.

Second, write $W(a,b):=a/(a+b)$, so that, by
\eqref{Appendix-eq:pilot-weight-intervals},
\[
 \underline w_m(r)
 =W(\underline f_{1,m}(r),\overline f_{0,m}(r)),
 \qquad
 \overline w_m(r)
 =W(\overline f_{1,m}(r),\underline f_{0,m}(r)),
 \qquad
 \hat w(r)=W(\hat f_1(r),\hat f_0(r)),
\]
where all three argument pairs lie in the box
$[f_{\star}/2,1]^2$, on which the partial derivatives
\[
 \frac{\partial W}{\partial a}=\frac{b}{(a+b)^2},
 \qquad
 \frac{\partial W}{\partial b}=-\frac{a}{(a+b)^2}
\]
are bounded in absolute value by $1/(a+b)\le1/f_{\star}$. Moreover,
$r_m\le f_{\star}/4<\hat f_x(r)$ makes the positive-part clip in
$\underline f_{x,m}(r)=[\hat f_x(r)-r_m]_+$ inactive, so
\[
 |\hat f_1(r)-\underline f_{1,m}(r)|=r_m,
 \qquad
 |\overline f_{0,m}(r)-\hat f_0(r)|\le r_m,
\]
and likewise
$|\overline f_{1,m}(r)-\hat f_1(r)|\le r_m$ and
$|\hat f_0(r)-\underline f_{0,m}(r)|=r_m$.
The box is convex, so the segments joining
$(\hat f_1(r),\hat f_0(r))$ to
$(\underline f_{1,m}(r),\overline f_{0,m}(r))$ and to
$(\overline f_{1,m}(r),\underline f_{0,m}(r))$ stay inside it, and the
mean-value theorem yields
\[
 \hat w(r)-\underline w_m(r)
 \le\frac{
   |\hat f_1(r)-\underline f_{1,m}(r)|
   +|\overline f_{0,m}(r)-\hat f_0(r)|}{f_{\star}}
 \le\frac{2r_m}{f_{\star}},
 \qquad
 \overline w_m(r)-\hat w(r)
 \le\frac{2r_m}{f_{\star}},
\]
so $\overline w_m(r)-\underline w_m(r)\le4r_m/f_{\star}$. Combining
the two steps,
\[
 \varepsilon_m^{\mathrm{LR}}
 \le C_{\phi}\cdot\frac{4r_m}{f_{\star}}
 =\frac{4C_{\phi}r_m}{f_{\star}}.
\]

\emph{(c) Show that $\max_h\sup_{s\in[0,1]}|\hat\Psi_h(s)-\Psi_h(s)| = O(r_m)$.} The true frontier $\Psi_h$ is defined by the
inf-representation \eqref{eq:Psi-frontier-def}, and the
dual-representation part of the proof of
Lemma~\ref{lem:capacity-frontier-properties} (a finite linear
programming duality argument) uses only that the report marginal is a
strictly positive mass function and that the follow-up gains are
nonnegative and finite, which $\hat g_h>0$ and
$0\le\hat d^{(h)}<\infty$ satisfy. Hence the plug-in frontier
\eqref{eq:pilot-hat-Psi} admits the same representation. Writing
\[
 \Phi_h(\lambda,s):=\lambda s+
   \sum_{r\in\calR}g_h(r)\pos{d^{(h)}(r)-\lambda},
 \qquad
 \hat\Phi_h(\lambda,s):=\lambda s+
   \sum_{r\in\calR}\hat g_h(r)\pos{\hat d^{(h)}(r)-\lambda},
\]
we therefore have $\Psi_h(s)=\inf_{\lambda\ge0}\Phi_h(\lambda,s)$ and
$\hat\Psi_h(s)=\inf_{\lambda\ge0}\hat\Phi_h(\lambda,s)$.

First we restrict both infima to a common compact interval. Set
$d_{\max}:=\max_{h,r}\max\{d^{(h)}(r),\hat d^{(h)}(r)\}$; since by
step (a) all posteriors $q_h(r),\hat q_h(r)$ lie in
$[c_{\mathrm{post}},1-c_{\mathrm{post}}]$, and
$\klbin(\cdot\Vert\cdot)$ is bounded on the compact square
$[c_{\mathrm{post}},1-c_{\mathrm{post}}]^2$ by a constant depending only on
$c_{\mathrm{post}}$, we have
$d_{\max}\le C_2$ for a primitive constant $C_2$ on the present
event. For $\lambda\ge d_{\max}$ every positive part vanishes, so
$\Phi_h(\lambda,s)=\lambda s\ge d_{\max}s=\Phi_h(d_{\max},s)$, and
likewise for $\hat\Phi_h$; the infima over $\lambda\ge d_{\max}$ are
thus attained at $\lambda=d_{\max}$, and
\[
 \Psi_h(s)=\inf_{0\le\lambda\le d_{\max}}\Phi_h(\lambda,s),
 \qquad
 \hat\Psi_h(s)=\inf_{0\le\lambda\le d_{\max}}\hat\Phi_h(\lambda,s).
\]

Next we show that the two objectives are uniformly close on this
common domain. Fix $\lambda\in[0,d_{\max}]$ and $s\in[0,1]$. The
$\lambda s$ terms coincide, and adding and subtracting
$\sum_r\hat g_h(r)\pos{d^{(h)}(r)-\lambda}$ splits the remainder into
two sums:
\[
 \Phi_h(\lambda,s)-\hat\Phi_h(\lambda,s)
 =\sum_r\bigl(g_h(r)-\hat g_h(r)\bigr)\pos{d^{(h)}(r)-\lambda}
 +\sum_r\hat g_h(r)
   \bigl(\pos{d^{(h)}(r)-\lambda}-\pos{\hat d^{(h)}(r)-\lambda}\bigr).
\]
In the first sum, $0\le\pos{d^{(h)}(r)-\lambda}\le d^{(h)}(r)\le
d_{\max}\le C_2$ for every $r$, so its absolute value is at most
$C_2\|g_h-\hat g_h\|_1\le C_2|\calR| r_m$, using
$\|g_h-\hat g_h\|_1\le|\calR| r_m$ from the proof of part (i). In
the second sum, the Lipschitz property $|\pos{a}-\pos{b}|\le|a-b|$
gives
$|\pos{d^{(h)}(r)-\lambda}-\pos{\hat d^{(h)}(r)-\lambda}|
\le|\hat d^{(h)}(r)-d^{(h)}(r)|$, so, since $\sum_r\hat g_h(r)=1$,
its absolute value is at most
$\max_r|\hat d^{(h)}(r)-d^{(h)}(r)|\le C_1r_m$ by step (a). Hence,
setting $\delta_\Psi:=(C_2|\calR|+C_1)r_m$,
\[
 \bigl|\Phi_h(\lambda,s)-\hat\Phi_h(\lambda,s)\bigr|\le\delta_\Psi
 \qquad\text{for all }\lambda\in[0,d_{\max}],\
 s\in[0,1],\ h\in\{0,1\}.
\]

Finally we pass from the objectives to their infima. Fix $s\in[0,1]$.
For every $\lambda\in[0,d_{\max}]$ the previous display gives
$\Phi_h(\lambda,s)\le\hat\Phi_h(\lambda,s)+\delta_\Psi$; taking the
infimum over $\lambda\in[0,d_{\max}]$ on both sides and using the
truncated representations above,
\[
 \Psi_h(s)
 =\inf_{0\le\lambda\le d_{\max}}\Phi_h(\lambda,s)
 \le\inf_{0\le\lambda\le d_{\max}}
   \bigl\{\hat\Phi_h(\lambda,s)+\delta_\Psi\bigr\}
 =\hat\Psi_h(s)+\delta_\Psi.
\]
Exchanging the roles of $\Phi_h$ and $\hat\Phi_h$ gives
$\hat\Psi_h(s)\le\Psi_h(s)+\delta_\Psi$, whence
$|\hat\Psi_h(s)-\Psi_h(s)|\le\delta_\Psi$. Since $\delta_\Psi$
depends on neither $s$ nor $h$,
\[
 \max_h\sup_{s\in[0,1]}|\hat\Psi_h(s)-\Psi_h(s)|
 \le\delta_\Psi=\bigl(C_2|\calR|+C_1\bigr)r_m.
\]
Summing the bounds of
(a), (b), and (c) proves \eqref{eq:pilot-local-bound} with
$L_{\mathrm{loc}}:=4C_{\phi}/f_{\star}+2C_1+C_2|\calR|$, a
primitive constant depending only on $p_0,p_1,f_{\star},|\calR|$.

\emph{Part (iii).}
If $n_{\AI}=0$, both perspective terms vanish and
$I^{(h)}_{\mathrm{plan}}=\hat I^{(h)}_{\mathrm{plan}}$, so the bound
is trivial. Otherwise, the $n_{\Hum}J_X^{(h)}$ terms in
\eqref{eq:pilot-plan-information-values} coincide because $p_0,p_1$
are known, so, for each $h$,
\[
 \bigl|I^{(h)}_{\mathrm{plan}}-\hat I^{(h)}_{\mathrm{plan}}\bigr|
 \le n_{\AI}\left(|\hat I_R^{(h)}-I_R^{(h)}|
   +\sup_{s\in[0,1]}|\hat\Psi_h(s)-\Psi_h(s)|\right)
 \le NL_{\mathrm{loc}}r_m,
\]
using \eqref{eq:pilot-local-bound} and $n_{\AI}\le N$. This proves
\eqref{eq:pilot-plan-information-bound}.
\end{proof}

\subsection{Proof of Theorem~\ref{thm:pilot-first-order}}\label{appendix-sec:pilot-first-order-1}
\begin{proof}{Proof of Theorem~\ref{thm:pilot-first-order}}
To prove Theorem~\ref{thm:pilot-first-order}, we show that on $\calE_{\pilot,m_{\pilot,k}}$ and for all sufficiently large $k$, there exist primitive constants $c>0$ and  $C_1, \dots, C_5<\infty$ such that
\begin{align}
&\frac{\max_h
    \E_h^{\hat\pi_{m_{\pilot,k},k}}
       [C^{\hat\pi_{m_{\pilot,k},k}}\mid\calD_{m_{\pilot,k}}]}
   {\LB_k}\notag\\
&\quad\le
  1+C_1\frac{\Delta_k}{L_k}
   +C_2r_{m_{\pilot,k}}
   +C_3\frac{z_k+1}{\Delta^{\eff}_{m_{\pilot,k},k}}
   +C_4\exp\left\{-c
       \frac{(\Delta^{\eff}_{m_{\pilot,k},k})^2}{L_k}\right\}
   +C_5\frac{\log L_k}{L_k},                      \label{eq:pilot-ratio-bound}
\end{align}
where 
\begin{equation}\label{eq:pilot-effective-buffer}
    \Delta^{\eff}_{m,k}:=\Delta_k-\omega_{m,k}.
\end{equation}

In the proof of Theorem~\ref{thm:pilot-first-order}, we always assume that the conditions of Theorem~\ref{thm:pilot-first-order} hold. That is, Assumption~\ref{ass:balanced-error-regime} holds, $f_{\fb,k}$ is bounded away
from zero and one, $(\Delta_k,z_k)$ satisfy
\eqref{eq:delta-admissible} and \eqref{eq:z-admissible}, and 
$m_{\pilot,k}\to\infty$ and $L_kr_{m_{\pilot,k}}=o(\Delta_k)$. For brevity, we state the auxiliary lemmas without explicitly rewriting those assumptions.

The proof of Theorem~\ref{thm:pilot-first-order} has an analogous roadmap to the proof of Theorem~\ref{thm:first-order}. The conditional expected cost of policy $\hat\pi_{m,k}$ conditional on $\calD_m$, $\E^{\hat\pi_{m,k}}_h[C^{\hat\pi_{m,k}}|\calD_m]$, has three parts: the data acquisition cost $\cdata \widehat N_{\main, m, k}$, the conditional expected main-stage sensing cost $\E^{\hat\pi_{m,k}}_h[\widehat C^{\main}|\calD_m]$ before fallback, and the conditional expected fallback test cost $\E^{\hat\pi_{m,k}}_h[\widehat C^{\fb}|\calD_m]$: $\E^{\hat\pi_{m,k}}_h[C^{\hat\pi_{m,k}}|\calD_m] = \cdata \widehat N_{\main, m, k} + \E^{\hat\pi_{m,k}}_h[\widehat C^{\main}|\calD_m] + \E^{\hat\pi_{m,k}}_h[\widehat C^{\fb}|\calD_m]$.

Lemma~\ref{lem:pilot-main-cost} below bounds the main-stage cost. Proof of Lemma~\ref{lem:pilot-main-cost} is in Appendix~\ref{appendix-sec:pilot-main-cost}.
\begin{lemma}[Main-stage cost]\label{lem:pilot-main-cost}
There exists a constant $C<\infty$ such that for every
$h\in\{0,1\}$ and for all sufficiently large $m,k$, on $\calE_{\pilot,m}$, 
\begin{equation*}
   \begin{aligned}
 \E^{\hat\pi_{m,k}}_h[\widehat C^{\main}\mid\calD_m]
 \le \hat\Gamma_h(
\widehat T^{\pilot}_{h,m,k}(\widehat{\overline N}_{\main,m,k}),
\widehat{\overline N}_{\main,m,k})
+CL_kr_m+C(z_k+1).
% +CL_k\left[
%    \frac{z_k+1}{\Delta^{\eff}_{m,k}}
%    +\exp\left\{-c
%     \frac{(\Delta^{\eff}_{m,k})^2}{L_k}\right\}
%  \right],
\end{aligned} 
\end{equation*}
\end{lemma}

Let $\hat E_{\fb}$ be the guarded fallback event,
as in the proof of Theorem~\ref{thm:pilot-validity}.
Lemma~\ref{lem:pilot-fallback-prob} below bounds the fallback probability $\Pp_h^{\hat\pi_{m,k}}(\hat E_{\fb}\mid\calD_m)$. Given that $\widehat N_{\main, m, k} = O(L_k)$ (see Lemma~\ref{appendix-lem:pilot-hat-N-main} in Appendix~\ref{appendix-sec:pilot-outer-stability}), the expected fallback cost is therefore $O(L_k[
    (z_k+1)/\Delta^{\eff}_{m,k}
    +\exp\{-c
       (\Delta^{\eff}_{m,k})^2/L_k\}])$. Proof of Lemma~\ref{lem:pilot-fallback-prob} is in Appendix~\ref{appendix-sec:pilot-fallback-prob}.
\begin{lemma}[Fallback probability]
\label{lem:pilot-fallback-prob}
Assume
$\Delta^{\eff}_{m,k}\ge4B_{\ell}$.
There are primitive
constants $C,c>0$ such that for every $h\in\{0,1\}$ and for sufficiently large $m, k$, on $\calE_{\pilot,m}$, 
\begin{equation}\label{eq:pilot-fallback-probability}
    \Pp_h^{\hat\pi_{m,k}}(\hat E_{\fb}\mid\calD_m)
    \le C\frac{z_k+1}{\Delta^{\eff}_{m,k}}
       +C\exp\left\{-c
          \frac{(\Delta^{\eff}_{m,k})^2}{L_k}\right\}.
\end{equation}
\end{lemma}
Combining Lemmas~\ref{lem:pilot-main-cost} and \ref{lem:pilot-fallback-prob} gives
\begin{align}
 \E_h^{\hat\pi_{m,k}}[C^{\hat\pi_{m,k}}\mid\calD_m]
 &\le \widehat N_{\main,m,k}\cdata
 +\hat\Gamma_h\!\left(
    \widehat T^{\pilot}_{h,m,k}
      (\widehat{\overline N}_{\main,m,k}),
    \widehat{\overline N}_{\main,m,k}\right)
 +CL_kr_m+C(z_k+1)\notag\\
 &\quad+CL_k\left[
    \frac{z_k+1}{\Delta^{\eff}_{m,k}}
    +\exp\left\{-c
       \frac{(\Delta^{\eff}_{m,k})^2}{L_k}\right\}
 \right].                                           \label{eq:pilot-total-cost}
\end{align}
The final lemma connects $\E^{\hat\pi_{m,k}}_h[C^{\hat\pi_{m,k}}]$ to $\LB_k$:
\begin{lemma}[Pilot outer stability]
\label{lem:pilot-outer-stability}
On $\calE_{\pilot,m}$, for all sufficiently large $m,k$, the guarded outer
problem \eqref{eq:pilot-Nbar-choice} is feasible and
\begin{align}
&\widehat N_{\main,m,k}\cdata
 +\max_h\hat\Gamma_h\!\left(
    \widehat T^{\pilot}_{h,m,k}
      (\widehat{\overline N}_{\main,m,k}),
    \widehat{\overline N}_{\main,m,k}\right)\notag\\
&\qquad\le
 \min_{N\in\mathbb{Z}_+, N\ge N_{\fixed,\Hum}(\alpha_{2,k},\beta_{2,k})}
 \left\{N\cdata+\max_h\Gamma_h(T_{h,k},N)\right\}
 +O(L_kr_m+1).                    \label{eq:pilot-outer-stability}
\end{align}
\end{lemma}
Proof of Lemma~\ref{lem:pilot-outer-stability} is in Appendix~\ref{appendix-sec:pilot-outer-stability}.
Finally, Lemma~\ref{appendix-lem:pilot-hat-N-main} in Appendix~\ref{appendix-sec:pilot-outer-stability} gives
$\omega_{m_{\pilot,k},k}=O(L_kr_{m_{\pilot,k}})=o(\Delta_k)$, so
$\Delta^{\eff}_{m_{\pilot,k},k}=\Delta_k\{1-o(1)\}$.
Combine \eqref{eq:pilot-outer-stability} and
\eqref{eq:pilot-total-cost}, and then applying Lemma~\ref{lem:LB-perturb} gives the desired argument.
\end{proof}

\subsection{Proof of Lemma~\ref{lem:pilot-outer-stability}}\label{appendix-sec:pilot-outer-stability}

\begin{proof}{Proof of Lemma~\ref{lem:pilot-outer-stability}}
Since $r_m\to0$ as $m\to\infty$, we have $r_m\le f_{\star}/4$ for
all sufficiently large $m$, so Lemma~\ref{lem:pilot-perturbation-2}(ii)
applies; moreover $\varepsilon_m^{\mathrm{dr}}=3|\calR|B_{\ell}r_m\to0$
and, by the first term of \eqref{eq:pilot-local-bound},
$\varepsilon_m^{\mathrm{LR}}\le L_{\mathrm{loc}}r_m$. Asymptotic
statements are understood along any pilot-size sequence
$m=m_k\to\infty$ satisfying \eqref{eq:pilot-main-rate-condition}, and
``for all sufficiently large $m,k$'' refers to such sequences.

In order to prove Lemma~\ref{lem:pilot-outer-stability}, we construct $N'\in \mathbb{Z}_+$ such that 
$$
\hat F^{\pilot}_{m,k}(N')\leq \min_{N\in\mathbb Z_+, N\ge N_{\fixed,\Hum}(\alpha_{2,k},\beta_{2,k})}
 \left\{N\cdata+\max_h\Gamma_h(T_{h,k},N)\right\}
 +O(L_kr_m+1).
$$
Then, the argument follows given that $\hat F^{\pilot}_{m,k}(\widehat{\overline N}_{\main,m,k})\leq \hat F^{\pilot}_{m,k}(N')$. Recall from \eqref{eq:Nbar-main-choice} that $\overline N_{\main,k}\in
    \argmin_{N\ge N_{\fixed,\Hum}(\alpha_{2,k},\beta_{2,k})}F_k^\Delta(N)$ and that according to Lemma~\ref{lemma-appendix:N-main}, $\overline N_{\main,k} = O(L_k)$.

% \emph{Step 1: the true outer minimum and its optimizer are $O(L_k)$.}
% The minimum on the right-hand side of \eqref{eq:pilot-outer-stability}
% is $\min_{N\ge N_{\fixed,\Hum}(\alpha_{2,k},\beta_{2,k})}F^\Delta_k(N)$ with
% $F^\Delta_k$ the buffered design value \eqref{eq:F-Delta-def}, taken
% over integer $N$; let
% $N_k^*:=\overline N_{\main,k}\in\argmin F^\Delta_k$ as in
% \eqref{eq:Nbar-main-choice}. By Lemma~\ref{lem:LB-perturb},
% \[
%  F^\Delta_k(N_k^*)
%  \le N_{\main,k}\cdata+\max_h\Gamma_h(T_{h,k},N_k^*)
%  \le\LB_k+C_P\Delta_k
%    +O\!\left(\log L_k+\log\frac1{f_{\fb,k}}
%      +\log\frac1{1-f_{\fb,k}}\right).
% \]
% Since $\LB_k=\Theta(L_k)$ by Theorem~\ref{thm:first-order}(i),
% $\Delta_k=o(L_k)$ by admissibility, and the logarithmic
% terms are $o(L_k)$ because $f_{\fb,k}$ is bounded away from zero and
% one, the true outer minimum satisfies $F^\Delta_k(N_k^*)=O(L_k)$;
% and since $F^\Delta_k(N_k^*)\ge N_k^*\cdata$, also $N_k^*=O(L_k)$.

For each $h$, let
$(n^*_{\Hum,h},n^*_{\AI,h},n^*_{\esc,h})$ be a $\Gamma_h$-optimizer at
$(T_{h,k},\overline N_{\main,k})$; its true information value
$I^{(h)}_{\mathrm{plan}}$ is at least $T_{h,k}$ by feasibility. By
\eqref{eq:pilot-plan-information-bound} with $N=\overline N_{\main,k}$, the plug-in
value of $(n^*_{\Hum,h},n^*_{\AI,h},n^*_{\esc,h})$ satisfies
$\hat I^{(h)}_{\mathrm{plan}}\ge I^{(h)}_{\mathrm{plan}} -L_{\mathrm{loc}}\overline N_{\main,k}r_m\geq  T_{h,k}-L_{\mathrm{loc}}\overline N_{\main,k}r_m$.
Append
\[
 R_k:=\left\lceil
 \frac{L_{\mathrm{loc}}\overline N_{\main,k}r_m+(\overline N_{\main,k}+1)
       \varepsilon_m^{\mathrm{dr}}}
      {J_{\min}-\varepsilon_m^{\mathrm{dr}}}\right\rceil,
 \qquad J_{\min}:=\min_hJ_X^{(h)},
\]
direct-human items; the denominator is positive for all large $m$
because $\varepsilon_m^{\mathrm{dr}}\to0$ while $J_{\min}>0$ is a
primitive constant. Let $N':=\overline N_{\main, k} + R_k$. 
We now verify that $(n^*_{\Hum,h}+R_k,\,n^*_{\AI,h},\,n^*_{\esc,h})$ is feasible to \eqref{eq:pilot-hat-Gamma} at $T = \widehat T^{\pilot}_{h,m,k}
      (N')$ and $N = N'$.  
First, by construction $n^*_{\esc,h}\leq n^*_{\AI,h}$ and 
$$
n^*_{\Hum,h}+R_k + n^*_{\AI,h}\leq \overline N_{\main, k} + R_k = N'.
$$
Additionally, the plug-in information value provided by the triple under $H_h$ satisfies
\begin{equation*}
    \begin{split}
       &(n^*_{\Hum,h}+R_k)J^{(h)}_X + n^*_{\AI,h}\hat I^{(h)}_R + n^*_{\AI,h}\hat \Psi_h(n^*_{\esc,h}/n^*_{\AI,h})\\
       \geq& n^*_{\Hum,h}J^{(h)}_X+R_kJ_{\min} + n^*_{\AI,h}\hat I^{(h)}_R + n^*_{\AI,h}\hat\Psi_h(n^*_{\esc,h}/n^*_{\AI,h})\\
       \geq& \hat I^{(h)}_{\mathrm{plan}}+R_kJ_{\min}\\
 \ge& T_{h,k}-L_{\mathrm{loc}}\overline N_{\main,k}r_m
   +R_k\bigl(J_{\min}-\varepsilon_m^{\mathrm{dr}}\bigr)
   +R_k\varepsilon_m^{\mathrm{dr}}\\
 \ge&T_{h,k}+(N'+1)\varepsilon_m^{\mathrm{dr}}\\
 =&\widehat T^{\pilot}_{h,m,k}(N').
    \end{split}
\end{equation*}

Since
$N'\ge \overline N_{\main,k}\ge N_{\fixed,\Hum}(\alpha_{2,k},\beta_{2,k})$, $N'$ is feasible to \eqref{eq:pilot-Nbar-choice}. Moreover,
since the objectives of $\Gamma_h$ and $\hat\Gamma_h$ share the
channel-free cost coefficients $(\cH,\cAI,\cH)$, the repaired triple
costs $\Gamma_h(T_{h,k},\overline N_{\main,k})+R_k\cH$ in the plug-in program, so
\[
 \widehat F^{\pilot}_{m,k}(N')
 \le N'\cdata+\max_h\Gamma_h(T_{h,k},\overline N_{\main,k})+R_k\cH
 =F^\Delta_k(\overline N_{\main,k})+R_k(\cdata+\cH).
\]
The left-hand side of \eqref{eq:pilot-outer-stability} equals
$\widehat F^{\pilot}_{m,k}(\widehat{\overline N}_{\main,m,k})+\cdata
\le\widehat F^{\pilot}_{m,k}(N')+\cdata$ by
\eqref{eq:pilot-Nbar-choice}. Since $\overline N_{\main,k}=O(L_k)$ and
$\varepsilon_m^{\mathrm{dr}}=O(r_m)$, the numerator of $R_k$ is
$O(L_kr_m)$, so $R_k=O(L_kr_m+1)$, and
\eqref{eq:pilot-outer-stability} follows.
\end{proof}

Lemma~\ref{appendix-lem:pilot-hat-N-main} is an immediate consequence of Lemma~\ref{lem:pilot-outer-stability}. 
\begin{lemma}
\label{appendix-lem:pilot-hat-N-main}
On $\calE_{\pilot,m}$ and for all sufficiently large $m,k$, $\widehat N_{\main,m,k}=O(L_k)$ and $\omega_{m,k}=O(L_kr_m)=o(\Delta_k)$.
\end{lemma}
\begin{proof}{Proof of Lemma~\ref{appendix-lem:pilot-hat-N-main}}
According to Lemma~\ref{lem:pilot-outer-stability}, 
$\widehat N_{\main,m,k}\cdata\leq F^\Delta_k(\overline N_{\main,k})+O(L_kr_m+1)=O(L_k)$ by Lemma~\ref{lemma-appendix:N-main} and
$L_kr_m=o(\Delta_k)=o(L_k)$. Hence
$\widehat N_{\main,m,k}=O(L_k)$, and, by the first term of
\eqref{eq:pilot-local-bound} with $L_{\mathrm{loc}}$ the constant of
Lemma~\ref{lem:pilot-perturbation-2}(ii),
$\omega_{m,k}
 =\widehat N_{\main,m,k}\varepsilon_m^{\mathrm{LR}}
 \le\widehat N_{\main,m,k}L_{\mathrm{loc}}r_m
 =O(L_kr_m)
 =o(\Delta_k)$. 
\end{proof}

\subsection{Proof of Lemma~\ref{lem:pilot-fallback-prob}}\label{appendix-sec:pilot-fallback-prob}

We first provide an auxiliary Lemma~\ref{lem:pilot-increment-conditions} together with its proof, and then prove Lemma~\ref{lem:pilot-fallback-prob} at the end of this section.

Write $\sigma_1:=1$ and
$\sigma_0:=-1$.  The signed one-item means under the true and plug-in
channels are, respectively, $\E^{\hat\pi_{m,k}}_h
   [\sigma_h\hat Z^{\mathrm{item}}_i(\hat\eta_{j,m,k})\mid\calD_m]$ and $\E^{\hat\pi_{m,k}}_{h,\hat f}
   [\sigma_h\hat Z^{\mathrm{item}}_i(\hat\eta_{j,m,k})\mid\calD_m]$.

According to \eqref{eq:pilot-drift-cost-perturbation} applied with
$\eta=\hat\eta_{j,m,k}$, on $\calE_{\pilot,m}$,
we have
\begin{equation}\label{eq:pilot-signed-mean-perturbation}
 \left|\E^{\hat\pi_{m,k}}_h
   [\sigma_h\hat Z^{\mathrm{item}}_i(\hat\eta_{j,m,k})\mid\calD_m]
 -\E^{\hat\pi_{m,k}}_{h,\hat f}
   [\sigma_h\hat Z^{\mathrm{item}}_i(\hat\eta_{j,m,k})\mid\calD_m]\right|
 \le\varepsilon_m^{\mathrm{dr}}.
\end{equation}

\begin{lemma}[Uniform pilot increment conditions]
\label{lem:pilot-increment-conditions}
Let $\widehat N^{\mathrm{plan}}_{h,m,k}:= \lceil\hat n^*_{\Hum,h,m,k}\rceil
      +\lceil\hat n^*_{\AI,h,m,k}\rceil$. For $h\in\{0, 1\}$, the correct-direction rounded block
also satisfies
\begin{equation}\label{eq:pilot-planned-true-drift}
    \widehat N^{\mathrm{plan}}_{h,m,k}\,
    \E^{\hat\pi_{m,k}}_h
      [\sigma_h\hat Z^{\mathrm{item}}_i(\hat\eta_{h,m,k})\mid\calD_m]\ge T_{h,k}.
\end{equation}

Moreover, there are
primitive constants $c_\mu>0$ and $C_V<\infty$ such that, on
$\calE_{\pilot,m}$ and for every $j\in\{0,1,*\}$ and all sufficiently large $m,k$, the following inequalities hold almost surely,
\begin{equation}\label{eq:pilot-increment-conditions-short}
\begin{split}
    |\hat Z^{\mathrm{item}}_i(\hat\eta_{j,m,k})|\le 2B_{\ell},\qquad
 \E^{\hat\pi_{m,k}}_h
      [\sigma_h\hat Z^{\mathrm{item}}_i(\hat\eta_{j,m,k})\mid\calD_m]
      \ge c_\mu,\\
 \Var_h(\hat Z^{\mathrm{item}}_i(\hat\eta_{j,m,k})\mid\calD_m)
      \le C_V\E^{\hat\pi_{m,k}}_h
      [\sigma_h\hat Z^{\mathrm{item}}_i(\hat\eta_{j,m,k})\mid\calD_m].
\end{split}
\end{equation}
\end{lemma}

\begin{proof}{Proof of Lemma~\ref{lem:pilot-increment-conditions}}
Work on $\calE_{\pilot,m}$ for all sufficiently large $m,k$,
such that the guarded outer
problem \eqref{eq:pilot-Nbar-choice} is feasible, 
$\widehat N_{\main,m,k}=O(L_k)$, 
$r_m\le f_{\star}/4$ and
$\varepsilon_m^{\mathrm{dr}}=3|\calR|B_{\ell} r_m\to0$. We have that
\[
 \E^{\hat\pi_{m,k}}_{h,\hat f}
   [\sigma_h\hat Z^{\mathrm{item}}_i(\hat\eta_{j,m,k})\mid\calD_m]
 =\hat\eta^{\Hum}_{j,m,k}J_X^{(h)}
 +(1-\hat\eta^{\Hum}_{j,m,k})
   \Bigl(\hat I_R^{(h)}
     +\sum_{r\in\calR}\hat g_h(r)\,
        \hat\eta^{\esc}_{j,m,k}(r)\,\hat d^{(h)}(r)\Bigr).
\]
\emph{Step 1: show that
$\widehat N^{\mathrm{plan}}_{h,m,k}\E^{\hat\pi_{m,k}}_h
[\sigma_h\hat Z^{\mathrm{item}}_i(\hat\eta_{h,m,k})\mid\calD_m]\ge T_{h,k}$.}
Repeating the derivation of
\eqref{eq:correct-rule-target} for the plug-in program gives
\[
 \widehat N^{\mathrm{plan}}_{h,m,k}\,
 \E^{\hat\pi_{m,k}}_{h,\hat f}
   [\sigma_h\hat Z^{\mathrm{item}}_i(\hat\eta_{h,m,k})\mid\calD_m]
 \ge\widehat T^{\pilot}_{h,m,k}(\widehat{\overline N}_{\main,m,k})
 =T_{h,k}+\widehat N_{\main,m,k}\varepsilon_m^{\mathrm{dr}},
\]
using $\widehat N_{\main,m,k}=\widehat{\overline N}_{\main,m,k}+1$ in
\eqref{eq:pilot-guarded-targets}. Combining the above display with
\eqref{eq:pilot-signed-mean-perturbation} and
$\widehat N^{\mathrm{plan}}_{h,m,k}\le\widehat N_{\main,m,k}$ (which follows since $\widehat N^{\mathrm{plan}}_{h,m,k}
 \le\lceil\hat n^*_{\Hum,h,m,k}+\hat n^*_{\AI,h,m,k}\rceil+1
 \le\lceil\widehat{\overline N}_{\main,m,k}\rceil+1
 =\widehat N_{\main,m,k}$) gives
\begin{align*}
 \widehat N^{\mathrm{plan}}_{h,m,k}\,
 \E^{\hat\pi_{m,k}}_h
   [\sigma_h\hat Z^{\mathrm{item}}_i(\hat\eta_{h,m,k})\mid\calD_m]
 &\ge\widehat N^{\mathrm{plan}}_{h,m,k}\,
 \E^{\hat\pi_{m,k}}_{h,\hat f}
   [\sigma_h\hat Z^{\mathrm{item}}_i(\hat\eta_{h,m,k})\mid\calD_m]
   -\widehat N^{\mathrm{plan}}_{h,m,k}\,\varepsilon_m^{\mathrm{dr}}\\
 &\ge T_{h,k}
   +\bigl(\widehat N_{\main,m,k}-\widehat N^{\mathrm{plan}}_{h,m,k}\bigr)
     \varepsilon_m^{\mathrm{dr}}\\
 &\ge T_{h,k},
\end{align*}
which is \eqref{eq:pilot-planned-true-drift}.

\emph{Step 2: show that
$\E^{\hat\pi_{m,k}}_h
[\sigma_h\hat Z^{\mathrm{item}}_i(\hat\eta_{h,m,k})\mid\calD_m]
\ge c_{\mathrm{plan}}$ for a primitive
$c_{\mathrm{plan}}>0$ for the \textit{correct} rule.}
According to Step 1, $\E^{\hat\pi_{m,k}}_h
   [\sigma_h\hat Z^{\mathrm{item}}_i(\hat\eta_{h,m,k})\mid\calD_m]\geq T_{h,k}/\widehat N^{\mathrm{plan}}_{h,m,k}\geq T_{h,k}/\widehat N_{\main,m,k}$ (note that $\widehat N^{\mathrm{plan}}_{h,m,k} \geq 1$ because
$\widehat T^{\pilot}_{h,m,k}(\widehat{\overline N}_{\main,m,k})>0$ forces a nonzero optimizer $(\hat n^*_{\Hum,h,m,k},\hat n^*_{\AI,h,m,k},
      \hat n^*_{\esc,h,m,k})\neq (0,0,0)$. Using that 
$T_{h,k}\ge\min\{a_k,b_k\}
\ge\min\{\log(1/\alpha_{1,k}),\log(1/\beta_{1,k})\}\ge c_TL_k$ for a
primitive $c_T>0$ and all large $k$ and that $\widehat N_{\main,m,k} = O(L_k)$, the correct rule
satisfies
$\E^{\hat\pi_{m,k}}_h
[\sigma_h\hat Z^{\mathrm{item}}_i(\hat\eta_{h,m,k})\mid\calD_m]\ge
T_{h,k}/\widehat N_{\main,m,k}\ge c_{\mathrm{plan}}$ for a primitive
$c_{\mathrm{plan}}>0$ with the \textit{correct} sensing rule $\hat\eta_{h,m,k}$ under $H_h$. 

\emph{Step 3: show that
$\E^{\hat\pi_{m,k}}_h
[\sigma_h\hat Z^{\mathrm{item}}_i(\hat\eta_{j,m,k})\mid\calD_m]\ge c_\mu$
for a primitive constant $c_\mu$ for every $h\in\{0,1\}$ and $j\in\{0,1,*\}$.}

We first show that
$\E^{\hat\pi_{m,k}}_{h,\hat f}
[\sigma_h\hat Z^{\mathrm{item}}_i(\hat\eta_{j,m,k})\mid\calD_m]\ge
c_\times^0\E^{\hat\pi_{m,k}}_{1-h,\hat f}
[\sigma_{1-h}\hat Z^{\mathrm{item}}_i(\hat\eta_{j,m,k})\mid\calD_m]$ for a primitive constant
$c^0_\times$.
Analogous to the proof of Lemma~\ref{lem:variance-drift}(ii), it
suffices to show that there exists a constant $c_{\mathrm{cmp}}>0$ such that for
both $h$ and every $r$,
\[
 J_X^{(h)}\ge c_{\mathrm{cmp}}J_X^{(1-h)},
 \qquad
 \hat I_R^{(h)}\ge c_{\mathrm{cmp}}\hat I_R^{(1-h)},
 \qquad
 \hat d^{(h)}(r)\ge c_{\mathrm{cmp}}\hat d^{(1-h)}(r)
\]
and 
\[
 m_g:=\min_{h,r}\frac{\hat g_h(r)}{\hat g_{1-h}(r)}>0.
\]
To verify these bounds uniformly over the realized pilot, work on
$\calE_{\pilot,m}$ and take $m$ sufficiently large that
$r_m\le f_{\star}/4$.  By Lemma~\ref{lem:pilot-perturbation-2}(ii), all
coordinates of the plug-in report laws $\hat g_h$ and the plug-in
conditional-label laws $\hat\rho_h(\cdot\mid r)$ are bounded below by
some primitive constant, let it be $q_\star$.  The same is true of the fixed
direct-label laws because $0<p_0<p_1<1$.

If two probability vectors $P,Q$ on a common finite alphabet have all
coordinates at least $q_\star$, then Pinsker's inequality and the
chi-square upper bound give
\[
 D(P\Vert Q)
 \ge \frac12\lVert P-Q\rVert_1^2
 \ge \frac12\lVert P-Q\rVert_2^2,
 \qquad
 D(Q\Vert P)
 \le \sum_x\frac{(Q(x)-P(x))^2}{P(x)}
 \le q_\star^{-1}\lVert P-Q\rVert_2^2.
\]
Consequently,
\[
 D(P\Vert Q)\ge\frac{q_\star}{2}D(Q\Vert P).
\]
This inequality also holds when $P=Q$, since then both divergences are
zero.  Applying it to the direct-label pair, the plug-in report pair,
and each plug-in conditional-label pair proves the three component
comparisons above with $c_{\mathrm{cmp}}:=q_\star/2$.  Moreover,
$\hat g_h(r)\ge q_\star$ and $\hat g_{1-h}(r)\le1$ imply
$m_g\ge q_\star>0$.  Thus for $j\in\{0,1,*\}$, we have $\E^{\hat\pi_{m,k}}_{h,\hat f}
 [\sigma_h\hat Z^{\mathrm{item}}_i(\hat\eta_{j,m,k})\mid\calD_m]
 \ge c_\times^0\E^{\hat\pi_{m,k}}_{1-h,\hat f}
 [\sigma_{1-h}\hat Z^{\mathrm{item}}_i(\hat\eta_{j,m,k})\mid\calD_m]$
where $c^0_\times:=c_{\mathrm{cmp}}\min\{1,m_g\}>0$ is a primitive constant that does not depend on $m$ or $k$.

Now,
analogous to the proof of Lemma~\ref{lem:variance-drift}(iii), we have
that $\E^{\hat\pi_{m,k}}_{h,\hat f}
[\sigma_h\hat Z^{\mathrm{item}}_i(\hat\eta_{h,m,k})\mid\calD_m]\ge
T_{h,k}/\widehat N_{\main,m,k}\ge c_{\mathrm{plan}}$ for a primitive
$c_{\mathrm{plan}}>0$ for $h\in\{0,1\}$.
Hence the preceding
comparison gives
$\E^{\hat\pi_{m,k}}_{h,\hat f}
[\sigma_h\hat Z^{\mathrm{item}}_i(\hat\eta_{1-h,m,k})\mid\calD_m]
\ge c^0_\times c_{\mathrm{plan}}$ for the \textit{wrong} rule, while
the dead-zone averaging rule gives
$\E^{\hat\pi_{m,k}}_{h,\hat f}
[\sigma_h\hat Z^{\mathrm{item}}_i(\hat\eta_{*,m,k})\mid\calD_m]\ge c_{\mathrm{plan}}/4$. Let
$c_{\mathrm{plug}}:=\min\{c_{\mathrm{plan}},
c_\times^0c_{\mathrm{plan}},c_{\mathrm{plan}}/4\}>0$.
Then
$\E^{\hat\pi_{m,k}}_{h,\hat f}
[\sigma_h\hat Z^{\mathrm{item}}_i(\hat\eta_{j,m,k})\mid\calD_m]\ge c_{\mathrm{plug}}$ for every
$h$ and $j\in\{0,1,*\}$. By
\eqref{eq:pilot-signed-mean-perturbation}, uniformly over these six
choices,
\[
 \E^{\hat\pi_{m,k}}_h
 [\sigma_h\hat Z^{\mathrm{item}}_i(\hat\eta_{j,m,k})\mid\calD_m]
 \ge\E^{\hat\pi_{m,k}}_{h,\hat f}
 [\sigma_h\hat Z^{\mathrm{item}}_i(\hat\eta_{j,m,k})\mid\calD_m]
 -\varepsilon_m^{\mathrm{dr}}.
\]
Since $\varepsilon_m^{\mathrm{dr}}\to0$, it is at most
$c_{\mathrm{plug}}/2$ for all sufficiently large $m$.  Thus
\[
 \E^{\hat\pi_{m,k}}_h
 [\sigma_h\hat Z^{\mathrm{item}}_i(\hat\eta_{j,m,k})\mid\calD_m]
 \ge c_{\mathrm{plug}}/2=:c_\mu>0,
\]
which proves the mean bound in
\eqref{eq:pilot-increment-conditions-short}.

\emph{Step 4: show that $|\hat Z^{\mathrm{item}}_i(\hat\eta_{j,m,k})|\le 2B_{\ell}$ and
$\Var_h(\hat Z^{\mathrm{item}}_i(\hat\eta_{j,m,k})\mid\calD_m)
      \le C_V\E^{\hat\pi_{m,k}}_h
      [\sigma_h\hat Z^{\mathrm{item}}_i(\hat\eta_{j,m,k})\mid\calD_m]$.}
The inequality $|\hat Z^{\mathrm{item}}_i(\hat\eta_{j,m,k})|\le 2B_{\ell}$ follows straightforwardly from the definition of $\hat Z^{\mathrm{item}}_i(\hat\eta_{j,m,k})$ and  Lemma~\ref{lem:pilot-perturbation-1}(i). Additionally
\[
 \Var_h(\hat Z^{\mathrm{item}}_i(\hat\eta_{j,m,k})\mid\calD_m)
 \le\E^{\hat\pi_{m,k}}_h[\hat Z^{\mathrm{item}}_i(\hat\eta_{j,m,k})^2\mid\calD_m]
 \le4B_{\ell}^2
 \le\frac{4B_{\ell}^2}{c_\mu}
       \E^{\hat\pi_{m,k}}_h
       [\sigma_h\hat Z^{\mathrm{item}}_i(\hat\eta_{j,m,k})\mid\calD_m].
\]
Taking $C_V:=4B_{\ell}^2/c_\mu$ proves the remaining assertion of
\eqref{eq:pilot-increment-conditions-short}.
\end{proof}

\begin{proof}{Proof of Lemma~\ref{lem:pilot-fallback-prob}}
We prove the claims under $H_1$.  The $H_0$ proof is identical after
applying the argument to the reflected process
$(-\hat S_i,-\hat Z^{\mathrm{item}}_i(\hat\eta_{j,m,k}))$, interchanging the indices $0$ and $1$, and replacing
$a_k$ by $b_k$; the upper guarded boundary then becomes the reflection
of the lower guarded boundary, and the direction-$1$ rule becomes the
direction-$0$ rule.  Work under $\Pp^{\hat\pi_{m,k}}_1(\cdot\mid\calD_m)$ for a realized
pilot in $\calE_{\pilot,m}$, for all sufficiently large $m,k$ as in
Lemma~\ref{lem:pilot-increment-conditions}. 

Conditionally on $\calD_m$, item $i$ is
fresh: its label and report are drawn independently of the previous
items by Assumption~\ref{ass:iid}, its randomization seeds are fresh
by construction, and all of these are independent of $\calD_m$ by
Assumption~\ref{ass:pilot-sample}. Since $J_i$ is
$\calH^m_{i-1}$-measurable, 
\[
 \E^{\hat\pi_{m,k}}_1[\hat Z^{\mathrm{item}}_i(\hat\eta_{J_i,m,k})\mid\calH^m_{i-1}]
 =\E^{\hat\pi_{m,k}}_1[\hat Z^{\mathrm{item}}_i(\hat\eta_{J_i,m,k})\mid\calD_m, J_i],
 \]
 and
 \[
 \Var_1(\hat Z^{\mathrm{item}}_i(\hat\eta_{J_i,m,k})\mid\calH^m_{i-1})
 =\Var_1(\hat Z^{\mathrm{item}}_i(\hat\eta_{J_i, m,k})\mid\calD_m, J_i).
\]
Since the inequalities in \eqref{eq:pilot-increment-conditions-short} hold for all $j\in\{0, 1, *\}$, the following holds almost surely:
\begin{equation*}
    \begin{split}
        \E^{\hat\pi_{m,k}}_1[\hat Z^{\mathrm{item}}_i(\hat\eta_{J_i,m,k})\mid\calH^m_{i-1}]\ge c_\mu,
 \qquad
 |\hat Z^{\mathrm{item}}_i(\hat\eta_{J_i,m,k})|\le 2B_{\ell},
\\
 \Var_1(\hat Z^{\mathrm{item}}_i(\hat\eta_{J_i,m,k})\mid\calH^m_{i-1})
 \le C_V\E^{\hat\pi_{m,k}}_1[\hat Z^{\mathrm{item}}_i(\hat\eta_{J_i,m,k})\mid\calH^m_{i-1}].
    \end{split}
\end{equation*}

\emph{Step 1: direction tracking.} The preceding display is exactly the
bounded-increment, positive-drift, and
variance--drift input \eqref{eq:dt-hyp} of
Lemma~\ref{lem:direction-tracking}, with $c_\mu$ playing the role of
the drift lower bound $\mu_0$ there, now verified for
$(\hat S_i,\hat Z^{\mathrm{item}}_i(\hat\eta_{J_i,m,k}))$ under the conditional measure.  The proof of
that lemma uses only these abstract hypotheses and the predictability
of the regime selection, which holds here because $J_i$ is determined
from $\hat S_{i-1}$.  It therefore shows that the number $W_{1,m,k}$
of items processed against the true direction under $H_1$ (the
analogue of $W_{1,k}$ there) satisfies
\[
 \E^{\hat\pi_{m,k}}_1[W_{1,m,k}\mid\calD_m]
 \le C(z_k+1),
\]
where the primitive constant absorbs the factor $1/c_\mu$.

\emph{Step 2: fallback probability.} The increment conditions above
also give the abstract input \eqref{eq:fb-hyp} of
Lemma~\ref{lem:fallback-low}.  Its
planned-block input \eqref{eq:fb-planned} is supplied by
Lemma~\ref{lem:pilot-increment-conditions}: by
\eqref{eq:pilot-planned-true-drift},
\[
 \widehat N^{\mathrm{plan}}_{1,m,k}\,
 \E^{\hat\pi_{m,k}}_1[\hat Z^{\mathrm{item}}_i(\hat\eta_{1,m,k})\mid\calD_m]
 \ \ge\ T_{1,k}=a_k+\Delta_k,
\]
while $\widehat N^{\mathrm{plan}}_{1,m,k}\le\widehat N_{\main,m,k}$ and
$\widehat N_{\main,m,k}=O(L_k)$ by Lemma~\ref{appendix-lem:pilot-hat-N-main}. The
guarded policy stops at the boundary $a_k+\omega_{m,k}$ (resp.\
$-(b_k+\omega_{m,k})$), so the
margin between the planned block drift and the stopping boundary is
\[
 (a_k+\Delta_k)-(a_k+\omega_{m,k})
 =\Delta_k-\omega_{m,k}
 =\Delta^{\eff}_{m,k}>0,
\]
which replaces $\Delta_k$ throughout: the Freedman margin becomes
$3\Delta^{\eff}_{m,k}/4$ and the occupation threshold
$\Delta^{\eff}_{m,k}/(8B_{\ell})$, exactly as in the proof of
Lemma~\ref{lem:fallback-low}. Running the proof of
Lemma~\ref{lem:fallback-low} with
the boundary $a_k+\omega_{m,k}$ and margin $\Delta^{\eff}_{m,k}$ in
place of $a_k$ and $\Delta_k$ --- the accumulated predictable
variance at the end of the block is at most
$2C_VB_\ell\,\widehat N^{\mathrm{plan}}_{1,m,k}=O(L_k)$ as in its display
\eqref{eq:fb-pqv}, so the Freedman exponent becomes
$(\Delta^{\eff}_{m,k})^2/L_k$ --- yields
\eqref{eq:pilot-fallback-probability}.
\end{proof}

\subsection{Proof of Lemma~\ref{lem:pilot-main-cost}}\label{appendix-sec:pilot-main-cost}

\begin{proof}{Proof of Lemma~\ref{lem:pilot-main-cost}}
We prove the claims under $H_1$.  The $H_0$ proof is identical after
applying the argument to the reflected process
$(-\hat S_i,-\hat Z^{\mathrm{item}}_i(\hat\eta_{J_i,m,k}))$, interchanging the indices $0$ and $1$, and replacing
$a_k$ by $b_k$; the upper guarded boundary then becomes the reflection
of the lower guarded boundary, and the direction-$1$ rule becomes the
direction-$0$ rule. 

Recall
\(\widehat C^{\main}\) is the sensing cost incurred during the main
stage, before the fallback completion is executed.  We follow the proof
of Lemma~\ref{lem:main-cost}.  Let \(I_i\) indicate that item \(i\) is
started and write
\(\hat\tau_{\main,m,k}:=\sum_{i=1}^{\widehat N_{\main,m,k}}I_i\), so
\(I_i=\ind\{i\le\hat\tau_{\main,m,k}\}\) is \(\calH^m_{i-1}\)-measurable.  Define
\(\hat{\mathcal C}:=\{i:I_i=1,\hat S_{i-1}>z_k\}\),
\(\hat{\mathcal W}:=\{i:I_i=1,\hat S_{i-1}\le z_k\}\), and
\(\hat N_c:=|\hat{\mathcal C}|\); then
\(|\hat{\mathcal W}|=W_{1,m,k}\).

For the completed-item continuation, 
recall that $C_i^{\mathrm{item}}(\hat\eta_{J_i, m, k})$ corresponds to the cost when item $i$ uses sensing rule $\hat \eta_{J_i, m, k}$. Set
\(c_{\max}:=\cH+\cAI+\cH\).  Analogous to the cost decomposition
\eqref{eq:mc-split},
\begin{equation*}
    \begin{split}
         \E^{\hat\pi_{m,k}}_1[\widehat C^{\main}\mid\calD_m]
 \le
\E^{\hat\pi_{m,k}}_1\bigg[\sum_{i\in\hat{\mathcal W}} C_i^{\mathrm{item}}(\hat\eta_{J_i, m, k})\mid\calD_m\bigg]
+\E^{\hat\pi_{m,k}}_1\bigg[\sum_{i\in\hat{\mathcal C}} C_i^{\mathrm{item}}(\hat\eta_{J_i, m, k})\mid\calD_m\bigg].
 % \\
 % +c_{\max}\widehat N_{\main,m,k}
 %   \Pp^{\hat\pi_{m,k}}_1(\hat E_{\fb}\mid\calD_m).
    \end{split}
\end{equation*}

\emph{Cost 1: wrong and dead-zone items.}  Analogous to Cost 1 in the
proof of Lemma~\ref{lem:main-cost},
\[
 \E^{\hat\pi_{m,k}}_1[\sum_{i\in\hat{\mathcal W}} C_i^{\mathrm{item}}(\hat\eta_{J_i, m, k})\mid\calD_m]
 \le c_{\max}
   \E^{\hat\pi_{m,k}}_1[W_{1,m,k}\mid\calD_m]
 \le C(z_k+1).
\]

\emph{Cost 2: correct items.}  
Under the correct sensing rule $\hat\eta_{1,m,k}$, the cost of item $i$ is $C_i^{\mathrm{item}}(\hat\eta_{1, m, k})=\ind\{U_i\le\hat\eta^{\Hum}_{1,m,k}\}\cH
 +\ind\{U_i>\hat\eta^{\Hum}_{1,m,k}\}
 \Bigl(\cAI+
 \ind\{V_i\le\hat\eta^{\esc}_{1,m,k}(R_i)\}\cH\Bigr)$.
Conditional on the realized pilot, its plug-in-channel and true-channel
expected costs are, respectively,
\begin{equation}\label{eq:pilot-direction-one-expected-costs}
\begin{aligned}
 \widehat c^{\mathrm{item}}_{1,m,k}
 :=\E^{\hat\pi_{m,k}}_{1,\hat f}
       [C_i^{\mathrm{item}}(\hat\eta_{1, m, k})\mid\calD_m]=\hat\eta^{\Hum}_{1,m,k}\cH
 +(1-\hat\eta^{\Hum}_{1,m,k})
 \left(\cAI+\cH\sum_{r\in\calR}
       \hat g_1(r)\hat\eta^{\esc}_{1,m,k}(r)\right),\\
 c^{\mathrm{item}}_{1,m,k}
 :=\E^{\hat\pi_{m,k}}_1
       [C_i^{\mathrm{item}}(\hat\eta_{1, m, k})\mid\calD_m]=\hat\eta^{\Hum}_{1,m,k}\cH
 +(1-\hat\eta^{\Hum}_{1,m,k})
 \left(\cAI+\cH\sum_{r\in\calR}
       g_1(r)\hat\eta^{\esc}_{1,m,k}(r)\right).
\end{aligned}
\end{equation}
The rounded-block calculation gives
\[
 \widehat c^{\mathrm{item}}_{1,m,k}
 \le
 \frac{\hat\Gamma_1(
    \widehat T^{\pilot}_{1,m,k}
      (\widehat{\overline N}_{\main,m,k}),
    \widehat{\overline N}_{\main,m,k})+O(1)}
      {\widehat N^{\mathrm{plan}}_{1,m,k}}.
\]
Moreover, \eqref{eq:pilot-direction-one-expected-costs} makes the
channel perturbation explicit:
\begin{align*}
 \left|c^{\mathrm{item}}_{1,m,k}
       -\widehat c^{\mathrm{item}}_{1,m,k}\right|
 &=(1-\hat\eta^{\Hum}_{1,m,k})\cH
   \left|\sum_{r\in\calR}
      \bigl(g_1(r)-\hat g_1(r)\bigr)
      \hat\eta^{\esc}_{1,m,k}(r)\right|\\
 &\le\varepsilon_m^{\mathrm{cost}}
   =\cH|\calR|r_m.
\end{align*}
Thus
\(c^{\mathrm{item}}_{1,m,k}
\le\widehat c^{\mathrm{item}}_{1,m,k}
+\varepsilon_m^{\mathrm{cost}}\).  Since membership in
\(\hat{\mathcal C}\) is predictable, by the tower property,
\begin{equation*}
\begin{split}
&\E^{\hat\pi_{m,k}}_1[\sum_{i\in\hat{\mathcal C}} C_i^{\mathrm{item}}(\hat\eta_{J_i, m, k})\mid\calD_m]\\
 =&c^{\mathrm{item}}_{1,m,k}\,
   \E^{\hat\pi_{m,k}}_1[\hat N_c\mid\calD_m]\\
   \leq &
 \frac{\hat\Gamma_1(
    \widehat T^{\pilot}_{1,m,k}
      (\widehat{\overline N}_{\main,m,k}),
    \widehat{\overline N}_{\main,m,k})+O(1)}{\widehat N^{\mathrm{plan}}_{1,m,k}}\,
 \E^{\hat\pi_{m,k}}_1[\hat N_c\mid\calD_m]
 +\varepsilon_m^{\mathrm{cost}}\widehat N_{\main,m,k},
\end{split}
\end{equation*}
where the last term is
\(O(L_kr_m)\) because \(\widehat N_{\main,m,k}=O(L_k)\).

The predictable drift calculation in Cost 2 of the proof of
Lemma~\ref{lem:main-cost} applies with \(S_i\) replaced by \(\hat S_i\).
The complementary items have nonnegative conditional drift, and the
completed continuation satisfies
\(\hat S_{\hat\tau_{\main,m,k}}\le a_k+\omega_{m,k}+2B_{\ell}\). Hence
\[
 \E^{\hat\pi_{m,k}}_1[\hat Z^{\mathrm{item}}_i(\hat\eta_{1,m,k})\mid\calD_m]\,
 \E^{\hat\pi_{m,k}}_1[\hat N_c\mid\calD_m]
 \le a_k+\omega_{m,k}+2B_{\ell}.
\]
Combining the last two displays with
\(\widehat N^{\mathrm{plan}}_{1,m,k}
\E^{\hat\pi_{m,k}}_1[\hat Z^{\mathrm{item}}_i(\hat\eta_{1,m,k})\mid\calD_m]
\ge a_k+\Delta_k\) from
\eqref{eq:pilot-planned-true-drift}, analogously to the end of Cost 2
in that proof,
\[
\begin{aligned}
 \E^{\hat\pi_{m,k}}_1[\sum_{i\in\hat{\mathcal C}} C_i^{\mathrm{item}}(\hat\eta_{J_i, m, k})\mid\calD_m]
 &\le(\hat\Gamma_1(
\widehat T^{\pilot}_{1,m,k}(\widehat{\overline N}_{\main,m,k}),
\widehat{\overline N}_{\main,m,k})+O(1))
   \frac{a_k+\omega_{m,k}+2B_{\ell}}{a_k+\Delta_k}
   +O(L_kr_m)\\
 &\le\hat\Gamma_1(
\widehat T^{\pilot}_{1,m,k}(\widehat{\overline N}_{\main,m,k}),
\widehat{\overline N}_{\main,m,k})+O(1)+O(L_kr_m),
\end{aligned}
\]
because
\(\Delta^{\eff}_{m,k}=\Delta_k-\omega_{m,k}\ge4B_{\ell}\).

% \emph{Cost 3: main-stage scan on the fallback event.}  Analogous to
% Cost 3 in the proof of Lemma~\ref{lem:main-cost}, 
% \(\widehat N_{\main,m,k}=O(L_k)\) gives
% \[
% \begin{aligned}
%  c_{\max}\widehat N_{\main,m,k}
%  \Pp^{\hat\pi_{m,k}}_1(\hat E_{\fb}\mid\calD_m)
%  \le CL_k\left[
%    \frac{z_k+1}{\Delta^{\eff}_{m,k}}
%    +\exp\left\{-c
%       \frac{(\Delta^{\eff}_{m,k})^2}{L_k}\right\}
%  \right].
% \end{aligned}
% \]
Combining Costs 1--2,
\begin{equation*}
    \begin{split}
        \E^{\hat\pi_{m,k}}_1[\widehat C^{\main}\mid\calD_m]
 \le \hat\Gamma_1(
\widehat T^{\pilot}_{1,m,k}(\widehat{\overline N}_{\main,m,k}),
\widehat{\overline N}_{\main,m,k})
+CL_kr_m+C(z_k+1),
% +CL_k\left[
%    \frac{z_k+1}{\Delta^{\eff}_{m,k}}
%    +\exp\left\{-c
%     \frac{(\Delta^{\eff}_{m,k})^2}{L_k}\right\}
 % \right],
    \end{split}
\end{equation*}
where \(O(1)\) is absorbed into \(C(z_k+1)\).
\end{proof}

\section{Additional Materials for Section~\ref{sec:unknown-ai-rates}}\label{appendix-sec:unknown-ai-rates}

\subsection{Proof of Corollary~\ref{cor:pilot-charged}}
\begin{proof}{Proof of Corollary~\ref{cor:pilot-charged}}
To prove Corollary~\ref{cor:pilot-charged}, we show that there exist primitive constants $c > 0$ and
$C_1,\cdots, C_5, C_6, C_7 < \infty$ such that for all sufficiently large $k$, 
\begin{equation}\label{eq:pilot-charged-total-cost}
    \begin{split}
    &\,\frac{\max_h \E_h^{\hat\pi_{m_{\pilot, k},k}}
  [C_{m_{\pilot, k}}^{\pilot} + C^{\hat\pi_{m_{\pilot, k},k}}]}{\LB_k}
  \leq
        1+C_1\frac{\Delta_k}{L_k}
   +C_2r_{m_{\pilot,k}}
   +C_3\frac{z_k+1}{\Delta_k}\\
   &\qquad\qquad+C_4\exp\left\{-c
       \frac{(\Delta_k)^2}{L_k}\right\}
   +C_5\frac{\log L_k}{L_k} + C_6\frac{m_{\pilot, k}}{L_k} + C_7m^{-2}_{\pilot, k}.
    \end{split}
\end{equation}
Then the argument follows by plugging in the values of $\Delta_k$. 

\emph{Step 1: the expected pilot charge.}
To prove \eqref{eq:pilot-charged-total-cost}, we first show that
$\E_h[M_{\tot}]=\Theta(m_{\pilot,k})$. By construction,
$M_{\tot}=\min\{n\geq 1: \min_{x\in\{0,1\}}\sum^n_{i=1}\ind\{X^{\pilot}_i=x\}\geq m\}$.
Thus $M_{\tot}\geq 2m=2m_{\pilot,k}$ almost surely. To calculate the
expectation, let \(B_h\sim\operatorname{Bin}(2m_{\pilot,k},p_h)\) be the
number of label-one observations among the first \(2m_{\pilot,k}\) pilot
items.  If \(B_h<m_{\pilot,k}\), then the label-zero quota has already been
reached and \(m_{\pilot,k}-B_h\) additional label-one observations are
required, with conditional expected waiting time
\((m_{\pilot,k}-B_h)/p_h\).  Similarly, if \(B_h>m_{\pilot,k}\), then
\(B_h-m_{\pilot,k}\) additional label-zero observations are required, with
conditional expected waiting time \((B_h-m_{\pilot,k})/(1-p_h)\).  Hence
\begin{equation}
\label{eq:pilot-duration-exact-mean}
\E_h[M_{\tot}]
=
2m_{\pilot,k}
+
\frac{\E_h[(m_{\pilot,k}-B_h)_+]}{p_h}
+
\frac{\E_h[(B_h-m_{\pilot,k})_+]}{1-p_h}.
\end{equation}

Suppose first that \(p_h<1/2\).  Since
\(\E_h[(B_h-m_{\pilot,k})_+]-\E_h[(m_{\pilot,k}-B_h)_+]
=\E_h[B_h-m_{\pilot,k}]=m_{\pilot,k}(2p_h-1)\),
equation~\eqref{eq:pilot-duration-exact-mean} can be rewritten as
\[
\E_h[M_{\tot}]
=
\frac{m_{\pilot,k}}{p_h}
+
\frac{\E_h[(B_h-m_{\pilot,k})_+]}{p_h(1-p_h)}.
\]
A binomial Chernoff bound gives
\(\Pp_h(B_h\ge m_{\pilot,k})\le e^{-c_h m_{\pilot,k}}\) for some
\(c_h>0\).  Since \((B_h-m_{\pilot,k})_+\le m_{\pilot,k}\), it follows that
\[
\E_h[M_{\tot}]
=
\frac{m_{\pilot,k}}{p_h}
+
O\!\left(m_{\pilot,k} e^{-c_h m_{\pilot,k}}\right).
\]
The case \(p_h>1/2\) is symmetric.
Thus, whenever \(p_h\ne1/2\),
\begin{equation}
\label{eq:pilot-duration-asymptotic-h}
\E_h[M_{\tot}]
=
\frac{m_{\pilot,k}}{\min\{p_h,1-p_h\}}
+
O\!\left(
m_{\pilot,k}e^{-c_hm_{\pilot,k}}
\right).
\end{equation}
If \(p_h=1/2\), a direct calculation using the symmetry of
\(B_h\sim\operatorname{Bin}(2m_{\pilot,k},1/2)\) gives
\[
\E_h[M_{\tot}]
=
2m_{\pilot,k}
+2m_{\pilot,k}\frac{\binom{2m_{\pilot,k}}{m_{\pilot,k}}}{4^{m_{\pilot,k}}}
=
2m_{\pilot,k}+2\sqrt{\frac{m_{\pilot,k}}{\pi}}+O(m_{\pilot,k}^{-1/2}).
\]
Because \(p_0<p_1\), we have
\(\min_{h\in\{0,1\}}\min\{p_h,1-p_h\}=\min\{p_0,1-p_1\}=:p_\star\), and the
preceding two displays combine into
\begin{equation}
\label{eq:pilot-duration-worst-expectation}
\max_{h\in\{0,1\}}\E_h[M_{\tot}]
=\frac{m_{\pilot,k}}{p_\star}\{1+o(1)\},
\qquad
2m_{\pilot,k}\le\max_{h\in\{0,1\}}\E_h[M_{\tot}]=O(m_{\pilot,k}).
\end{equation}
Since every pilot item is acquired and labeled by both the AI and the
human, \(C^{\pilot}_{m_{\pilot,k}}=(\cdata+\cAI+\cH)M_{\tot}\), and therefore
\begin{equation}
\label{eq:pilot-expected-charge}
\max_{h\in\{0,1\}}\E_h[C^{\pilot}_{m_{\pilot,k}}]
=\frac{\cdata+\cAI+\cH}{p_\star}\,m_{\pilot,k}\{1+o(1)\}
=\Theta(m_{\pilot,k}).
\end{equation}

\emph{Step 2: the unconditional main-stage cost.}
 Let \(\mathcal E_k:=\mathcal E_{\pilot,m_{\pilot,k}}\).  By the pilot
concentration result, \(\Pp_h(\mathcal E_k^c)\le m_{\pilot,k}^{-2}\) for
\(h\in\{0,1\}\). On \(\mathcal E_k\), when \eqref{eq:pilot-main-rate-condition} holds, we have that $\Delta^{\eff}_{m,k} = \Delta_k - \omega_{m,k} = \Delta_k(1-o(1))$ (see Lemma~\ref{appendix-lem:pilot-hat-N-main}). Additionally, Theorem~\ref{thm:pilot-first-order} bounds the expected cost
conditional on the pilot data \(\calD_{m_{\pilot,k}}\); averaging that
bound over pilot realizations in \(\mathcal E_k\) gives, for primitive
constants \(C<\infty\) and \(c>0\),
\begin{equation}
\label{eq:pilot-good-event-cost}
\E_h^{\hat\pi_{m_{\pilot,k},k}}
\bigl[C^{\hat\pi_{m_{\pilot,k},k}}\mid\mathcal E_k\bigr]
\le
\LB_k\left(1+C\,\Xi_k\right),
\end{equation}
where
\begin{equation*}
\Xi_k:=
\frac{\Delta_k}{L_k}
+r_{m_{\pilot,k}}
+\frac{z_k+1}{\Delta_k}
+\exp\left\{-c\frac{(\Delta_k)^2}{L_k}\right\}
+\frac{\log L_k}{L_k}.
\end{equation*}
On the complement \(\mathcal E_k^c\), in the worst case we may have to query both the AI and human on all $\widehat N_{\main,m_{\pilot,k},k}$ items. We now show that 
\(\widehat N_{\main,m_{\pilot,k},k}=O(L_k)\) holds uniformly over all
pilot realizations. Let $J_{\min}:=\min_hJ_X^{(h)}>0$ and recall that
$\varepsilon_m^{\mathrm{dr}}=3|\calR|B_\ell r_m$ is deterministic with
$\varepsilon_m^{\mathrm{dr}}\to0$, so
$\varepsilon_m^{\mathrm{dr}}<J_{\min}/2$ for all large $k$.  Consider the
human-only candidate
\[
N_{H,k}:=\max\left\{N_{\fixed,\Hum}(\alpha_{2,k},\beta_{2,k}),
\left\lceil\max_{h\in\{0,1\}}
\frac{T_{h,k}+\varepsilon_m^{\mathrm{dr}}}
     {J_X^{(h)}-\varepsilon_m^{\mathrm{dr}}}\right\rceil\right\}.
\]
For every $h$,
$N_{H,k}J_X^{(h)}\ge T_{h,k}+(N_{H,k}+1)\varepsilon_m^{\mathrm{dr}}$, and
this constraint involves only the known information rate $J_X^{(h)}$ and
the deterministic guard $\varepsilon_m^{\mathrm{dr}}$, so the all-human
plan $(n_{\Hum},n_{\AI},n_{\esc})=(N_{H,k},0,0)$ is feasible for the
guarded plug-in program at $N=N_{H,k}$ regardless of the realized pilot;
in particular the guarded outer problem \eqref{eq:pilot-Nbar-choice} is
feasible for all large $k$.  By optimality of
$\widehat{\overline N}_{\main,m,k}$,
\[
\cdata\widehat{\overline N}_{\main,m,k}
\le\widehat F^{\pilot}_{m,k}\bigl(\widehat{\overline N}_{\main,m,k}\bigr)
\le\widehat F^{\pilot}_{m,k}(N_{H,k})
\le(\cdata+\cH)N_{H,k}.
\]
Since $T_{h,k}=O(L_k)$ and $N_{\fixed,\Hum}(\alpha_{2,k},\beta_{2,k})=O(L_k)$, we
have $N_{H,k}=O(L_k)$ and hence $\widehat N_{\main,m,k}=O(L_k)$; if the
guarded outer problem is infeasible, the safe default
gives
$\widehat N_{\main,m,k}=N_{\fixed,\Hum}(\alpha_{2,k},\beta_{2,k})=O(L_k)$ directly.

It thus follows from the previous argument that on $\calE^c_k$ the worst-case cost is at most $O(L_k)$. Combining the two events,
\begin{equation}
\label{eq:pilot-unconditional-main-cost}
\begin{split}
\E_h^{\hat\pi_{m_{\pilot,k},k}}
\bigl[C^{\hat\pi_{m_{\pilot,k},k}}\bigr]
&=\Pp_h(\mathcal E_k)\,
\E_h^{\hat\pi_{m_{\pilot,k},k}}
\bigl[C^{\hat\pi_{m_{\pilot,k},k}}\mid\mathcal E_k\bigr]
+\Pp_h(\mathcal E_k^c)\,
\E_h^{\hat\pi_{m_{\pilot,k},k}}
\bigl[C^{\hat\pi_{m_{\pilot,k},k}}\mid\mathcal E_k^c\bigr]\\
&\le
\LB_k\left(1+C\,\Xi_k\right)
+O\!\left(L_km_{\pilot,k}^{-2}\right).
\end{split}
\end{equation}

Adding \eqref{eq:pilot-unconditional-main-cost} to \eqref{eq:pilot-expected-charge}, and dividing by
\(\LB_k=\Theta(L_k)\) (Theorem~\ref{thm:first-order}(i)) yields \eqref{eq:pilot-charged-total-cost}.

\emph{Step 3: Plug in $\Delta_k=L_k^{2/3}$.}
For the choice in Corollary~\ref{cor:pilot-rates}, $\Delta_k=L_k^{2/3}$ and $m_{\pilot,k}
=
\left\lceil
L_k^{2/3}(\log L_k)^2
\right\rceil$.
In this case,
\[
r_{m_{\pilot,k}}
=
O\!\left(
L_k^{-1/3}(\log L_k)^{-1/2}
\right),
\qquad
L_kr_{m_{\pilot,k}}=o(\Delta_k),
\]
and, for \(z_k=1\), 
substituting into \eqref{eq:pilot-charged-total-cost} therefore
gives
\[
\frac{
\max_h\E_h[
C_{m_{\pilot,k}}^{\pilot}+C^{\hat\pi_{m_{\pilot,k},k}}
]}
{\LB_k}
\le
1+
O\!\left(
L_k^{-1/3}(\log L_k)^2
\right)
=
1+\widetilde O(L_k^{-1/3}).
\]
\end{proof}

\end{document}